\documentclass[acmsmall]{acmart}
\AtBeginDocument{}
\setcopyright{acmcopyright}
\copyrightyear{2025}
\acmYear{2025}
\acmDOI{xx.xxxx/xxxxxxx}
\acmJournal{TODS}
\acmVolume{99}
\acmNumber{9}
\acmArticle{99}
\acmMonth{9}

\usepackage[all]{nowidow}
\usepackage{amsthm}
\usepackage{amsmath}
\usepackage{algorithm}
\usepackage[noend]{algpseudocode}
\usepackage{enumitem}
\usepackage{thm-restate}
\setlist{noitemsep,parsep=2pt,topsep=2pt}
\usepackage{tikz}
\usetikzlibrary{positioning,shadows,arrows,trees,shapes,fit,shapes.geometric,backgrounds,arrows.meta,quotes,angles,calc,positioning,scopes,quotes,angles}
\usepackage{pgfplots}
\usepgfplotslibrary{statistics,groupplots}
\pgfplotsset{compat=1.12}
\usepackage{multirow}
\usepackage{macros}
\allowdisplaybreaks

\newif\ifappendix
\appendixtrue

\begin{document}

\title{Goal-Driven Query Answering over First- and Second-Order Dependencies with Equality}

%

\author{Efthymia Tsamoura}
\authornote{This research started before the author joined Huawei Labs.}
\orcid{0009-0008-8302-1902}
\affiliation{%
    \institution{Huawei Labs}
    \city{Cambridge}
    \country{United Kingdom}
}
\email{efthymia.tsamoura@huawei.com}

\author{Boris Motik}
\orcid{0000-0003-2506-4118}
\affiliation{%
    \institution{Department of Computer Science, Oxford University}
    \city{Oxford}
    \country{United Kingdom}
}
\email{boris.motik@cs.ox.ac.uk}

\begin{abstract}
\emph{Dependencies} are used in databases to express integrity constraints,
deal with data incompleteness, or specify transformations between database
schemas. While they are usually expressed in fragments of first-order logic,
second-order dependencies provide existentially quantified function variables,
which are needed to express composition of certain classes of schema mappings
\cite{DBLP:journals/tods/FaginKPT05, DBLP:journals/jcss/Arenas0RR13}. When
further extended with the ability to derive equality statements, second-order
dependencies can express compositions of very expressive classes of schema
mappings \cite{DBLP:journals/tods/NashBM07, DBLP:journals/corr/abs-1106-3745}.
Answering queries over data with dependencies is commonly solved by using a
suitable chase variant to compute a universal model of the dependencies and the
data, after which any conjunctive query can be answered by evaluating it in the
universal model. If, however, the query to be answered is known in advance,
computing the universal model is often inefficient as many inferences made
during this process can be irrelevant to the query at hand. In such cases, a
\emph{goal-driven} approach, which avoids drawing unnecessary inferences, can
be more efficient and thus preferable in practice.

In this paper we present the first goal-driven query answering technique for
first- and second-order dependencies with equality. Our technique transforms
the input dependencies so that applying the chase to the output avoids many
inferences that are irrelevant to the query. The transformation proceeds in
several steps, which comprise the following three novel techniques. First, we
present a variant of the \emph{singularisation} technique by
\citet{DBLP:conf/pods/Marnette09} that can handle function variables and that
corrects an incompleteness of a related formulation by
\citet{DBLP:journals/jodsn/CateHK16}. Second, we present a \emph{relevance
analysis} technique that can eliminate dependencies that provably do not
contribute to query answers. Third, we present a variant of the \emph{magic
sets} algorithm \cite{DBLP:journals/jlp/BeeriR91} that can handle second-order
dependencies with equality. We also present the results of an extensive
empirical evaluation, which show that goal-driven query answering can be orders
of magnitude faster than computing the full universal model.

\end{abstract}



\received{14 October 1066}
\received[revised]{14 October 1066}
\received[accepted]{14 October 1066}

\maketitle

\section{Introduction}\label{sec:introduction}

The need to describe a domain of interest using formal statements naturally
arises in many areas of databases and knowledge representation. Such
descriptions are usually formulated in a suitable fragment of first- or
second-order logic, and are, depending on one's perspective and background,
called \emph{dependencies} \cite{DBLP:conf/icalp/BeeriV81},
\emph{$\forall\exists$-rules} \cite{DBLP:journals/ai/BagetLMS11},
\emph{existential rules} \cite{DBLP:conf/rr/Mugnier11}, or \emph{Datalog$^\pm$}
\cite{DBLP:journals/ws/CaliGL12}.

\subsection{Background: Dependencies}\label{sec:introduction:dependencies}

Dependencies have many uses in databases. They can express integrity
constraints---statements that describe valid database states
\cite{DBLP:journals/tods/Fagin77, DBLP:journals/tods/Delobel78}. They can also
be used to complete an incomplete database with missing facts and thus provide
richer answers to queries \cite{Greco2012}. Finally, they are used extensively
in declarative data integration \cite{Levy2000, DBLP:conf/vldb/HalevyRO06,
DBLP:conf/pods/Lenzerini02}, as well as to specify mappings between database
schemas \cite{DBLP:journals/tcs/FaginKMP05}---that is, how to transform any
database expressed in one schema to a database in another schema. In knowledge
representation, it was shown that ontology languages such as
$\mathcal{EL}{+}{+}$ \cite{babl05} and certain languages of the DL-Lite family
\cite{DBLP:journals/jar/CalvaneseGLLR07, DBLP:journals/jair/ArtaleCKZ09} can be
expressed as dependencies \cite{DBLP:journals/ws/CaliGL12}.

Dependencies were initially expressed using ad hoc languages
\cite{DBLP:journals/tods/Fagin77, DBLP:journals/tods/Delobel78}, but
\citet{DBLP:conf/sigmod/Nicolas78} observed that many dependency classes can be
expressed in first-order logic. \citet{DBLP:conf/icalp/BeeriV81} generalised
this idea and proposed \emph{tuple and equality generating dependencies} (TGDs
and EGDs). Intuitively, TGDs state that existence of certain tuples implies
existence of other tuples, and EGDs specify that existence of certain tuples
implies uniqueness of certain values.

The expressivity of first-order logic can sometimes be insufficient.
\citet{DBLP:journals/tods/FaginKPT05} observed this in the context of schema
mappings. If $\Sigma_1$ and $\Sigma_2$ are source-to-target TGDs describing
mappings from a schema $\mathcal{S}_1$ to a schema $\mathcal{S}_2$, and from
$\mathcal{S}_2$ to a schema $\mathcal{S}_3$, respectively, the equivalent
mapping from $\mathcal{S}_1$ to $\mathcal{S}_3$ may not be expressible in
first-order logic; however, the composition of $\Sigma_1$ and $\Sigma_2$ can be
expressed using second-order (SO) dependencies. Instead of existentially
quantified first-order variables, SO dependencies provide existentially
quantified function variables, which considerably increases the expressive
power. Moreover, \citet{DBLP:journals/tods/FaginKPT05} have shown that SO
dependencies themselves are closed under composition.

Following this seminal result, composition properties of other dependency
classes have been studied. \citet{DBLP:journals/tods/NashBM07} studied mappings
that do not distinguish the source and the target schema, and they identified
cases that do and do not support composition; some of these involve
second-order dependencies. \citet{DBLP:journals/corr/abs-1106-3745} studied
mappings that distinguish the source from the target schemas, but where the
target schema also uses TGDs and EGDs. They showed that SO dependencies alone
cannot express composition of such mappings, but this can be achieved using
TGDs, EGDs, and \emph{source-to-target SO dependencies}---an extension of SO
dependencies that allows for equalities between terms in the consequents.
\citet{DBLP:journals/jcss/Arenas0RR13} introduced \emph{plain source-to-target
SO TGDs}, which disallow equalities and nesting of function variables in SO
dependencies, and they showed that mappings expressed using this language can
be composed as well as inverted.

\subsection{Background: Query Answering over Dependencies}\label{sec:introduction:qa}

The problem of answering a query over a database extended with facts that are
logically implied by a set of dependencies plays a central role in areas such
as declarative data integration \cite{Levy2000, DBLP:conf/vldb/HalevyRO06,
DBLP:conf/pods/Lenzerini02}, answering queries using views
\cite{DBLP:journals/vldb/Halevy01}, accessing data sources with restrictions
\cite{DBLP:journals/sigmod/DeutschPT06, DBLP:journals/vldb/Meier14}, and
answering queries over ontologies \cite{DBLP:journals/ws/CaliGL12}. Thus,
identifying dependency classes that are sufficiently expressive but also
support effective query answering has received considerable attention.

This problem can be solved by using a suitable variant of the \emph{chase
algorithm} to compute a \emph{universal model}---a set of facts that satisfies
both the dataset and all dependencies. Any conjunctive query can be answered by
evaluating it in the universal model and returning the answers that consist of
constants only. The chase was first introduced by
\citet{DBLP:journals/tods/MaierMS79}, and many variants have been developed
since, such as the \emph{restricted} \cite{DBLP:conf/icalp/BeeriV81,
DBLP:journals/tcs/FaginKMP05}, \emph{oblivious} \cite{DBLP:conf/kr/CaliGK08},
\emph{semioblivious} \cite{DBLP:journals/mst/CalauttiP21}, \emph{Skolem}
\cite{DBLP:conf/pods/Marnette09, DBLP:journals/pvldb/CateCKT09}, \emph{core}
\cite{DBLP:conf/pods/DeutschNR08}, \emph{parallel}
\cite{DBLP:conf/pods/DeutschNR08}, and \emph{frugal}
\cite{DBLP:journals/pvldb/KonstantinidisA14} chase.
\citet{DBLP:conf/pods/BenediktKMMPST17} discuss the similarities and
differences of many chase variants for first-order dependencies. The chase was
also extended to handle second-order dependencies without
\cite{DBLP:journals/tods/FaginKPT05} and with
\cite{DBLP:journals/corr/abs-1106-3745, DBLP:journals/tods/NashBM07} equality
in the consequents.

Such an approach is practicable only if the chase terminates and produces a
finite universal model. This is not guaranteed in general, and in fact checking
whether the chase terminates on all data sets is undecidable
\cite{DBLP:conf/icalp/GogaczM14}. However, the problem is decidable for some
dependency classes, such as linear \cite{DBLP:conf/icdt/LeclereMTU19} and
sticky \cite{DBLP:journals/mst/CalauttiP21} TGDs. Moreover, many sufficient
termination conditions have been proposed, such as \emph{weak}
\cite{DBLP:journals/tcs/FaginKMP05}, \emph{superweak}
\cite{DBLP:conf/pods/Marnette09}, \emph{joint}
\cite{DBLP:conf/ijcai/KrotzschR11}, \emph{argument restricted}
\cite{DBLP:conf/iclp/LierlerL09}, and \emph{model-summarising} and
\emph{model-faithful} \cite{DBLP:journals/jair/GrauHKKMMW13} acyclicity.
Finally, the chase terminates for all classes of second-order dependencies used
in the literature on compositions of schema mappings
\cite{DBLP:journals/tods/FaginKPT05, DBLP:journals/tods/NashBM07,
DBLP:journals/corr/abs-1106-3745, DBLP:journals/jcss/Arenas0RR13}.

The query answering problem can be solved for some dependency classes even if
the chase does not terminate. For example, one can develop a finite
representation of infinite universal models; then, one can use a chase variant
with \emph{blocking} \cite{DBLP:journals/ws/CaliGL12,
DBLP:journals/pvldb/BellomariniSG18, hirschthesis, dl-handbook-2} to construct
such a representation, or one can construct a tree automaton that accepts such
representations \cite{DBLP:conf/lics/GradelW99}. Alternatively, one can
\emph{rewrite} the query and the dependencies into first-order
\cite{DBLP:journals/jar/CalvaneseGLLR07, DBLP:journals/ai/BagetLMS11,
DBLP:conf/ijcai/CaliLR03, DBLP:conf/ijcai/BarceloBLP18} or Datalog
\cite{DBLP:journals/corr/abs-1212-0254, DBLP:conf/mfcs/BaranyBC13,
DBLP:journals/jsyml/BaranyBC18, DBLP:journals/tkde/WangXWZW23} queries. In this
paper, however, we focus on dependency classes where the chase terminates.

\subsection{The Need for Goal-Driven Query Answering}\label{sec:introduction:need}

A universal model can be used to answer an arbitrary query. The cost of
computing the chase in a preprocessing step is thus often amortised over time
in applications where the data changes infrequently or the query workload is
unknown in advance. However, answers to a particular query often rely only on a
relatively small portion of the universal model. Thus, if the query workload is
known in advance, computing the chase in full can be inefficient as many
inferences made by the algorithm may be irrelevant to the query answers. This
problem is exacerbated if the data changes frequently, so the universal model
needs to be frequently recomputed. When dependencies do not contain existential
quantifiers (and are thus equivalent to Datalog rules possibly extended with
the equality predicate), recomputing the chase in full after each change can be
mitigated by using an \emph{incremental maintenance} algorithm
\cite{DBLP:conf/sigmod/LuMSS95, mnph19maintenance-revisited,
mnph15incremental-BF-sameAs}; however, to the best of our knowledge, no such
algorithm is known for first- and second-order dependencies with equalities.

In this paper we thus turn our attention to \emph{goal-driven} query answering
techniques, which typically start from the query and work backwards through
dependencies to identify the relevant inferences. Some of the rewriting
techniques mentioned in Section~\ref{sec:introduction:qa}, such as the query
rewriting algorithm for DL-Lite \cite{DBLP:journals/jar/CalvaneseGLLR07} or the
piece-based backward chaining \cite{DBLP:journals/ai/BagetLMS11}, can be seen
as being goal-driven; however, these are applicable only to syntactically
restricted dependency classes that can be insufficiently expressive in certain
applications. SLD resolution \cite{DBLP:conf/ifip/Kowalski74} provides
goal-driven query answering for logic programs. Furthermore, the \emph{magic
sets} technique for logic programs \cite{DBLP:conf/pods/BancilhonMSU86,
DBLP:journals/jlp/BeeriR91, DBLP:journals/jlp/BalbinPRM91} optimises the
tuple-at-a-time style of processing of SLD resolution. The idea behind the
magic sets is to analyse the program's inferences and modify the program so
that applying the chase to the transformation result simulates backward
chaining. This is achieved by introducing auxiliary \emph{magic} predicates
that accumulate the bindings that would be produced during backward chaining,
and by using these predicates as guards to restrict the program's rules to the
relevant bindings. This idea has been adapted to many contexts, such as
finitely recursive programs \cite{DBLP:conf/lpnmr/CalimeriCIL09}, programs with
builtins \cite{DBLP:journals/tods/MumickFPR96} and aggregates
\cite{DBLP:conf/lpnmr/AlvianoGL11}, disjunctive programs
\cite{DBLP:journals/ai/AlvianoFGL12}, and \emph{Shy} dependencies
\cite{DBLP:conf/datalog/AlvianoLMTV12}.

As we argued in Section~\ref{sec:introduction:dependencies}, second-order
dependencies with equality atoms are necessary to capture many relevant data
management tasks, but, to the best of our knowledge, none of the goal-driven
techniques we outlined thus far are applicable to this class of dependencies. A
na{\"i}ve approach might be to explicitly axiomatise equality as an ordinary
predicate \cite{theorem-proving} (see Section~\ref{sec:preliminaries}), and to
use the standard magic sets technique for logic programs (possibly containing
function symbols) \cite{DBLP:journals/jlp/BeeriR91}; however, reasoning with
the explicit axiomatisation of equality can be very inefficient in practice
\cite{mnph15owl-sameAs-rewriting}. Consequently, efficient and practically
successful goal-driven query answering over first- and second-order
dependencies with equality remains an open problem.

\subsection{Our Contribution}\label{sec:introduction:contribution}

In this paper, we present what we believe to be the first goal-driven approach
to answering queries over first- and second-order dependencies with equality.
Our technique takes as input a dataset, a set of dependencies, and a query, and
it modifies the dependencies so that chase computation answers the query while
reducing irrelevant inferences. Our technique is inspired by the magic sets
variants mentioned in Section~\ref{sec:introduction:need}, but is much more
involved. The key problem is that equality inferences are prolific (i.e., they
can affect any predicate in any dependency) and highly redundant (i.e., the
same conclusion is derived in many different ways), both of which make the
analysis of equality inferences hard. Our solution is based on the following
three main contributions.

First, to facilitate a precise and efficient analysis of equality inferences,
we use the \emph{singularisation} technique by
\citet{DBLP:conf/pods/Marnette09}, which axiomatises equality without the
congruence axioms \cite{theorem-proving} and thus avoids a key source of
inefficiency in practice \cite{mnph15owl-sameAs-rewriting}. To compensate for
the lack of congruence axioms, the dependencies need to be modified too. One
can intuitively understand this as `pruning' redundant inferences from the
original dependencies, which in turn allows for an efficient analysis of
equality inferences. \citet{DBLP:conf/pods/Marnette09} introduced
singularisation for first-order dependencies. Moreover,
\citet{DBLP:journals/jodsn/CateHK16} applied this technique to second-order
dependencies, but their result is incomplete: in
Section~\ref{sec:so-dependencies} we present an example where singularisation
by \citet{DBLP:journals/jodsn/CateHK16} does \emph{not} preserve all query
answers because it does not take into account \emph{functional reflexivity} of
equality. The latter property ensures that function variables behave like
functions: if we derive that $a$ and $b$ are equal, then we must derive that
$f(a)$ and $f(b)$ are equal too. Completeness of the technique by
\citet{DBLP:journals/jodsn/CateHK16} can be easily recovered by axiomatising
functional reflexivity, but this prevents chase termination: if $f(a)$ and
$f(b)$ are equal, then $f(f(a))$ and $f(f(b))$ should be equal as well, and so
on ad infinitum. We overcome this by presenting a novel singularisation variant
where functional reflexivity is constrained to derive `just the right'
equalities: sufficient to derive all relevant answers, but without necessarily
making the universal model infinite.

Second, we present a \emph{relevance analysis} technique that can identify and
eliminate dependencies for which no conclusion is relevant to the query.
Roughly speaking, this technique computes an abstraction of a universal
model---that is, a model that contains a homomorphic image of a universal model
computed by the chase on the given dataset. This abstraction is then used to
perform a backward analysis of the inferences and dependencies that contribute
to query answers.

Third, we present a modification of the magic sets technique for logic programs
that handles equality efficiently. Roughly speaking, our technique takes into
account the reflexivity, symmetry, and transitivity of equality to reduce the
number of rules produced. A careful handling of reflexivity is particularly
essential: an equality of the form ${t \equals t}$ can be derived from any
predicate, so the standard magic sets algorithm necessarily considers each
input dependency. In contrast, if the input dependencies are \emph{safe} (which
intuitively ensures that the conclusions of the dependencies do not depend on
the interpretation domain), we show that reflexivity does not need to be taken
into account in the magic sets transformation, which usually reduces the output
size considerably.

Our three techniques are complementary: singularisation facilitates the use of
the relevance analysis and the magic sets transformation, and neither of the
latter two techniques subsumes the other. Thus, all three techniques are
required to facilitate efficient goal-driven query answering.

We stress that singularisation is used only to \emph{facilitate our
transformation}: it is `undone' at the end of our transformation and equality
is treated as `true' equality in the result. In other words, the final chase
step (which is likely to be critical to the performance of query answering)
does not suffer from any overheads associated with axiomatising equality.

Our aim is to demonstrate that our techniques are practical, but we faced a
significant obstacle towards this goal: whereas publicly available first-order
benchmarks exist \cite{DBLP:conf/pods/BenediktKMMPST17}, we are unaware of any
such benchmarks for second-order dependencies with equalities. We therefore
turned to synthetic benchmark generators described in the literature:
STBenchmark \cite{DBLP:journals/pvldb/AlexeTV08} can synthesise data mapping
benchmarks; ToXgene \cite{DBLP:conf/sigmod/BarbosaMKL02} can produce XML
datasets that are easily converted into a relational form; iBench
\cite{DBLP:journals/pvldb/ArocenaGCM15} can produce data integration benchmarks
consisting of first- and second-order dependencies; and WatDiv
\cite{DBLP:conf/semweb/AlucHOD14} can generate RDF data. To the best of our
knowledge, only iBench can produce dependencies with function variables, but
these are much simpler in structure than what we consider in this paper or what
was considered by \citet{DBLP:journals/tods/FaginKPT05} in their foundational
work on SO dependencies. Extending iBench to produce more complex dependencies
is unlikely to be adequate: iBench randomly produces each dependency in
isolation, so independently generated dependencies are not guaranteed to `fire'
and produce an interesting chain of inferences. An analogous problem occurs
when dependencies are generated independently from the data.

To overcome these problems, we developed a new benchmark generation technique
that produces dependencies and datasets in combination. Roughly speaking, our
technique randomly generates derivation trees of instantiated
dependencies---that is, dependencies with ground terms of a bounded depth. The
tree leaves are then converted into a dataset, and the internal nodes are
converted into second-order dependencies by replacing certain subterms with
variables. Applying the resulting dependencies to the corresponding dataset is
guaranteed to perform at least the inferences from the generated derivation
trees. We are not aware of any related technique that can provide analogous
guarantees about the minimum amount of incurred work.

We evaluated our techniques on a range of existing and new test scenarios
involving first- and second-order dependencies. Our objective was to verify
whether a goal-driven technique can answer a single query faster than computing
a universal model in full. Moreover, to isolate the contributions of different
techniques, we compared answering the query by relevance analysis only, magic
sets only, and with both techniques combined. Our results show that goal-driven
query answering can be very effective, sometimes improving the performance by
orders of magnitude. Our relevance analysis technique seems to be the main
reason behind these improvements, but magic sets can be beneficial too.
Moreover, both techniques have the potential to considerably reduce the number
of facts derived in practice.

\subsection{Summary of Contributions and Paper Structure}

The results presented in this paper extend our earlier work published at the
2018 AAAI conference \cite{DBLP:conf/aaai/BenediktMT18}, and the main novelty
The main novelty of this work can be summarised as follows:
\begin{itemize}
    \item an extension of the singularisation technique by
    \citet{DBLP:conf/pods/Marnette09} to dependencies with function symbols,
    which corrects the incompleteness in the work by
    \citet{DBLP:journals/jodsn/CateHK16};

    \item an extension of the relevance analysis and magic sets techniques to
    second-order dependencies;

    \item a generator of second-order dependencies and datasets that are
    guaranteed to exhibit nontrivial inferences;

    \item the first implementation of the chase for second-order dependencies
    with equality; and

    \item an extensive empirical evaluation showing that our techniques can be
    effective in practice.
\end{itemize}

The rest of our paper is structured as follows. In
Section~\ref{sec:preliminaries} we recapitulate definitions, terminology, and
notation that we use in the rest of the paper. In
Section~\ref{sec:so-dependencies} we introduce second-order dependencies, we
discuss the chase variant applicable to such dependencies, and we present a
running example that illustrates the key difficulties in our work. In
Section~\ref{sec:answering} we present our approach in detail and prove its
correctness. In Section~\ref{sec:evaluation} we discuss the results of our
experimental evaluation. Finally, in Section~\ref{sec:conclusion} we
recapitulate our main findings and discuss possible avenues for further work.
\ifappendix
The proofs of our results are given in
Appendices~\ref{sec:proof:sg}--\ref{sec:proof:desg}, and the description of the
benchmark generator is given in Appendix~\ref{sec:generating}.
\fi
\section{Preliminaries}\label{sec:preliminaries}

We formalise our results using first- and second-order logic, as well as logic
programming. To avoid defining each formalism separately, we ground all
definitions in the framework of second-order logic as presented by
\citet{DBLP:books/daglib/0076838}. We next recapitulate the relevant
terminology and notation.

\myparagraph{First- and Second-Order Logic: Syntax}
We fix arbitrary, countably infinite, and mutually disjoint sets of
\emph{constants}, \emph{individual variables}, \emph{function symbols},
\emph{function variables}, and \emph{predicates}. Each function symbol,
function variable, and predicate is associated with a nonnegative integer
\emph{arity}. A \emph{term} is inductively defined as a constant, an individual
variable, or an expression of the form ${f(t_1,\dots,t_n)}$ where
${t_1,\dots,t_n}$ are terms and $f$ is an $n$-ary function symbol or an $n$-ary
function variable. An \emph{atom} is an expression of the form
${R(t_1,\dots,t_n)}$ where $R$ is an $n$-ary predicate and ${t_1,\dots,t_n}$
are terms called the atom's \emph{arguments}. We assume that there exists a
distinct binary \emph{equality predicate} $\equals$. Atoms of the form
${\equals}(t_1,t_2)$ are typically written as ${t_1 \equals t_2}$ and are
called \emph{equality atoms} (or just \emph{equalities}), and atoms with a
predicate different from $\equals$ are called \emph{relational}. Formulas of
second-order logic are constructed as usual using Boolean connectives $\wedge$,
$\vee$, and $\neg$, first-order quantifiers $\exists x$ and $\forall x$ where
$x$ is an individual variable, and second-order quantifiers $\forall f$ and
$\exists f$ where $f$ is a function variable. Implication ${\varphi \rightarrow
\psi}$ abbreviates ${\neg \varphi \vee \psi}$. A first-order formula does not
contain function variables. A \emph{sentence} is a formula with no free
(individual or function) variables. In the context of first-order formulas,
individual variables are typically called just \emph{variables}. Unless stated
otherwise, we use possibly subscripted letters ${a, b, c, \dots}$ for
constants, ${s, t,\dots}$ for terms, ${x, y, z, \dots}$ for individual
variables, and ${f, g, h,\dots}$ for either function symbols or function
variables; in the latter case, the intended use will always be clear from the
context.

\myparagraph{First- and Second-Order Logic: Semantics}
An \emph{interpretation} ${I = (\Delta^I, \cdot^I)}$ consists of a nonempty
\emph{domain} set $\Delta^I$ and a function $\cdot^I$ that maps each constant
$a$ to a domain element ${a^I \in \Delta^I}$, each $n$-ary function symbol $f$
to a function ${f^I : (\Delta^I)^n \to \Delta^I}$, and each $n$-ary predicate
$R$ to a relation ${R^I \subseteq (\Delta^I)^n}$. A \emph{valuation} $\pi$ on
$I$ maps each individual variable $x$ to a domain element ${x^\pi \in
\Delta^I}$, and each $n$-ary function variable $f$ to an $n$-ary function
${f^\pi : (\Delta^I)^n \to \Delta^I}$. Given an interpretation $I$ and a
valuation $\pi$ in $I$, each term $t$ is assigned a value ${t^{I,\pi} \in
\Delta^I}$ as follows:
\begin{displaymath}
    t^{I,\pi} = \begin{cases}
        a^I                                     & \text{if } t \text{ is a constant } a, \\
        x^\pi                                   & \text{if } t \text{ is an individual variable } x, \\
        f^I(t_1^{I, \pi},\dots,t_n^{I, \pi})    & \text{if } t = f(t_1,\dots,t_n) \text{ with } f \text{ an $n$-ary function symbol, and} \\
        f^\pi(t_1^{I, \pi},\dots,t_n^{I, \pi})  & \text{if } t = f(t_1,\dots,t_n) \text{ with } f \text{ an $n$-ary function variable.} \\
    \end{cases}
\end{displaymath}

Let $\varphi$ be a first- or second-order formula (possibly containing free
individual and/or function variables), let $I$ be an interpretation, and let
$\pi$ be a valuation defined on all free (individual and function) variables of
$\varphi$. We can determine whether $\varphi$ is \emph{satisfied} in $I$ and
$\pi$, written ${I,\pi \modelsEq \varphi}$, using the standard definitions of
first- and second-order logic \cite{DBLP:books/daglib/0082516,
DBLP:books/daglib/0076838}; we recapitulate below only the cases for the
existential first- and second-order quantifiers.
\begin{displaymath}
\begin{array}{@{}l@{\;}l@{}}
    I,\pi \modelsEq \exists x.\psi  & \text{iff there exists a domain element } \alpha \in \Delta^I \text{ such that} \\
                                    & \quad I,\pi' \modelsEq \psi \text{ where } \pi' \text{ is obtained from } \pi \text{ by mapping } x \text{ to } \alpha. \\[2ex]
    I,\pi \modelsEq \exists f.\psi  & \text{iff there exists a function } \alpha : (\Delta^I)^n \to \Delta^I \text{ such that} \\
                                    & \quad I,\pi' \modelsEq \psi \text{ where } \pi' \text{ is obtained from } \pi \text{ by mapping } f \text{ to } \alpha. \\
\end{array}
\end{displaymath}
The subscript $\equals$ in $\modelsEq$ stipulates that $\equals$ is interpreted
as `true equality'---that is, ${\equals^I = \{ \langle \alpha, \alpha \rangle
\mid \alpha \in \Delta^I \}}$, so two domain elements are equal if and only if
they are identical. The truth of $\varphi$ in $I$ does not depend on a
valuation if $\varphi$ is a sentence, so we simply write ${I \modelsEq
\varphi}$ and call $I$ a \emph{model} of $\varphi$. Moreover, a sentence
$\varphi$ is \emph{satisfiable} if it has a model; a sentence $\varphi$
\emph{entails} a sentence $\psi$, written ${\varphi \modelsEq \psi}$, if each
model of $\varphi$ is also a model of $\psi$; finally, sentences $\varphi$ and
$\psi$ are \emph{equivalent} if they are satisfied in exactly the same models
(so ${\varphi \modelsEq \psi}$ and ${\psi \modelsEq \varphi}$). A model $I$ of
a sentence $\varphi$ is \emph{universal} if, for each model ${J = (\Delta^J,
\cdot^J)}$ of $\varphi$, there exists a mapping ${\mu : \Delta^I \to \Delta^J}$
such that
\begin{itemize}
    \item ${\mu(a^I) = a^J}$ for each constant $a$,
    
    \item ${\mu(f^I(\alpha_1,\dots,\alpha_n)) =
    f^J(\mu(\alpha_1),\dots,\mu(\alpha_n))}$ for each $n$-ary function symbol
    $f$ and each $n$-tuple ${\langle \alpha_1, \dots, \alpha_n \rangle \in
    (\Delta^I)^n}$, and

    \item ${\langle \mu(\alpha_1), \dots, \mu(\alpha_n) \rangle \in R^J}$ for
    each $n$-ary predicate $R$ and each $n$-tuple ${\langle \alpha_1, \dots,
    \alpha_n \rangle \in R^I}$.
\end{itemize}

First- and second-order logic typically do not assume the \emph{unique name
assumption} (UNA): distinct constants can be interpreted as the same domain
element. However, UNA is commonly used in databases: an attempt to equate two
constants results in a contradiction. For the sake of generality, we do not
assume UNA in this paper. For example, sentence ${\varphi = R(a,b) \wedge
\forall x,y.[R(x,y) \rightarrow x \equals y]}$ is satisfiable and it implies
that $a$ and $b$ are the same constant---that is, ${\varphi \modelsEq a \equals
b}$. If desired, one can always check whether UNA is satisfied by explicitly
querying the equality predicate.

\myparagraph{Auxiliary Definitions}
A term $t_1$ is a \emph{subterm} of a term $t_2$ if $t_1$ syntactically occurs
inside $t_2$, and $t_1$ is a \emph{proper subterm} of $t_2$ if $t_1$ is a
subterm of $t_2$ and ${t_1 \neq t_2}$. The \emph{depth} of a term $t$ is
defined as ${\dep{t} = 0}$ if $t$ is a variable or a constant, and ${\dep{t} =
1 + \max \{ \dep{t_i} \mid 1 \leq i \leq n \}}$ if ${t = f(t_1,\dots,t_n)}$.
The \emph{depth} of an atom is equal to the maximum depth of its arguments. We
often abbreviate a tuple ${t_1,\dots,t_n}$ of terms as $\vec t$, and we often
treat $\vec t$ as a set; for example, we write ${t \in \vec t}$ to indicate
that ${t = t_i}$ for some ${i \in \{ 1, \dots, n \}}$. Also, we often
abbreviate a tuple ${f_1,\dots,f_n}$ of function symbols or function variables
as $\vec f$. For $\alpha$ a term, a formula, or a set thereof, $\vars{\alpha}$
is the set of all free variables of $\alpha$. A \emph{substitution} is a
mapping of finitely many variables to terms. For $\sigma$ a substitution and
$\alpha$ a term, a formula, or a set thereof, $\alpha\sigma$ is the result of
replacing each free occurrence of a variable ${x \in \vars{\alpha}}$ in
$\alpha$ with $\sigma(x)$ provided the latter is defined. A term or an atom is
\emph{ground} if it does not contain a variable. A \emph{fact} is a ground
atom, and a \emph{base fact} is a fact that does not use the equality predicate
and does not contain a function symbol. An \emph{instance} is a (possibly
infinite) set of facts. A \emph{base instance} is a finite set of base facts,
and it corresponds to the notion of a database instance.

\myparagraph{Equality as an Ordinary Predicate}
Our algorithms will need to analyse inferences that use the equality predicate,
which, as we discuss in Section~\ref{sec:so-dependencies:goal-driven}, can be
very challenging. We overcome this issue by explicitly axiomatising the
properties of equality and treating $\equals$ as an ordinary predicate. To
clearly distinguish the two uses of equality, we use the symbol $\models$ for
satisfaction and entailment whenever we assume that $\equals$ is an ordinary
predicate without any special meaning. For example, let ${\varphi = R(a,b)
\wedge S(a,c) \wedge \forall x.[R(x,x_1) \wedge S(x,x_2) \rightarrow x_1
\equals x_2]}$. Then, ${\varphi \modelsEq b \equals c}$, but also ${\varphi
\modelsEq c \equals b}$ because $\modelsEq$ interprets $\equals$ as a symmetric
predicate. In contrast, ${\varphi \models b \equals c}$, but ${\varphi
\not\models c \equals b}$ because $\models$ interprets $\equals$ as just
another predicate with no special meaning. Even when $\equals$ is interpreted
as an ordinary predicate, we still distinguish relational and equality atoms as
outlined earlier.

\myparagraph{Explicit Axiomatisation of Equality}
When equality is treated as an ordinary predicate, the properties of equality
can be explicitly axiomatised so that there is no distinction in the entailed
facts. In particular, for $\Sigma$ a first-order formula and $B$ a base
instance, let $\D$ be a fresh predicate not occurring in $\Sigma$, and let
$\EQ{\Sigma}$ be a conjunction containing a \emph{domain axiom}
\eqref{eq:dom-c} for each constant $c$ occurring in $\Sigma$, a \emph{domain
axiom} \eqref{eq:dom-R} for each $n$-ary predicate $R$ occurring in $\Sigma$
distinct from $\equals$ and each ${i \in \{ 1, \dots, n \}}$, the
\emph{reflexivity axiom} \eqref{eq:ref}, the \emph{symmetry axiom}
\eqref{eq:sym}, the \emph{transitivity axiom} \eqref{eq:trans}, a
\emph{functional reflexivity axiom} \eqref{eq:fnref} for each $n$-ary function
symbol $f$ occurring in $\Sigma$, and a \emph{congruence axiom} \eqref{eq:cong}
for each $n$-ary predicate $R$ distinct from $\equals$ and each ${i \in \{ 1,
\dots, n \}}$.
\begin{align}
                                                                                                    & \rightarrow \D(c)                                             \label{eq:dom-c} \\
    \forall x_1, \dots, x_n.[R(x_1, \dots, x_n)                                                     & \rightarrow \D(x_i)]                                          \label{eq:dom-R} \\
    \forall x.[\D(x)                                                                                & \rightarrow x \equals x]                                      \label{eq:ref}   \\
    \forall x_1,x_2.[x_1 \equals x_2                                                                & \rightarrow x_2 \equals x_1]                                  \label{eq:sym}   \\
    \forall x_1,x_2,x_3.[x_1 \equals x_2 \wedge x_2 \equals x_3                                     & \rightarrow x_1 \equals x_3]                                  \label{eq:trans} \\
    \forall x_1,\dots,x_n,x_1',\dots,x_n'.[x_1 \equals x_1' \wedge \dots \wedge x_n \equals x_n'    & \rightarrow f(x_1, \dots, x_n) \equals f(x_1', \dots, x_n')]  \label{eq:fnref} \\
    \forall x_1,\dots,x_n,x_i'.[R(x_1,\dots,x_n) \wedge x_i \equals x_i'                            & \rightarrow R(x_1,\dots,x_{i-1},x_i',x_{i+1},\dots,x_n)]      \label{eq:cong}
\end{align}
Intuitively, axioms \eqref{eq:dom-c} and \eqref{eq:dom-R} ensure that $\D$
enumerates the domain of ${\{ \Sigma \} \cup \EQ{\Sigma} \cup B}$, axioms
\eqref{eq:ref}--\eqref{eq:trans} and \eqref{eq:cong} axiomatise equality as a
congruence relation, and axioms \eqref{eq:fnref} ensure that function symbols
are interpreted as functions. It is well known that ${\{ \Sigma \} \cup B
\modelsEq F}$ if and only if ${\{ \Sigma \wedge \EQ{\Sigma} \} \cup B \models
F}$ for each fact $F$ \cite{theorem-proving}---that is, explicit axiomatisation
of equality preserves fact entailment. Nevertheless, reasoning with ${\Sigma
\wedge \EQ{\Sigma}}$ instead of handling equality `natively' can be very
inefficient in practice \cite{mnph15owl-sameAs-rewriting} and is thus rarely
used when the performance of query answering is critical.

\myparagraph{Herbrand Interpretations}
When the equality predicate has no special meaning, we can interpret
first-order sentences in interpretations of a specific form. In particular, a
\emph{Herbrand} interpretation $I$ uses the domain set $\Delta^I$ that contains
all ground terms constructed from available constants and function symbols, and
it interprets each term by itself---that is, ${c^I = c}$ for each constant $c$,
and ${f^I(t_1,\dots,t_n) = f(t_1,\dots,t_n)}$ for each $n$-ary function symbol
$f$ and all ground terms ${t_1, \dots, t_n}$. Such $I$ can equivalently be seen
as a possibly infinite instance that contains a fact $R(t_1,\dots,t_n)$ for
each tuple of ground terms ${\langle t_1, \dots, t_n \rangle \in R^I}$; thus,
we use the two views of a Herbrand interpretation interchangeably. It is well
known that, for all first-order sentences $\varphi$ and $\psi$, we have
${\varphi \models \psi}$ if and only if ${I \models \psi}$ for each Herbrand
interpretation $I$ such that ${I \models \varphi}$.

\myparagraph{Logic Programming}
A \emph{(logic programming) rule} is a formula of the form \eqref{eq:rule},
where $R(\vec t)$ and $R_i(\vec t_i)$ are atoms possibly containing function
symbols and/or the equality predicate.
\begin{align}
    R(\vec t) \leftarrow R_1(\vec t_1) \wedge \dots \wedge R_n(\vec t_n)    \label{eq:rule}
\end{align}
Atom ${\head{r} = R(\vec t)}$ and conjunction ${\body{r} = R_1(\vec t_1) \wedge
\dots \wedge R_n(\vec t_n)}$ are the \emph{head} and \emph{body} of the rule
$r$, respectively. We often treat $\body{r}$ as a set of atoms; for example,
${R_i(\vec t_i) \in \body{r}}$ means that atom $R_i(\vec t_i)$ is a conjunct of
$\body{r}$. For $r$ a rule and $\sigma$ a substitution mapping all variables of
$r$ to ground terms, the ground rule $r\sigma$ is an \emph{instance} of $r$
\emph{via} $\sigma$. Predicate $\equals$ is always ordinary in rules, but we
still distinguish relational and equality atoms as outlined earlier. Each rule
must be \emph{safe}, meaning that each variable in the rule occurs in at least
one (not necessarily relational) body atom. A \emph{(logic) program} $P$ is a
finite set of rules. Rules are interpreted as first-order implications where
all variables are universally quantified, so the notions of satisfaction and
entailment (written $\models$ since $\equals$ is ordinary) are inherited from
first-order logic. A program $P$ and a base instance $B$ satisfy UNA if ${a =
b}$ holds for each pair of constants $a$ and $b$ such that ${P \cup B \models a
\equals b}$.

\myparagraph{Fixpoint Computation for Logic Programs}
For $I$ an instance, $T_P(I)$ is the result of extending $I$ with
$\sigma(\head{r})$ for each rule ${r \in P}$ and substitution $\sigma$ such
that ${\sigma(\body{r}) \subseteq I}$. Given a base instance $B$, we
inductively define a sequence of interpretations where ${I^0 = B}$ and
${I^{i+1} = T_P(I^i)}$ for each ${i \geq 0}$. The least fixpoint of $P$ on $B$
is given by ${\fixpoint{P}{B} = \bigcup_{i \geq 0} I^i}$. Instance
$\fixpoint{P}{B}$ is a subset-minimal Herbrand model of $P$ so, for each fact
$F$, we have ${P \cup B \models F}$ if and only if ${F \in \fixpoint{P}{B}}$.

\myparagraph{Rule Subsumption}
A rule $r$ \emph{subsumes} a rule $r'$ if there is a substitution $\sigma$ such
that ${\body{r\sigma} \subseteq \body{r'}}$ and ${\head{r\sigma} =
\head{r'\sigma}}$. Intuitively, $r'$ is then redundant in the presence of
$r$---that is, all consequences of $r'$ are derived by $r$ on any dataset.
Checking rule subsumption is $\textsc{NP}$-complete \cite{GottlobLeitsch85}.

\section{Query Answering over Second-Order Dependencies}\label{sec:so-dependencies}

In this section we discuss the problem of query answering over second-order
dependencies. Towards this end, in
Section~\ref{sec:so-dependencies:so-dependencies} we first introduce the notion
of \emph{generalised second-order dependencies} and present a motivating
example. In Section~\ref{sec:so-dependencies:chase} we discuss the chase
variant applicable to such dependencies. Finally, in
Section~\ref{sec:so-dependencies:goal-driven} we discuss the key difficulties
that need to be overcome to facilitate practically effective goal-driven query
answering for second-order dependencies.

\subsection{Generalised Second-Order Dependencies}\label{sec:so-dependencies:so-dependencies}

The algorithms we present in this paper are applicable to dependencies of the
form specified by the following definition. A key aspect of this definition is
that it allows for second-order quantification and equalities in implication
consequents, which is necessary to express certain kinds of schema compositions
\cite{DBLP:journals/tods/NashBM07, DBLP:journals/corr/abs-1106-3745}, as well
as to generalise the standard notion of EGDs.

\begin{definition}\label{def:generalised-so-dep}
    A \emph{generalised second-order (SO) dependency} is a second-order formula
    of the form ${\exists \vec f.(\delta_1 \wedge \dots \wedge \delta_n)}$
    where $\vec f$ is a tuple of function variables, and each conjunct
    $\delta_i$ with ${i \in \{ 1, \dots, n \}}$ is of the form
    \eqref{eq:generalised-SO-dep:conjunct} and it satisfies the following
    conditions.
    \begin{align}
        \forall \vec x_i.\big[\varphi_i(\vec x_i) \rightarrow \exists \vec y_i.\psi_i(\vec x_i, \vec y_i)\big] \label{eq:generalised-SO-dep:conjunct}
    \end{align}
    \begin{itemize}
        \item $\vec x_i$ and $\vec y_i$ are tuples of distinct individual
        variables.
    
        \item Formula ${\varphi_i(\vec x_i)}$ is the \emph{body} of $\delta_i$,
        written $\body{\delta_i}$. The formula is a conjunction of
        (i)~relational atoms whose arguments are constructed using constants
        and individual variables in $\vec x_i$, and (ii)~equality atoms whose
        arguments are terms of depth at most one constructed using constants,
        individual variables in $\vec x_i$, and function variables in $\vec f$.
    
        \item Formula ${\exists \vec y_i.\psi_i(\vec x_i,\vec y_i)}$ is the
        \emph{head} of $\delta_i$, written $\head{\delta_i}$. Formula
        $\psi_i(\vec x_i,\vec y_i)$ is a conjunction of atoms whose arguments
        are terms of depth at most one constructed using constants, individual
        variables in ${\vec x_i \cup \vec y_i}$, and function variables in
        $\vec f$.
    
        \item Conjunct $\delta_i$ must satisfy a variant of the \emph{safety}
        condition: each individual variable in $\vec x_i$ must appear in
        $\varphi_i(\vec x_i)$ in a relational atom.
    \end{itemize}
\end{definition}

Safety is needed to ensure \emph{domain independence}---that is, that the
satisfaction of a formula in an interpretation does not depend on the choice of
the interpretation domain \cite{DBLP:journals/tods/FaginKPT05}. Since a
generalised SO dependency can contain an arbitrary number of conjuncts, it
suffices to consider just one generalised second-order dependency (instead of a
set of dependencies). However, to simplify the notation, we often write down a
generalised SO dependency as a set of formulas of the form ${\varphi_i(\vec
x_i) \rightarrow \exists \vec y_i.\psi_i(\vec x_i, \vec y_i)}$ where
quantifiers $\exists \vec f$ and $\forall \vec x_i$ are left implicit.

As is customary in the related literature,
Definition~\ref{def:generalised-so-dep} allows a conjunct body to contain
functional terms only in equality atoms. However, unlike most related
definitions, we assume that all terms are of depth at most one---that is, terms
with nested function variables, such as $f(g(x))$, are disallowed. This
assumption allows us to simplify the presentation of our results, and it is
without loss of generality: \citet[Theorems 7.1 and
7.2]{DBLP:journals/corr/abs-1106-3745} have shown that, for each generalised SO
dependency $\Sigma$, there exists an equivalent generalised SO dependency
$\Sigma'$ where all terms are of depth at most one. This transformation does
not introduce fresh predicates, which is important when dependencies are used
to describe transformations between database instances. Alternatively, an SO
dependency containing terms of depth more than one can be reduced to the form
from Definition~\ref{def:generalised-so-dep} as described in the following
example. Analogously to the unfolding technique by \citet[Theorem
8.4]{DBLP:journals/tods/FaginKPT05}, this transformation introduces fresh
predicates that represent `intermediary' terms. While this changes the database
instance transformation, it does not affect query answers.

\begin{example}
Consider the generalised SO dependency $\Sigma$ containing only the conjunct
\eqref{ex:transform:1}. We can remove nesting of function variables in the head
by transforming the conjunct into
\eqref{ex:transform:2}--\eqref{ex:transform:3}, where $X$ is a fresh predicate.
Moreover, we can remove the nesting of function variables in the body by
transforming \eqref{ex:transform:3} into \eqref{ex:transform:4}. Formula
\eqref{ex:transform:4} is not safe, so we further transform it into
\eqref{ex:transform:5}, and we also introduce formulas \eqref{eq:dom-c} and
\eqref{eq:dom-R} to ensure that the predicate $\D$ enumerates the active domain.
\begin{align}
    A(f(g(x)))                                                                          & \rightarrow B(h(i(x)))    \label{ex:transform:1} \\
    X(x)                                                                                & \rightarrow B(h(x))       \label{ex:transform:2} \\
    A(f(g(x)))                                                                          & \rightarrow X(i(x))       \label{ex:transform:3} \\
    A(x'') \wedge x'' \equals f(x') \wedge x' \equals g(x)                              & \rightarrow X(i(x))       \label{ex:transform:4} \\
    A(x'') \wedge \D(x') \wedge \D(x) \wedge x'' \equals f(x') \wedge x' \equals g(x)   & \rightarrow X(i(x))       \label{ex:transform:5}
\end{align}
Formulas \eqref{ex:transform:2}, \eqref{ex:transform:3}, and
\eqref{ex:transform:5} satisfy all conditions of
Definition~\ref{def:generalised-so-dep}, and it is straightforward to see that,
on each base instance, they entail the same facts not involving the predicate
$X$ as the original generalised SO dependency $\Sigma$. This approach can be
easily extended to arbitrarily nested terms.
\end{example}

Apart from the technical assumption on term depth,
Definition~\ref{def:generalised-so-dep} generalises all dependency notions we
are aware of, which is why we call our dependencies `generalised'. In
particular, tuple-generating and equality-generating dependencies (TGDs and
EGDs) \cite{DBLP:journals/tcs/FaginKMP05} are obtained by disallowing function
variables. The existential SO dependencies ($\exists$SOEDs) by
\citet{DBLP:journals/tods/NashBM07} are obtained by disallowing first-order
existential quantifiers ($\exists \vec y_i$). The source-to-target SO
dependencies by \citet{DBLP:journals/corr/abs-1106-3745} are obtained by
further disallowing constants in all (relational and equality) atoms and
requiring all equality atoms in $\psi_i(\vec x_i,\vec y_i)$ to be of the form
${x \equals x'}$. The SO-TGDs by \citet{DBLP:journals/tods/FaginKPT05} are
obtained by further disallowing equality atoms in formulas $\psi_i(\vec
x_i,\vec y_i)$. Finally, the plain SO-TGDs by
\citet{DBLP:journals/jcss/Arenas0RR13} are obtained by further disallowing
equality atoms in $\varphi_i(\vec x_i)$.

Conjunctive queries (CQs) and unions of conjunctive queries (UCQs) are commonly
used in the literature to query dependencies
\cite{DBLP:journals/tcs/FaginKMP05}. While queries are often defined separately
from dependencies in the literature, we `absorb' queries into generalised SO
dependencies as follows.

\begin{definition}\label{def:queries}
    Let the set of predicates contain a distinct \emph{query predicate}
    $\predQ$ that can occur in a generalised SO dependency only in conjuncts of
    the form ${\forall \vec x_i.[\varphi_i(\vec x_i) \rightarrow \predQ(\vec
    x_i')]}$ where ${\vec x_i' \subseteq \vec x_i}$ and $\predQ$ does not occur
    in $\varphi_i(\vec x_i)$. A tuple of constants $\vec a$ is an \emph{answer}
    to $\predQ$ over a generalised SO dependency $\Sigma$ and a base instance
    $B$ if ${\{ \Sigma \} \cup B \modelsEq \predQ(\vec a)}$.\footnote{Note
    that, in the case of second-order dependencies, $\Sigma$ is a single
    formula and not a set; thus, ${\Sigma \cup B}$ is ill-defined, so we write
    ${\{ \Sigma \} \cup B}$ instead.}
\end{definition}

Predicate $\predQ$ is not allowed to occur in bodies, so it provides a `name'
that identifies a query inside $\Sigma$. Thus, a UCQ $\bigvee_{i=1}^n \exists
\vec y_i.\varphi_i(\vec x,\vec y_i)$ in the work by
\citet{DBLP:journals/tcs/FaginKMP05} is represented in our setting using
conjuncts ${\forall \vec x, \vec y_i.[\varphi_i(\vec x,\vec y_i) \rightarrow
\predQ(\vec x)]}$ of a generalised SO dependency. This is convenient because
our techniques analyse propagation of information through conjunctions, so
`absorbing' queries into dependencies allows us to apply this analysis to
queries in a seamless way.

As explained in Section~\ref{sec:preliminaries}, ${\{ \Sigma \} \cup B
\modelsEq \predQ(\vec a)}$ if and only if ${I \modelsEq \predQ(\vec a)}$ for
each interpretation $I$ such that ${I \modelsEq \Sigma}$ and ${I \modelsEq B}$.
If we interpret $I$ as a \emph{data exchange solution} by
\citet{DBLP:journals/tcs/FaginKMP05}, then Definition~\ref{def:queries}
captures exactly the notion of certain answers commonly used in the dependency
literature. A minor detail is that models are generally allowed to be infinite,
whereas solutions are necessarily finite; however, this distinction is
irrelevant to our work since we consider only generalised SO dependencies for
which the chase terminates and are thus finitely satisfiable. Another detail is
that models generally do not adopt the unique name assumption (i.e., distinct
constants can be interpreted as the same domain object); however, as explained
in Section~\ref{sec:preliminaries}, we can always query the equality predicate
to verify whether UNA is satisfied.

We next present a running example that we use throughout this paper to
illustrate key technical difficulties of goal-driven query answering, as well
as our proposed solutions.

\begin{example}\label{ex:run}
Let ${\exB = \{ C(a_1) \} \cup \{ S(a_{i-1},a_i) \mid 1 < i \leq k \}}$ for
arbitrary ${k \geq 1}$, and let $\exSigma$ be the generalised SO dependency
consisting of conjuncts \eqref{ex:run:Q}--\eqref{ex:run:U-eq} where $f$ is a
function variable.
\begin{align}
    R(x_1,x_2) \wedge f(x_1) \equals x_3 \wedge A(x_3) \wedge B(x_3)    & \rightarrow \predQ(x_1)           \label{ex:run:Q}      \\
    S(x_1,x_2)                                                          & \rightarrow \exists y.R(x_1,y)    \label{ex:run:S-R}    \\
    R(x_2,x_1) \wedge S(x_2,x_3) \wedge R(x_3,x_4)                      & \rightarrow x_1 \equals x_4       \label{ex:run:RSR-eq} \\
    C(x)                                                                & \rightarrow A(f(x))               \label{ex:run:C-Af}   \\
    C(x)                                                                & \rightarrow U(x,f(x))             \label{ex:run:C-Uf}   \\
    U(x_1,x_2)                                                          & \rightarrow B(f(x_2))             \label{ex:run:U-Bf}   \\
    U(x_1,x_2)                                                          & \rightarrow x_1 \equals x_2       \label{ex:run:U-eq}
\end{align}

Figure~\ref{fig:ex:run:int} shows an interpretation $I$ that satisfies
$\exSigma$ and $\exB$. The domain of $I$ consists of the vertices in the
figure, and the interpretation of unary and binary predicates consists of
vertices and arcs, respectively, labelled with the respective predicate. We
show in Example~\ref{ex:run-chase} that ${\{ \exSigma \} \cup \exB \modelsEq
\predQ(a_1)}$.
\end{example}

\begin{figure}[tb]
\begin{center}
\begin{tikzpicture}[scale=0.75, every node/.style={scale=0.75}]
    \tikzset{>=latex}

    \node[label=below:{$A, B, C, \predQ$}] (a1)    at (0,0)  {$a_1$}     ;
    \node                                  (a2)    at (2,0)  {$a_2$}     ;
    \node                                  (a3)    at (4,0)  {$a_3$}     ;
    \node                                  (adots) at (6,0)  {$\ldots$}  ;
    \node                                  (ak1)   at (8,0)  {$a_{k-1}$} ;
    \node                                  (ak)    at (10,0) {$a_k$}     ;
    \node                                  (n1)    at (0,2)  {$n_1$}     ;

    \draw[->,loop left,looseness=20] (a1)    to node[left]        {$U$} (a1)    ;
    \draw[->]                        (a1)    to node[below]       {$S$} (a2)    ;
    \draw[->]                        (a2)    to node[below]       {$S$} (a3)    ;
    \draw[->,dotted]                 (a3)    to                         (adots) ;
    \draw[->,dotted]                 (adots) to                         (ak1)   ;
    \draw[->]                        (ak1)   to node[below]       {$S$} (ak)    ;

    \draw[->]                        (a1)    to node[left]        {$R$} (n1)    ;
    \draw[->]                        (a2)    to node[below left]  {$R$} (n1)    ;
    \draw[->]                        (a3)    to node[below left]  {$R$} (n1)    ;
    \draw[->]                        (ak1)   to node[below left]  {$R$} (n1)    ;
    \draw[->]                        (ak)    to node[above right] {$R$} (n1)    ;

\end{tikzpicture}
\end{center}
\caption{A universal model for the generalised SO dependency $\exSigma$ and dataset $\exB$ from Example~\ref{ex:run}}\label{fig:ex:run:int}
\end{figure}

\subsection{The Chase for Generalised SO Dependencies}\label{sec:so-dependencies:chase}

A standard way to answer the query $\predQ$ over a dependency $\Sigma$ and a
base instance $B$ is to compute a universal model for $\Sigma$ and $B$ using an
appropriate variant of the chase algorithm, and then to simply `read off' the
facts that use the $\predQ$ predicate and consist of constants only. Numerous
chase variants have been proposed \cite{DBLP:journals/tods/MaierMS79,
DBLP:journals/jacm/BeeriV84, DBLP:journals/jcss/JohnsonK84,
DBLP:journals/tcs/FaginKMP05, DBLP:conf/kr/CaliGK08, DBLP:conf/pods/Marnette09,
DBLP:journals/pvldb/CateCKT09, DBLP:journals/is/MeccaPR12,
DBLP:conf/pods/DeutschNR08, DBLP:journals/tods/FaginKPT05,
DBLP:journals/corr/abs-1106-3745}. To the best of our knowledge,
\citet{DBLP:journals/corr/abs-1106-3745} presented the first chase variant that
is applicable to second-order dependencies with equality atoms in the heads.
Our key results are largely independent of the details of the chase procedure:
all that matters is that a suitable chase procedure exists and terminates on
$\Sigma$ and $B$. Nevertheless, to frame this paper properly, we next
recapitulate the chase variant by \citet{DBLP:journals/corr/abs-1106-3745},
which we extend in the obvious way to handle both first- and second-order
quantification. As we discuss later, distinguishing first- and second-order
quantification will provide opportunities for optimisation that would be lost
if we simply converted all first-order quantifiers into second-order ones. The
proof that this chase variant produces a universal model is a straightforward
variation of the proof by \citet{DBLP:journals/corr/abs-1106-3745}, so we do
not discuss it any further.

The chase algorithm takes as input a generalised second-order dependency
$\Sigma$ and a base instance $B$. The algorithm uses \emph{labelled
nulls}---objects whose existence is implied by first- and second-order
quantifiers. We distinguish a countably infinite set of \emph{base labelled
nulls} disjoint with the set of constants, and a countably infinite set of
\emph{functional labelled nulls} defined inductively as the smallest set
containing a distinct object $\fnnull{t}$ for each term of the form ${t =
f(u_1,\dots,u_n)}$ where $f$ is a function variable and ${u_1,\dots,u_n}$ are
constants or (base or functional) labelled nulls. Intuitively, base and
functional labelled nulls will be used to satisfy first-order and second-order
quantifiers, respectively. Furthermore, we assume that all constants and
labelled nulls are totally ordered using an arbitrary, but fixed ordering
$\prec$, where all constants precede all labelled nulls. Finally, for each
$n$-ary function variable $f$, the algorithm introduces a distinct fresh
$n+1$-ary predicate $\fnval{f}$.

We next specify how to evaluate a term in a set of facts. To this end, let $I$
be an instance consisting of facts constructed using the predicates of $\Sigma$
and $B$, the equality predicate $\equals$, and the predicates $\fnval{f}$, and
furthermore assume that ${\{ \fnval{f}(u_1,\dots,u_n,v),
\fnval{f}(u_1,\dots,u_n,v') \} \subseteq I}$ implies ${v = v'}$ for all $f$ and
${u_1,\dots,u_n}$. For $t$ a term of depth at most one, we define the value of
$t$ in $I$, written $t^I$, as follows.
\begin{itemize}
    \item For $t$ a constant or a (base or functional) labelled null, we define
    ${t^I = t}$.

    \item For ${t = f(u_1,\dots,u_n)}$, we define ${t^I = v}$ if there exists a
    fact $\fnval{f}(u_1,\dots,u_n,v) \in I$, and we define ${t^I = \fnnull{t}}$
    if $I$ does not contain a fact of the form $\fnval{f}(u_1,\dots,u_n,v)$.
\end{itemize}
Moreover, let ${F = R(t_1,\dots,t_n)}$ be an equality or a relational ground
atom where terms ${t_1,\dots,t_n}$ are all of depth at most one. We specify
whether $F$ is satisfied in $I$, written ${I \vdash F}$, as follows.
\begin{itemize}
    \item If $F$ is an equality ${t_1 \equals t_2}$, then ${I \vdash t_1
    \equals t_2}$ if ${t_1^I = t_2^I}$.
    
    \item Otherwise, ${I \vdash R(t_1,\dots,t_n)}$ if ${R(t_1^I,\dots,t_n^I)
    \in I}$.
\end{itemize}

Given $\Sigma$ and $B$, the chase algorithm constructs a sequence of pairs
${\langle I^0,\mu^0 \rangle, \langle I^1,\mu^1 \rangle, \dots}$. Each $I^i$ is
an instance containing facts constructed using constants and labelled nulls,
and using the predicates of $\Sigma$ and $B$, the equality predicate $\equals$,
and the predicates $\fnval{f}$. Each $\mu^i$ is a mapping from constants to
constants, and it will record constant representatives. For $\alpha$ an atom,
$\mu^i(\alpha)$ is the result of replacing each occurrence of a constant $c$ in
$\alpha$ on which $\mu^i$ is defined with $\mu^i(c)$; note that $c$ is replaced
even if it occurs as a proper subterm of an argument of $\alpha$. Moreover, for
$\alpha$ a conjunction, $\mu^i(\alpha)$ is the result of applying $\mu^i$ to
each conjunct of $\alpha$. The algorithm initialises $I^0$ to $B$ and the
mapping $\mu^0$ to the identity on all the constants occurring in $I^0$ and
$\Sigma$. For ${i \geq 0}$, a pair ${\langle I^{i+1}, \mu^{i+1} \rangle}$ is
obtained from ${\langle I^i, \mu^i \rangle}$ by applying one of the following
two steps.

To apply an \emph{equality step}, choose an equality ${s_1 \equals s_2 \in
I^i}$. Let ${t_1 = s_1}$ and ${t_2 = s_2}$ if ${s_1 \preceq s_2}$, and let
${t_1 = s_2}$ and ${t_2 = s_1}$ otherwise. Let $I^{i+1}_a$ be the instance
obtained from ${I^i \setminus \{ s_1 \equals s_2 \}}$ by replacing $t_2$ with
$t_1$, and let $I^{i+1}$ and $\mu^{i+1}(c)$ for each constant $c$ in the domain
of $\mu^i$ be defined as follows.
\begin{align*}
    I^{i+1}         & = I^{i+1}_a \cup \Big\{ v \equals v' \mid v \neq v' \text{ and } \exists u_1,\dots,u_n \text{ s.t. } \{ \fnval{f}(u_1,\dots,u_n,v), \fnval{f}(u_1,\dots,u_n,v') \} \subseteq I^{i+1}_a \Big \} \\
    \mu^{i+1}(c)    & = \begin{cases}
                            t_1         & \text{if } \mu^i(c) = t_2, \\
                            \mu^i(c)    & \text{otherwise}.
                        \end{cases}
\end{align*}

To apply a \emph{dependency step}, choose a conjunct ${\forall \vec
x.\big[\varphi(\vec x) \rightarrow \exists \vec y.\psi(\vec x, \vec y)]}$ of
$\Sigma$ and a substitution $\sigma$ such that ${I^i \vdash \mu^i(\varphi(\vec
x)\sigma)}$, and ${I^i \not\vdash \mu^i(\psi(\vec x, \vec y)\sigma')}$ for each
substitution $\sigma'$ that extends $\sigma$ by mapping the variables $\vec y$
to the terms of $I^i$. Let $\sigma''$ be the substitution that extends $\sigma$
by mapping the variables $\vec y$ to fresh base labelled nulls not occurring in
$I^i$. Furthermore, let $I^{i+1}_a$ be $I^i$ extended as follows: for each term
of the form ${t = f(u_1,\dots,u_n)}$ that occurs as an argument of an atom in
$\mu^i(\psi(\vec x, \vec y)\sigma'')$ such that $I^i$ does not contain a fact
of the form $\fnval{f}(u_1,\dots,u_n,v)$, add the fact
$\fnval{f}(u_1,\dots,u_n,\fnnull{t})$ to $I^{i+1}_a$. Finally, let ${\mu^{i+1}
= \mu^i}$, and let
\begin{displaymath}
    I^{i+1} = I^{i+1}_a \cup \Big\{ R(t_1{}^{I^{i+1}_a},\dots,t_n{}^{I^{i+1}_a}) \mid R(t_1,\dots,t_n) \text{ is a fact in } \mu^i(\psi(\vec x,\vec y)\sigma'') \Big\}.
\end{displaymath}

These steps are applied exhaustively as long as possible, but the equality step
is applied with higher priority. This process does not terminate in general,
but if it does, the final pair ${\langle I^n,\mu^n \rangle}$ in the sequence is
called \emph{the chase} of $\Sigma$ and $B$ and it contains all query answers:
${\{ \Sigma \} \cup B \modelsEq \predQ(\vec a)}$ if and only if
${\predQ(\mu^n(\vec a)) \in I^n}$ for each fact of the form $\predQ(\vec a)$
where $\vec a$ is a vector of constants and $\mu^n(\vec a)$ is the vector of
constants obtained by applying $\mu^n$ to each element of $\vec a$.

\begin{example}\label{ex:run-chase}
We next discuss how this chase variant is applied to $\exSigma$ and $\exB$ from
Example~\ref{ex:run}. As in the standard chase, applying a dependency step to
the TGD \eqref{ex:run:S-R} and facts $S(a_i,a_{i+1})$ introduces facts
$R(a_i,n_i)$, where $n_i$ are base labelled nulls since they are introduced by
first-order quantifiers in \eqref{ex:run:S-R}. Applying the EGD
\eqref{ex:run:RSR-eq} derives equalities of the form ${n_i \equals n_{i+1}}$,
and applying the equality step eventually merges all $n_i$ into one object
($n_1$ in Figure~\ref{fig:ex:run:int}).

Consider now applying dependency steps to the second-order conjuncts
\eqref{ex:run:C-Af}--\eqref{ex:run:U-Bf}, all of which contain the function
variable $f$. The chase algorithm needs to produce the same labelled null
whenever $f$ is applied to the same argument. The algorithm first applies
\eqref{ex:run:C-Af} to $C(a_1)$. Since this is the first time $f$ is applied to
$a_1$, the algorithm introduces a fresh functional labelled null $n_2$ and adds
the fact $A(n_2)$; moreover, the algorithm also introduces an auxiliary fact
$\fnval{f}(a_1,n_2)$ (not shown in Figure~\ref{fig:ex:run:int}) to record that
applying $f$ to $a_1$ produces $n_2$. Next, the algorithm applies
\eqref{ex:run:C-Uf} to $C(a_1)$. The presence of $\fnval{f}(a_1,n_2)$ states
that applying $f$ to $a_1$ produces $n_2$, and so the dependency step
introduces $U(a_1,n_2)$ instead of introducing a fresh labelled null. The
algorithm next applies \eqref{ex:run:U-Bf} to $U(a_1,n_2)$; since $f$ has
previously not been applied to $n_2$, the algorithm introduces a fresh labelled
null $n_3$, records this application using an auxiliary fact
$\fnval{f}(n_2,n_3)$, and adds a fact $B(n_3)$.

Applying the EGD \eqref{ex:run:U-eq} to $U(a_1,n_2)$ derives the equality ${a_1
\equals n_2}$. As in the standard chase, the equality rule handles this by
replacing $n_2$ with $a_1$; hence, facts $A(n_2)$, $U(a_1,n_2)$,
$\fnval{f}(a_1,n_2)$, and $\fnval{f}(n_2,n_3)$ are replaced by $A(a_1)$,
$U(a_1,a_1)$, $\fnval{f}(a_1,a_1)$, and $\fnval{f}(a_1,n_3)$, respectively.
However, note that $\fnval{f}(a_1,a_1)$ and $\fnval{f}(a_1,n_3)$ no longer
provide a unique value for $f$ on $a_1$. The equality step corrects all such
discrepancies before applying further dependency steps. In our example, this is
achieved by deriving the equality ${a_1 \equals n_3}$, which is handled by
replacing $n_3$ with $a_1$. Thus, $B(n_3)$ and $\fnval{f}(a_1,n_3)$ are
replaced by $B(a_1)$ and $\fnval{f}(a_1,a_1)$, respectively (and duplicate
facts are removed).

Finally, consider applying \eqref{ex:run:Q}. The safety condition of
Definition~\ref{def:generalised-so-dep} ensures that each variable in a body of
a conjunct occurs in a relational atom, so evaluating the relational body atoms
instantiates all body variables. In our example, variables $x_1$ and $x_3$ are
matched to $a_1$, and variable $x_2$ is matched to the labelled null $n_1$.
Furthermore, auxiliary fact $\fnval{f}(a_1,a_1)$ reflects that the value of $f$
on $a_1$ is $a_1$, so the body equality ${f(x_1) \equals x_3}$ is satisfied.
Thus, the dependency step derives $\predQ(a_1)$, as shown in
Figure~\ref{fig:ex:run:int}.
\end{example}

We next briefly discuss the role of the mappings $\mu^i$. Since we do not
assume UNA, the chase can replace a constant with a constant, and mappings
$\mu^i$ keep track of such replacements. Thus, ${\mu^i(c) = c'}$ means that
constant $c'$ represents $c$---that is, $c$ and $c'$ were equated (possibly
indirectly), and so each occurrence of $c$ in a fact in $I^i$ can also be
understood as an occurrence of $c'$. These mappings are used in two ways.
First, the dependency step applies $\mu^i$ to both the body and the head of a
dependency conjunct, which ensures that all constants in the conjunct are
normalised to the constants occurring in $I^i$. Second, the final mapping
$\mu^n$ is used to normalise the answer $\vec a$ when checking
${\predQ(\mu^n(\vec a)) \in I^n}$. If UNA is desired, the algorithm can be
modified to report a contradiction when two constants are equated, in which
case mappings $\mu^i$ need not be maintained.

\citet{DBLP:journals/tods/FaginKPT05} presented a chase variant for SO
dependencies without head equality atoms. This chase variant never derives any
equalities, so the facts with $\fnval{f}$ predicates do not need to be
maintained; instead, the chase derives facts with functional terms of arbitrary
depth. It is possible to extend this chase variant to our generalised SO
dependencies. Roughly speaking, when the equality step is applied to ${t_1
\equals t_2}$ with ${t_1 \prec t_2}$ and thus term $t_2$ is to be replaced by
term $t_1$, replacements should be made at proper subterm positions as well;
for example, a fact $A(g(f(t_2)))$ must be normalised to $A(g(f(t_1)))$. To
ensure completeness, the ordering $\prec$ on terms must be \emph{well-founded},
which intuitively ensures that each term can be normalised in a finite number
of steps. Such an algorithm is equivalent to the chase variant by
\citet{DBLP:journals/corr/abs-1106-3745} that we outlined above, and it can be
seen as an instance of the \emph{paramodulation} calculus used in first-order
theorem proving \cite{NieuwenhuisRubio:HandbookAR:paramodulation:2001}.

Although the chase does not terminate in general, numerous sufficient
conditions for chase termination on first-order dependencies are known
\cite{DBLP:journals/tcs/FaginKMP05, DBLP:conf/ijcai/KrotzschR11,
DBLP:journals/jair/GrauHKKMMW13}. Moreover, the chase is known to terminate for
all classes of second-order dependencies studied in the literature on schema
mappings \cite{DBLP:journals/corr/abs-1106-3745, DBLP:journals/tods/FaginKPT05,
DBLP:journals/tods/NashBM07, DBLP:journals/jcss/Arenas0RR13,
DBLP:journals/jodsn/CateHK16}, and the results we present in this paper are
applicable in all of these cases.

\subsection{Goal-Driven Query Answering}\label{sec:so-dependencies:goal-driven}

Computing the chase in full can be expensive, and the computation needs to be
repeated whenever the base instance changes; consequently, chase-based query
answering can be unsuitable for applications where data changes frequently.
However, the amount of work needed to answer a query can be significantly
reduced by using an algorithm that aims to draw only the inferences relevant to
the query. Example~\ref{ex:run} illustrates two kinds of opportunity for
optimisation. First, $\predQ(a_1)$ remains entailed even if we remove
\eqref{ex:run:RSR-eq} from $\exSigma$. In other words, all logical consequences
of a conjunct of a generalised SO dependency can be irrelevant to the query, so
the entire conjunct can be deleted. Second, the derivation of $\predQ(a_1)$
depends on the conclusion $R(a_1,n_1)$ of the conjunct \eqref{ex:run:S-R}, but
not on the conclusions $R(a_i,n_1)$ with ${i \geq 2}$. Thus, we cannot simply
delete conjunct \eqref{ex:run:S-R}, but must instead analyse its instances and
identify the relevant ones.

Various \emph{backward chaining} techniques realise goal-driven query answering
by working backwards from the query to identify the premises relevant to the
query answers. For example, SLD resolution \cite{DBLP:conf/ifip/Kowalski74} in
logic programming searches backwards for proofs from potential query answers.
For relational facts, backward chaining can be roughly realised by matching the
fact in question to the head of a dependency conjunct, and treating the
conjunct's body as a query whose answers determine all relevant premises---that
is, we `invert' the dependency chase step. However, the premises producing a
fact via the equality chase step are not even present in the step's result, so
we do not see a practical way to `invert' equality chase steps.

A na{\"i}ve attempt to overcome this problem might be to treat $\equals$ as an
ordinary predicate and axiomatise its properties as discussed in
Section~\ref{sec:preliminaries}; then, we can use backward chaining for all
inferences, including the ones that involve the equality predicate. However,
such an approach exhibits at least two problems. First, the congruence axioms
\eqref{eq:cong} often produce many redundant inferences: given ${a \equals b}$,
these axioms `copy' all facts containing $a$ in some position into facts
containing $b$ at that position. This can make the fixpoint computation of
$\Sigma$ and $\EQ{\Sigma}$ very inefficient \cite{mnph15owl-sameAs-rewriting},
and it can prevent effective backward chaining. Second, as we discussed in
Example~\ref{ex:run-chase}, to ensure that function variables are interpreted
as functions, we must introduce a functional reflexivity axiom \eqref{eq:fnref}
for each function variable. But then, as soon as we derive ${a \equals b}$, the
functional reflexivity axiom for $f$ derives ${f(a) \equals f(b)}$, ${f(f(a))
\equals f(f(b))}$, and so on. In other words, even if the chase of a set of
generalised SO dependencies $\Sigma$ terminates, the fixpoint computation of
$\Sigma$ and $\EQ{\Sigma}$ does not terminate whenever $\Sigma$ contains at
least one function variable.

For first-order dependencies, the first problem can be overcome using
\emph{singularisation} \cite{DBLP:conf/pods/Marnette09}---a technique that
omits congruence axioms \eqref{eq:cong} and modifies the dependency bodies so
that joins among body atoms take equalities into account. The technique was
shown to be effective in practice \cite{DBLP:journals/jair/GrauHKKMMW13}, and
\citet{DBLP:conf/aaai/BenediktMT18} also used it successfully for goal-driven
query answering over first-order dependencies. This idea was used with
second-order dependencies \cite{DBLP:journals/jodsn/CateHK16}, but the
resulting approach is \emph{incomplete} because it does not take functional
reflexivity into account---that is, on Example~\ref{ex:run}, singularisation as
used by \citet{DBLP:journals/jodsn/CateHK16} \emph{does not} derive
$\predQ(a_1)$. Completeness can be easily recovered by introducing functional
reflexivity axioms, but this prevents termination of fixpoint computation in
the same way as for explicit equality axiomatisation.

To summarise, existing goal-driven query answering techniques are difficult to
apply to dependencies containing equality and function variables. In
Section~\ref{sec:answering} we present an approach that overcomes these
challenges and is effective in practice.

\section{Query Answering over Generalised Second-Order Dependencies}\label{sec:answering}

We now present our goal-driven approach for answering queries over generalised
second-order dependencies. We first discuss in
Section~\ref{sec:answering:overview} the approach end-to-end, and then in
Sections~\ref{sec:answering:second-order}--\ref{sec:answering:desg} we discuss
individual steps in detail.

\subsection{Solution Overview}\label{sec:answering:overview}

\begin{algorithm}[tb]
\caption{Compute the answers to a query defined by predicate $\predQ$ over a generalised second-order dependency $\Sigma$ and a base instance $B$}\label{alg:answer-query}
\begin{algorithmic}[1]
    \State $\Sigma_1 \defeq \fol{\Sigma}$                                                                                                   \label{alg:answer-query:fol}
    \State $\Sigma_2 \defeq \text{a singularisation of } \Sigma_1$                                                                          \label{alg:answer-query:sg}
    \State $P_3 \defeq \sk{\Sigma_2}$                                                                                                       \label{alg:answer-query:sk}
    \State $P_4 \defeq \relevance{P_3}{B}$                                                                                                  \label{alg:answer-query:relevance}
    \State $P_5 \defeq \magic{P_4}$                                                                                                         \label{alg:answer-query:magic}
    \State $P_6 \defeq \defun{P_5}$                                                                                                         \label{alg:answer-query:defun}
    \State $P_7 \defeq \desg{P_6}$                                                                                                          \label{alg:answer-query:desg}
    \State Compute the chase $\langle I,\mu \rangle$ of $P_7$ and $B$ using the chase variant from Section~\ref{sec:so-dependencies:chase}  \label{alg:answer-query:chase}
    \State Output each fact $\predQ(\vec a)$ such that $\predQ(\mu(\vec a)) \in I$                                                          \label{alg:answer-query:output}
\end{algorithmic}
\end{algorithm}

Our approach is shown in Algorithm~\ref{alg:answer-query}. The algorithm takes
as input a generalised SO dependency $\Sigma$ and a base instance $B$, and it
outputs all facts of the form $\predQ(\vec a)$ such that ${\{ \Sigma \} \cup B
\modelsEq \predQ(\vec a)}$. This is achieved by transforming $\Sigma$ into a
logic program $P_7$ such that the query answers of ${P_7 \cup \EQ{P_7}}$ and
$\Sigma$ on $B$ coincide. All steps apart from the one in
line~\ref{alg:answer-query:relevance} are independent of $B$: they depend only
on $\Sigma$ and the definition of $\predQ$ within $\Sigma$. Furthermore, while
the transformation in line~\ref{alg:answer-query:relevance} can be optimised
for a specific base instance, it can always be applied in a way that is
independent of $B$. To analyse equality inferences, several steps of our
transformation treat equality as an ordinary predicate. However, the equality
predicate is treated in the resulting program as `true' equality (modulo
certain details that we discuss shortly), so line~\ref{alg:answer-query:chase}
can use the chase variant from Section~\ref{sec:so-dependencies:chase} that
uses the equality step. In other words, our approach does not impose a
performance penalty in the handling of equality. We next describe the various
stages of our transformation and provide forward pointers to subsections where
each is discussed in depth.

First, the input generalised SO dependency $\Sigma$ is transformed to a set of
first-order formulas $\Sigma_1$ by dropping all second-order quantifiers and
converting function variables into function symbols
(line~\ref{alg:answer-query:fol}). We show that this transformation does not
affect query answers. Thus, we can develop our algorithms in the framework of
first-order logic. In particular, many of our proofs use Herbrand models, whose
universe is constructed using function symbols from the signature; however,
$\vec f$ in a generalised SO dependency are \emph{quantified second-order
variables} rather than a part of the signature, so the notion of a Herbrand
model does not apply directly. Although we transform the SO dependencies to
first-order formulas, the resulting formulas contain function symbols, so they
are strictly more expressive than standard TGDs and EGDs, and in fact this is
the main source of technical complexity in our work. Although most of our
results are expressed in first-order logic, we begin with second-order
dependencies to position our work within the literature. Second-order
dependencies have attracted significant interest, whereas first-order
dependencies with function symbols, to the best of our knowledge, have not been
studied in this context. As we discuss in
Section~\ref{sec:answering:second-order}, the difference between function
variables and function symbols largely disappears in the context of query
answering, but we retain it to keep the formal development precise.

Next, the algorithm applies an extension of the singularisation technique by
\citet{DBLP:conf/pods/Marnette09} that correctly handles second-order
dependencies. This step uses a set of axioms $\SG{\Sigma}$ that relaxes
$\EQ{\Sigma}$: set $\SG{\Sigma}$ does not contain the congruence axioms
\eqref{eq:cong}, and it uses a modified version of the functional reflexivity
axioms \eqref{eq:fnref} that does not necessarily imply infinitely many
equalities. The step transforms $\Sigma_1$ into a singularised set of
dependencies $\Sigma_2$ (line~\ref{alg:answer-query:sg}); this transformation
is not deterministic in the sense that many different $\Sigma_2$ can be
produced from $\Sigma_1$, but any of these can be used in the rest of the
algorithm. We prove that ${\{ \Sigma \} \cup B \modelsEq \predQ(\vec a)}$ if
and only if ${\Sigma_2 \cup \SG{\Sigma_2} \cup B \models \predQ(\vec a)}$ for
each fact $\predQ(\vec a)$---that is, singularisation preserves query answers.
A key issue is to ensure that the modified functional reflexivity axioms derive
all relevant equalities, and this is perhaps the most technically challenging
result of this paper. We discuss singularisation in detail in
Section~\ref{sec:answering:singularisation}.

The first-order existential quantifiers of $\Sigma_2$ are next replaced by
function symbols using \emph{Skolemisation}. The transformation is standard and
has been widely used in the literature. However, we apply Skolemisation to
$\Sigma_2$, where equality is an ordinary predicate. Thus, functional
reflexivity axioms have already been added for the function symbols that
correspond to the function variables in $\Sigma_1$, and so no such axioms are
needed for the function symbols introduced by Skolemisation. This observation
opens the door to important optimisations of our approach. Moreover, if
$\Sigma$ consists only of standard TGDs and EGDs, then no functional
reflexivity axioms are needed at all, and the approach presented in this paper
reduces to the approach for first-order dependencies by
\citet{DBLP:conf/aaai/BenediktMT18}. The Skolemisation of $\Sigma_2$ produces a
logic program $P_3$ (line~\ref{alg:answer-query:sk}) that preserves all query
answers: ${\Sigma_2 \cup \SG{\Sigma_2} \cup B \models \predQ(\vec a)}$ if and
only if ${P_3 \cup \SG{P_3} \cup B \models \predQ(\vec a)}$ for each fact
$\predQ(\vec a)$.

Program $P_3$ is now amenable to analysing inferences by backward chaining,
which is done in two steps: the relevance analysis
(line~\ref{alg:answer-query:relevance}) removes from $P_3$ the rules that
provably do not contribute to query answers, and an adaptation of the magic
sets transformation by \citet{DBLP:journals/jlp/BeeriR91}
(line~\ref{alg:answer-query:magic}) prunes irrelevant inferences from $P_4$. We
discuss our techniques in Sections~\ref{sec:answering:relevance}
and~\ref{sec:answering:magic}, respectively, and show that they preserve query
answers: ${P_i \cup \SG{P_i} \cup B \models \predQ(\vec a)}$ if and only if
${P_{i+1} \cup \SG{P_{i+1}} \cup B \models \predQ(\vec a)}$ for each fact
$\predQ(\vec a)$ and ${i \in \{ 3, 4 \}}$. As we shall see, the two techniques
are complementary: neither technique alone subsumes the other one on our
example.

The original query can now be answered over ${P_5 \cup \SG{P_5}}$, and this
will typically involve fewer inferences than computing the chase of $\Sigma$.
However, equality is still treated as an ordinary predicate in $P_5$, so
reasoning with $P_5$ can be inefficient. The remaining two steps bring $P_5$
into a form that can be evaluated using the chase variant from
Section~\ref{sec:so-dependencies:chase}.

The magic sets transformation can introduce relational body atoms containing
function symbols, which cannot be handled by the chase variant outlined in
Section~\ref{sec:so-dependencies:chase}. Thus, function symbols are next
eliminated from all relational body atoms (line~\ref{alg:answer-query:defun}).
For example, a rule ${A(f(x)) \rightarrow B(f(x))}$ is transformed into
${A(z_{f(x)}) \wedge \fnpred{f}(x,z_{f(x)}) \rightarrow B(f(x)) \wedge
\fnpred{f}(x,f(x))}$ where $\fnpred{f}$ is a fresh binary predicate uniquely
associated with $f$, and $z_{f(x)}$ is a fresh variable uniquely associated
with the term $f(x)$. Intuitively, $\fnpred{f}$ is axiomatised to associate the
terms containing the function symbol $f$ with the terms' arguments, which
allows us to refer to functional terms in the rule body via atoms with the
$\fnpred{f}$ predicate. We discuss this step in detail in
Section~\ref{sec:answering:defun} and show that ${P_5 \cup \SG{P_5} \cup B
\models \predQ(\vec a)}$ if and only if ${P_6 \cup \SG{P_6} \cup B \models
\predQ(\vec a)}$ for each fact $\predQ(\vec a)$.

The magic sets transformation can introduce rules that are not safe (i.e.,
where a rule variable does not occur in a relational body atom), which also
cannot be handled by the chase variant outlined in
Section~\ref{sec:so-dependencies:chase}. Moreover, singularisation can
introduce many equalities in a rule body, which increases the number of joins
needed to evaluate a rule. The final step, which we discuss in
Section~\ref{sec:answering:desg}, reverses singularisation
(line~\ref{alg:answer-query:desg}): it eliminates all equalities from rule
bodies, thus making the produced rules safe and easier to evaluate. The
resulting program $P_7$ preserves all query answers. Note, however, that
program $P_7$ can contain function symbols that must be handled as outlined in
Examples~\ref{ex:run} and~\ref{ex:run-chase}. Hence, standard chase variants
for TGDs and EGDs are not sufficient, but query answers can be computed
efficiently as described in Section~\ref{sec:so-dependencies:chase}.

The chase computation in line~\ref{alg:answer-query:chase} terminates whenever
the fixpoint of ${P_3 \cup \SG{P_3} \cup B}$ is finite, which can be ensured in
several ways. First, \citet{DBLP:conf/pods/Marnette09} showed that, if $\Sigma$
is a set of super-weakly acyclic TGDs and EGDs (which is more general than weak
acyclicity by \citet{DBLP:journals/tcs/FaginKMP05}), then the fixpoint of ${P_3
\cup \SG{P_3} \cup B}$ is finite for each possible singularisation. Second, for
all dependency classes from the schema mapping literature
\cite{DBLP:journals/corr/abs-1106-3745, DBLP:journals/tods/FaginKPT05,
DBLP:journals/tods/NashBM07, DBLP:journals/jcss/Arenas0RR13,
DBLP:journals/jodsn/CateHK16}, the fixpoint of ${P_3 \cup \SG{P_3} \cup B}$ is
finite regardless of which singularisation is chosen in
line~\ref{alg:answer-query:sg}. Third, \citet{DBLP:journals/jair/GrauHKKMMW13}
studied chase termination for singularised dependencies. They showed that the
finiteness of the fixpoint of ${P_3 \cup \SG{P_3} \cup B}$ sometimes depends on
the choice of singularisation, and they discussed ways to increase the
likelihood of termination. Our approach is applicable to (at least) all of
these cases.

\subsection{Eliminating Second-Order Quantification}\label{sec:answering:second-order}

In the first step, we drop all second-order quantifiers in $\Sigma$ and convert
the function variables into function symbols. This produces formulas of the
form specified in Definition~\ref{def:generalised-fo-dep}.

\begin{definition}\label{def:generalised-fo-dep}
    A \emph{generalised first-order (FO) dependency} has the form ${\forall
    \vec x.[\varphi(\vec x) \rightarrow \exists \vec y.\psi(\vec x, \vec y)]}$
    where $\varphi(\vec x)$ and $\psi(\vec x, \vec y)$ satisfy the same
    conditions as in Definition~\ref{def:generalised-so-dep}, but all terms use
    function symbols instead of function variables.
\end{definition}

The notions of the head and body of a generalised FO dependency carry over
naturally from the SO case, and we assume that any query is `absorbed' into a
set of generalised FO dependencies in the same way as in
Definition~\ref{def:queries}. The only difference between generalised SO and FO
dependencies is that the former contain function variables, whereas the latter
contain function symbols; we discuss shortly the subtleties of this
distinction. Moreover, note that first-order dependencies such as TGDs and EGDs
do not support function symbols, whereas generalised FO dependencies support
both function symbols and existentially quantified individual variables. Thus,
generalised FO dependencies are strictly more expressive than standard
TGDs and EGDs.

We next show that transforming a generalised SO dependency to a set of
generalised FO dependencies does not affect the query answers.

\begin{definition}\label{def:fol}
    For $\Sigma$ a generalised SO dependency as in
    Definition~\ref{def:generalised-so-dep}, set $\fol{\Sigma}$ contains the
    corresponding generalised FO dependency ${\forall \vec x_i.[\varphi_i(\vec
    x_i) \rightarrow \exists \vec y_i.\psi_i(\vec x_i, \vec y_i)]}$ for each
    conjunct of $\Sigma$.
\end{definition}

\begin{proposition}\label{prop:fol}
    For each generalised second-order dependency $\Sigma$, each base instance
    $B$, and each fact of the form $\predQ(\vec a)$, it holds that ${\{ \Sigma
    \} \cup B \modelsEq \predQ(\vec a)}$ if and only if ${\fol{\Sigma} \cup B
    \modelsEq \predQ(\vec a)}$.
\end{proposition}

\begin{proof}
We show the contrapositive: there exists an interpretation $I$ such that ${I
\modelsEq B}$, ${I \modelsEq \Sigma}$, and ${I \not\modelsEq \predQ(\vec a)}$
if and only if there exists an interpretation $I'$ such that ${I' \modelsEq
\fol{\Sigma} \cup B}$ and ${I' \not\modelsEq \predQ(\vec a)}$. For the
($\Leftarrow$) direction, we can obtain $I$ from $I'$ by simply `erasing' the
interpretation of the first-order function symbols of $\fol{\Sigma}$. For the
($\Rightarrow$) direction, we can always extend $I$ to $I'$ by interpreting the
function symbols of $\fol{\Sigma}$ in a way whose existence is guaranteed by
${I \modelsEq \Sigma}$.
\end{proof}

While Proposition~\ref{prop:fol} may seem to suggest that second-order
quantification is irrelevant to the expressivity of dependencies, the following
example illustrates the differences.

\begin{example}
Second-order dependencies were introduced to describe data mappings. For
example, one can use the second-order formula ${\varphi = \exists f.\forall
x_1,x_2.[R(x_1,x_2) \rightarrow S(x_1,f(x_2))]}$ to specify a mapping between
base instances where the source and target instances contain facts with the $R$
and $S$ predicates, respectively. One possible use of the formula $\varphi$ is
to check whether the base instance ${B_2 = \{ S(a,d), S(b,d) \}}$ can be
obtained by applying $\varphi$ to the base instance ${B_1 = \{ R(a,c), R(b,c)
\}}$. To achieve this in our framework, we can form an interpretation $I$ as
the union of $B_1$ and $B_2$---that is, we set the interpretation domain to
${\Delta^I = \{ a, b, c, d \}}$, and we interpret the predicates as ${R^I = \{
\langle a, c \rangle, \langle b, c \rangle \}}$ and ${S^I = \{ \langle a, d
\rangle, \langle b, d \rangle \}}$; then, our task corresponds to checking
whether $I$ satisfies $\varphi$. For $\pi$ a valuation where ${f^\pi(c) = d}$,
we have ${I,\pi \modelsEq \forall x_1,x_2.[R(x_1,x_2) \rightarrow
S(x_1,f(x_2))]}$, so ${I \modelsEq \varphi}$. Now let ${\varphi' = \forall
x_1,x_2.[R(x_1,x_2) \rightarrow S(x_1,f(x_2))]}$ be the first-order formula
obtained from $\varphi$ by the transformation in Definition~\ref{def:fol}; note
that $f$ in $\varphi'$ is a function symbol rather than a function variable.
The question whether $I$ satisfies $\varphi'$ is ill-formed since $f^I$ is not
defined in $I$; thus, formula $\varphi'$ does not describe a transformation of
$B_1$ into $B_2$.

Since base instances cannot contain function symbols, first-order formulas with
function symbols cannot describe transformations between base instances.
Second-order quantification overcomes this problem. In our example, function
variable $f$ is not interpreted in $I$, so, to check whether ${I \modelsEq
\varphi}$ holds, we need to guess a valuation for $f$. In contrast, the
interpretation for the function symbol $f$ is given explicitly in the
interpretation $I'$, so there is nothing to guess. Consequently, checking
satisfaction of second-order dependencies is $\normalfont\textsc{NP}$-hard in
data complexity \cite{DBLP:journals/tods/FaginKPT05}, whereas the problem is
polynomial for first-order dependencies with function symbols.

Query answering, however, is based on entailment: we identify facts that hold
in \emph{all} models of a base instance and a generalised SO dependency. Since
second-order quantifiers do not occur in the scope of other quantifiers in a
generalised SO dependency, we can eliminate SO quantifiers as in
Definition~\ref{def:fol} in a way that can be understood as `second-order
Skolemisation'. This transformation does not produce an equivalent formula
(e.g., $\varphi$ and $\varphi'$ are not equivalent), but it does not affect
query answers.
\end{example}

This discussion reveals that the chase variant from
Section~\ref{sec:so-dependencies:chase} can also be used to reason with
generalised FO dependencies by handling function symbols in exactly the same
way as function variables. The only difference is how to interpret the facts
with the auxiliary $\fnval{f}$ predicates in the chase result: with generalised
SO dependencies, these describe a valuation for the existentially quantified
function variables $\vec f$; in contrast, with generalised FO dependencies,
they provide an interpretation for the function symbols in a universal model.
Either way, a query can be answered by computing the chase result and then
reading off the $\predQ$ facts consisting of only constants.

\subsection{Singularisation of Generalised First-Order Dependencies}\label{sec:answering:singularisation}

As we explained in Section~\ref{sec:answering:overview}, we use the
\emph{singularisation} technique by \citet{DBLP:conf/pods/Marnette09} to
axiomatise the properties of equality in order to be able to analyse equality
inferences. The key idea is that replicating the consequences of equality via
congruence axioms can be avoided if we modify the dependency bodies so that
joins among body atoms take equalities into account. We discuss this intuition
in detail in Example~\ref{ex:sg}. This idea was used to optimise sufficient
conditions for chase termination and was shown to be very effective in practice
\cite{DBLP:journals/jair/GrauHKKMMW13}. Furthermore,
\citet{DBLP:journals/jodsn/CateHK16} applied singularisation to second-order
dependencies, but their formulation does not take into account functional
reflexivity axioms and is thus incomplete: it \emph{does not} derive
$\predQ(a_1)$ on Example~\ref{ex:run}. Completeness can easily be recovered by
simply including functional reflexivity axioms, but this prevents chase
termination as we discuss in Example~\ref{ex:sg} below.

We next present a variant of singularisation that overcomes both issues. Our
formulation includes functional reflexivity restricted to only the `relevant'
domain elements that occur in the interpretation of a relational predicate. In
other words, functional reflexivity is restricted to derive only equalities
that can affect a fact with a relational predicate.

\begin{definition}\label{def:sg}
    Let ${\delta = \forall \vec x.[\varphi(\vec x) \rightarrow \exists \vec
    y.\psi(\vec x, \vec y)]}$ be a generalised first-order dependency. A
    \emph{singularisation step} is applicable to argument $t_k$ of an atom
    ${R(t_1, \dots, t_k, \dots, t_m)}$ occurring in $\varphi(\vec x)$ if
    predicate $R$ is different from $\equals$ and term $t_k$ is
    \begin{itemize}
        \item a constant,

        \item a variable that occurs as the $i$-th argument of ${R(t_1, \dots,
        t_k, \dots, t_m)}$ where ${i \neq k}$, or

        \item a variable that occurs in $\varphi(\vec x)$ as an argument of a
        relational atom other than ${R(t_1, \dots, t_k, \dots, t_m)}$.
    \end{itemize}
    Applying such a step to $\delta$ replaces ${R(t_1, \dots, t_k, \dots,
    t_m)}$ with ${R(t_1, \dots, x', \dots, t_m) \wedge x' \equals t_k}$ where
    $x'$ is a fresh variable. A generalised FO dependency $\delta'$ is a
    \emph{singularisation} of $\delta$ if there exists a sequence of
    generalised FO dependencies ${\delta_0, \delta_1, \dots, \delta_n}$ with
    ${n \geq 1}$ satisfying the following conditions.
    \begin{itemize}
        \item ${\delta_0 = \delta}$ and ${\delta_n = \delta'}$.

        \item If $\psi(\vec x, \vec y)$ does not contain $\predQ$, then
        ${\delta_1 = \delta_0}$; otherwise, $\psi(\vec x, \vec y)$ is of the
        form ${\predQ(x_1,\dots,x_m)}$, and $\delta_1$ is defined as follows
        where ${x_1',\dots,x_m'}$ are fresh variables.
        \begin{displaymath}
            \delta_1 = \forall \vec x, x_1', \dots, x_m'.[\varphi(\vec x) \wedge x_1 \equals x_1' \wedge \dots \wedge x_m \equals x_m' \rightarrow \predQ(x_1',\dots,x_m')]
        \end{displaymath}

        \item For each ${i \in \{ 2, \dots, n \}}$, dependency $\delta_i$ is
        obtained by applying a singularisation step to $\delta_{i-1}$.

        \item No singularisation step is applicable to $\delta_n$.
    \end{itemize}

    Let $\Sigma$ be a set of generalised first-order dependencies. A set
    $\Sigma'$ is a \emph{singularisation} of $\Sigma$ if each element of
    $\Sigma'$ is a singularisation of a distinct dependency of $\Sigma$.
    Moreover, set $\DOM{\Sigma}$ contains a domain axiom \eqref{eq:dom-c}
    instantiated for each constant $c$ occurring in $\Sigma$, and a domain
    axiom \eqref{eq:dom-R} instantiated for each $n$-ary predicate $R$
    occurring in $\Sigma$ different from $\equals$ and $\predQ$ and each ${i
    \in \{ 1, \dots, n \}}$. Set $\Rfl$ contains the reflexivity axiom
    \eqref{eq:ref}. Set $\ST$ contains the symmetry axiom \eqref{eq:sym} and
    the transitivity axiom \eqref{eq:trans}. Set $\FR{\Sigma}$ contains a
    \emph{$\D$-restricted functional reflexivity} axiom \eqref{eq:Dfnref}
    instantiated for each $n$-ary function symbol $f$ occurring in $\Sigma$.
    Finally, ${\SG{\Sigma} = \DOM{\Sigma} \cup \Rfl \cup \ST \cup \FR{\Sigma}}$.
    \begin{align}
        \D(x_1) \wedge x_1 \equals x_1' \wedge \D(x_1') \wedge \dots \wedge \D(x_n) \wedge x_n \equals x_n' \wedge \D(x_n') & \rightarrow f(x_1, \dots, x_n) \equals f(x_1', \dots, x_n')  \label{eq:Dfnref}
    \end{align}
\end{definition}

Note that $\DOM{\Sigma}$ does not contain the domain rules for $\predQ$ because
this predicate occurs only in the rule head. We define $\EQ{\Sigma}$ as in
Section~\ref{sec:preliminaries} while treating $\Sigma$ as a conjunction;
again, the domain axioms are not instantiated for $\predQ$ or $\equals$.

A singularisation of a set of generalised first-order dependencies $\Sigma$ is
not unique because the result depends on the order in which singularisation
steps are applied. However, the set $\SG{\Sigma}$ depends only on the function
and predicate symbols in $\Sigma$, which are unaffected by singularisation;
thus, ${\SG{\Sigma'} = \SG{\Sigma}}$ for each singularisation $\Sigma'$ of
$\Sigma$. Theorem~\ref{thm:sg} states that, regardless of how singularisation
steps are applied, query answers remain preserved even though functional
reflexivity axioms are $\D$-restricted. We discuss the intuition behind this
result in Example~\ref{ex:sg}.

\begin{restatable}{theorem}{thmSG}\label{thm:sg}
    For each finite set of generalised first-order dependencies $\Sigma$, each
    singularisation $\Sigma'$ of $\Sigma$, each base instance $B$, and each
    fact of the form $\predQ(\vec a)$, it is the case that ${\Sigma \cup B
    \modelsEq \predQ(\vec a)}$ if and only if ${\Sigma' \cup \SG{\Sigma'} \cup
    B \models \predQ(\vec a)}$.
\end{restatable}

\begin{example}\label{ex:sg}
For the generalised second-order dependency $\exSigma$ from
Example~\ref{ex:run}, set ${\Sigma' = \fol{\exSigma}}$ consists of generalised
first-order dependencies \eqref{ex:run:Q}--\eqref{ex:run:U-eq}. As we discuss
shortly, singularisation is applicable only to dependencies \eqref{ex:run:Q}
and \eqref{ex:run:RSR-eq}, which we restate below for convenience.
\begin{align}
    R(x_1,x_2) \wedge f(x_1) \equals x_3 \wedge A(x_3) \wedge B(x_3)                                                    & \rightarrow \predQ(x_1)       \tag{\ref{ex:run:Q} restated} \\
    R(x_1,x_2) \wedge f(x_1) \equals x_3 \wedge A(x_3) \wedge x_3 \equals x_3' \wedge B(x_3') \wedge x_1 \equals x_1'   & \rightarrow \predQ(x_1')      \label{ex:sg:Q} \\
    R(x_2,x_1) \wedge S(x_2,x_3) \wedge R(x_3,x_4)                                                                      & \rightarrow x_1 \equals x_4   \tag{\ref{ex:run:RSR-eq} restated} \\
    R(x_2,x_1) \wedge x_2 \equals x_2' \wedge S(x_2',x_3) \wedge x_3 \equals x_3' \wedge R(x_3',x_4)                    & \rightarrow x_1 \equals x_4   \label{ex:sg:RSR-eq}
\end{align}
To singularise dependency \eqref{ex:run:Q}, since the head atom contains the
query predicate, we introduce a fresh variable $x_1'$, we replace $\predQ(x_1)$
with $\predQ(x_1')$, and we add ${x_1 \equals x_1'}$ to the body. Moreover,
variable $x_3$ occurs in relational body atoms $A(x_3)$ and $B(x_3)$, so we
replace $B(x_3)$ with $B(x_3')$ and add ${x_3 \equals x_3'}$ to the body. Thus,
dependency \eqref{ex:sg:Q} is a singularisation of \eqref{ex:run:Q}.
Analogously, dependency \eqref{ex:sg:RSR-eq} is a possible singularisation of
dependency \eqref{ex:run:RSR-eq}. Singularisation is not applicable to any
other dependency of $\Sigma'$, so let $\Sigma''$ be obtained from $\Sigma'$ by
replacing \eqref{ex:run:Q} and \eqref{ex:run:RSR-eq} with \eqref{ex:sg:Q} and
\eqref{ex:sg:RSR-eq}, respectively.

Figure~\ref{fig:ex:sg} shows a universal model of ${\Sigma'' \cup
\SG{\Sigma''}}$ and $\exB$ constructed using the chase variant from
Section~\ref{sec:so-dependencies:chase}. Equalities are shown as double
$\equals$-labelled arcs, where reflexive and transitive edges are omitted for
readability. Moreover, $n_1, \dots, n_k$ are base labelled nulls, whereas
$\fnnull{s}$, $\fnnull{t}$, and $\fnnull{u}$ are functional labelled nulls for
the terms ${s = f(a_1)}$, ${t = f(\fnnull{s})}$, and ${u = f(\fnnull{t})}$,
respectively.

We first consider how $\predQ(a_1)$ is derived when equality is axiomatised
explicitly as discussed in Section~\ref{sec:preliminaries}. Then, $\Sigma'$ and
$B$ derive ${\fnnull{s} \equals \fnnull{t}}$, ${A(\fnnull{s})}$, and
${B(\fnnull{t})}$, and the congruence axioms in $\EQ{\Sigma'}$ `copy' the last
two facts to ${A(\fnnull{t})}$ and ${B(\fnnull{s})}$. We can thus match the
conjunction ${A(x_3) \wedge B(x_3)}$ in \eqref{ex:run:Q} to ${A(\fnnull{s})}$
and ${B(\fnnull{s})}$, or to ${A(\fnnull{t})}$ and ${B(\fnnull{t})}$; either
way, we derive $\predQ(a_1)$.

Singularisation avoids fact `copying'. Apart from the restricted functional
reflexivity axioms that we discuss shortly, the main difference between
$\SG{\Sigma''}$ and $\EQ{\Sigma'}$ is that the former does not contain any
congruence axioms. As one can see in Figure~\ref{fig:ex:sg}, applying the chase
to ${\Sigma'' \cup \SG{\Sigma''}}$ and $\exB$ derives ${A(\fnnull{s})}$ and
${B(\fnnull{t})}$, but not ${A(\fnnull{t})}$ or ${B(\fnnull{s})}$. To
compensate for this, conjunction ${A(x_3) \wedge B(x_3)}$ in \eqref{ex:run:Q}
is replaced with conjunction ${A(x_3) \wedge x_3 \equals x_3' \wedge B(x_3')}$
in \eqref{ex:sg:Q}. By matching these atoms to $A(\fnnull{s})$, ${\fnnull{s}
\equals \fnnull{t}}$, and ${B(\fnnull{t})}$, respectively, we derive
$\predQ(a_1)$ without any fact `copying'.

Although $x_1$ occurs in dependency \eqref{ex:sg:Q} in atoms ${f(x_1) \equals
x_3}$ and $R(x_1,x_2)$, the two occurrences of $x_1$ \emph{do not} need to be
renamed apart because singularisation preserves equality facts. In particular,
atom ${f(x_1) \equals x_3}$ is true for the same values of $x_1$ and $x_3$
after singularisation, so the atom can be matched to the values of $x_1$ and
$x_3$ that make $R(x_1,x_2)$ and $A(x_3)$ true.

The functional reflexivity axioms of $\EQ{\Sigma'}$ and $\SG{\Sigma''}$ differ
in that axioms \eqref{eq:Dfnref} contain atoms $\D(x_i)$ and $\D(x_i')$ for
each ${i \in \{ 1, \dots, n \}}$. If $\SG{\Sigma''}$ were to contain
\eqref{eq:fnref} instead of \eqref{eq:Dfnref}, then the chase of ${\Sigma''
\cup \SG{\Sigma''}}$ and $\exB$ would introduce facts of the form
${\fnnull{v_i} \equals \fnnull{v_{i+1}}}$ for each $i \geq 0$, where $v_0 = u$
and ${v_{i+1} = f(\fnnull{v_i})}$. In other words, a universal model would be
necessarily infinite. Our formulation of $\SG{\Sigma''}$ makes the universal
model finite as follows. First, the domain axioms \eqref{eq:dom-c} and
\eqref{eq:dom-R} ensure that the chase of ${\Sigma'' \cup \SG{\Sigma''}}$ and
$\exB$ derives $\D(t)$ for each $t$ occurring in a relational fact; in other
words, predicate $\D$ is axiomatised to enumerate the relevant subset of the
universal model's domain. Second, the $\D$-atoms in \eqref{eq:Dfnref} prevent
the functional reflexivity axioms from being applied beyond this relevant
subset. The missing equalities that would be produced by the unrestricted
functional reflexivity axioms are irrelevant due to safety. For example,
${\fnnull{v_1} \equals \fnnull{v_0}}$ cannot be matched to atom ${x_3 \equals
x_3'}$ of dependency \eqref{ex:sg:Q} since variable $x_3$ also occurs in a
relational atom, and $\fnnull{v_1}$ does not occur in any relational fact in a
universal model of ${\Sigma'' \cup \SG{\Sigma''}}$ and $\exB$.
\end{example}

\begin{figure}[tb]
\begin{center}
\begin{tikzpicture}[scale=0.75, every node/.style={scale=0.75}]
    \tikzset{>=latex}

    \node[label=below:{$C, \predQ$}] (a1)    at (0,0)  {$a_1$}     ;
    \node                            (a2)    at (2,0)  {$a_2$}     ;
    \node                            (a3)    at (4,0)  {$a_3$}     ;
    \node                            (adots) at (6,0)  {$\ldots$}  ;
    \node                            (ak1)   at (8,0)  {$a_{k-1}$} ;
    \node                            (ak)    at (10,0) {$a_k$}     ;

    \node  (a1p)    at (0,2)   {$n_1$}     ;
    \node  (a2p)    at (2,2)   {$n_2$}     ;
    \node  (a3p)    at (4,2)   {$n_3$}     ;
    \node  (adotsp) at (6,2)   {}          ;
    \node  (ak1p)   at (8,2)   {$n_{k-1}$} ;
    \node  (akp)    at (10,2)  {$n_k$}     ;

    \node[label=below:{$A$}] (fa1)   at (-2,0)  {$\fnnull{s}$} ;
    \node[label=below:{$B$}] (ffa1)  at (-4,0)  {$\fnnull{t}$} ;
    \node                    (fffa1) at (-6,0)  {$\fnnull{u}$} ;

    \draw[->]        (a1)    to node[below] {$S$} (a2)    ;
    \draw[->]        (a2)    to node[below] {$S$} (a3)    ;
    \draw[->,dotted] (a3)    to                   (adots) ;
    \draw[->,dotted] (adots) to                   (ak1)   ;
    \draw[->]        (ak1)   to node[below] {$S$} (ak)    ;
    \draw[->]        (a1)    to node[below] {$U$} (fa1)   ;

    \draw[->]        (a1)    to node[right] {$R$} (a1p)   ;
    \draw[->]        (a2)    to node[right] {$R$} (a2p)   ;
    \draw[->]        (a3)    to node[right] {$R$} (a3p)   ;
    \draw[->]        (ak1)   to node[right] {$R$} (ak1p)  ;
    \draw[->]        (ak)    to node[right] {$R$} (akp)   ;

    \draw[double]    (a1p)    to[bend left = 30] node[above] {$\equals$} (a2p)    ;
    \draw[double]    (a2p)    to[bend left = 30] node[above] {$\equals$} (a3p)    ;
    \draw[double]    (a3p)    to[bend left = 30] node[above] {$\equals$} (adotsp) ;
    \draw[double]    (adotsp) to[bend left = 30] node[above] {$\equals$} (ak1p)   ;
    \draw[double]    (ak1p)   to[bend left = 30] node[above] {$\equals$} (akp)    ;

    \draw[double]    (fa1)    to[bend left = 30] node[above] {$\equals$} (a1)     ;
    \draw[double]    (ffa1)   to[bend left = 30] node[above] {$\equals$} (fa1)    ;
    \draw[double]    (fffa1)  to[bend left = 30] node[above] {$\equals$} (ffa1)   ;

\end{tikzpicture}
\end{center}
\caption{A universal model for the singularised dependencies and dataset from Example~\ref{ex:sg}}\label{fig:ex:sg}
\end{figure}

\subsection{Skolemisation and Transformation to a Logic Program}\label{sec:answering:skolem}

In the rest, we assume that the set of function symbols is partitioned into
disjoint subsets of \emph{true function symbols} and \emph{Skolem function
symbols}. We also assume that the result $\Sigma_2$ of singularisation contains
only true function symbols. Set $\Sigma_2$ is transformed into a logic program
${P_3 = \sk{\Sigma_2}}$ as specified in Definition~\ref{def:sk}, which uses
Skolem function symbols.

\begin{definition}\label{def:sk}
    The \emph{Skolemisation} of a generalised FO dependency ${\delta = \forall
    \vec x.[\varphi(\vec x) \rightarrow \exists \vec y.\psi(\vec x, \vec y)]}$
    is the set $\sk{\delta}$ that contains precisely a rule ${H \leftarrow
    \varphi(\vec x)}$ for each atom ${H \in \psi'(\vec x')}$, where $\vec x'$
    are the \emph{frontier variables} of $\delta$ (i.e., all variables of $\vec
    x$ that occur in $\psi(\vec x, \vec y)$) and $\psi'(\vec x')$ is the
    conjunction obtained from $\psi(\vec x, \vec y)$ by replacing each
    existentially quantified variable ${y \in \vec y}$ with $g(\vec x')$ for
    $g$ a Skolem function symbol freshly introduced for $y$. The
    \emph{Skolemisation} of a set $\Sigma$ of generalised FO dependencies is
    the logic program ${\sk{\Sigma} = \bigcup_{\delta \in \Sigma} \sk{\delta}}$.
\end{definition}

Note that one can define Skolemisation to use $g(\vec x)$ instead of $g(\vec
x')$, which is also semantically correct. However, the formulation from
Definition~\ref{def:sk} can be more efficient in practice since it reduces the
arity of the Skolem function symbols.

Let $P$ be a logic program that can contain true and Skolem function symbols.
Then, $\EQ{P}$ is a set of rules defined analogously to
Section~\ref{sec:preliminaries}, but where functional reflexivity rules
\eqref{eq:fnref} are instantiated only for true function symbols. Moreover,
$\DOM{P}$, $\ST$, $\FR{P}$, and $\SG{P}$ are sets of rules as in
Definition~\ref{def:sg}, but the $\D$-restricted functional reflexivity rules
\eqref{eq:Dfnref} are again instantiated only for true function symbols. We
next show that Skolemisation does not affect query answers.

\begin{proposition}\label{prop:sk}
    For each set of generalised FO dependencies $\Sigma$, each base instance
    $B$, each fact of the form $\predQ(\vec a)$, and ${P = \sk{\Sigma}}$, it
    holds that ${\Sigma \cup \SG{\Sigma} \cup B \models \predQ(\vec a)}$ if and
    only if ${P \cup \SG{P} \cup B \models \predQ(\vec a)}$.
\end{proposition}

\begin{proof}
Entailment relation $\models$ treats equality as an ordinary predicate, so
${\Sigma \cup \SG{\Sigma} \cup B \models \predQ(\vec a)}$ if and only if
${\sk{\Sigma \cup \SG{\Sigma}} \cup B \models \predQ(\vec a)}$ for each fact of
the form $\predQ(\vec a)$ \cite{theorem-proving}. Furthermore, $\SG{\Sigma}$
does not contain an existential quantifier, so ${\sk{\Sigma \cup \SG{\Sigma}} =
\sk{\Sigma} \cup \SG{\Sigma} = P \cup \SG{P}}$.
\end{proof}

Definition~\ref{def:well-formed} captures the syntactic form of the rules
produced from a generalised SO dependency by removing second-order
quantification, and then singularising and Skolemising the result. The
definition allows relational atoms in rule bodies to contain constants (but not
functional terms), which is more general than the form produced solely by the
transformations presented thus far. However, this allows us to capture the
result of the transformation in Section~\ref{sec:answering:relevance} using the
same definition. As we shall see, this syntactic form plays a central role in
our correctness proofs.

\begin{definition}\label{def:well-formed}
    A conjunction of atoms $\phi$ is \emph{well-formed} if all relational atoms
    of $\phi$ are of depth zero, all equality atoms of $\phi$ are of depth at
    most one, and each variable occurring in $\phi$ in an equality atom also
    occurs in $\phi$ in a relational atom. A rule $r$ is \emph{well-formed} if
    it is of the form
    \begin{itemize}
        \item ${H \leftarrow \phi}$ where conjunction $\phi$ is well-formed and
        atom $H$ is of depth at most one, or

        \item ${\predQ(\vec t) \leftarrow \phi \wedge \phi'}$ where atom
        $\predQ(\vec t)$ is of depth zero, conjunction $\phi$ is well-formed,
        and conjunction $\phi'$ contains zero or more equalities of the form
        ${x \equals x'}$ where $x$ is a variable that occurs in $\phi$ and $x'$
        is a variable that occurs in $\vec t$ but not $\phi$.
    \end{itemize}
    A program $P$ is \emph{well-formed} if each rule in $P$ is
    well-formed.
\end{definition}

\begin{example}\label{ex:sk}
Skolemising the singularised dependencies from Example~\ref{ex:sg} produces a
program $P$ that contains rules \eqref{ex:rule:Q}--\eqref{ex:rule:U-eq}, all of
which are well-formed. Note that $f$ is a true function symbol, whereas $g$ is
a Skolem function symbol; thus, program $\SG{P}$ contains the $\D$-restricted
functional reflexivity axioms \eqref{eq:Dfnref} for $f$, but not for $g$.
Figure~\ref{fig:ex:sk} shows a Herbrand model of ${P \cup \SG{P} \cup \exB}$.
\begin{align}
    \predQ(x_1')        & \leftarrow R(x_1,x_2) \wedge f(x_1) \equals x_3 \wedge A(x_3) \wedge x_3 \equals x_3' \wedge B(x_3') \wedge x_1 \equals x_1'  \label{ex:rule:Q}      \\
    R(x_1,g(x_1))       & \leftarrow S(x_1,x_2)                                                                                                         \label{ex:rule:S-R}    \\
    x_1 \equals x_4     & \leftarrow R(x_2,x_1) \wedge x_2 \equals x_2' \wedge S(x_2',x_3) \wedge x_3 \equals x_3' \wedge R(x_3',x_4)                   \label{ex:rule:RSR-eq} \\
    A(f(x))             & \leftarrow C(x)                                                                                                               \label{ex:rule:C-Af}   \\
    U(x,f(x))           & \leftarrow C(x)                                                                                                               \label{ex:rule:C-Uf}   \\
    B(f(x_2))           & \leftarrow U(x_1,x_2)                                                                                                         \label{ex:rule:U-Bf}   \\
    x_1 \equals x_2     & \leftarrow U(x_1,x_2)                                                                                                         \label{ex:rule:U-eq}
\end{align}
\end{example}

\begin{figure}[tb]
\begin{center}
\begin{tikzpicture}[scale=0.75, every node/.style={scale=0.75}]
    \tikzset{>=latex}

    \node[label=below:{$C, \predQ$}] (a1)    at (0,0)  {$a_1$}     ;
    \node                            (a2)    at (2,0)  {$a_2$}     ;
    \node                            (a3)    at (4,0)  {$a_3$}     ;
    \node                            (adots) at (6,0)  {$\ldots$}  ;
    \node                            (ak1)   at (8,0)  {$a_{k-1}$} ;
    \node                            (ak)    at (10,0) {$a_k$}     ;

    \node  (a1p)    at (0,2)   {$g(a_1)$}     ;
    \node  (a2p)    at (2,2)   {$g(a_2)$}     ;
    \node  (a3p)    at (4,2)   {$g(a_3)$}     ;
    \node  (adotsp) at (6,2)   {}             ;
    \node  (ak1p)   at (8,2)   {$g(a_{k-1})$} ;
    \node  (akp)    at (10,2)  {$g(a_k)$}     ;

    \node[label=below:{$A$}] (fa1)   at (-2,0)  {$f(a_1)$}       ;
    \node[label=below:{$B$}] (ffa1)  at (-4,0)  {$f(f(a_1))$}    ;
    \node                    (fffa1) at (-6,0)  {$f(f(f(a_1)))$} ;

    \draw[->]        (a1)    to node[below] {$S$} (a2)    ;
    \draw[->]        (a2)    to node[below] {$S$} (a3)    ;
    \draw[->,dotted] (a3)    to                   (adots) ;
    \draw[->,dotted] (adots) to                   (ak1)   ;
    \draw[->]        (ak1)   to node[below] {$S$} (ak)    ;
    \draw[->]        (a1)    to node[below] {$U$} (fa1)   ;

    \draw[->]        (a1)    to node[right] {$R$} (a1p)   ;
    \draw[->]        (a2)    to node[right] {$R$} (a2p)   ;
    \draw[->]        (a3)    to node[right] {$R$} (a3p)   ;
    \draw[->]        (ak1)   to node[right] {$R$} (ak1p)  ;
    \draw[->]        (ak)    to node[right] {$R$} (akp)   ;

    \draw[double]    (a1p)    to[bend left = 30] node[above] {$\equals$} (a2p)    ;
    \draw[double]    (a2p)    to[bend left = 30] node[above] {$\equals$} (a3p)    ;
    \draw[double]    (a3p)    to[bend left = 30] node[above] {$\equals$} (adotsp) ;
    \draw[double]    (adotsp) to[bend left = 30] node[above] {$\equals$} (ak1p)   ;
    \draw[double]    (ak1p)   to[bend left = 30] node[above] {$\equals$} (akp)    ;

    \draw[double]    (fa1)    to[bend left = 30] node[above] {$\equals$} (a1)     ;
    \draw[double]    (ffa1)   to[bend left = 30] node[above] {$\equals$} (fa1)    ;
    \draw[double]    (fffa1)  to[bend left = 30] node[above] {$\equals$} (ffa1)   ;

\end{tikzpicture}
\end{center}
\caption{A Herbrand model for the logic program and dataset from Example~\ref{ex:sk}}\label{fig:ex:sk}
\end{figure}

We next discuss why distinguishing true and Skolem function symbols is
important. This distinction is enabled by the fact that axiomatisation of
equality (either in full or via singularisation) is applied \emph{before}
Skolemisation. In particular, at the point when equality is axiomatised, the
formulas in question contain only true function symbols, and so the result
contains only the ($\D$-restricted) functional reflexivity axioms for the true
function symbols. After this step, the equality predicate is ordinary, and
Skolemisation in first-order logic without equality does not require
introducing any functional reflexivity axioms. Thus, after equality
axiomatisation, both $\EQ{\Sigma}$ and $\SG{P}$ contain ($\D$-restricted)
functional reflexivity axioms only for the true function symbols, but not for
the Skolem function symbols. Example~\ref{ex:skolem-fnref} clarifies this
important detail.

\begin{example}\label{ex:skolem-fnref}
Let $\Sigma$ contain the first-order dependencies \eqref{ex:skolem-fnref:R} and
\eqref{ex:skolem-fnref:S}; note that $\Sigma$ is a singularisation of itself as
per Definition~\ref{def:sg}. Moreover, let $\Sigma'$ be the Skolemisation of
$\Sigma$---that is, $\Sigma'$ contains dependencies \eqref{ex:skolem-fnref:R}
and \eqref{ex:skolem-fnref:skS}. Finally, let ${B = \{ R(a,b), A(a), A(b) \}}$.
\begin{align}
    R(x_1,x_2)  & \rightarrow x_1 \equals x_2   \label{ex:skolem-fnref:R}   \\
    A(x)        & \rightarrow \exists y.S(x,y)  \label{ex:skolem-fnref:S}   \\
    A(x)        & \rightarrow S(x,f(x))         \label{ex:skolem-fnref:skS}
\end{align}

Let $I$ be the interpretation where ${\Delta^I = \{ \alpha, \beta, \gamma \}}$,
${a^I = b^I = \alpha}$, ${A^I = \{ \alpha \}}$, and ${S^I = \{ \langle
\alpha,\beta \rangle, \langle \alpha,\gamma \rangle \}}$. One can see that ${I
\modelsEq \Sigma}$ and ${I \modelsEq B}$; in fact, $I$ is a universal model of
$\Sigma$ and $B$. Note that $\exists y.S(x,y)$ in dependency
\eqref{ex:skolem-fnref:S} says that, for each value of $x$, atom $S(x,y)$ must
be satisfied for \emph{at least} one value of $y$, but there is no upper bound
on the number of such values. In fact, when $x$ is mapped to the domain element
$\alpha$ of $I$, two values for $y$ satisfy $S(x,y)$: $\beta$ and $\gamma$.

Now let $J$ be the Herbrand interpretation as specified below. One can check
that $J$ is a model of ${\Sigma' \cup \SG{\Sigma'} \cup B}$. Constants $a$ and
$b$ are equal in $J$, which is analogous to the condition ${a^I = b^I =
\alpha}$ from the previous paragraph. Furthermore, $f$ is a Skolem function
symbol, so $\SG{\Sigma'}$ does not contain a $\D$-restricted functional
reflexivity rule for $f$ and thus $J$ \emph{does not} need to contain ${f(a)
\equals f(b)}$. Intuitively, $f(a)$ and $f(b)$ play in $J$ the role of domain
elements $\beta$ and $\gamma$ in $I$; hence, just as $I$ is allowed to contain
distinct $\beta$ and $\gamma$, so is $J$ allowed to contain distinct $f(a)$ and
$f(b)$.
\begin{align*}
    J = \{ R(a,b), A(a), A(b), S(a,f(a)), S(b,f(b)), a \equals a, a \equals b, b \equals a, b \equals b, f(a) \equals f(a), f(b) \equals f(b) \}
\end{align*}
Analogous observations apply when equality is axiomatised fully: ${\Sigma' \cup
\EQ{\Sigma'} \cup B}$ is satisfied in a model that extends $J$ but does not
need to contain ${f(a) \equals f(b)}$.

Now assume that $\Sigma'$ is given in the input---that is, the input contains
$f$ as a true function symbol. We must then interpret $f$ as a function from
the very beginning. For example, to define $f^I$, we must decide whether
$f^I(\alpha)$ produces $\beta$ or $\gamma$. To reflect this when axiomatising
equality, set $\SG{\Sigma'}$ must include the $\D$-restricted functional
reflexivity axiom for $f$; but then, to obtain a Herbrand model of ${\Sigma'
\cup \SG{\Sigma'} \cup B}$, we must extend $J$ with ${f(a) \equals f(b)}$ and
${f(b) \equals f(a)}$.
\end{example}

The observations from Example~\ref{ex:skolem-fnref} have repercussions
throughout this paper. First, line~\ref{alg:answer-query:chase} of
Algorithm~\ref{alg:answer-query} is applied to a program $P_7$ that contains
true and Skolem function symbols, so the chase variant from
Section~\ref{sec:so-dependencies:chase} must be extended to distinguish these
two types. In particular, the algorithm should interpret a ground term ${t =
f(u_1,\dots,u_n)}$ with $f$ a Skolem function symbol as the functional labelled
null $\fnnull{t}$ unique for $t$, but it should \emph{not} maintain facts of
the form $\fnval{f}(u_1,\dots,u_n,v)$ for $f$ since Skolem function symbols are
not subjected to functional reflexivity axioms.

Second, the need to distinguish true and Skolem function symbols explains why
the correctness of Algorithm~\ref{alg:answer-query} is stated in
Theorem~\ref{thm:desg} the way it is; we discuss this further in
Section~\ref{sec:answering:desg}.

Third, this distinction ensures that, when our approach is applied to
first-order dependencies only, no ($\D$-restricted) functional reflexivity
axioms are necessary. Thus, Algorithm~\ref{alg:answer-query} then reduces to
the work by \citet{DBLP:conf/aaai/BenediktMT18}, where the chase in
line~\ref{alg:answer-query:chase} uses a more standard formulation.

\subsection{Relevance Analysis}\label{sec:answering:relevance}

The objective of the steps presented thus far was to bring the input into a
form that allows us to identify the inferences that contribute to the query
answers. We now present our first such technique that we call \emph{relevance
analysis}. The technique eliminates entire rules that provably do not
contribute to the query answers. Top-down query answering techniques such as
SLD resolution \cite{DBLP:conf/ifip/Kowalski74} can naturally identify relevant
inferences by working backwards from the query. However, our objective is to
optimise the Skolemised program rather than answer the query itself, so our
technique must be cheaper to compute than top-down query answering to be of
practical use.

Our technique can be summarised as follows. First, we homomorphically embed the
input base instance $B$ into a much smaller instance $B'$ that we call an
\emph{abstraction} of $B$. Second, we compute the fixpoint of the Skolemised
program and $B'$; since $B'$ is orders of magnitude smaller than $B$, the
overhead needed to achieve this should be negligible. Third, inspired by SLD
resolution, we use this fixpoint to analyse inferences of the Skolemised
program. The existence of a homomorphism from $B$ to $B'$ guarantees that each
rule instance that fires on $B$ also fires on $B'$; hence, a rule that is not
relevant on $B'$ cannot be relevant on $B$. To present this approach formally,
we next introduce the notion of an abstraction of a base instance.

\begin{definition}\label{def:abstraction}
    A base instance $B'$ is an \emph{abstraction} of a base instance $B$ with
    respect to a program $P$ if there exists a mapping $\eta$, called a
    \emph{homomorphism from} $B$ \emph{to} $B'$ \emph{w.r.t.} $P$, of constants
    to constants such that ${\eta(c) = c}$ for each constant $c$ occurring in
    $P$, and ${\eta(B) \subseteq B'}$ where
    \begin{displaymath}
        \eta(B) = \{ R(\eta(c_1),\dots,\eta(c_n)) \mid R(c_1,\dots,c_n) \in B \}.
    \end{displaymath}
\end{definition}

Any abstraction of $B$ can be used in our relevance analysis algorithm, and the
choice of abstraction can affect the algorithm's precision. However, the
\emph{critical instance} for a base instance $B$ and program $P$ allows us to
abstract $B$ in a way that depends on the predicates but not the facts of $B$.
Let $C$ be the set of constants of $P$, and let $\ast$ be a fresh constant;
then, the critical instance for $B$ and $P$ is the instance $B'$ that contains
a fact ${R(\vec a)}$ for each $n$-ary predicate $R$ occurring in $B$ and each
${\vec a \in (C \cup \{ \ast \})^n}$. We can further refine this approach if
predicates are \emph{sorted} (i.e., each predicate position is associated with
a sort such as strings or integers): we introduce a distinct fresh constant
$\ast_i$ per sort, and we form $B'$ as before while honouring the sorting
requirements.

Our relevance analysis technique is shown in Algorithm~\ref{alg:relevance}. It
takes as input a well-formed program $P$ and a base instance $B$, and it
returns the rules that are relevant to answering $\predQ$ on $B$. It selects an
abstraction $B'$ of $B$ with respect to $P$
(line~\ref{alg:relevance:abstraction}), computes the consequences $I$ of $P$ on
$B'$ (line~\ref{alg:relevance:abstraction-fixpoint}), and identifies the rules
of $P$ that contribute to the answers of $\predQ$ on $B'$ by a form of backward
chaining. It populates the `ToDo' set $\mathcal{T}$ with facts that represent
the answers to $\predQ$ on $B$ (line~\ref{alg:relevance:init}) and then
iteratively explores $\mathcal{T}$
(lines~\ref{alg:relevance:while:start}--\ref{alg:relevance:while:end}). In each
iteration, it extracts a fact $F$ from $\mathcal{T}$
(line~\ref{alg:relevance:choose-F}) and identifies each rule ${r \in P \cup
\DOM{P} \cup \ST \cup \FR{P}}$ and substitution $\nu$ that matches the head of
$r$ to $F$ and the body of $r$ to facts in $I$
(line~\ref{alg:relevance:rule:start}). Such $\nu$ captures ways of deriving a
fact represented by $F$ from $B$ via $r$. Thus, if such $\nu$ exists, rule $r$
is relevant and is added to the set $\R$ of relevant rules provided that ${r
\not\in \SG{P}}$ (line~\ref{alg:relevance:add-r}). Rule $r$ is not added to
$\R$ if ${r \in \SG{P}}$: such $r$ is part of the axiomatisation of equality,
which can always be recovered from the algorithm's output. Either way, the
matched body atoms $\body{r}\nu$ must all be derivable as well, so they are
added to $\mathcal{T}$ (line~\ref{alg:relevance:add-T-D}). The `done' set
$\mathcal{D}$ ensures that each fact is added to $\mathcal{T}$ just once, which
ensures termination.

All rules from ${P \cup \SG{P}}$ apart from the reflexivity rule \eqref{eq:ref}
are considered in line~\ref{alg:relevance:rule:start}, which is possible
because the rules are well-formed. For example, assume that the body of a rule
$r$ contains an equality ${x \equals x'}$ that is matched to a fact ${t \equals
t}$. There may be many ways to derive ${t \equals t}$, but most of them are
irrelevant: rule $r$ is well-formed, so variable $x$ occurs in the body of $r$
in a relational atom $A(\dots,x\dots)$ that is matched to a fact of the form
$A(\dots,t,\dots)$; hence, we can always derive ${t \equals t}$ from
$A(\dots,t,\dots)$, the domain rule for $A$, and the reflexivity rule. In
contrast, if we considered the reflexivity rule in
line~\ref{alg:relevance:rule:start}, our algorithm would necessarily identify
each rule that derives a fact of the form $B(\dots,t,\dots)$ as relevant and
thus be considerably less effective. Since the reflexivity rule is not
considered, the domain rules \eqref{eq:dom-R} can be considered in
line~\ref{alg:relevance:rule:start} only via atoms $\D(x_i)$ and $\D(x_i')$ of
the $\D$-restricted functional reflexivity axioms \eqref{eq:Dfnref}, which in
turn are considered only if the input program contains true function symbols.
In other words, if the input dependencies are all first-order, neither the
reflexivity rule nor the domain rules are considered.

We can further optimise this algorithm if we know that ${P \cup B}$ satisfies
UNA (e.g., if an earlier UNA check was conducted). If $\nu$ matches an equality
atom ${t_1 \equals t_2 \in \body{r}}$ to ${c \equals c}$, then any
corresponding derivation from $P$ and $B$ necessarily matches ${t_1 \equals
t_2}$ to ${d \equals d}$ for some constant $d$. Moreover, program $P$ is
well-formed, so ${d \equals d}$ can be derived using the reflexivity rule, and
no other proofs for ${d \equals d}$ need to be considered
(line~\ref{alg:relevance:UNA-opt}). Finally, if all matches of ${t_1 \equals
t_2}$ are of the form ${c \equals c}$ and at least one of $t_1$ or $t_2$ is a
variable and the other term is of depth zero, then we can eliminate ${t_1
\equals t_2}$ from the rule and replace the variable everywhere else in the
rule. To this end, Algorithm~\ref{alg:relevance} uses a set $\mathcal{B}$ to
keep track of the body equality atoms that can be matched to an equality of the
form other than ${d \equals d}$ (line~\ref{alg:relevance:add-B}) and are thus
`blocked' from being eliminated. After considering all ways to derive query
answers, body equality atoms that have not been blocked are removed
(lines~\ref{alg:relevance:UNA:start}--\ref{alg:relevance:UNA:end}).

\begin{algorithm}[tb]
\caption{$\relevance{P}{B}$}\label{alg:relevance}
\begin{small}
\begin{algorithmic}[1]\footnotesize
    \State \textbf{choose} an abstraction $B'$ of $B$ with respect to $P$                                                                                                   \label{alg:relevance:abstraction}
    \State $I \defeq \fixpoint{P'}{B'}$ for $P' = P \cup \SG{P}$                                                                                                            \label{alg:relevance:abstraction-fixpoint}
    \State $\mathcal{D} \defeq \mathcal{T} \defeq \{ \predQ(\vec a) \in I \mid \vec a \text{ is a tuple of constants} \}$                                                   \label{alg:relevance:init}
    \State $\R \defeq \emptyset$ \;\; and \;\; $\mathcal{B} \defeq \emptyset$
    \While{$\mathcal{T} \neq \emptyset$}                                                                                                                                    \label{alg:relevance:while:start}
        \State \textbf{choose and remove} a fact $F$ from $\mathcal{T}$                                                                                                     \label{alg:relevance:choose-F}
        \For{\textbf{each} ${r \in P \cup \DOM{P} \cup \ST \cup \FR{P}}$ and \textbf{each} substitution $\nu$ such that $\head{r}\nu = F$ and $\body{r}\nu \subseteq I$}    \label{alg:relevance:rule:start}
            \If{$r \not\in \R \cup \DOM{P} \cup \ST \cup \FR{P}$}
                \textbf{add} $r$ to $\R$                                                                                                                                    \label{alg:relevance:add-r}
            \EndIf
            \For{\textbf{each} $G_i \in \body{r}\nu$}                                                                                                                       \label{alg:relevance:body:start}
                \If{$G_i$ is not of the form $c \equals c$ with $c$ a constant, or $P \cup B$ is not known to satisfy UNA}                                                  \label{alg:relevance:UNA-opt}
                    \If{$G_i \not\in \mathcal{D}$}
                        \textbf{add} $G_i$ to $\mathcal{T}$ and $\mathcal{D}$                                                                                               \label{alg:relevance:add-T-D}
                    \EndIf
                    \If{$P \cup B$ is known to satisfy UNA, $G_i$ is an equality, and $r \not\in \DOM{P} \cup \ST \cup \FR{P}$}
                        \textbf{add} $\langle r,i \rangle$ to $\mathcal{B}$                                                                                                 \label{alg:relevance:add-B}
                    \EndIf
                \EndIf
            \EndFor                                                                                                                                                         \label{alg:relevance:body:end}
        \EndFor                                                                                                                                                             \label{alg:relevance:rule:end}
    \EndWhile                                                                                                                                                               \label{alg:relevance:while:end}
    \If{$P \cup B$ is known to satisfy UNA}                                                                                                                                 \label{alg:relevance:UNA:start}
        \For{\textbf{each} rule $r \in \R$ and \textbf{each} atom of the form $x \equals t$ or $t \equals x$ at position $i$ of $\body{r}$
        \Statex \hspace{1.4cm} such that $\langle r,i \rangle \not\in \mathcal{B}$, variable $x$ occurs in $\body{r}$ in a relational atom, and $\dep{t} = 0$}              \label{alg:relevance:UNA:form}
            \State \textbf{remove} $x \equals t$ from $r$ and replace $x$ with $t$ in $r$                                                                                   \label{alg:relevance:UNA:desg}
        \EndFor
    \EndIf                                                                                                                                                                  \label{alg:relevance:UNA:end}
    \State \Return $\R$
\end{algorithmic}
\end{small}
\end{algorithm}

The computation of $I$ in line~\ref{alg:relevance:abstraction-fixpoint} may not
terminate in general, but termination is guaranteed for all dependency classes
that we mentioned in Section~\ref{sec:answering:overview}.
Theorem~\ref{thm:relevance} states that the algorithm is \emph{sound} in the
sense that the returned rules are sufficient to correctly answer the query. In
Section~\ref{sec:evaluation} we show empirically that the technique can be very
effective in practice.

\begin{restatable}{theorem}{thmRelevance}\label{thm:relevance}
    For each well-formed program $P$ defining a query $\predQ$ and each base
    instance $B$, program ${\R = \relevance{P}{B}}$ is well-formed; moreover,
    for each fact of the form ${\predQ(\vec a)}$, it holds that ${P \cup \SG{P}
    \cup B \models \predQ(\vec a)}$ if and only if ${\R \cup \SG{\R} \cup B
    \models \predQ(\vec a)}$.
\end{restatable}

If computing $\fixpoint{P'}{B'}$ in
line~\ref{alg:relevance:abstraction-fixpoint} is difficult (as is the case in
some of our experiments), we can replace each term ${f(\vec x)}$ in ${P \cup
\SG{P}}$ with a fresh constant $c_f$ unique for $f$; note that, for each true
function symbol $f$, terms $f(x_1,\dots,x_n)$ and $f(x_1',\dots,x_n')$ in the
$\D$-restricted functional reflexivity axioms must also be replaced by $c_f$.
This does not affect the algorithm's correctness since $\fixpoint{P \cup
\SG{P}}{B}$ can still be homomorphically embedded into $\fixpoint{P'}{B'}$.
Moreover, the resulting program then does not contain function symbols, so the
computation in line~\ref{alg:relevance:abstraction-fixpoint} necessarily
terminates.

\begin{example}\label{ex:relevance}
The chase construction in Example~\ref{ex:run-chase} never merges constants, so
$\exSigma$ and $\exB$ from Example~\ref{ex:run} satisfy UNA and we shall assume
that this is known in advance. Consider applying Algorithm~\ref{alg:relevance}
to the rules \eqref{ex:rule:Q}--\eqref{ex:rule:U-eq} from Example~\ref{ex:sk}.
We use in line~\ref{alg:relevance:abstraction} the critical instance ${B' = \{
C(\ast), S(\ast,\ast) \}}$ for $B$ and $P$, so the computation in
line~\ref{alg:relevance:abstraction-fixpoint} produces the instance $I$ shown
in Figure~\ref{fig:relevance:I}.

\begin{figure}[!tb]
\tiny
\begin{displaymath}
\begin{array}{@{}c@{\;\;\;}c@{\;\;\;}c@{\;\;\;}c@{\;\;\;}c@{\;\;\;}c@{\;\;\;}c@{\;\;\;}c@{}}
    C(\ast)         & S(\ast,\ast)      & \predQ(\ast)  & \D(\ast)          & \ast \equals \ast \\
    R(\ast,g(\ast)) &                   &               & \D(g(\ast))       & g(\ast) \equals g(\ast)              \\
    A(f(\ast))      & U(\ast,f(\ast))   &               & \D(f(\ast))       & f(\ast) \equals f(\ast)               & f(\ast) \equals \ast            \\
                    &                   &               &                   &                                       & \ast \equals f(\ast)       \\
    B(f(f(\ast)))   &                   &               & \D(f(f(\ast)))    & f(f(\ast)) \equals f(f(\ast))         & f(f(\ast)) \equals f(\ast)       & f(f(\ast)) \equals \ast         \\
                    &                   &               &                   &                                       & f(\ast) \equals f(f(\ast))       & \ast \equals f(f(\ast))     \\
                    &                   &               &                   & f(f(f(\ast))) \equals f(f(f(\ast)))   & f(f(f(\ast))) \equals f(f(\ast)) & f(f(f(\ast))) \equals f(\ast)  & f(f(f(\ast))) \equals \ast   \\
                    &                   &               &                   &                                       & f(f(\ast)) \equals f(f(f(\ast))) & f(\ast) \equals f(f(f(\ast)))  & \ast \equals f(f(f(\ast)))   \\
\end{array}
\end{displaymath}
\caption{Instance $I$ from Example~\ref{ex:relevance}}\label{fig:relevance:I}
\end{figure}

The top-down part of Algorithm~\ref{alg:relevance} starts by matching fact
$\predQ(\ast)$ to the head of rule \eqref{ex:rule:Q} and evaluating the
instantiated body in $I$, which produces the rule instance
\eqref{ex:relevance:1}.
\begin{align}
    \predQ(\ast) \leftarrow R(\ast,g(\ast)) \wedge f(\ast) \equals f(\ast) \wedge A(f(\ast)) \wedge f(\ast) \equals f(f(\ast)) \wedge B(f(f(\ast))) \wedge \ast \equals \ast \label{ex:relevance:1}
\end{align}
Atom ${\ast \equals \ast}$ is not considered any further since UNA is known to
hold; moreover, matching atoms ${R(\ast,g(\ast))}$, ${f(\ast) \equals
f(\ast)}$, $A(f(\ast))$, ${f(\ast) \equals f(f(\ast))}$, and $B(f(f(\ast)))$ to
the heads of rule \eqref{ex:rule:S-R}, the functional reflexivity rule
\eqref{eq:Dfnref} for $f$, \eqref{ex:rule:C-Af}, again the functional
reflexivity rule \eqref{eq:Dfnref} for $f$, and \eqref{ex:rule:U-Bf}, produces
rule instances \eqref{ex:relevance:2}--\eqref{ex:relevance:6}, respectively.
Atom ${\ast \equals \ast}$ in \eqref{ex:relevance:3} again does not need to be
explored since UNA is known to hold. Furthermore, matching atom ${\ast \equals
f(\ast)}$ in the body of \eqref{ex:relevance:5} to the head of rule
\eqref{ex:rule:U-eq} produces the rule instance \eqref{ex:relevance:7}.
Finally, matching atom $U(\ast,f(\ast))$ in the bodies of
\eqref{ex:relevance:6} and \eqref{ex:relevance:7} to the head of rule
\eqref{ex:rule:C-Uf} produces the rule instance \eqref{ex:relevance:8}. Atoms
$\D(\ast)$ and $\D(f(\ast))$ can be matched to the heads of the domain rules
\eqref{eq:dom-R}, but all thus instantiated body atoms have already been
considered. Thus, the loop in
lines~\ref{alg:relevance:while:start}--\ref{alg:relevance:while:end} finishes
with the set $\R$ containing all rules apart from rule \eqref{ex:rule:RSR-eq}.
\begin{align}
    R(\ast,g(\ast))             & \leftarrow S(\ast,\ast)                                               \label{ex:relevance:2} \\
    f(\ast) \equals f(\ast)     & \leftarrow \D(\ast) \wedge \ast \equals \ast \wedge \D(\ast)          \label{ex:relevance:3} \\
    A(f(\ast))                  & \leftarrow C(\ast)                                                    \label{ex:relevance:4} \\
    f(\ast) \equals f(f(\ast))  & \leftarrow \D(\ast) \wedge \ast \equals f(\ast) \wedge \D(f(\ast))    \label{ex:relevance:5} \\
    B(f(f(\ast)))               & \leftarrow U(\ast,f(\ast))                                            \label{ex:relevance:6} \\
    \ast \equals f(\ast)        & \leftarrow U(\ast,f(\ast))                                            \label{ex:relevance:7} \\
    U(\ast,f(\ast))             & \leftarrow C(\ast)                                                    \label{ex:relevance:8}
\end{align}

Facts ${f(\ast) \equals f(\ast)}$, ${f(\ast) \equals f(f(\ast))}$, and ${\ast
\equals f(\ast)}$ can be matched to the head of rule \eqref{ex:rule:RSR-eq} to
produce instantiated bodies \eqref{ex:relevance:9}--\eqref{ex:relevance:11},
respectively, but none of these have a match in $I$. Thus, no instance of rule
\eqref{ex:rule:RSR-eq} derives any of these facts, so the rule is irrelevant to
$\predQ$ and is not added to $\R$.
\begin{align}
    R(x_2,f(\ast)) \wedge x_2 \equals x_2' \wedge S(x_2',x_3) \wedge x_3 \equals x_3' \wedge R(x_3',f(\ast))    \label{ex:relevance:9}  \\
    R(x_2,f(\ast)) \wedge x_2 \equals x_2' \wedge S(x_2',x_3) \wedge x_3 \equals x_3' \wedge R(x_3',f(f(\ast))) \label{ex:relevance:10} \\
    R(x_2,\ast) \wedge x_2 \equals x_2' \wedge S(x_2',x_3) \wedge x_3 \equals x_3' \wedge R(x_3',f(\ast))       \label{ex:relevance:11}
\end{align}

In lines~\ref{alg:relevance:UNA:start}--\ref{alg:relevance:UNA:end}, the
algorithm further tries to identify irrelevant body equalities. Equality atom
${x_1 \equals x_1'}$ in rule \eqref{ex:rule:Q} is matched in the above process
\emph{only} to a fact ${\ast \equals \ast}$ that represents an equality between
two constants; moreover, UNA is known to hold, so ${\ast \equals \ast}$ can
represent only facts of the form ${c \equals c}$ (i.e., where both constants
are the same). Hence, atom ${x_1 \equals x_1'}$ in \eqref{ex:rule:Q} is
irrelevant and the rule is replaced with rule \eqref{ex:rel:Q}. In contrast,
atom ${f(x_1) \equals x_3}$ in rule \eqref{ex:rule:Q} is matched to a fact
${f(\ast) \equals f(\ast)}$, which, despite UNA being satisfied, can represent
a fact of the form ${f(c) \equals f(d)}$; thus, atom ${f(x_1) \equals x_3}$ is
relevant. Similarly, atom ${x_3 \equals x_3'}$ in rule \eqref{ex:rule:Q} is
matched to ${f(\ast) \equals f(f(\ast))}$, so it is clearly relevant. Thus, our
algorithm finishes and returns rules \eqref{ex:rel:Q}--\eqref{ex:rel:U-eq}.
\begin{align}
    \predQ(x_1)     & \leftarrow R(x_1,x_2) \wedge f(x_1) \equals x_3 \wedge A(x_3) \wedge x_3 \equals x_3' \wedge B(x_3')  \label{ex:rel:Q}      \\
    R(x_1,g(x_1))   & \leftarrow S(x_1,x_2)                                                                                 \label{ex:rel:S-R}    \\
    A(f(x))         & \leftarrow C(x)                                                                                       \label{ex:rel:C-Af}   \\
    U(x,f(x))       & \leftarrow C(x)                                                                                       \label{ex:rel:C-Uf}   \\
    B(f(x_2))       & \leftarrow U(x_1,x_2)                                                                                 \label{ex:rel:U-Bf}   \\
    x_1 \equals x_2 & \leftarrow U(x_1,x_2)                                                                                 \label{ex:rel:U-eq}
\end{align}
\end{example}

\subsection{Magic Sets Transformation}\label{sec:answering:magic}

We now present a variant of the well-known magic sets transformation that can
eliminate certain irrelevant rule instances. The transformation modifies the
rules of a program and introduces so-called \emph{magic} predicates so that
fixpoint computation of the resulting program simulates backward chaining on
the original program. This transformation does not subsume the relevance
analysis: irrelevant rules in the input give rise to rules that populate
irrelevant magic predicates, which introduces overheads that can be avoided by
first eliminating irrelevant rules. We show empirically in
Section~\ref{sec:evaluation} that the two techniques complement each other.

Equality is treated as an ordinary predicate, so we could use the standard
magic sets transformation by \citet{DBLP:journals/jlp/BeeriR91}. However, we
improve the algorithm's effectiveness in two ways: we reduce the number of
magic rules by taking the symmetry of equality into account, and we show that
the reflexivity rule \eqref{eq:ref} does not need to be taken into account in
the transformation. Before discussing our approach, we illustrate the idea
behind the standard magic sets algorithm.

\begin{example}\label{ex:mgc-standard}
Rule \eqref{ex:mgc:2} shown below derives a fact $B(a)$ for each constant $a$
that is reachable from some fact $B(b)$ via facts with the $R$ predicate; thus,
computing the full fixpoint of rule \eqref{ex:mgc:2} could be costly. However,
rule \eqref{ex:mgc:1} derives $\predQ(a)$ only if both $A(a)$ and $B(a)$ are
derived. Thus, if we are interested only in facts of the form $\predQ(a)$, and
if only a small fraction of the facts derivable by rule \eqref{ex:mgc:2} are
relevant to the query, magic sets can avoid computing the full fixpoint of the
rule.
\begin{align}
    \predQ(x)   & \leftarrow A(x) \wedge B(x)           \label{ex:mgc:1} \\
    B(x_1)      & \leftarrow R(x_1,x_2) \wedge B(x_2)   \label{ex:mgc:2}
\end{align}

The idea is to simulate sideways information passing from techniques such as
SLD resolution. To evaluate the body of rule \eqref{ex:mgc:1}, we shall
identify all values of $x_1$ that make $A(x_1)$ true, and then pass those
values to atom $B(x_1)$. Furthermore, we shall use these bindings for $B(x_1)$
to constrain the evaluation of rule \eqref{ex:mgc:2}: for each binding $x_1$,
we shall identify all values $x_2$ that make the atom $R(x_1,x_2)$ true, and we
shall pass those values of $x_2$ to the atom $B(x_2)$. To determine the order
in which to evaluate the body atoms and how to propagate bindings, the magic
sets algorithm uses a \emph{sideways information passing strategy} (SIPS),
which \emph{adorns} each predicate in a rule with a string consisting of
letters $\ad{b}$ and $\ad{f}$. If the $i$-th letter is $\ad{b}$, this indicates
that the $i$-th argument is \emph{bound}---that is, a value for the argument is
passed from the rule head or an earlier atom in the rule body; otherwise, the
argument is \emph{free}. Only predicates occurring in a rule head are adorned;
in our example, this includes predicates $\predQ$ and $B$. When applied to our
example, the SIPS can produce the following rules.
\begin{align}
    \predQ^\ad{f}(x)    & \leftarrow A(x) \wedge B^\ad{b}(x)            \label{ex:mgc:1:ad} \\
    B^\ad{b}(x_1)       & \leftarrow R(x_1,x_2) \wedge B^\ad{b}(x_2)    \label{ex:mgc:2:ad}
\end{align}
In rule \eqref{ex:mgc:1:ad}, the adornment of atom $\predQ^\ad{f}(x)$ indicates
that all values of $x$ are important, whereas the adornment of atom
$B^\ad{b}(x)$ indicates that the value for $x$ is determined by an earlier atom
$A(x)$.

To simulate sideways information passing, the magic sets algorithm introduces a
\emph{magic} predicate $\mgc{R}{\alpha}$ for each predicate $R$ adorned with a
string $\alpha$. Such predicates accumulate the relevant bindings using
so-called \emph{magic rules}. In our example, $\predQ$ is annotated with
$\ad{f}$ in rule \eqref{ex:mgc:1:ad}, so we introduce the magic rule
\eqref{ex:mgc:res} to indicate that all bindings of $\predQ$ should be
produced. Furthermore, rule \eqref{ex:mgc:1:B} computes the bindings that are
passed from atom $A(x)$ to atom $B(x)$ in rule \eqref{ex:mgc:1}; these bindings
should also be propagated from the head of rule \eqref{ex:mgc:2} to the atom
$B(x_2)$, which is achieved by rule \eqref{ex:mgc:2:B}.
\begin{align}
    \mgc{\predQ}{\ad{f}}    & \leftarrow                                        \label{ex:mgc:res} \\
    \mgc{B}{\ad{b}}(x)      & \leftarrow \mgc{\predQ}{\ad{f}} \wedge A(x)       \label{ex:mgc:1:B} \\
    \mgc{B}{\ad{b}}(x_2)    & \leftarrow \mgc{B}{\ad{b}}(x_1) \wedge R(x_1,x_2) \label{ex:mgc:2:B}
\end{align}

Finally, to answer the query, the magic sets transformation introduces the
following \emph{modified rules}, which constrain the original rules
\eqref{ex:mgc:1} and \eqref{ex:mgc:2} using magic predicates.
\begin{align}
    \predQ(x)   & \leftarrow \mgc{\predQ}{\ad{f}} \wedge A(x) \wedge B(x)         \label{ex:mgc:mod:1} \\
    B(x_1)      & \leftarrow \mgc{B}{\ad{b}}(x_1) \wedge R(x_1,x_2) \wedge B(x_2) \label{ex:mgc:mod:2}
\end{align}

One can compute the fixpoint of the transformed program as described in
Section~\ref{sec:preliminaries}, but the magic rules
\eqref{ex:mgc:res}--\eqref{ex:mgc:2:B} derive a fact of the form
$\mgc{B}{\ad{b}}(a)$ only if constant $a$ is relevant to the query answer.
This, in turn, constrains the evaluation of the modified rules
\eqref{ex:mgc:mod:1} and \eqref{ex:mgc:mod:2} only to the relevant bindings,
which ensures that the computation simulates SLD resolution.
\end{example}

We next formalise a variant of the magic sets transformation that deals with
equality in an optimised fashion. We start by introducing the notions of
adornment and magic predicates.

\begin{definition}\label{def:adornment}
    An \emph{adornment} for an $n$-ary relational predicate $R$ is a string of
    length $n$ over the alphabet $\ad{b}$ (`bound') and $\ad{f}$ (`free'). For
    each adornment $\alpha$ for $R$ with $k$ $\ad{b}$-symbols,
    $\mgc{R}{\alpha}$ is a fresh \emph{magic} predicate of arity $k$ unique for
    $R$ and $\alpha$. An \emph{adornment} for the equality predicate $\equals$
    is a string $\ad{bf}$ or $\ad{fb}$, and $\mgc{\equals}{\adEqb}$ is a fresh
    \emph{magic} predicate for $\equals$ of arity one.

    Given an adornment $\alpha$ of length $n$ and an $n$-tuple ${\vec t}$ of
    terms, ${\vec t^\alpha}$ is a tuple of terms that contains, in the same
    relative order, each ${t_i \in \vec t}$ for which the $i$-th element of
    $\alpha$ is $\ad{b}$.
\end{definition}

Relational predicates are thus adorned in the same way as in the standard magic
sets transformation, whereas the equality predicate can be adorned only by
$\ad{bf}$ or $\ad{fb}$. In particular, adorning $\equals$ by $\ad{ff}$ would
indicate that the entire equality relation should be computed, and this is
generally very inefficient since the equality relation typically affects most
predicates of a program. Furthermore, adorning the head of the transitivity
rule \eqref{eq:trans} by $\ad{bb}$ produces the rule
\begin{align}
    x_1 \equals^\ad{bb} x_3 \leftarrow x_1 \equals^\ad{bf} x_2 \wedge x_2 \equals^\ad{bb} x_3,
\end{align}
which introduces the adornment $\ad{bf}$. Now, adornment $\ad{bf}$ is more
general than $\ad{bb}$: whenever a pair of constants $\langle c_1,c_2 \rangle$
is a relevant binding for $\equals^\ad{bb}$, constants $c_1$ and $c_2$ are
relevant bindings for $\equals^\ad{bf}$ and $\equals^\ad{fb}$. Adornment
$\ad{bb}$ is thus subsumed by $\ad{bf}$.

Since the predicate $\equals$ is symmetric, whenever a constant $c$ is a
relevant binding for the first argument of $\equals$, constant $c$ is also a
relevant binding for the second argument; thus, introducing distinct magic
predicates $\mgc{\equals}{\ad{bf}}$ and $\mgc{\equals}{\ad{fb}}$ is redundant.
Consequently, Definition~\ref{def:adornment} introduces just one magic
predicate $\mgc{\equals}{\adEqb}$ for $\equals$. Notation $\adEqb$ suggests
that one argument is bound and the other one is free, but without fixing the
order of the arguments.

We next adapt the notion of a sideways information passing strategy to our
setting. As in the equality-free case, a SIPS describes how bindings are passed
sideways through rule bodies.

\begin{definition}\label{def:SIPS}
    A \emph{sideways information passing strategy} is a function that takes a
    well-formed conjunction $C$ of $n$ atoms and a set $V$ of variables.
    Applying the function to $C$ and $V$ produces
    \begin{displaymath}
        \SIPS(C,V) = \langle \langle R_1(\vec t_1), \dots, R_n(\vec t_n) \rangle, \langle \gamma_1, \dots, \gamma_n \rangle \rangle,
    \end{displaymath}
    where ${\langle R_1(\vec t_1), \dots, R_n(\vec t_n) \rangle}$ is a
    permutation of the atoms of $C$, and ${\langle \gamma_1, \dots, \gamma_n
    \rangle}$ is a sequence where each $\gamma_i$ is an adornment for $R_i$.
    For each ${i \in \{ 1, \dots, n \}}$, property \eqref{eq:SIPS:1} must hold.
    \begin{align}
        \vars{\vec t_i^{\gamma_i}} \subseteq V \cup \bigcup \Big\{ \vars{\vec t_k} \mid 1 \leq k < i \Big \} \label{eq:SIPS:1}
    \end{align}
    Furthermore, for each ${i \in \{ 1, \dots, n \}}$ such that $R_i(\vec t_i)$
    is of the form ${t_i^1 \equals t_i^2}$, and for each ${j \in \{ 1, 2 \}}$
    such that the $j$-th element of $\gamma_i$ is $\ad{b}$, property
    \eqref{eq:SIPS:2} must hold.
    \begin{align}
        \vars{t_i^j} \subseteq \bigcup \Big\{ \vars{\vec t_k} \mid 1 \leq k < i \text{ and predicate } R_k \text{ is relational} \Big \} \label{eq:SIPS:2}
    \end{align}
\end{definition}

A SIPS is given a conjunction $C$ corresponding to a rule body, and a set of
variables $V$ that will be bound when the body is evaluated (i.e., variables
$V$ will be passed as bindings from the head). The SIPS reorders the conjuncts
of $C$ and adorns each atom's predicate to indicate which arguments of the atom
are bound. The result of a SIPS must satisfy two conditions. The condition in
equation \eqref{eq:SIPS:1} captures the essence of sideways information
passing: each bound variable that is passed sideways to an atom $R_i(\vec t_i)$
must occur in either $V$ or a preceding (relational or equality) atom $R_k(\vec
t_k)$. The condition in equation \eqref{eq:SIPS:2} requires that each variable
bound in an equality atom also occurs in a preceding relational atom. This is
needed to ensure termination of the fixpoint computation of the transformed
program, but the details are quite technical and rely on the exact formulation
of our approach; thus, we defer a detailed discussion to
Example~\ref{ex:magic:termination:2}. For now, we just observe that this
requirement can always be satisfied when conjunction $C$ is well formed: each
variable occurring in an equality atom in $C$ also occurs in a relational atom,
so we can always reorder all relevant relational atoms before the respective
equality atoms.

With these definitions in place, we present our variant of the magic sets
transformation in Algorithm~\ref{alg:magic}. The algorithm takes as input a
well-formed program $P$, and it outputs a transformed program $\R$. In
line~\ref{alg:magic:init}, the algorithm initialises the `ToDo' set
$\mathcal{T}$ with the magic predicate $\mgc{\predQ}{\alpha}$ for $\alpha$
where all positions are free, and it initialises $\R$ to contain the rule
$\mgc{\predQ}{\alpha} \leftarrow$ with the empty body; thus, the resulting
program will entail all facts of the form $\predQ(\vec a)$ entailed by $P$. The
algorithm then enters a loop
(lines~\ref{alg:magic:T:start}--\ref{alg:magic:T:end}) where it successively
processes a magic predicate from $\mathcal{T}$. The set $\Di{}$ ensures that
each magic predicate is processed only once. In each iteration, the algorithm
selects a magic predicate $\mgc{R}{\alpha}$ to process
(line~\ref{alg:magic:R:start}), and it processes each rule ${r \in P \cup
\DOM{P}}$ containing $R$ in the head
(lines~\ref{alg:magic:r:start}--\ref{alg:magic:r:end}). If $R$ is the equality
predicate, then $\alpha$ is $\adEqb$ so the rule $r$ is processed by adorning
$\head{r}$ with both $\ad{bf}$ and $\ad{fb}$
(lines~\ref{alg:magic:process:bf}--\ref{alg:magic:process:fb}); otherwise, $r$
is processed by adorning the head of $r$ by $\alpha$
(line~\ref{alg:magic:process}). Either way, the algorithm extends the result
set $\R$ with the modified rule for $r$ (line~\ref{alg:magic:mod-rule}),
consults the SIPS to determine how to propagate the bindings among the atoms of
$r$ (line~\ref{alg:magic:SIPS}), generates the magic rule for each body atom of
$r$ that also occurs in $P$ in a rule head (line~\ref{alg:magic:magic-rule}),
and adds all freshly generated magic predicates to $\mathcal{T}$ for further
processing (line~\ref{alg:magic:add-S}). In line~\ref{alg:magic:S}, the
algorithm takes into account that $\mgc{\equals}{\adEqb}$ is the magic
predicate for $\equals$ adorned by $\ad{bf}$ or $\ad{fb}$.

If we wanted to use the standard magic sets transformation, we should apply it
to all the rules of ${P \cup \SG{P}}$. Now assume that $P$ contains a rule ${r
= C(x) \leftarrow A(x) \wedge x \equals x' \wedge B(x')}$, and that $r$ is
matched to an instance using a substitution $\sigma$ where ${\sigma(x) =
\sigma(x') = t}$---that is, the equality atom ${x \equals x'}$ is matched to a
fact of the form ${t \equals t}$. This fact can be produced by the reflexivity
rule \eqref{eq:ref} and the domain rule \eqref{eq:dom-R} for an arbitrary
predicate $R$ occurring in $P$. Therefore, when applied to ${P \cup \SG{P}}$,
the standard magic sets algorithm would process $r$, the reflexivity rule
\eqref{eq:ref}, and the domain rules \eqref{eq:dom-R} for each predicate $R$.
Consequently, the resulting program would necessarily `touch' all predicates
from the input program and would thus likely be large.

Line~\ref{alg:magic:r:start} of Algorithm~\ref{alg:magic} addresses the issue
from the previous paragraph by excluding the reflexivity rule from processing
and thus considerably reducing the size of the result. This is possible because
the rules of $P$ are well formed. In the rule $r$ from the previous paragraph,
variables $x$ and $x'$ of equality ${x \equals x'}$ occur in relational atoms
$A(x)$ and $B(x')$. Thus, if ${x \equals x'}$ is matched to ${t \equals t}$,
atoms $A(x)$ and $B(x')$ are matched to $A(t)$ and $B(t)$; but then, we can
recover ${t \equals t}$ using the domain rules for $A$ and $B$ only, without
considering any other predicate in the input program.

In fact, the symmetry, the transitivity, and the $\D$-restricted functional
reflexivity rules are not processed in line~\ref{alg:magic:r:start} either: the
resulting program $\R$ will be considered together with $\SG{\R}$, so there is
no point in adding the modified rules to $\R$. The algorithm only needs the
magic rules obtained from the rules of $\SG{P}$: line~\ref{alg:magic:eq:mgc:1}
introduces the magic rule obtained by annotating the head of the transitivity
rule \eqref{eq:trans} with $\ad{bf}$ or $\ad{fb}$, and
lines~\ref{alg:magic:eq:mgc:2}--\ref{alg:magic:eq:mgc:4} introduce the magic
rules obtained by annotating the head of the $\D$-restricted functional
reflexivity rules \eqref{eq:Dfnref}.

Finally, the magic rules produced in lines~\ref{alg:magic:eq:mgc:3}
and~\ref{alg:magic:eq:mgc:4} contain the domain predicate $\D$ in the body;
hence, if at least one such rule is produced, the domain rules need to be
considered in line~\ref{alg:magic:r:start} to produce the magic rules that
identify the relevant bindings for the $\D$ predicate. To ensure that the
domain rules are indeed analysed, the magic predicate $\mgc{\D}{\ad{b}}$ is
added to $\mathcal{T}$ in line~\ref{alg:magic:add-D} whenever the $\equals$
predicate is processed and the input program $P$ contains at least one true
function symbol. In other words, the domain rules are not considered if $P$ is
obtained from first-order dependencies.

\begin{algorithm}[tb]
\caption{$\magic{P}$}\label{alg:magic}
\begin{small}
\begin{algorithmic}[1]\footnotesize
    \State $\Di{} \defeq \mathcal{T} \defeq \{ \mgc{\predQ}{\alpha} \}$ and $\R \defeq \{ \mgc{\predQ}{\alpha} \leftarrow \}$, where $\alpha = \ad{f} \cdots \ad{f}$                    \label{alg:magic:init}
    \While{$\mathcal{T} \neq \emptyset$}                                                                                                                                                \label{alg:magic:T:start}
        \State \textbf{choose and remove} some $\mgc{R}{\alpha}$ from $\mathcal{T}$                                                                                                     \label{alg:magic:R:start}
        \For{\textbf{each} ${r \in P \cup \DOM{P}}$ such that $\head{r} = R(\vec t)$}                                                                                                   \label{alg:magic:r:start}
            \If{$R = {\equals}$}                                                    \Comment{Note that $\alpha = \adEqb$}
                \State $\process(r,\alpha,\ad{bf})$                                                                                                                                     \label{alg:magic:process:bf}
                \State $\process(r,\alpha,\ad{fb})$                                                                                                                                     \label{alg:magic:process:fb}
                \If{$\FR{P} \neq \emptyset$ and $\mgc{\D}{\ad{b}} \not\in \Di{}$}
                    \textbf{add} $\mgc{\D}{\ad{b}}$ to $\mathcal{T}$ and $\Di{}$                                                                                                        \label{alg:magic:add-D}
                \EndIf
            \Else
                \State $\process(r,\alpha,\alpha)$                                                                                                                                      \label{alg:magic:process}
            \EndIf
        \EndFor                                                                                                                                                                         \label{alg:magic:r:end}
    \EndWhile                                                                                                                                                                           \label{alg:magic:T:end}
    \If{$\mgc{\equals}{\adEqb} \in \Di{}$}                                                                                                                                              \label{alg:magic:eq:start}
        \State \textbf{add} $\mgc{\equals}{\adEqb}(x_2) \leftarrow \mgc{\equals}{\adEqb}(x_1) \wedge x_1 \equals x_2$ to $\R$                                                           \label{alg:magic:eq:mgc:1}
        \For{\textbf{each} true function symbol $f$ in $P$ of arity $n$ and \textbf{each} $i \in \{ 1, \dots, n \}$}
            \State \textbf{add} $\mgc{\D}{\ad{b}}(x_i) \leftarrow \mgc{\equals}{\adEqb}(f(x_1,\dots,x_n))$ to $\R$                                                                      \label{alg:magic:eq:mgc:2}
            \State \textbf{add} $\mgc{\equals}{\adEqb}(x_i) \leftarrow \mgc{\equals}{\adEqb}(f(x_1,\dots,x_n)) \wedge \D(x_i)$ to $\R$                                                  \label{alg:magic:eq:mgc:3}
            \State \textbf{add} $\mgc{\D}{\ad{b}}(x_i') \leftarrow \mgc{\equals}{\adEqb}(f(x_1,\dots,x_n)) \wedge \D(x_i) \wedge x_i \equals x_i'$ to $\R$                              \label{alg:magic:eq:mgc:4}
        \EndFor
    \EndIf                                                                                                                                                                              \label{alg:magic:eq:end}
    \State \Return $\R$                                                                                                                                                                 \label{alg:magic:return}
    \Statex
    \Procedure{$\process$}{$r, \alpha, \beta$} where $\head{r} = R(\vec t)$                                                                                                             \label{alg:magic:fn-process}
        \If{$r \not\in \DOM{P}$}
            \textbf{add} $\head{r} \leftarrow \mgc{R}{\alpha}(\vec t^\beta) \wedge \body{r}$ to $\R$                                                                                    \label{alg:magic:mod-rule}
        \EndIf
        \State $\langle \langle R_1(\vec t_1), \dots, R_n(\vec t_n) \rangle, \langle \gamma_1, \dots, \gamma_n \rangle \rangle \defeq \SIPS(\body{r},\vars{\vec t^\beta})$              \label{alg:magic:SIPS}
        \For{\textbf{each} $1 \leq i \leq n$ such that $R_i = {\equals}$ or $R_i$ occurs in $P$ in a rule head}                                                                         \label{alg:magic:body:start}
            \State \textbf{if} $R_i = {\equals}$ and $\gamma_i \in \{ \ad{bf}, \ad{fb} \}$ \textbf{then} $S \defeq \mgc{\equals}{\adEqb}$ \textbf{else} $S \defeq \mgc{R_i}{\gamma_i}$  \label{alg:magic:S}
            \State \textbf{add} $S(\vec t_i^{\gamma_i}) \leftarrow \mgc{R}{\alpha}(\vec t^\beta) \wedge R_1(\vec t_1) \wedge \dots \wedge R_{i-1}(\vec t_{i-1})$ to $\R$                \label{alg:magic:magic-rule}
            \If{$S \not \in \Di{}$}
                \textbf{add} $S$ to $\mathcal{T}$ and $\Di{}$                                                                                                                           \label{alg:magic:add-S}
            \EndIf
        \EndFor                                                                                                                                                                         \label{alg:magic:body:end}
    \EndProcedure
\end{algorithmic}
\end{small}
\hrule
\begin{tabular}{p{0.95\textwidth}}
    \textbf{Note:} Equality atoms of the form $t_1 \equals t_2$ are
    abbreviations for ${\equals}(t_1,t_2)$, so $R(\vec t)$ and $R_i(\vec t_i)$
    can be equality atoms.
\end{tabular}
\end{algorithm}

As we already mentioned, programs ${P \cup \SG{P}}$ and ${\R \cup \SG{\R}}$
entail the same query answers, but there is an important detail: according to
Definition~\ref{def:sg}, program $\SG{\R}$ would need to contain the domain
rules \eqref{eq:dom-R} even for the magic predicates. This would not be a
problem from a correctness point of view; however, as we discuss in
Example~\ref{ex:magic:termination:1}, this can adversely affect the termination
of the fixpoint computation for ${\R \cup \SG{\R}}$. Therefore, in
Definition~\ref{def:magic-SG} we refine the definition of $\SG{\R}$ and
stipulate that the domain rules \eqref{eq:dom-R} are never instantiated for
magic predicates. Theorem~\ref{thm:magic} shows that the query answers
nevertheless remain preserved.

\begin{definition}\label{def:magic-SG}
    For $P$ a well-formed program and ${\R = \magic{P}}$, program $\SG{\R}$ is
    defined as in Definition~\ref{def:sg} but without instantiating the domain
    rules of the form \eqref{eq:dom-R} for the magic predicates.
\end{definition}

\begin{restatable}{theorem}{thmMagic}\label{thm:magic}
    For each well-formed program $P$ defining a query $\predQ$, each base
    instance $B$, each fact of the form ${\predQ(\vec a)}$, and ${\R =
    \magic{P}}$, it is the case that ${P \cup \SG{P} \cup B \models \predQ(\vec
    a)}$ if and only if ${\R \cup \SG{\R} \cup B \models \predQ(\vec a)}$.
\end{restatable}

Theorem~\ref{thm:magic} holds regardless of whether the fixpoint of ${P \cup
\SG{P} \cup B}$ is finite. If, however, it is, then
Theorem~\ref{thm:magic:termination} guarantees that the fixpoint of ${\R \cup
\SG{\R} \cup B}$ is finite too, so the fixpoint computation terminates. This is
not a trivial observation: the rules of $P$ can contain function symbols in the
body, which are transformed into function symbols in the head of the magic
rules. Thus, $\R$ can contain rules with function symbols in the head at
positions not originally found in $P$, so it is not obvious that the fixpoint
of ${\R \cup \SG{\R} \cup B}$ is necessarily finite.

\begin{restatable}{theorem}{thmMagicTermination}\label{thm:magic:termination}
    For each well-formed program $P$, each base instance $B$, and programs ${\R
    = \magic{P}}$, ${P_1 = P \cup \SG{P}}$, and ${P_2 = \R \cup \SG{\R}}$, if
    ${\fixpoint{P_1}{B}}$ is finite, then ${\fixpoint{P_2}{B}}$ is finite as
    well.
\end{restatable}

Example~\ref{ex:mgc} illustrates how Algorithm~\ref{alg:magic} is applied to
the rules produced by our relevance analysis algorithm in
Example~\ref{ex:relevance}.

\begin{example}\label{ex:mgc}
We now apply Algorithm~\ref{alg:magic} to rules
\eqref{ex:rel:Q}--\eqref{ex:rel:U-eq} from Example~\ref{ex:relevance}. We use
horizontal lines to separate the rules produced in each invocation of
$\mathsf{process}$. Rule \eqref{ex:mgc:start} is produced in
line~\ref{alg:magic:init}. Next, processing the magic predicate
$\mgc{\predQ}{\ad{f}}$ produces rules \eqref{ex:mgc:Q}--\eqref{ex:mgc:Q:R},
where we assume that the SIPS reorders the body of rule \eqref{ex:rel:Q} as
shown in rule \eqref{ex:mgc:Q}. Moreover, processing the magic predicates
$\mgc{B}{\ad{ff}}$, $\mgc{A}{\ad{b}}$, $\mgc{R}{\ad{bf}}$, and
$\mgc{U}{\ad{ff}}$ produces rules \eqref{ex:mgc:U-Bf}--\eqref{ex:mgc:U-Bf:U},
\eqref{ex:mgc:C-Af}, \eqref{ex:mgc:S-R:1}, and \eqref{ex:mgc:C-Uf},
respectively. Now consider the magic predicate $\mgc{\equals}{\adEqb}$.
Predicate $\equals$ occurs in the head of rule \eqref{ex:rel:U-eq} so
$\mathsf{process}$ is called for this rule twice: first with ${\beta =
\ad{bf}}$, and then with ${\beta = \ad{fb}}$. A maximal SIPS would thus adorn
$U(x_1,x_2)$ with $\ad{bf}$ and $\ad{fb}$, respectively. However, rules
\eqref{ex:mgc:U-Bf} and \eqref{ex:mgc:U-Bf:U} already use the adornment
$\ad{ff}$ for $U$, meaning that all facts with the $U$ predicate should be
computed. Thus, it is more efficient to not propagate the bindings for $x_1$
and $x_2$ in rule \eqref{ex:rel:U-eq} to $U(x_1,x_2)$, and for the two
invocations of $\mathsf{process}$ to produce rules
\eqref{ex:mgc:U-eq:1}--\eqref{ex:mgc:U-eq:1:U} and
\eqref{ex:mgc:U-eq:2}--\eqref{ex:mgc:U-eq:2:U}, respectively.
\begin{align}
    \mgc{\predQ}{\ad{f}}            & \leftarrow                                                                                                                        \label{ex:mgc:start}    \\
    \hline
    \predQ(x_1)                     & \leftarrow \mgc{\predQ}{\ad{f}} \wedge B(x_3') \wedge x_3 \equals x_3' \wedge A(x_3) \wedge f(x_1) \equals x_3 \wedge R(x_1,x_2)  \label{ex:mgc:Q}        \\
    \mgc{B}{\ad{f}}                 & \leftarrow \mgc{\predQ}{\ad{f}}                                                                                                   \label{ex:mgc:Q:B}      \\
    \mgc{\equals}{\adEqb}(x_3')     & \leftarrow \mgc{\predQ}{\ad{f}} \wedge B(x_3')                                                                                    \label{ex:mgc:Q:eq1}    \\
    \mgc{A}{\ad{b}}(x_3)            & \leftarrow \mgc{\predQ}{\ad{f}} \wedge B(x_3') \wedge x_3 \equals x_3'                                                            \label{ex:mgc:Q:A}      \\
    \mgc{\equals}{\adEqb}(x_3)      & \leftarrow \mgc{\predQ}{\ad{f}} \wedge B(x_3') \wedge x_3 \equals x_3' \wedge A(x_3)                                              \label{ex:mgc:Q:eq2}    \\
    \mgc{R}{\ad{bf}}(x_1)           & \leftarrow \mgc{\predQ}{\ad{f}} \wedge B(x_3') \wedge x_3 \equals x_3' \wedge A(x_3) \wedge f(x_1) \equals x_3                    \label{ex:mgc:Q:R}      \\
    \hline
    B(f(x_2))                       & \leftarrow \mgc{B}{\ad{f}} \wedge U(x_1,x_2)                                                                                      \label{ex:mgc:U-Bf}     \\
    \mgc{U}{\ad{ff}}                & \leftarrow \mgc{B}{\ad{f}}                                                                                                        \label{ex:mgc:U-Bf:U}   \\
    \hline
    A(f(x))                         & \leftarrow \mgc{A}{\ad{b}}(f(x)) \wedge C(x)                                                                                      \label{ex:mgc:C-Af}     \\
    \hline
    R(x_1,g(x_1))                   & \leftarrow \mgc{R}{\ad{bf}}(x_1) \wedge S(x_1,x_2)                                                                                \label{ex:mgc:S-R:1}    \\
    \hline
    U(x,f(x))                       & \leftarrow \mgc{U}{\ad{ff}} \wedge C(x)                                                                                           \label{ex:mgc:C-Uf}     \\
    \hline
    x_1 \equals x_2                 & \leftarrow \mgc{\equals}{\adEqb}(x_1) \wedge U(x_1,x_2)                                                                           \label{ex:mgc:U-eq:1}   \\
    \mgc{U}{\ad{ff}}                & \leftarrow \mgc{\equals}{\adEqb}(x_1)                                                                                             \label{ex:mgc:U-eq:1:U} \\
    \hline
    x_1 \equals x_2                 & \leftarrow \mgc{\equals}{\adEqb}(x_2) \wedge U(x_1,x_2)                                                                           \label{ex:mgc:U-eq:2}   \\
    \mgc{U}{\ad{ff}}                & \leftarrow \mgc{\equals}{\adEqb}(x_2)                                                                                             \label{ex:mgc:U-eq:2:U}
\end{align}

The input program contains the true function symbol $f$. Hence, the
$\D$-restricted functional reflexivity rule \eqref{eq:Dfnref} for $f$ can also
derive facts involving the equality predicate, and we need to take this rule
into account when processing $\mgc{\equals}{\adEqb}$ in
lines~\ref{alg:magic:process:bf}--\ref{alg:magic:add-D}. Binding propagation
for such rules is always the same, so the relevant rules are added directly in
lines~\ref{alg:magic:eq:mgc:2}--\ref{alg:magic:eq:mgc:4}. Furthermore, the
magic predicate $\mgc{\D}{\ad{b}}$ is added to the set $\mathcal{T}$ in
line~\ref{alg:magic:add-D}, which ensures that the relevant domain rules are
processed too. Domain rules will be added in their entirety to the result of
the magic transformation, so only the relevant magic rules
\eqref{ex:mgc:dom-B}--\eqref{ex:mgc:dom-R:2} are produced to ensure that the
bindings for $\D$ are correctly propagated to the relevant predicates. In rules
\eqref{ex:mgc:dom-B} and \eqref{ex:mgc:dom-U}, we again take into account that
we have already processed $\mgc{B}{\ad{f}}$ and $\mgc{U}{\ad{ff}}$, so it would
be inefficient to additionally process $\mgc{B}{\ad{b}}$, $\mgc{U}{\ad{bf}}$,
or $\mgc{U}{\ad{fb}}$. However, rule \eqref{ex:mgc:dom-R:2} introduces the
magic predicate $\mgc{R}{\ad{fb}}$, which in turn introduces rule
\eqref{ex:mgc:S-R:2}.
\begin{align}
    \mgc{B}{\ad{f}}                 & \leftarrow \mgc{\D}{\ad{b}}(x)                        \label{ex:mgc:dom-B}   \\
    \hline
    \mgc{A}{\ad{b}}(x)              & \leftarrow \mgc{\D}{\ad{b}}(x)                        \label{ex:mgc:dom-A}   \\
    \hline
    \mgc{U}{\ad{ff}}                & \leftarrow \mgc{\D}{\ad{b}}(x)                        \label{ex:mgc:dom-U}   \\
    \hline
    \mgc{R}{\ad{bf}}(x)             & \leftarrow \mgc{\D}{\ad{b}}(x)                        \label{ex:mgc:dom-R:1} \\
    \hline
    \mgc{R}{\ad{fb}}(x)             & \leftarrow \mgc{\D}{\ad{b}}(x)                        \label{ex:mgc:dom-R:2} \\
    \hline
    R(x_1,g(x_1))                   & \leftarrow \mgc{R}{\ad{fb}}(g(x_1)) \wedge S(x_1,x_2) \label{ex:mgc:S-R:2}
\end{align}

At this point all relevant bindings have been processed, so the loop in
lines~\ref{alg:magic:T:start}--\ref{alg:magic:T:end} finishes. Since
$\mgc{\equals}{\adEqb}$ was processed, the following rules are added in
lines~\ref{alg:magic:eq:mgc:1}--\ref{alg:magic:eq:mgc:4}. In particular, rule
\eqref{ex:mgc:eq} captures binding propagation through the transitivity rule
\eqref{eq:trans}, and rules \eqref{ex:mgc:f:1}--\eqref{ex:mgc:f:3} capture
binding propagation through the $\D$-restricted functional reflexivity rule
\eqref{eq:Dfnref} for $f$.
\begin{align}
    \mgc{\equals}{\adEqb}(x_2)      & \leftarrow \mgc{\equals}{\adEqb}(x_1) \wedge x_1 \equals x_2                      \label{ex:mgc:eq}  \\
    \mgc{\D}{\ad{b}}(x_1)           & \leftarrow \mgc{\equals}{\adEqb}(f(x_1))                                          \label{ex:mgc:f:1} \\
    \mgc{\equals}{\adEqb}(x_1)      & \leftarrow \mgc{\equals}{\adEqb}(f(x_1)) \wedge \D(x_1)                           \label{ex:mgc:f:2} \\
    \mgc{\D}{\ad{b}}(x_1')          & \leftarrow \mgc{\equals}{\adEqb}(f(x_1)) \wedge \D(x_1) \wedge x_1 \equals x_1'   \label{ex:mgc:f:3}
\end{align}

The result $\R$ of our magic algorithm thus consists of rules
\eqref{ex:mgc:start}--\eqref{ex:mgc:f:3}. Now, it is straightforward to see
that instance ${\fixpoint{\R \cup \SG{\R}}{\exB}}$ contains the fact
$\predQ(a_1)$, but no fact of the form $R(a_i,g(a_i))$ with ${i \geq 2}$. In
other words, our magic sets transformation prunes the derivations of rule
\eqref{ex:rel:S-R} that involve $a_i$ with ${i \geq 2}$ since these do not
contribute to the query answers.
\end{example}

Examples~\ref{ex:magic:termination:1} and~\ref{ex:magic:termination:2} explain
the role that Definition~\ref{def:magic-SG} and condition \eqref{eq:SIPS:2} of
Definition~\ref{def:SIPS} play in the proof of
Theorem~\ref{thm:magic:termination}.

\begin{example}\label{ex:magic:termination:1}
Let $f^i$ abbreviate $i$ applications of $f$; for example, $f^2(c)$ abbreviates
$f(f(c))$. To understand why the domain rules \eqref{eq:dom-R} \emph{should
not} be instantiated for the magic predicates (cf.\
Definition~\ref{def:magic-SG}), note that the rules of ${\R \cup \SG{\R}}$ from
Example~\ref{ex:mgc} derive $\D(f(a_1))$, $\D(f^2(a_1))$, and
$\mgc{\equals}{\adEqb}(f^2(a_1))$; thus, the $\D$-restricted reflexivity rule
\eqref{eq:Dfnref} for $f$ derives ${f^2(a_1) \equals f^3(a_1)}$, and so the
magic rule \eqref{ex:mgc:eq} derives $\mgc{\equals}{\adEqb}(f^3(a))$. Now, if
$\SG{\R}$ contained rule \eqref{eq:dom-m-eq}, we would derive $\D(f^3(a))$, so
rule \eqref{eq:Dfnref} for $f$ would derive ${f^3(a_1) \equals f^4(a_1)}$; but
then, rule \eqref{ex:mgc:eq} would derive $\mgc{\equals}{\adEqb}(f^4(a))$, rule
\eqref{eq:dom-m-eq} would derive $\D(f^4(a))$, and the process would never
terminate.
\begin{align}
    \D(x) \leftarrow \mgc{\equals}{\adEqb}(x) \label{eq:dom-m-eq}
\end{align}

Definition~\ref{def:magic-SG} is thus critical to
Theorem~\ref{thm:magic:termination}. Intuitively, by not instantiating the
domain rules \eqref{eq:dom-R} for the magic predicates, we ensure that each
fact with the $\D$ predicate that is derived by ${\R \cup \SG{\R}}$ is also
derived by ${P \cup \SG{P}}$. Hence, the application of the $\D$-restricted
functional reflexivity axioms is constrained in the same way as discussed in
Section~\ref{sec:answering:singularisation}.
\end{example}

\begin{example}\label{ex:magic:termination:2}
To understand the condition \eqref{eq:SIPS:2} from Definition~\ref{def:SIPS},
consider processing rule \eqref{ex:magic:termination:rule} in
line~\ref{alg:magic:process:bf} (so ${\beta = \ad{bf}}$) where the SIPS keeps
the order of the body atoms unchanged and adorns them with $\ad{bf}$, $\ad{b}$,
and $\ad{b}$, respectively. Line~\ref{alg:magic:magic-rule} then produces rule
\eqref{ex:magic:termination:mag:1}, which makes the fixpoint of the transformed
program infinite whenever a fact of the form $\mgc{\equals}{\adEqb}(t)$ is
derived.
\begin{align}
    x_1 \equals x_2                 & \leftarrow f(x_1) \equals g(x_2) \wedge A(x_1) \wedge B(x_2)  \label{ex:magic:termination:rule} \\
    \mgc{\equals}{\adEqb}(f(x_1))   & \leftarrow \mgc{\equals}{\adEqb}(x_1)                         \label{ex:magic:termination:mag:1}
\end{align}

To address this problem, condition \eqref{eq:SIPS:2} of
Definition~\ref{def:SIPS} requires that, if a term in a body equality is
adorned as bound, then each variable in the term must occur in a preceding
relational atom. Thus, on our example, the SIPS can order the body atoms as in
rule \eqref{ex:magic:termination:rule:1} and adorn them by $\ad{b}$, $\ad{bf}$,
and $\ad{b}$, respectively. Crucially, the term $f(x_1)$ in the equality atom
${f(x_1) \equals g(x_2)}$ is adorned as bound, but its variable $x_1$ occurs in
the preceding relational atom $A(x_1)$; hence, line~\ref{alg:magic:magic-rule}
now produces the magic rule \eqref{ex:magic:termination:mag:2}. Although this
rule is similar to \eqref{ex:magic:termination:mag:1}, atom $A(x_1)$ solves the
termination problem. In particular, all facts derived by the transformed
program with a predicate that is not a magic predicate are also derived by the
input program. Hence, if the depth of the facts $A(t)$ derived by the input
program is bounded by $d$, then the depth of such facts derived by the
transformed program is bounded by $d$ as well; but then, the depth of the facts
derived by rule \eqref{ex:magic:termination:mag:2} is bounded by ${d+1}$.
\begin{align}
    x_1 \equals x_2                 & \leftarrow A(x_1) \wedge f(x_1) \equals g(x_2) \wedge B(x_2)  \label{ex:magic:termination:rule:1} \\
    \mgc{\equals}{\adEqb}(f(x_1))   & \leftarrow \mgc{\equals}{\adEqb}(x_1) \wedge A(x_1)           \label{ex:magic:termination:mag:2}
\end{align}
\end{example}

\subsection{Removal of Constants and Function Symbols in Rule Bodies}\label{sec:answering:defun}

The programs produced by the relevance analysis and/or the magic sets
transformation can be used to answer a query, but treating $\equals$ as an
ordinary predicate can make this process inefficient. As we explained in
Section~\ref{sec:answering:overview}, the main objective of our work is to
optimise the input dependencies so that they can be evaluated with `true'
equality. Towards this goal, we next present two final transformations that
bring the resulting program into a required form.

The magic sets transformation can introduce rules such as \eqref{ex:mgc:S-R:2}
and \eqref{ex:mgc:f:3} containing relational atoms with function symbols in the
body. The chase variant from Section~\ref{sec:so-dependencies:chase} cannot
process such rules, so we next show how to remove the constants and function
symbols from the rule bodies. Intuitively, for each $n$-ary function symbol
$f$, we introduce a fresh $n+1$-ary predicate $\fnpred{f}$, and we introduce
rules that, for each relevant ground term $f(\vec t)$, ensure that $\fnpred{f}$
associates $\vec t$ with $f(\vec t)$. This allows us to rewrite references to
the functional terms in the rule bodies as references to atoms with the
$\fnpred{f}$ predicate. This is formalised in Definition~\ref{def:defun}.

\begin{definition}\label{def:defun}
    Program $\defun{P}$ is obtained from a program $P$ by exhaustively applying
    each of the following steps in the following sequence.
    \begin{enumerate}
        \item In the body of each rule, replace each occurrence of a constant
        $c$ with a fresh variable $z_c$ unique for $c$, add the atom
        ${\fnpred{c}(z_c)}$ to the body, and add the rule ${\fnpred{c}(c)
        \leftarrow}$.

        \item In the body of each rule, replace each occurrence of a term of
        the form ${f(\vec s)}$ with a fresh variable $z_{f(\vec s)}$ unique for
        $f(\vec s)$, and add the atom ${\fnpred{f}(\vec s,z_{f(\vec s)})}$ to
        the rule body.

        \item For each function symbol $f$ considered in the previous step, and
        for each rule $r$ that contains a term of the form $f(\vec t)$ in its
        head, add the rule ${\fnpred{f}(\vec t, f(\vec t)) \leftarrow
        \body{r}}$.

        \item For each true function symbol $f$ occurring in $P$, add the rule
        \begin{align}
            \fnpred{f}(x_1,\dots,x_n,f(x_1,\dots,x_n)) \leftarrow \D(x_1) \wedge x_1 \equals x_1' \wedge \D(x_1') \wedge \dots \wedge \D(x_n) \wedge x_n \equals x_n' \wedge \D(x_n').  \label{eq:defun:Dfnref}
        \end{align}
    \end{enumerate}
\end{definition}

The first step of Definition~\ref{def:defun} eliminates all constants, and the
second step eliminates all functional terms from all relational atoms in rule
bodies. The third and the fourth step axiomatise $\fnpred{f}$: the third step
accounts for the rules of $P$, and the fourth step accounts for the
$\D$-restricted functional reflexivity rules. It is straightforward to see that
all query answers remain preserved.

\begin{proposition}\label{prop:defun}
    For each program $P$, each base instance $B$, each fact $R(\vec t)$ with
    $R$ not of the form $\fnpred{f}$, and ${\R = \defun{P}}$, it is the case
    that ${P \cup \SG{P} \cup B \models R(\vec t)}$ if and only if ${\R \cup
    \SG{\R} \cup B \models R(\vec t)}$.
\end{proposition}

\begin{proof}
Let ${I = \fixpoint{P \cup \SG{P}}{B}}$ and let ${J = \fixpoint{\R \cup
\SG{\R}}{B}}$. By routine inductions on the construction of $I$ and $J$, one
can show that
\begin{displaymath}
    J = I \cup \{ \fnpred{f}({\vec t, f(\vec t)}) \mid f(\vec t) \text{ occurs in } I \text{ and } \fnpred{f} \text{ occurs in } \R \}.
\end{displaymath}
Each derivation of a fact $F$ in $I$ by a rule ${r \in P \cup \SG{P}}$
corresponds to the derivation of $F$ and ${\fnpred{f}(\vec t, f(\vec t))}$ for
each term ${f(\vec t)}$ occurring in $F$ in $J$ by the rules obtained from $r$
by transformations in Definition~\ref{def:defun}; moreover, each term occurring
in a fact with predicate of the form $\fnpred{f}$ also occurs in a fact with a
predicate not of the form $\fnpred{f}$ so the rules in ${\DOM{\R} \setminus
\DOM{P}}$ do not derive any new facts. Thus, ${P \cup \SG{P} \cup B}$ and ${\R
\cup \SG{\R} \cup B}$ entail the same facts over the predicates of $P$.
\end{proof}

\begin{example}\label{ex:defun}
Consider again rules \eqref{ex:mgc:start}--\eqref{ex:mgc:f:3} produced by
applying the magic sets transformation to our running example. Rule
\eqref{ex:mgc:Q} contains a functional term in a body equality atom, and it can
be handled directly by the chase for generalised SO dependencies from
Section~\ref{sec:so-dependencies:chase}. In contrast, rule \eqref{ex:mgc:Q:R}
also contains a functional term only in the equality atom ${f(x_1) \equals
x_3}$, but the variable $x_1$ does not occur in a relational body atom; thus,
it is unclear how to match the body variables of this rule as described in
Section~\ref{sec:so-dependencies:chase}. Furthermore, rules
\eqref{ex:mgc:C-Af}, \eqref{ex:mgc:S-R:2}, and
\eqref{ex:mgc:f:1}--\eqref{ex:mgc:f:3} all contain functional terms in
relational atoms in the body. However, the transformation described in
Definition~\ref{def:defun} brings these rules into a suitable form as shown
below.
\begin{align}
    \predQ(x_1)                 & \leftarrow \mgc{\predQ}{\ad{f}} \wedge B(x_3') \wedge x_3 \equals x_3' \wedge A(x_3) \wedge z_{f(x_1)} \equals x_3 \wedge R(x_1,x_2) \wedge \fnpred{f}(x_1,z_{f(x_1)})    \label{ex:defun:Q} \\
    \mgc{R}{\ad{bf}}(x_1)       & \leftarrow \mgc{\predQ}{\ad{f}} \wedge B(x_3') \wedge x_3 \equals x_3' \wedge A(x_3) \wedge z_{f(x_1)} \equals x_3 \wedge \fnpred{f}(x_1,z_{f(x_1)})                      \label{ex:defun:Q:R} \\
    A(f(x))                     & \leftarrow \mgc{A}{\ad{b}}(z_{f(x)}) \wedge C(x) \wedge \fnpred{f}(x,z_{f(x)})                                                                                            \label{ex:defun:C-Af} \\
    R(x_1,g(x_1))               & \leftarrow \mgc{R}{\ad{fb}}(z_{g(x_1)}) \wedge S(x_1,x_2) \wedge \fnpred{g}(x_1,z_{g(x_1)})                                                                               \label{ex:defun:S-R:2} \\
    \mgc{\D}{\ad{b}}(x_1)       & \leftarrow \mgc{\equals}{\adEqb}(z_{f(x_1)}) \wedge \fnpred{f}(x_1,z_{f(x_1)})                                                                                            \label{ex:defunf:1} \\
    \mgc{\equals}{\adEqb}(x_1)  & \leftarrow \mgc{\equals}{\adEqb}(z_{f(x_1)}) \wedge \D(x_1) \wedge \fnpred{f}(x_1,z_{f(x_1)})                                                                             \label{ex:defunf:2} \\
    \mgc{\D}{\ad{b}}(x_1')      & \leftarrow \mgc{\equals}{\adEqb}(z_{f(x_1)}) \wedge \D(x_1) \wedge x_1 \equals x_1' \wedge \fnpred{f}(x_1,z_{f(x_1)})                                                     \label{ex:defunf:3}
\end{align}
Furthermore, to axiomatise predicates $\fnpred{f}$ and $\fnpred{g}$, the
transformation introduces rules
\eqref{ex:defun:p:U-Bf}--\eqref{ex:defun:p:S-R:2} for rules
\eqref{ex:mgc:U-Bf}, \eqref{ex:defun:C-Af}, \eqref{ex:mgc:S-R:1},
\eqref{ex:mgc:C-Uf}, and \eqref{ex:defun:S-R:2} that contain a function symbol
in the rule head.
\begin{align}
    \fnpred{f}(x_2,f(x_2))  & \leftarrow \mgc{B}{\ad{f}} \wedge U(x_1,x_2)                                                  \label{ex:defun:p:U-Bf} \\
    \fnpred{f}(x,f(x))      & \leftarrow \mgc{A}{\ad{b}}(z_{f(x)}) \wedge C(x) \wedge \fnpred{f}(x,z_{f(x)})                \label{ex:defun:p:C-Af} \\
    \fnpred{g}(x_1,g(x_1))  & \leftarrow \mgc{R}{\ad{bf}}(x_1) \wedge S(x_1,x_2)                                            \label{ex:defun:p:S-R:1} \\
    \fnpred{f}(x,f(x))      & \leftarrow \mgc{U}{\ad{ff}} \wedge C(x)                                                       \label{ex:defun:p:C-Uf} \\
    \fnpred{g}(x_1,g(x_1))  & \leftarrow \mgc{R}{\ad{fb}}(z_{g(x_1)}) \wedge S(x_1,x_2) \wedge \fnpred{g}(x_1,z_{g(x_1)})   \label{ex:defun:p:S-R:2}
\end{align}
Finally, $f$ is a true function symbol, so the transformation also introduces
the following rule. In contrast, $g$ is a Skolem function symbol, so no such
rule is needed.
\begin{align}
    \fnpred{f}(x_1,f(x_1)) \leftarrow \D(x_1) \wedge x_1 \equals x_1' \wedge \D(x_1')
\end{align}
\end{example}

\subsection{Reversing the Effects of Singularisation}\label{sec:answering:desg}

All equalities in the body of a rule obtained by the transformation in
Definition~\ref{def:defun} contain only variables, which enables the
transformation presented in the following definition.

\begin{definition}\label{def:desg}
    The \emph{desingularisation} of a rule is obtained by removing a body atom
    of the form ${x \equals t}$ where $x$ is a variable, replacing $x$ with $t$
    everywhere else in the rule, and repeating this process as long as
    possible. For $P$ a program, $\desg{P}$ contains a desingularisation of
    each rule of $P$.
\end{definition}

Desingularisation is useful because it removes equality atoms from the rule
bodies and thus reduces the number of joins. Theorem~\ref{thm:desg} establishes
correctness of our entire pipeline.

\begin{restatable}{theorem}{thmDESG}\label{thm:desg}
    Consider an arbitrary generalised second-order dependency $\Sigma$, a base
    instance $B$, and the program $P_7$ obtained by applying
    Algorithm~\ref{alg:answer-query} to $\Sigma$ and $B$. Program $P_7$ is safe
    and contains no equalities, function symbols, or constants in the body.
    Moreover, for each fact of the form ${\predQ(\vec a)}$, it is the case that
    ${\{ \Sigma \} \cup B \modelsEq \predQ(\vec a)}$ if and only if ${P_7 \cup
    \EQ{P_7} \cup B \models \predQ(\vec a)}$.
\end{restatable}

One may wonder whether the condition in Theorem~\ref{thm:desg} could have been
stated more simply as ${\{ \Sigma \} \cup B \modelsEq \predQ(\vec a)}$ if and
only if ${P_7 \cup B \modelsEq \predQ(\vec a)}$. The latter property indeed
holds, but it is stronger than what is needed: the entailment relation
$\modelsEq$ uses the standard definitions of first-order logic that do not
distinguish true and Skolem function symbols. By using ${P_7 \cup \EQ{P_7} \cup
B \models \predQ(\vec a)}$, we do not wish to suggest that one should compute a
fixpoint of ${P_7 \cup \EQ{P_7}}$ on $B$ to answer a query. Rather, we stress
that the query answers can be computed using the chase variant outlined in
Section~\ref{sec:so-dependencies:chase}; this algorithm treats the $\equals$
predicate as `true' equality, but it does not enforce functional reflexivity
for the Skolem function symbols.

\begin{example}\label{ex:desg}
On our running example, applying desingularisation to rules
\eqref{ex:defun:Q:R} and \eqref{ex:defunf:3}, which we both restate for
convenience, produces rules \eqref{ex:desg:1} and \eqref{ex:desg:2},
respectively.
\begin{align}
    \mgc{R}{\ad{bf}}(x_1)   & \leftarrow \mgc{\predQ}{\ad{f}} \wedge B(x_3') \wedge x_3 \equals x_3' \wedge A(x_3) \wedge z_{f(x_1)} \equals x_3 \wedge \fnpred{f}(x_1,z_{f(x_1)})  \tag{\ref{ex:defun:Q:R} restated} \\
    \mgc{R}{\ad{bf}}(x_1)   & \leftarrow \mgc{\predQ}{\ad{f}} \wedge B(x_3) \wedge A(x_3) \wedge \fnpred{f}(x_1,x_3)                                                                \label{ex:desg:1} \\
    \mgc{\D}{\ad{b}}(x_1')  & \leftarrow \mgc{\equals}{\adEqb}(z_{f(x_1)}) \wedge \D(x_1) \wedge x_1 \equals x_1' \wedge \fnpred{f}(x_1,z_{f(x_1)})                                 \tag{\ref{ex:defunf:3} restated} \\
    \mgc{\D}{\ad{b}}(x_1)   & \leftarrow \mgc{\equals}{\adEqb}(z_{f(x_1)}) \wedge \D(x_1) \wedge \fnpred{f}(x_1,z_{f(x_1)})                                                         \label{ex:desg:2}
\end{align}
Note that the variable $x_1'$ occurs in the body of rule \eqref{ex:defunf:3}
only in the equality atom ${x_1 \equals x_1'}$, but not in a relational atom;
hence, such a rule cannot be evaluated using the chase variant described in
Section~\ref{sec:so-dependencies:chase}. In contrast, rule \eqref{ex:desg:2}
produced by desingularisation does not contain equality atoms in the body and
can therefore be processed without any problems. All remaining rules are
transformed analogously and we do not show the result for the sake of brevity.
\end{example}

\section{Evaluation}\label{sec:evaluation}

To see whether our techniques can answer a query more efficiently than by
computing the chase in full, we have implemented a prototype and have evaluated
it on a subset of the first-order benchmarks by
\citet{DBLP:conf/pods/BenediktKMMPST17} and several second-order benchmarks
produced by our generator. We next present our test setting and discuss our
findings. The implementation of our algorithms, the datasets, and the raw
output produced by our system are available
online.\footnote{\url{https://krr-nas.cs.ox.ac.uk/2026/goal-driven-QA/}}

\subsection{Test Setting}\label{sec:evaluation:test-setting}

Our prototype can load a base instance and a set of dependencies, and either
compute the chase in full, or answer a specific query by applying our relevance
analysis and/or magic sets algorithms. The system was written in Java. We used
the API by \citet{DBLP:conf/pods/BenediktKMMPST17} for loading and manipulating
instances and dependencies, and we implemented our algorithms on top of this
API exactly as specified in Section~\ref{sec:answering}. In the magic sets
algorithm, a SIPS is produced by a greedy algorithm that eagerly binds as many
variables as possible while avoiding cross-products or violating the
requirements of Definition~\ref{def:SIPS}.

We computed the fixpoint of a logic program and the chase of first- and
second-order dependencies with equality using a research version of the RDFox
system. RDFox was implemented in C++, but it provides a Java API. The system
can be configured to treat $\equals$ as either ordinary or `true' equality, and
in either case it can efficiently parallelise the fixpoint computation for a
set of rules written in an extension of Datalog
\cite{mnpho14parallel-materialisation-RDFox, mnph15owl-sameAs-rewriting}. While
RDFox does not implement the chase variant from
Section~\ref{sec:so-dependencies:chase} natively, we used the following
techniques to overcome any discrepancies.

First, RDFox natively handles only RDF triples (i.e., facts of the form
$\langle s, p, o \rangle$), so we encode facts of varying arity into triples.
In particular, we represent unary and binary facts $A(a)$ and $R(a,b)$ as
${\langle a, \mathit{rdf}{:}\mathit{type}, A \rangle}$ and ${\langle a, R, b
\rangle}$, respectively. Moreover, we reify facts of higher arity; for example,
we transform $S(a,b,c)$ into triples ${\langle a,S{:}1,t \rangle}$, ${\langle
b,S{:}2,t \rangle}$, and ${\langle c,S{:}3,t \rangle}$, where $t$ is a fresh
labelled null that is uniquely determined by the fact's arguments, and $S{:}1$,
$S{:}2$, and $S{:}3$ are fresh properties that associate $t$ with the
corresponding arguments of $S(a,b,c)$. The following example illustrates how
dependencies containing atoms of higher arity are translated into RDFox rules
over triples.
\begin{center}
\begin{tabular}{@{}l@{}l@{}}
    $S(x,y,z) \rightarrow T(x,y,z) \quad \rightsquigarrow \quad$    & \tt [?x,T:1,?w], [?y,T:2,?w], [?z,T:3,?w] :- \\
                                                                    & \qquad \tt [?x,S:1,?v], [?y,S:2,?v], [?z,S:3,?v], \\
                                                                    & \qquad \tt BIND(SKOLEM("S",?x,?y,?z) AS ?v), \\
                                                                    & \qquad \tt BIND(SKOLEM("T",?x,?y,?z) AS ?w). \\
\end{tabular}
\end{center}
Literal \texttt{BIND(SKOLEM("S",?x,?y,?z) AS ?v)} of the rule binds the
variable \texttt{?v} to a fresh labelled null that is unique for each
instantiation of \texttt{?x}, \texttt{?y}, and \texttt{?z}.

Second, RDFox does not support function symbols natively, so we simulated them
using a builtin function. In the following rule, the \texttt{BIND} construct
assigns to \texttt{?z} a value that is uniquely computed from $f$ and
\texttt{?x}; hence, the value of \texttt{?z} can be seen as the result of
applying $f$ to the value of \texttt{?x}.
\begin{center}
\begin{tabular}{@{}l@{}l@{}}
    $A \rightarrow R(x,f(x)) \quad \rightsquigarrow \quad $ & \tt [?x,R,?z] :- \\
                                                            & \qquad \tt [?x,rdf:type,A], BIND(SKOLEM("f",?x) AS ?z) . \\
\end{tabular}
\end{center}

Third, RDFox does not natively provide inferences that realise functional
reflexivity. To overcome this issue, we removed all function symbols from the
rule bodies as explained in Section~\ref{sec:answering:desg}; please recall
that this introduces a fresh $n+1$-ary predicate $\fnpred{f}$ for each $n$-ary
(Skolem or true) function symbol $f$. Moreover, we introduced the rule
${\fnpred{f}(x_1,\dots,x_n,y) \wedge \fnpred{f}(x_1,\dots,x_n,y') \rightarrow y
\equals y'}$ for each true function symbol $f$ in order to ensure that $f$ is
interpreted as a function.

Fourth, RDFox does not handle existential quantifiers over individual variables
directly, so we Skolemised all such quantifiers before processing.

With these modifications in place, RDFox realises the chase variant described
in Section~\ref{sec:so-dependencies:chase} with Skolem and true function
symbols, but no first-order existential quantification. We used this algorithm
both as a baseline and for computing the chase in
line~\ref{alg:answer-query:chase} of Algorithm~\ref{alg:answer-query}.

We conducted all experiments on a server equipped with 256~GB RAM and an
Intel(R) Xeon(R) Silver 4116 CPU with 48 cores and clock frequency 2.10~GHz
running Ubuntu 22.04.4 LTS.

\subsection{Test Scenarios}

\begin{table}[tb!]
    \caption{Numbers of TGDs, EGDs, Facts, and Queries per Test Scenario}\label{table:test-scenarios}
    \centering
    \footnotesize
    \begin{tabular}{l|r|r|r|rc@{\qquad}l|r|r|r|r}
        \multicolumn{5}{c}{First-Order Scenarios}                               &   & \multicolumn{5}{c}{Second-Order Scenarios} \\
        \cline{1-5} \cline{7-11}
                      & TGDs    & EGDs  & Facts                     & Queries   &   &               & TGDs    & EGDs  & Facts                     & Queries \\
        \cline{1-5} \cline{7-11}
        \LUBMonehun   & $136$   & $0$    & \SI{12}{\mega\nothing}   & $14$      &   & \Gensix-1k    & $477$   & $23$   & \SI{400}{\kilo\nothing}  & $120$ \\
        \LUBMonek     & $136$   & $0$    & \SI{120}{\mega\nothing}  & $14$      &   & \Gensix-10k   & $477$   & $23$   & \SI{4}{\mega\nothing}    & $120$ \\
        \Deeptwohun   & $1,200$ & $0$    & \SI{1}{\kilo\nothing}    & $20$      &   & \Gensix-50k   & $477$   & $23$   & \SI{20}{\mega\nothing}   & $120$ \\
        \Deepthreehun & $1,300$ & $0$    & \SI{1}{\kilo\nothing}    & $20$      &   & \Genseven-1k  & $460$   & $44$   & \SI{400}{\kilo\nothing}  & $120$ \\
        \Ont          & $529$   & $348$  & \SI{2}{\mega\nothing}    & $20$      &   & \Genseven-10k & $460$   & $44$   & \SI{4}{\mega\nothing}    & $120$ \\
        \cline{1-5}
        \multicolumn{5}{c}{}                                                    &   & \Genseven-50k & $460$   & $44$   & \SI{20}{\mega\nothing}   & $120$ \\
        \multicolumn{5}{c}{}                                                    &   & \Geneight-1k  & $369$   & $12$   & \SI{330}{\kilo\nothing}  & $120$ \\
        \multicolumn{5}{c}{}                                                    &   & \Geneight-10k & $369$   & $12$   & \SI{3.3}{\mega\nothing}  & $120$ \\
        \multicolumn{5}{c}{}                                                    &   & \Geneight-50k & $369$   & $12$   & \SI{16.5}{\mega\nothing} & $120$ \\
        \cline{7-11}
    \end{tabular}
\end{table}

Our evaluation involves five first-order and nine second-order test scenarios
shown in Table~\ref{table:test-scenarios}. We consider first- and second-order
scenarios separately to evaluate our techniques in the more common setting of
TGDs and EGDs, as well as on the more advanced SO dependencies. Each scenario
consists of a set of dependencies, a base instance, and a set of queries. For
SO scenarios, we slightly abuse the terminology and use the terms `TGD' and
`EGD' to refer to dependency conjuncts with relational and equality head atoms,
respectively.

We used the first-order scenarios by \citet{DBLP:conf/pods/BenediktKMMPST17}
for benchmarking chase systems. In particular, \LUBMonehun and \LUBMonek are
derived from the well-known Lehigh University Benchmark
\cite{DBLP:journals/ws/GuoPH05} that has extensively been used to evaluate
performance of various reasoning systems. Moreover, \Deeptwohun and
\Deepthreehun are synthetic benchmarks designed to stress-test chase systems:
universal models are large despite the base instance being quite small.
Finally, the \Ont scenario is a synthetic benchmark designed to test reasoning
with equality-generating dependencies.

We produced second-order dependencies using our generator from
Appendix~\ref{sec:generating} parameterised as follows; please refer to the
appendix for a discussion of the various parameters. We used $\maxPredArity=2$,
$\maxFnSymArity=1$, $\preds=456,976$, $\fnSyms=17,576$, $\constants=10,000$,
$\seedQFacts=120$, and $\maxTermDepth=5$ in all cases, and we produced three
sets of dependencies using the following parameters.
\begin{itemize}
    \item \Gensix:
    $\maxNumberRules=26$, $\maxNumberRulesPerFact=2$,
    $\maxNumberRelBodyAtoms=3$, and $\maxNumberEqBodyAtoms=1$.

    \item \Genseven: 
    $\maxNumberRules=27$, $\maxNumberRulesPerFact=2$,
    $\maxNumberRelBodyAtoms=4$, and $\maxNumberEqBodyAtoms=2$.

    \item \Geneight: 
    $\maxNumberRules=28$, $\maxNumberRulesPerFact=1$,
    $\maxNumberRelBodyAtoms=3$, and $\maxNumberEqBodyAtoms=1$.
\end{itemize}
For each of these three sets, we generated three accompanying base instances by
using $\maxNumberTuples$ as 1k, 10k, and 50k. We thus obtained a total of nine
second-order test scenarios shown in Table~\ref{table:test-scenarios}.

For each scenario, we computed the chase, and we answered each query as shown
in Algorithm~\ref{alg:answer-query} using just the relevance analysis, just the
magic sets transformation, or using both techniques. When computing the chase
or using just the relevance analysis, the rules did not contain functional
terms in relational body atoms, so we skipped the elimination of function
symbols (see Section~\ref{sec:answering:defun}). In each case we measured the
wall-clock processing time (i.e., the time needed to compute the chase or to
run Algorithm~\ref{alg:answer-query}); this excludes the time needed to load
the base instance, but it includes the time needed to run the relevance
analysis (Algorithm~\ref{alg:relevance}) and the magic sets transformation
(Algorithm~\ref{alg:magic}). We also recorded the number of rules in the
program $P_7$ in line~\ref{alg:answer-query:desg} of
Algorithm~\ref{alg:answer-query}, and the numbers of all and `useful' facts
(i.e., facts not involving any auxiliary predicates such as $\D$, $\fnpred{f}$,
or magic predicates) in the instance $I$ in line~\ref{alg:answer-query:chase}.

\subsection{Experiments with First-Order Dependencies}

\begin{table}[tb!]
    \caption{Summary of the evaluation results on FO scenarios}\label{tab:fo:results}
    \centering\tiny
    \newcommand{\scen}[1]{\multirow{2}{*}{\ensuremath{#1}}}
    \newcommand{\val}[2]{#1 & #2}
    \newcommand{\hmin}{\multicolumn{2}{c|}{min}}
    \newcommand{\hmax}{\multicolumn{2}{c|}{max}}
    \newcommand{\hmed}{\multicolumn{2}{c||}{median}}
    \newcommand{\hlmed}{\multicolumn{2}{c}{median}}
    \begin{tabular}{c|c|r@{\,}l||r@{\,}l|r@{\,}l|r@{\,}l||r@{\,}l|r@{\,}l|r@{\,}l||r@{\,}l|r@{\,}l|r@{\,}l}
        \hline
        \multicolumn{2}{c|}{}   & \\[-1ex]
        \multicolumn{2}{c|}{}   & \multicolumn{20}{c}{\textbf{Query times (seconds)}} \\[1ex]
        \cline{3-22}
        \multicolumn{2}{c|}{}   & \multicolumn{2}{c||}{\matrun}  & \multicolumn{6}{c||}{\relevancerun}                       & \multicolumn{6}{c||}{\magicrun}                          & \multicolumn{6}{c}{\allrun} \\
        \cline{5-22}
        \multicolumn{2}{c|}{}   & \multicolumn{2}{c||}{}         & \hmin             & \hmax             & \hmed             & \hmin            & \hmax             & \hmed             & \hmin             & \hmax             & \hlmed          \\
        \hline
        \scen{\LUBM}    & 100   & \val{38.85}{}                  & \val{0.07}{}      & \val{25.53}{}     & \val{24.91}{}     & \val{0.05}{}     & \val{76.86}{}     & \val{1.23}{}      & \val{0.03}{}      & \val{73.64}{}     & \val{1.11}{}    \\
                        & 01k   & \val{376.64}{}                 & \val{0.67}{}      & \val{291.26}{}    & \val{260.19}{}    & \val{0.33}{}     & \val{919.41}{}    & \val{12.59}{}     & \val{0.13}{}      & \val{842.95}{}    & \val{11.95}{}   \\
        \hline
        \scen{\Deep}    & 200   & \val{8.61}{}                   & \val{3.61}{}      & \val{4.58}{}      & \val{3.70}{}      & \val{0.02}{}     & \val{1.89}{}      & \val{0.51}{}      & \val{3.69}{}      & \val{7.08}{}      & \val{3.93}{}    \\
                        & 300   & \val{N/A}{}                    & \val{3.98}{}      & \val{97.49}{}     & \val{4.07}{}      & \val{0.09}{}     & \val{408.56}{}    & \val{11.66}{}     & \val{4.10}{}      & \val{105.08}{}    & \val{4.36}{}    \\
        \hline
        \Ont            &       & \val{59.58}{}                  & \val{0.19}{}      & \val{4.02}{}      & \val{0.55}{}      & \val{499.98}{}   & \val{607.41}{}    & \val{531.49}{}    & \val{0.19}{}      & \val{3.70}{}      & \val{0.37}{}    \\
        \hline
        \multicolumn{2}{c|}{}   & \\[-1ex]
        \multicolumn{2}{c|}{}   & \multicolumn{20}{c}{\textbf{The numbers of total derived facts}} \\[1ex]
        \cline{3-22}
        \multicolumn{2}{c|}{}   & \multicolumn{2}{c||}{\matrun}  & \multicolumn{6}{c||}{\relevancerun}                       & \multicolumn{6}{c||}{\magicrun}                          & \multicolumn{6}{c}{\allrun} \\
        \cline{5-22}
        \multicolumn{2}{c|}{}   & \multicolumn{2}{c||}{}         & \hmin             & \hmax             & \hmed             & \hmin            & \hmax             & \hmed             & \hmin             & \hmax             & \hlmed          \\
        \hline
        \scen{\LUBM}    & 100   & \val{37.07}{M}                 & \val{154.41}{k}   & \val{33.74}{M}    & \val{29.45}{M}    & \val{136.39}{k}  & \val{69.62}{M}    & \val{10.44}{M}    & \val{63.47}{k}    & \val{67.23}{M}    & \val{10.41}{M}  \\
                        & 01k   & \val{369.34}{M}                & \val{1.54}{M}     & \val{336.32}{M}   & \val{293.49}{M}   & \val{1.35}{M}    & \val{693.89}{M}   & \val{103.55}{M}   & \val{622.01}{k}   & \val{670.08}{M}   & \val{103.54}{M} \\
        \hline
        \scen{\Deep}    & 200   & \val{927.00}{k}                & \val{0.00}{}      & \val{1.00}{k}     & \val{77.00}{}     & \val{258.00}{}   & \val{227.25}{k}   & \val{1.01}{k}     & \val{1.00}{}      & \val{493.00}{}    & \val{69.00}{}   \\
                        & 300   & \val{N/A}{}                    & \val{0.00}{}      & \val{2.44}{k}     & \val{92.00}{}     & \val{508.00}{}   & \val{31.90}{M}    & \val{1.22}{M}     & \val{1.00}{}      & \val{1.42}{k}     & \val{78.00}{}   \\
        \hline
        \Ont            &       & \val{7.82}{M}                  & \val{0.00}{}      & \val{157.43}{k}   & \val{79.63}{k}    & \val{68.60}{M}   & \val{329.96}{M}   & \val{287.77}{M}   & \val{1.00}{}      & \val{134.73}{k}   & \val{59.98}{k}  \\
        \hline
        \multicolumn{2}{c|}{}   & \\[-1ex]
        \multicolumn{2}{c|}{}   & \multicolumn{20}{c}{\textbf{The numbers of useful derived facts}} \\[1ex]
        \cline{3-22}
        \multicolumn{2}{c|}{}   & \multicolumn{2}{c||}{\matrun}  & \multicolumn{6}{c||}{\relevancerun}                       & \multicolumn{6}{c||}{\magicrun}                          & \multicolumn{6}{c}{\allrun} \\
        \cline{5-22}
        \multicolumn{2}{c|}{}   & \multicolumn{2}{c||}{}         & \hmin             & \hmax             & \hmed             & \hmin            & \hmax             & \hmed             & \hmin             & \hmax             & \hlmed          \\
        \hline
        \scen{\LUBM}    & 100   & \val{37.07}{M}                 & \val{154.41}{k}   & \val{33.74}{M}    & \val{29.45}{M}    & \val{135.91}{k}  & \val{27.89}{M}    & \val{10.41}{M}    & \val{62.99}{k}    & \val{27.89}{M}    & \val{10.39}{M}  \\
                        & 01k   & \val{369.34}{M}                & \val{1.54}{M}     & \val{336.32}{M}   & \val{293.49}{M}   & \val{1.35}{M}    & \val{277.97}{M}   & \val{103.52}{M}   & \val{621.52}{k}   & \val{277.96}{M}   & \val{103.5}{M}  \\
        \hline
        \scen{\Deep}    & 200   & \val{927.00}{k}                & \val{0.00}{}      & \val{1.00}{k}     & \val{77.00}{}     & \val{211.00}{}   & \val{138.93}{k}   & \val{897.00}{}    & \val{0.00}{}      & \val{383.00}{}    & \val{54.00}{}   \\
                        & 300   & \val{N/A}{}                    & \val{0.00}{}      & \val{2.44}{k}     & \val{92.00}{}     & \val{404.00}{}   & \val{24.18}{M}    & \val{791.00}{k}   & \val{0.00}{}      & \val{740.00}{}    & \val{62.00}{}   \\
        \hline
        \Ont            &       & \val{7.82}{M}                  & \val{0.00}{}      & \val{157.43}{k}   & \val{79.63}{k}    & \val{2.43}{M}    & \val{2.48}{M}     & \val{2.45}{M}     & \val{0.00}{}      & \val{95.09}{k}    & \val{50.09}{k}  \\
        \hline
    \end{tabular}
\end{table}

\begin{figure}[tb!]
    \centering
    \scalebox{0.9}{
        \begin{tabular}{cccc}
                                                        & Query Times (ms)                   & \# Derived Facts                          & \# Rules                            \\[1ex]
            \rotatebox{90}{\hspace{0.8cm}\LUBMonehun}   & \begin{tikzpicture}[xscale=0.25, yscale=0.20]
\pgfplotsset{compat=newest}

\begin{axis}[
    width=15cm,
    height=15cm,
    enlargelimits=false,
    ymode=log,
    ytick pos=left,
    ylabel={\huge Time in ms},
    xmin=1,
    xmax=14,
    xtick pos=left,
    xlabel={\huge \#Queries},
    tick label style={font=\huge},
    tick align=outside,
    legend cell align=left,
    legend style={at={(0,1)}, anchor=north west, font=\huge}
]
\addplot+[color=black, mark=none] table {
0 31859
14 31859
};
\addplot+[color=green, mark=none] table {
1 50
2 458
3 725
4 1090
5 1140
6 1150
7 1214
8 1313
9 1328
10 1487
11 1772
12 8777
13 48546
14 76869
};
\addplot+[color=blue, mark=none]  table {
1 72
2 663
3 2087
4 2388
5 4037
6 24558
7 24821
8 25013
9 25024
10 25032
11 25069
12 25189
13 25395
14 25534
};
\addplot+[color=red, mark=none]  table {
1 36
2 472
3 569
4 738
5 1132
6 1149
7 1183
8 1269
9 1349
10 1445
11 1750
12 5103
13 47130
14 73646
};
\end{axis}
\end{tikzpicture} & \begin{tikzpicture}[xscale=0.25, yscale=0.20]
\pgfplotsset{compat=newest}

\begin{axis}[
    width=15cm,
    height=15cm,
    enlargelimits=false,
    ymode=log,
    xtick pos=left,
    ylabel={\huge \#Total facts},
    xmin=1,
    xmax=14,
    xtick pos=left,
    xlabel={\huge \#Queries},
    tick label style={font=\huge},
    tick align=outside,
    legend cell align=left,
    legend style={at={(0,1)}, anchor=north west, font=\huge}
]
\addplot+[color=black, mark=none] table {
0 37061682
14 37061682
};
\addplot+[color=green, mark=none] table {
1 136398
2 2250976
3 2387912
4 10387561
5 10387883
6 10391658
7 10431423
8 10455861
9 10674457
10 10906984
11 12541810
12 13582507
13 42908312
14 69624726
};
\addplot+[color=blue, mark=none]  table {
1 154419
2 2387910
3 4501857
4 5921730
5 6294627
6 29451773
7 29451784
8 29451836
9 29452241
10 29452488
11 29459559
12 29479015
13 30500300
14 33749463
};
\addplot+[color=red, mark=none]  table {
1 63471
2 2250976
3 2387912
4 3147330
5 4907891
6 10387875
7 10391518
8 10429743
9 10453641
10 10672359
11 10886996
12 12541642
13 41702390
14 67235954
};
\end{axis}
\end{tikzpicture} & \begin{tikzpicture}[xscale=0.25, yscale=0.20]
\pgfplotsset{compat=newest}

\begin{axis}[
    width=15cm,
    height=15cm,
    enlargelimits=false,
    ymode=log,
    ytick pos=left,
    ylabel={\huge \#Rules},
    xmin=1,
    xmax=14,
    xtick pos=left,
    xlabel={\huge \#Queries},
    tick label style={font=\huge},
    tick align=outside,
    legend cell align=left,
    legend style={at={(0,1)}, anchor=north west, font=\huge}
]
\addplot+[color=black, mark=none] table {
0 144
14 144
};
\addplot+[color=green, mark=none] table {
1 4
2 14
3 27
4 274
5 276
6 276
7 278
8 278
9 279
10 279
11 280
12 283
13 283
14 285
};
\addplot+[color=blue, mark=none]  table {
1 2
2 4
3 6
4 20
5 26
6 110
7 110
8 111
9 111
10 111
11 111
12 111
13 111
14 112
};
\addplot+[color=red, mark=none]  table {
1 4
2 7
3 10
4 27
5 49
6 241
7 243
8 243
9 245
10 245
11 246
12 247
13 250
14 252
};
\end{axis}
\end{tikzpicture} \\
            \rotatebox{90}{\hspace{0.9cm}\LUBMonek}     & \begin{tikzpicture}[xscale=0.25, yscale=0.20]
\pgfplotsset{compat=newest}

\begin{axis}[
    width=15cm,
    height=15cm,
    enlargelimits=false,
    ymode=log,
    ytick pos=left,
    ylabel={\huge Time in ms},
    xmin=1,
    xmax=14,
    xtick pos=left,
    xlabel={\huge \#Queries},
    tick label style={font=\huge},
    tick align=outside,
    legend cell align=left,
    legend style={at={(0,1)}, anchor=north west, font=\huge}
]
\addplot+[color=black, mark=none] table {
0 376648
14 376648
};
\addplot+[color=green, mark=none] table {
1 338
2 4205
3 10721
4 10833
5 11619
6 12129
7 12341
8 12854
9 13097
10 13245
11 14144
12 599247
13 899705
14 919410
};
\addplot+[color=blue, mark=none]  table {
1 676
2 11628
3 26485
4 33032
5 42715
6 258199
7 258893
8 261502
9 262140
10 263333
11 265920
12 268945
13 274488
14 291264
};
\addplot+[color=red, mark=none]  table {
1 134
2 4314
3 5357
4 10569
5 10746
6 11040
7 11820
8 12080
9 12233
10 13069
11 14367
12 550611
13 640252
14 842955
};
\end{axis}
\end{tikzpicture} & \begin{tikzpicture}[xscale=0.25, yscale=0.20]
\pgfplotsset{compat=newest}

\begin{axis}[
    width=15cm,
    height=15cm,
    enlargelimits=false,
    ymode=log,
    xtick pos=left,
    ylabel={\huge \#Total facts},
    xmin=1,
    xmax=14,
    xtick pos=left,
    xlabel={\huge \#Queries},
    tick label style={font=\huge},
    tick align=outside,
    legend cell align=left,
    legend style={at={(0,1)}, anchor=north west, font=\huge}
]
\addplot+[color=black, mark=none] table {
0 369346294
14 369346294
};
\addplot+[color=green, mark=none] table {
1 1350891
2 22435572
3 23774297
4 103495347
5 103495669
6 103499444
7 103539209
8 103563647
9 104014770
10 106424432
11 124915575
12 321038085
13 428179885
14 693894330
};
\addplot+[color=blue, mark=none]  table {
1 1541760
2 23774295
3 44871049
4 59029058
5 62691521
6 293497170
7 293497181
8 293497233
9 293497885
10 293501926
11 293504956
12 293770147
13 303944546
14 336326818
};
\addplot+[color=red, mark=none]  table {
1 622010
2 22435572
3 23774297
4 31345777
5 103495661
6 103499304
7 103537529
8 103561427
9 103994782
10 106403470
11 124915407
12 240084711
13 416135085
14 670084044
};
\end{axis}
\end{tikzpicture} & \begin{tikzpicture}[xscale=0.25, yscale=0.20]
\pgfplotsset{compat=newest}

\begin{axis}[
    width=15cm,
    height=15cm,
    enlargelimits=false,
    ymode=log,
    ytick pos=left,
    ylabel={\huge \#Rules},
    xmin=1,
    xmax=14,
    xtick pos=left,
    xlabel={\huge \#Queries},
    tick label style={font=\huge},
    tick align=outside,
    legend cell align=left,
    legend style={at={(0,1)}, anchor=north west, font=\huge}
]
\addplot+[color=black, mark=none] table {
0 144
14 144
};
\addplot+[color=green, mark=none] table {
1 4
2 14
3 27
4 274
5 276
6 276
7 278
8 278
9 279
10 279
11 280
12 283
13 283
14 285
};
\addplot+[color=blue, mark=none]  table {
1 2
2 4
3 6
4 20
5 26
6 110
7 110
8 111
9 111
10 111
11 111
12 111
13 111
14 112
};
\addplot+[color=red, mark=none]  table {
1 4
2 7
3 10
4 27
5 49
6 241
7 243
8 243
9 245
10 245
11 246
12 247
13 250
14 252
};
\end{axis}
\end{tikzpicture} \\
            \rotatebox{90}{\hspace{0.9cm}\Deeptwohun}   & \begin{tikzpicture}[xscale=0.25, yscale=0.20]
    \pgfplotsset{compat=newest}

    \begin{axis}[
            width=15cm,
            height=15cm,
            enlargelimits=false,
            ymode=log,
            ytick pos=left,
            ylabel={\huge Time in ms},
            xmin=1,
            xmax=20,
            ymax=50000,
            xtick pos=left,
            xlabel={\huge \#Queries},
            tick label style={font=\huge},
            tick align=outside,
            legend cell align=left,
            legend style={at={(0,1)}, anchor=north west, font=\huge}
        ]
        \addplot+[color=black, mark=none] table {
                0 8619
                20 8619
            };
        \addplot+[color=green, mark=none] table {
                1 27
                2 56
                3 96
                4 126
                5 214
                6 290
                7 365
                8 389
                9 437
                10 506
                11 529
                12 554
                13 580
                14 741
                15 790
                16 928
                17 962
                18 963
                19 972
                20 1890
            };
        \addplot+[color=blue, mark=none]  table {
                1 3612
                2 3654
                3 3658
                4 3661
                5 3662
                6 3663
                7 3665
                8 3667
                9 3685
                10 3695
                11 3709
                12 3713
                13 3719
                14 3739
                15 3744
                16 3748
                17 3780
                18 3832
                19 3860
                20 4583
            };
        \addplot+[color=red, mark=none]  table {
                1 3698
                2 3857
                3 3864
                4 3872
                5 3875
                6 3878
                7 3887
                8 3888
                9 3914
                10 3920
                11 3940
                12 4008
                13 4093
                14 4094
                15 4095
                16 4111
                17 4117
                18 4202
                19 4768
                20 7089
            };
    \end{axis}
\end{tikzpicture} & \begin{tikzpicture}[xscale=0.25, yscale=0.20]
    \pgfplotsset{compat=newest}

    \begin{axis}[
            width=15cm,
            height=15cm,
            enlargelimits=false,
            ymode=log,
            xtick pos=left,
            ylabel={\huge \#Total facts},
            xmin=1,
            xmax=20,
            ymax=5000000,
            xtick pos=left,
            xlabel={\huge \#Queries},
            tick label style={font=\huge},
            tick align=outside,
            legend cell align=left,
            legend style={at={(0,1)}, anchor=north west, font=\huge}
        ]
        \addplot+[color=black, mark=none] table {
                0 927324
                20 927324
            };
        \addplot+[color=green, mark=none] table {
                1 258
                2 440
                3 508
                4 615
                5 654
                6 681
                7 882
                8 928
                9 972
                10 1002
                11 1027
                12 1043
                13 1118
                14 1128
                15 1280
                16 2697
                17 2780
                18 11221
                19 14165
                20 227251
            };
        \addplot+[color=blue, mark=none]  table {
                1 0
                2 33
                3 38
                4 39
                5 42
                6 44
                7 44
                8 45
                9 59
                10 66
                11 88
                12 90
                13 90
                14 107
                15 348
                16 355
                17 365
                18 389
                19 436
                20 1007
            };
        \addplot+[color=red, mark=none]  table {
                1 1
                2 40
                3 48
                4 49
                5 53
                6 53
                7 53
                8 55
                9 63
                10 66
                11 72
                12 86
                13 100
                14 165
                15 213
                16 280
                17 300
                18 347
                19 418
                20 493
            };
    \end{axis}
\end{tikzpicture} & \begin{tikzpicture}[xscale=0.25, yscale=0.20]
    \pgfplotsset{compat=newest}

    \begin{axis}[
            width=15cm,
            height=15cm,
            enlargelimits=false,
            ymode=log,
            ytick pos=left,
            ylabel={\huge \#Rules},
            xmin=1,
            xmax=20,
            ymax=10000,
            xtick pos=left,
            xlabel={\huge \#Queries},
            tick label style={font=\huge},
            tick align=outside,
            legend cell align=left,
            legend style={at={(0,1)}, anchor=north west, font=\huge}
        ]
        \addplot+[color=black, mark=none] table {
                0 4541
                20 4541
            };
        \addplot+[color=green, mark=none] table {
                1 180
                2 405
                3 582
                4 782
                5 1138
                6 1652
                7 1791
                8 1934
                9 2098
                10 2153
                11 2455
                12 2464
                13 2537
                14 2691
                15 3468
                16 3599
                17 4354
                18 4361
                19 4415
                20 4508
            };
        \addplot+[color=blue, mark=none]  table {
                1 0
                2 17
                3 19
                4 21
                5 21
                6 21
                7 21
                8 21
                9 30
                10 33
                11 44
                12 44
                13 45
                14 52
                15 147
                16 150
                17 169
                18 172
                19 189
                20 353
            };
        \addplot+[color=red, mark=none]  table {
                1 1
                2 21
                3 23
                4 26
                5 26
                6 26
                7 27
                8 28
                9 39
                10 43
                11 54
                12 55
                13 56
                14 66
                15 213
                16 267
                17 270
                18 270
                19 331
                20 657
            };
    \end{axis}
\end{tikzpicture} \\
            \rotatebox{90}{\hspace{0.9cm}\Deepthreehun} & \begin{tikzpicture}[xscale=0.25, yscale=0.20]
\pgfplotsset{compat=newest}

\begin{axis}[
    width=15cm,
    height=15cm,
    enlargelimits=false,
    ymode=log,
    ytick pos=left,
    ylabel={\huge Time in ms},
    xmin=1,
    xmax=20,
    xtick pos=left,
    xlabel={\huge \#Queries},
    tick label style={font=\huge},
    tick align=outside,
    legend cell align=left,
    legend style={at={(0,1)}, anchor=north west, font=\huge}
]
\addplot+[color=green, mark=none] table {
1 98
2 482
3 1474
4 1631
5 1858
6 1860
7 2144
8 2171
9 3342
10 9458
11 13879
12 17378
13 22907
14 28079
15 28359
16 56775
17 87045
18 191397
19 222665
20 408561
};
\addplot+[color=blue, mark=none]  table {
1 3987
2 3991
3 4001
4 4008
5 4012
6 4037
7 4038
8 4046
9 4047
10 4050
11 4094
12 4101
13 4136
14 4141
15 4147
16 4172
17 4354
18 4637
19 12724
20 97494
};
\addplot+[color=red, mark=none]  table {
1 4105
2 4261
3 4265
4 4284
5 4298
6 4315
7 4328
8 4350
9 4362
10 4362
11 4364
12 4379
13 4411
14 4473
15 4482
16 4492
17 4621
18 4880
19 12955
20 105085
};
\end{axis}
\end{tikzpicture} & \begin{tikzpicture}[xscale=0.25, yscale=0.20]
\pgfplotsset{compat=newest}

\begin{axis}[
    width=15cm,
    height=15cm,
    enlargelimits=false,
    ymode=log,
    xtick pos=left,
    ylabel={\huge \#Total facts},
    xmin=1,
    xmax=20,
    xtick pos=left,
    xlabel={\huge \#Queries},
    tick label style={font=\huge},
    tick align=outside,
    legend cell align=left,
    legend style={at={(0,1)}, anchor=north west, font=\huge}
]
\addplot+[color=green, mark=none] table {
1 508
2 1118
3 1220
4 1224
5 1261
6 1383
7 1551
8 1997
9 313573
10 877939
11 1571724
12 1931912
13 1931924
14 2515492
15 3110247
16 3112444
17 7401835
18 14650417
19 16502214
20 31904587
};
\addplot+[color=blue, mark=none]  table {
1 0
2 38
3 39
4 42
5 45
6 48
7 54
8 58
9 59
10 90
11 95
12 97
13 107
14 111
15 535
16 604
17 604
18 714
19 811
20 2445
};
\addplot+[color=red, mark=none]  table {
1 1
2 48
3 49
4 53
5 53
6 61
7 66
8 72
9 72
10 73
11 84
12 102
13 105
14 165
15 333
16 365
17 607
18 768
19 902
20 1420
};
\end{axis}
\end{tikzpicture} & \begin{tikzpicture}[xscale=0.25, yscale=0.20]
    \pgfplotsset{compat=newest}

    \begin{axis}[
            width=15cm,
            height=15cm,
            enlargelimits=false,
            ymode=log,
            ytick pos=left,
            ylabel={\huge \#Rules},
            xmin=1,
            xmax=20,
            xtick pos=left,
            xlabel={\huge \#Queries},
            tick label style={font=\huge},
            tick align=outside,
            legend cell align=left,
            legend style={at={(0,1)}, anchor=north west, font=\huge}
        ]
        \addplot+[color=green, mark=none] table {
                1 582
                2 2098
                3 2309
                4 2334
                5 2338
                6 3761
                7 4560
                8 4739
                9 5045
                10 5760
                11 6197
                12 6528
                13 6786
                14 7327
                15 7644
                16 7802
                17 8772
                18 9009
                19 10311
                20 10352
            };
        \addplot+[color=blue, mark=none]  table {
                1 0
                2 19
                3 21
                4 21
                5 21
                6 23
                7 27
                8 27
                9 30
                10 44
                11 47
                12 48
                13 52
                14 55
                15 255
                16 256
                17 258
                18 304
                19 340
                20 667
            };
        \addplot+[color=red, mark=none]  table {
                1 1
                2 23
                3 26
                4 26
                5 28
                6 28
                7 34
                8 34
                9 39
                10 55
                11 58
                12 60
                13 66
                14 69
                15 314
                16 414
                17 424
                18 515
                19 518
                20 1447
            };
    \end{axis}
\end{tikzpicture} \\
            \rotatebox{90}{\hspace{1.1cm}\Ont}          & \begin{tikzpicture}[xscale=0.25, yscale=0.20]
    \pgfplotsset{compat=newest}

    \begin{axis}[
            width=15cm,
            height=15cm,
            enlargelimits=false,
            ymode=log,
            ytick pos=left,
            ylabel={\huge Time in ms},
            xmin=1,
            xmax=20,
            ymax=1000000,
            xtick pos=left,
            xlabel={\huge \#Queries},
            tick label style={font=\huge},
            tick align=outside,
            legend cell align=left,
            legend style={at={(0,1)}, anchor=north west, font=\huge}
        ]
        \addplot+[color=black, mark=none] table {
                0 59580
                20 59580
            };
        \addplot+[color=green, mark=none] table {
                1 499985
                2 501778
                3 505570
                4 516381
                5 518092
                6 522314
                7 528483
                8 529679
                9 531284
                10 531496
                11 536398
                12 537240
                13 537324
                14 538026
                15 541536
                16 552055
                17 557685
                18 558335
                19 607410
            };
        \addplot+[color=blue, mark=none]  table {
                1 191
                2 200
                3 202
                4 202
                5 355
                6 360
                7 391
                8 414
                9 493
                10 554
                11 563
                12 627
                13 650
                14 706
                15 717
                16 772
                17 3733
                18 4013
                19 4026
            };
        \addplot+[color=red, mark=none]  table {
                1 196
                2 202
                3 204
                4 205
                5 306
                6 349
                7 357
                8 364
                9 368
                10 372
                11 396
                12 406
                13 475
                14 510
                15 537
                16 702
                17 3633
                18 3659
                19 3703
            };
    \end{axis}
\end{tikzpicture}  & \begin{tikzpicture}[xscale=0.25, yscale=0.20]
    \pgfplotsset{compat=newest}

    \begin{axis}[
            width=15cm,
            height=15cm,
            enlargelimits=false,
            ymode=log,
            xtick pos=left,
            ylabel={\huge \#Total facts},
            xmin=1,
            xmax=20,
            ymax=1000000000,
            xtick pos=left,
            xlabel={\huge \#Queries},
            tick label style={font=\huge},
            tick align=outside,
            legend cell align=left,
            legend style={at={(0,1)}, anchor=north west, font=\huge}
        ]
        \addplot+[color=black, mark=none] table {
                0 7820377
                20 7820377
            };
        \addplot+[color=green, mark=none] table {
                1 268606204
                2 270407771
                3 271008102
                4 271347381
                5 280538593
                6 282805872
                7 285207257
                8 287522370
                9 287759820
                10 287771788
                11 287782785
                12 287807009
                13 288254268
                14 290516881
                15 294814947
                16 301558356
                17 301572084
                18 304186839
                19 329966506
            };
        \addplot+[color=blue, mark=none]  table {
                1 1
                2 1
                3 1
                4 1
                5 49671
                6 59237
                7 78284
                8 78570
                9 79060
                10 79631
                11 98156
                12 106201
                13 108624
                14 127909
                15 128198
                16 140228
                17 140323
                18 147486
                19 157433
            };
        \addplot+[color=red, mark=none]  table {
                1 1
                2 1
                3 1
                4 1
                5 39780
                6 58823
                7 59088
                8 59642
                9 59891
                10 59985
                11 65013
                12 68847
                13 69241
                14 69425
                15 75729
                16 79013
                17 79286
                18 88965
                19 134737
            };
    \end{axis}
\end{tikzpicture}  & \begin{tikzpicture}[xscale=0.25, yscale=0.20]
    \pgfplotsset{compat=newest}

    \begin{axis}[
            width=15cm,
            height=15cm,
            enlargelimits=false,
            ymode=log,
            ytick pos=left,
            ylabel={\huge \#Rules},
            xmin=1,
            xmax=20,
            ymax=100000,
            xtick pos=left,
            xlabel={\huge \#Queries},
            tick label style={font=\huge},
            tick align=outside,
            legend cell align=left,
            legend style={at={(0,1)}, anchor=north west, font=\huge}
        ]
        \addplot+[color=black, mark=none] table {
                0 1492
                20 1492
            };
        \addplot+[color=green, mark=none] table {
                1 8375
                2 8376
                3 8379
                4 8380
                5 8380
                6 8380
                7 8381
                8 8381
                9 8382
                10 8383
                11 8384
                12 8384
                13 8385
                14 8387
                15 8388
                16 8390
                17 8395
                18 8398
                19 8399
            };
        \addplot+[color=blue, mark=none]  table {
                1 1
                2 1
                3 1
                4 1
                5 4
                6 5
                7 5
                8 5
                9 5
                10 6
                11 7
                12 7
                13 8
                14 8
                15 9
                16 10
                17 10
                18 10
                19 10
            };
        \addplot+[color=red, mark=none]  table {
                1 1
                2 1
                3 1
                4 1
                5 8
                6 9
                7 9
                8 10
                9 10
                10 11
                11 14
                12 14
                13 15
                14 15
                15 18
                16 19
                17 19
                18 19
                19 19
            };
    \end{axis}
\end{tikzpicture}  \\
        \end{tabular}
    }
    \caption{The results for \textcolor{black}{\matrun}, \textcolor{blue}{\relevancerun}, \textcolor{green}{\magicrun}, and \textcolor{red}{\allrun} on first-order scenarios}\label{fig:fo:results}
\end{figure}

Figure~\ref{fig:fo:results} shows the distributions of our results on
first-order scenarios using cactus plots: the horizontal axis enumerates all
queries sorted by the value being shown, and the vertical axis shows the
corresponding value (i.e., the query time in milliseconds, the number of
derived facts, and the number of the rules in the program $P_7$). In each
figure, \relevancerun (shown in blue), \magicrun (shown in green), and \allrun
(shown in red) refer to the variant of Algorithm~\ref{alg:answer-query} that
uses the relevance analysis only, the magic sets only, and both techniques in
combination, respectively. Furthermore, \matrun (shown in black) provides the
baseline of computing the chase in full, which is equivalent to using
Algorithm~\ref{alg:answer-query} with no relevance analysis or magic sets:
singularisation and desingularisation then cancel each other out, and
Skolemisation and elimination of function symbols in the body are already
required by the RDFox-based chase computation procedure described in
Section~\ref{sec:evaluation:test-setting}. Finally, Table~\ref{tab:fo:results}
summarises the distributions by showing the minimum, maximum, and median
values. Computing the chase in full for \Deepthreehun was aborted after two
hours.

As one can see, our techniques are generally very effective: a query can
sometimes be answered orders of magnitude faster than by computing the chase in
full. In fact, goal-driven query answering can mean the difference between
success and failure: whereas \matrun failed for \Deepthreehun, we were able to
successfully answer all queries of this scenario. This may not seem surprising:
both relevance analysis and magic sets identify a subset of the relevant rules,
and so computing the chase in line~\ref{alg:answer-query:chase} of
Algorithm~\ref{alg:answer-query} should intuitively be more efficient than
computing the chase in full. However, a more detailed analysis of our results
reveals a much more nuanced picture, as we discuss next.

Relevance analysis always returns a subset of the input rules, so one may
naturally expect it to be always faster than computing the chase in full.
However, the query times for \relevancerun are the same for most queries of
\Deeptwohun and \Deepthreehun, which suggests that the technique contains a
fixed overhead. Indeed, our analysis revealed that, on these scenarios, the
fixpoint computation on the critical instance in
line~\ref{alg:relevance:abstraction-fixpoint} of Algorithm~\ref{alg:relevance}
dominates the overall processing time. As a result, \relevancerun is only
marginally faster than \matrun, and considerably slower than \magicrun on most
queries. In all other cases, the overhead of relevance analysis is negligible
and the technique seems very effective.

The rules relevant to a query could be determined by a simple reachability
analysis: we identify the query predicate $\predQ$ as reachable; we
transitively propagate reachability over the rules (i.e., if a predicate
occurring in a head atom of a rule is reachable, then all predicates occurring
in the rule body are reachable too); and we select all rules containing a
reachable predicate in the head. However, as soon as the input dependencies
contain equality, rules \eqref{eq:dom-R}, \eqref{eq:ref}, and \eqref{eq:cong}
ensure that all predicates, and thus all rules, are reachable. In contrast, our
relevance analysis technique computes much smaller programs on the \Ont
scenario, suggesting that the technique is considerably more effective than
simple reachability. This is because our technique does not consider body atoms
in isolation, but instead tries to identify whether a rule has relevant
instances. This, in turn, seems to be key to an effective analysis of
inferences with equality.

The magic sets algorithm implicitly performs the reachability analysis but also
identifies relevant rule instances. This allows \magicrun to be faster than
\matrun on all scenarios apart from \Ont. However, by introducing magic rules,
the algorithm can sometimes produce programs that are larger than the input,
and the evaluation of these additional rules can be a source of overhead. For
example, \magicrun is consistently slower than \matrun on \Ont, primarily
because of the magic rules. In fact, as Table~\ref{tab:fo:results} shows, the
ratios of the numbers of total and useful facts are orders of magnitude higher
for \magicrun than for \matrun and \relevancerun, which suggests further that
deriving magic facts can be costly.

The combination of both techniques often produces the best results (e.g., on
\LUBM and some \Ont queries). Intuitively, the relevance analysis can reduce
the number of rules sufficiently so that the introduction of magic rules leads
to further improvements in query answering performance.

\subsection{Experiments with Second-Order Dependencies}

\begin{table}[tb!]
    \caption{Summary of the evaluation results on second-order scenarios}\label{tab:so:results}
    \centering\tiny
    \newcommand{\scen}[1]{\multirow{2}{*}{\ensuremath{#1}}}
    \newcommand{\val}[2]{#1 & #2}
    \newcommand{\hmin}{\multicolumn{2}{c|}{min}}
    \newcommand{\hmax}{\multicolumn{2}{c|}{max}}
    \newcommand{\hmed}{\multicolumn{2}{c||}{median}}
    \newcommand{\hlmed}{\multicolumn{2}{c}{median}}
    \begin{tabular}{c|c|r@{\,}l||r@{\,}l|r@{\,}l|r@{\,}l||r@{\,}l|r@{\,}l|r@{\,}l||r@{\,}l|r@{\,}l|r@{\,}l}
        \hline
        \multicolumn{2}{c|}{}       & \\[-1ex]
        \multicolumn{2}{c|}{}       & \multicolumn{20}{c}{\textbf{Query times (seconds)}} \\[1ex]
        \cline{3-22}
        \multicolumn{2}{c|}{}       & \multicolumn{2}{c||}{\matrun} & \multicolumn{6}{c||}{\relevancerun}                       & \multicolumn{6}{c||}{\magicrun}                               & \multicolumn{6}{c}{\allrun} \\
        \cline{5-22}
        \multicolumn{2}{c|}{}       & \multicolumn{2}{c||}{}        & \hmin             & \hmax             & \hmed             & \hmin             & \hmax             & \hmed                 & \hmin             & \hmax             & \hlmed          \\
        \hline
        \scen{\Gensix}      & 1k    & \val{1.78}{}                  & \val{0.02}{}      & \val{1.72}{}      & \val{0.03}{}      & \val{0.98}{}      & \val{2.54}{}      & \val{1.03}{}          & \val{0.02}{} & \val{2.07}{}           & \val{0.03}{}    \\
                            & 10k   & \val{26.54}{}                 & \val{0.04}{}      & \val{18.61}{}     & \val{0.09}{}      & \val{14.87}{}     & \val{35.61}{}     & \val{15.34}{}         & \val{0.04}{} & \val{27.05}{}          & \val{0.09}{}    \\
                            & 50k   & \val{147.49}{}                & \val{0.09}{}      & \val{120.92}{}    & \val{0.30}{}      & \val{84.78}{}     & \val{191.05}{}    & \val{87.67}{}         & \val{0.10}{} & \val{138.36}{}         & \val{0.35}{}    \\
        \hline
        \scen{\Genseven}    & 1k    & \val{5.60}{}                  & \val{0.03}{}      & \val{5.09}{}      & \val{0.04}{}      & \val{4.13}{}      & \val{7.47}{}      & \val{4.55}{}          & \val{0.03}{} & \val{6.58}{}           & \val{0.04}{}    \\
                            & 10k   & \val{73.80}{}                 & \val{0.05}{}      & \val{62.04}{}     & \val{0.10}{}      & \val{53.92}{}     & \val{98.78}{}     & \val{58.77}{}         & \val{0.05}{} & \val{73.87}{}          & \val{0.10}{}    \\
                            & 50k   & \val{397.35}{}                & \val{0.06}{}      & \val{281.57}{}    & \val{0.26}{}      & \val{265.60}{}    & \val{448.89}{}    & \val{295.24}{}        & \val{0.09}{} & \val{325.11}{}         & \val{0.28}{}    \\
        \hline
        \scen{\Geneight}    & 1k    & \val{3.91}{}                  & \val{0.02}{}      & \val{3.84}{}      & \val{0.02}{}      & \val{4.31}{}      & \val{5.41}{}      & \val{4.38}{}          & \val{0.02}{} & \val{5.12}{}           & \val{0.03}{}    \\
                            & 10k   & \val{55.00}{}                 & \val{0.04}{}      & \val{51.66}{}     & \val{0.07}{}      & \val{56.41}{}     & \val{73.20}{}     & \val{57.29}{}         & \val{0.04}{} & \val{60.77}{}          & \val{0.08}{}    \\
                            & 50k   & \val{288.05}{}                & \val{0.08}{}      & \val{253.21}{}    & \val{0.20}{}      & \val{289.70}{}    & \val{370.28}{}    & \val{297.47}{}        & \val{0.08}{} & \val{307.23}{}         & \val{0.25}{}    \\
        \hline
        \multicolumn{2}{c|}{}       & \\[-1ex]
        \multicolumn{2}{c|}{}       & \multicolumn{20}{c}{\textbf{The numbers of total derived facts}} \\[1ex]
        \cline{3-22}
        \multicolumn{2}{c|}{}       & \multicolumn{2}{c||}{\matrun} & \multicolumn{6}{c||}{\relevancerun}                       & \multicolumn{6}{c||}{\magicrun}                               & \multicolumn{6}{c}{\allrun} \\
        \cline{5-22}
        \multicolumn{2}{c|}{}       & \multicolumn{2}{c||}{}        & \hmin             & \hmax             & \hmed             & \hmin             & \hmax                 & \hmed             & \hmin             & \hmax             & \hlmed          \\
        \hline
        \scen{\Gensix}      & 1k    & \val{1.59}{M}                 & \val{4.00}{k}     & \val{1.48}{M}     & \val{12.50}{k}    & \val{1.01}{M}     & \val{2.46}{M}         & \val{1.02}{M}     & \val{4.00}{k}     & \val{1.72}{M}     & \val{13.50}{k}  \\
                            & 10k   & \val{15.9}{M}                 & \val{40.00}{k}    & \val{14.82}{M}    & \val{125.00}{k}   & \val{10.13}{M}    & \val{24.65}{M}        & \val{10.25}{M}    & \val{40.00}{k}    & \val{17.27}{M}    & \val{135.00}{k} \\
                            & 50k   & \val{79.5}{M}                 & \val{200.00}{k}   & \val{74.10}{M}    & \val{625.00}{k}   & \val{50.65}{M}    & \val{123.25}{M}       & \val{51.25}{M}    & \val{200.00}{k}   & \val{86.40}{M}    & \val{675.00}{k} \\
        \hline
        \scen{\Genseven}    & 1k    & \val{2.78}{M}                 & \val{4.00}{k}     & \val{2.76}{M}     & \val{12.00}{k}    & \val{2.35}{M}     & \val{5.10}{M}         & \val{2.39}{M}     & \val{4.00}{k}     & \val{3.68}{M}     & \val{13.00}{k}  \\
                            & 10k   & \val{27.87}{M}                & \val{40.00}{k}    & \val{27.69}{M}    & \val{120.00}{k}   & \val{23.57}{M}    & \val{51.04}{M}        & \val{23.96}{M}    & \val{40.00}{k}    & \val{36.87}{M}    & \val{130.00}{k} \\
                            & 50k   & \val{139.35}{M}               & \val{200.00}{k}   & \val{138.45}{M}   & \val{600.00}{k}   & \val{117.85}{M}   & \val{255.20}{M}       & \val{119.82}{M}   & \val{200.00}{k}   & \val{184.35}{M}   & \val{650.00}{k} \\
        \hline
        \scen{\Geneight}    & 1k    & \val{3.08}{M}                 & \val{4.00}{k}     & \val{2.94}{M}     & \val{9.50}{k}     & \val{2.39}{M}     & \val{3.52}{M}         & \val{2.42}{M}     & \val{4.00}{k}     & \val{2.93}{M}     & \val{11.00}{k}  \\
                            & 10k   & \val{30.81}{M}                & \val{40.00}{k}    & \val{29.45}{M}    & \val{95.00}{k}    & \val{23.94}{M}    & \val{35.21}{M}        & \val{24.25}{M}    & \val{40.00}{k}    & \val{29.38}{M}    & \val{110.00}{k} \\
                            & 50k   & \val{154.05}{M}               & \val{200.00}{k}   & \val{147.25}{M}   & \val{475.00}{k}   & \val{119.7}{M}    & \val{176.05}{M}       & \val{121.25}{M}   & \val{200.00}{k}   & \val{146.9}{M}    & \val{550.00}{k} \\
        \hline
        \multicolumn{2}{c|}{}       & \\[-1ex]
        \multicolumn{2}{c|}{}       & \multicolumn{20}{c}{\textbf{The numbers of useful derived facts}} \\[1ex]
        \cline{3-22}
        \multicolumn{2}{c|}{}       & \multicolumn{2}{c||}{\matrun} & \multicolumn{6}{c||}{\relevancerun}                       & \multicolumn{6}{c||}{\magicrun}                           & \multicolumn{6}{c}{\allrun} \\
        \cline{5-22}
        \multicolumn{2}{c|}{}       & \multicolumn{2}{c||}{}        & \hmin             & \hmax             & \hmed             & \hmin             & \hmax             & \hmed             & \hmin             & \hmax             & \hlmed           \\
        \hline
        \scen{\Gensix}      & 1k    & \val{1.01}{M}                 & \val{3.00}{k}     & \val{868.00}{k}   & \val{12.50}{k}    & \val{402.00}{k}   & \val{440.00}{k}   & \val{404.00}{k}   & \val{3.00}{k}     & \val{410.00}{k}   & \val{8.00}{k}    \\
                            & 10k   & \val{10.10}{M}                & \val{30.00}{k}    & \val{8.68}{M}     & \val{125.00}{k}   & \val{4.02}{M}     & \val{4.40}{M}     & \val{4.04}{M}     & \val{30.00}{k}    & \val{4.10}{M}     & \val{80.00}{k}   \\
                            & 50k   & \val{50.50}{M}                & \val{150.00}{k}   & \val{43.40}{M}    & \val{625.00}{k}   & \val{20.10}{M}    & \val{22.00}{M}    & \val{20.00}{M}    & \val{150.00}{k}   & \val{20.50}{M}    & \val{400.00}{k}  \\
        \hline
        \scen{\Genseven}    & 1k    & \val{999.00}{k}               & \val{3.00}{k}     & \val{837.00}{k}   & \val{9.00}{k}     & \val{406.00}{k}   & \val{500.00}{k}   & \val{409.00}{k}   & \val{3.00}{k}     & \val{416.00}{k}   & \val{8.00}{k}    \\
                            & 10k   & \val{10.00}{M}                & \val{30.00}{k}    & \val{8.37}{M}     & \val{90.00}{k}    & \val{4.06}{M}     & \val{5.00}{M}     & \val{4.09}{M}     & \val{30.00}{k}    & \val{4.16}{M}     & \val{80.00}{k}   \\
                            & 50k   & \val{49.95}{M}                & \val{150.00}{k}   & \val{41.85}{M}    & \val{450.00}{k}   & \val{20.30}{M}    & \val{25.00}{M}    & \val{20.47}{M}    & \val{150.00}{k}   & \val{20.80}{M}    & \val{400.00}{k}  \\
        \hline
        \scen{\Geneight}    & 1k    & \val{825.00}{k}               & \val{3.00}{k}     & \val{691.00}{k}   & \val{7.00}{k}     & \val{333.00}{k}   & \val{362.00}{k}   & \val{335.00}{k}   & \val{3.00}{k}     & \val{341.00}{k}   & \val{7.00}{k}    \\
                            & 10k   & \val{8.25}{M}                 & \val{30.00}{k}    & \val{6.91}{M}     & \val{70.00}{k}    & \val{3.33}{M}     & \val{3.62}{M}     & \val{3.35}{M}     & \val{30.00}{k}    & \val{3.41}{M}     & \val{70.00}{k}   \\
                            & 50k   & \val{41.25}{M}                & \val{150.00}{k}   & \val{34.55}{M}    & \val{350.00}{k}   & \val{16.65}{M}    & \val{18.10}{M}    & \val{16.75}{M}    & \val{150.00}{k}   & \val{17.05}{M}    & \val{350.00}{k}  \\
        \hline
    \end{tabular}
\end{table}

\begin{figure}[tb!]
    \centering
    \scalebox{0.9}{
        \begin{tabular}{ccc}
            \Gensix                                     & \Genseven                                   & \Geneight                                   \\
            \begin{tikzpicture}[xscale=0.25, yscale=0.20]
    \pgfplotsset{compat=newest}

    \begin{axis}[
            width=15cm,
            height=15cm,
            enlargelimits=false,
            ymode=log,
            ytick pos=left,
            ylabel={\huge \#Rules},
            xmin=0,
            xmax=120,
            ymax=10000,
            xtick={1,20,40,60,80,100,120},
            xtick pos=left,
            xlabel={\huge \#Queries},
            tick label style={font=\huge},
            tick align=outside,
            legend cell align=left,
            legend style={at={(0,1)}, anchor=north west, font=\huge}
        ]
        \addplot+[color=black, mark=none] table {
                0 1987
                120 1987
            };
        \addplot+[color=green, mark=none] table {
                1 3504
                2 3505
                3 3505
                4 3506
                5 3506
                6 3507
                7 3508
                8 3508
                9 3508
                10 3508
                11 3510
                12 3510
                13 3512
                14 3513
                15 3513
                16 3513
                17 3514
                18 3514
                19 3514
                20 3514
                21 3514
                22 3514
                23 3514
                24 3515
                25 3515
                26 3515
                27 3515
                28 3515
                29 3515
                30 3515
                31 3515
                32 3515
                33 3516
                34 3516
                35 3516
                36 3516
                37 3516
                38 3516
                39 3516
                40 3516
                41 3517
                42 3517
                43 3517
                44 3518
                45 3518
                46 3519
                47 3519
                48 3519
                49 3519
                50 3519
                51 3519
                52 3519
                53 3519
                54 3519
                55 3520
                56 3520
                57 3520
                58 3520
                59 3520
                60 3520
                61 3520
                62 3520
                63 3520
                64 3521
                65 3521
                66 3521
                67 3521
                68 3521
                69 3521
                70 3521
                71 3521
                72 3521
                73 3521
                74 3521
                75 3521
                76 3521
                77 3521
                78 3521
                79 3521
                80 3521
                81 3521
                82 3521
                83 3521
                84 3521
                85 3521
                86 3521
                87 3521
                88 3521
                89 3521
                90 3521
                91 3521
                92 3521
                93 3521
                94 3521
                95 3521
                96 3521
                97 3521
                98 3522
                99 3522
                100 3522
                101 3522
                102 3522
                103 3522
                104 3522
                105 3522
                106 3522
                107 3522
                108 3522
                109 3522
                110 3522
                111 3522
                112 3522
                113 3523
                114 3523
                115 3523
                116 3523
                117 3523
                118 3523
                119 3524
                120 3524
            };
        \addplot+[color=blue, mark=none]  table {
                1 4
                2 4
                3 4
                4 4
                5 6
                6 6
                7 6
                8 6
                9 7
                10 7
                11 8
                12 8
                13 9
                14 9
                15 9
                16 9
                17 9
                18 9
                19 9
                20 9
                21 9
                22 9
                23 9
                24 9
                25 9
                26 9
                27 9
                28 10
                29 11
                30 11
                31 11
                32 11
                33 11
                34 11
                35 11
                36 11
                37 11
                38 11
                39 12
                40 12
                41 12
                42 12
                43 12
                44 12
                45 14
                46 14
                47 14
                48 14
                49 14
                50 14
                51 15
                52 15
                53 15
                54 15
                55 15
                56 16
                57 16
                58 16
                59 16
                60 16
                61 17
                62 17
                63 17
                64 17
                65 18
                66 18
                67 18
                68 19
                69 19
                70 19
                71 20
                72 20
                73 20
                74 21
                75 22
                76 23
                77 23
                78 24
                79 25
                80 25
                81 28
                82 28
                83 28
                84 28
                85 30
                86 31
                87 31
                88 31
                89 32
                90 32
                91 37
                92 39
                93 39
                94 40
                95 40
                96 40
                97 41
                98 42
                99 43
                100 43
                101 43
                102 43
                103 44
                104 47
                105 47
                106 57
                107 60
                108 61
                109 68
                110 70
                111 80
                112 1816
                113 1816
                114 1822
                115 1822
                116 1822
                117 1822
                118 1822
                119 1822
                120 1822
            };
        \addplot+[color=red, mark=none]  table {
                1 6
                2 6
                3 6
                4 6
                5 8
                6 8
                7 8
                8 9
                9 10
                10 10
                11 11
                12 11
                13 11
                14 12
                15 12
                16 12
                17 12
                18 12
                19 12
                20 12
                21 12
                22 12
                23 12
                24 12
                25 12
                26 12
                27 12
                28 13
                29 14
                30 14
                31 14
                32 14
                33 14
                34 14
                35 14
                36 14
                37 14
                38 14
                39 15
                40 16
                41 16
                42 16
                43 16
                44 16
                45 17
                46 18
                47 18
                48 18
                49 18
                50 18
                51 19
                52 19
                53 20
                54 20
                55 20
                56 20
                57 20
                58 20
                59 20
                60 20
                61 20
                62 21
                63 21
                64 21
                65 22
                66 23
                67 23
                68 23
                69 23
                70 24
                71 24
                72 25
                73 26
                74 27
                75 28
                76 29
                77 29
                78 29
                79 31
                80 31
                81 34
                82 36
                83 36
                84 36
                85 38
                86 38
                87 40
                88 40
                89 42
                90 42
                91 47
                92 49
                93 49
                94 50
                95 51
                96 51
                97 53
                98 53
                99 53
                100 53
                101 54
                102 55
                103 56
                104 58
                105 59
                106 71
                107 73
                108 75
                109 84
                110 86
                111 98
                112 3227
                113 3230
                114 3230
                115 3232
                116 3232
                117 3233
                118 3233
                119 3235
                120 3235
            };
    \end{axis}
\end{tikzpicture} & \begin{tikzpicture}[xscale=0.25, yscale=0.20]
\pgfplotsset{compat=newest}

\begin{axis}[
    width=15cm,
    height=15cm,
    enlargelimits=false,
    ymode=log,
    ytick pos=left,
    ylabel={\huge \#Rules},
    xmin=0,
    xmax=120,
    xtick={1,20,40,60,80,100,120},
    xtick pos=left,
    xlabel={\huge \#Queries},
    tick label style={font=\huge},
    tick align=outside,
    legend cell align=left,
    legend style={at={(0,1)}, anchor=north west, font=\huge}
]
\addplot+[color=black, mark=none] table {
0 2005
120 2005
};
\addplot+[color=green, mark=none] table {
1 3506
2 3508
3 3509
4 3514
5 3516
6 3517
7 3518
8 3518
9 3519
10 3519
11 3520
12 3520
13 3521
14 3521
15 3521
16 3521
17 3521
18 3522
19 3522
20 3522
21 3522
22 3522
23 3522
24 3522
25 3523
26 3523
27 3523
28 3523
29 3523
30 3523
31 3523
32 3523
33 3524
34 3524
35 3524
36 3525
37 3525
38 3525
39 3525
40 3525
41 3525
42 3525
43 3525
44 3525
45 3525
46 3526
47 3526
48 3526
49 3526
50 3527
51 3527
52 3527
53 3527
54 3527
55 3527
56 3527
57 3527
58 3527
59 3527
60 3527
61 3528
62 3528
63 3528
64 3528
65 3528
66 3528
67 3528
68 3528
69 3528
70 3528
71 3528
72 3528
73 3529
74 3529
75 3529
76 3529
77 3529
78 3529
79 3529
80 3529
81 3529
82 3529
83 3529
84 3529
85 3529
86 3529
87 3529
88 3529
89 3529
90 3529
91 3529
92 3529
93 3529
94 3529
95 3529
96 3529
97 3529
98 3530
99 3530
100 3530
101 3530
102 3530
103 3530
104 3530
105 3530
106 3530
107 3530
108 3530
109 3530
110 3530
111 3530
112 3530
113 3531
114 3531
115 3531
116 3531
117 3531
118 3531
119 3532
120 3532
};
\addplot+[color=blue, mark=none]  table {
1 4
2 4
3 4
4 4
5 4
6 4
7 4
8 4
9 4
10 4
11 4
12 4
13 4
14 6
15 6
16 6
17 6
18 6
19 6
20 6
21 6
22 7
23 8
24 9
25 9
26 9
27 9
28 9
29 9
30 9
31 9
32 9
33 9
34 9
35 9
36 11
37 11
38 11
39 11
40 12
41 12
42 12
43 12
44 12
45 12
46 12
47 12
48 14
49 14
50 14
51 14
52 14
53 14
54 14
55 14
56 15
57 15
58 15
59 16
60 16
61 17
62 17
63 18
64 18
65 18
66 19
67 19
68 19
69 19
70 20
71 20
72 20
73 21
74 21
75 23
76 24
77 24
78 25
79 27
80 29
81 30
82 31
83 31
84 32
85 32
86 33
87 35
88 36
89 37
90 40
91 41
92 43
93 48
94 50
95 54
96 54
97 56
98 56
99 59
100 60
101 95
102 1801
103 1801
104 1801
105 1801
106 1801
107 1801
108 1801
109 1801
110 1801
111 1801
112 1801
113 1804
114 1804
115 1805
116 1805
117 1805
118 1807
119 1807
120 1807
};
\addplot+[color=red, mark=none]  table {
1 6
2 6
3 6
4 6
5 6
6 6
7 6
8 6
9 6
10 6
11 6
12 6
13 6
14 8
15 8
16 8
17 8
18 8
19 8
20 8
21 8
22 10
23 11
24 11
25 11
26 11
27 12
28 12
29 12
30 12
31 12
32 12
33 12
34 12
35 12
36 14
37 14
38 14
39 14
40 14
41 16
42 16
43 16
44 16
45 16
46 16
47 16
48 17
49 17
50 18
51 18
52 18
53 18
54 18
55 18
56 19
57 19
58 19
59 20
60 20
61 20
62 21
63 22
64 23
65 23
66 23
67 23
68 23
69 24
70 24
71 24
72 25
73 26
74 27
75 29
76 31
77 32
78 32
79 35
80 36
81 37
82 38
83 40
84 40
85 42
86 42
87 43
88 45
89 48
90 52
91 52
92 54
93 60
94 61
95 65
96 68
97 70
98 71
99 72
100 75
101 114
102 3212
103 3212
104 3213
105 3214
106 3214
107 3215
108 3216
109 3216
110 3217
111 3217
112 3219
113 3219
114 3220
115 3221
116 3221
117 3221
118 3222
119 3223
120 3223
};
\end{axis}
\end{tikzpicture} & \begin{tikzpicture}[xscale=0.25, yscale=0.20]
\pgfplotsset{compat=newest}

\begin{axis}[
    width=15cm,
    height=15cm,
    enlargelimits=false,
    ymode=log,
    ytick pos=left,
    ylabel={\huge \#Rules},
    xmin=0,
    xmax=120,
    ymax=10000,
    xtick={1,20,40,60,80,100,120},
    xtick pos=left,
    xlabel={\huge \#Queries},
    tick label style={font=\huge},
    tick align=outside,
    legend cell align=left,
    legend style={at={(0,1)}, anchor=north west, font=\huge}
]
\addplot+[color=black, mark=none] table {
0 1568
120 1568
};
\addplot+[color=green, mark=none] table {
1 2561
2 2561
3 2561
4 2564
5 2564
6 2565
7 2567
8 2567
9 2568
10 2568
11 2569
12 2570
13 2570
14 2570
15 2571
16 2571
17 2571
18 2572
19 2572
20 2572
21 2572
22 2572
23 2572
24 2574
25 2574
26 2574
27 2574
28 2574
29 2575
30 2575
31 2575
32 2575
33 2575
34 2575
35 2575
36 2575
37 2575
38 2575
39 2575
40 2575
41 2576
42 2576
43 2576
44 2576
45 2576
46 2576
47 2576
48 2576
49 2576
50 2576
51 2577
52 2577
53 2577
54 2577
55 2577
56 2577
57 2577
58 2577
59 2577
60 2577
61 2577
62 2577
63 2577
64 2577
65 2577
66 2577
67 2577
68 2577
69 2577
70 2577
71 2577
72 2577
73 2577
74 2577
75 2577
76 2577
77 2577
78 2577
79 2577
80 2577
81 2577
82 2577
83 2577
84 2577
85 2578
86 2578
87 2578
88 2578
89 2578
90 2578
91 2578
92 2578
93 2578
94 2578
95 2578
96 2578
97 2578
98 2578
99 2578
100 2578
101 2578
102 2578
103 2578
104 2578
105 2578
106 2578
107 2578
108 2578
109 2578
110 2578
111 2578
112 2578
113 2578
114 2578
115 2578
116 2578
117 2579
118 2579
119 2579
120 2579
};
\addplot+[color=blue, mark=none]  table {
1 4
2 4
3 4
4 4
5 4
6 4
7 4
8 4
9 4
10 4
11 6
12 6
13 6
14 6
15 6
16 6
17 7
18 9
19 9
20 9
21 9
22 9
23 9
24 9
25 9
26 9
27 9
28 9
29 9
30 9
31 9
32 9
33 9
34 10
35 11
36 11
37 11
38 11
39 11
40 11
41 12
42 12
43 12
44 12
45 12
46 12
47 12
48 12
49 12
50 12
51 12
52 12
53 12
54 12
55 12
56 12
57 12
58 12
59 12
60 12
61 12
62 13
63 14
64 14
65 14
66 14
67 14
68 14
69 14
70 14
71 14
72 15
73 15
74 16
75 16
76 16
77 16
78 17
79 17
80 17
81 17
82 17
83 17
84 17
85 17
86 17
87 18
88 18
89 18
90 18
91 19
92 19
93 19
94 20
95 20
96 20
97 21
98 21
99 22
100 25
101 27
102 29
103 31
104 35
105 35
106 36
107 37
108 38
109 38
110 43
111 43
112 46
113 57
114 81
115 1409
116 1409
117 1409
118 1409
119 1410
120 1415
};
\addplot+[color=red, mark=none]  table {
1 6
2 6
3 6
4 6
5 6
6 6
7 6
8 6
9 6
10 6
11 8
12 8
13 8
14 8
15 8
16 8
17 10
18 11
19 11
20 12
21 12
22 12
23 12
24 12
25 12
26 12
27 12
28 12
29 12
30 12
31 12
32 12
33 12
34 13
35 14
36 14
37 14
38 14
39 14
40 14
41 15
42 15
43 15
44 15
45 15
46 15
47 15
48 15
49 16
50 16
51 16
52 16
53 16
54 16
55 16
56 16
57 16
58 16
59 16
60 16
61 16
62 16
63 17
64 17
65 18
66 18
67 18
68 18
69 18
70 18
71 18
72 19
73 20
74 20
75 20
76 20
77 20
78 20
79 20
80 20
81 21
82 21
83 21
84 21
85 21
86 21
87 21
88 22
89 22
90 23
91 23
92 23
93 23
94 25
95 25
96 25
97 26
98 27
99 27
100 32
101 35
102 36
103 39
104 44
105 45
106 46
107 47
108 48
109 49
110 54
111 55
112 55
113 68
114 97
115 2371
116 2373
117 2381
118 2381
119 2382
120 2415
};
\end{axis}
\end{tikzpicture} \\
        \end{tabular}
    }
    \caption{The numbers of rules for \matrun, \textcolor{green}{\magicrun}, \textcolor{blue}{\relevancerun}, and \textcolor{red}{\allrun} on SO scenarios}\label{fig:so:results:1}
\end{figure}

\begin{figure}[p]
    \centering
    \scalebox{0.9}{
        \begin{tabular}{cccc}
                                                           & \Gensix-1k                                        & \Gensix-10k                                        & \Gensix-50k                                        \\
            \rotatebox{90}{\hspace{0.5cm}Query Times (ms)} & \begin{tikzpicture}[xscale=0.25, yscale=0.20]
\pgfplotsset{compat=newest}

\begin{axis}[
    width=15cm,
    height=15cm,
    enlargelimits=false,
    ymode=log,
    ytick pos=left,
    ylabel={\huge Time in ms},
    xmin=0,
    xmax=120,
    xtick={1,20,40,60,80,100,120},
    xtick pos=left,
    xlabel={\huge \#Queries},
    tick label style={font=\huge},
    tick align=outside,
    legend cell align=left,
    legend style={at={(0,1)}, anchor=north west, font=\huge}
]
\addplot+[color=black, mark=none] table {
0 1782
120 1782
};
\addplot+[color=green, mark=none] table {
1 981
2 986
3 995
4 997
5 998
6 999
7 999
8 1000
9 1000
10 1000
11 1002
12 1002
13 1003
14 1005
15 1006
16 1006
17 1007
18 1007
19 1007
20 1008
21 1008
22 1009
23 1010
24 1010
25 1010
26 1010
27 1010
28 1011
29 1011
30 1011
31 1012
32 1012
33 1014
34 1014
35 1014
36 1015
37 1017
38 1017
39 1017
40 1018
41 1018
42 1019
43 1019
44 1019
45 1020
46 1020
47 1022
48 1022
49 1023
50 1024
51 1024
52 1025
53 1025
54 1026
55 1026
56 1026
57 1028
58 1029
59 1030
60 1030
61 1030
62 1030
63 1031
64 1032
65 1032
66 1032
67 1033
68 1034
69 1034
70 1034
71 1037
72 1040
73 1042
74 1042
75 1045
76 1048
77 1050
78 1051
79 1052
80 1055
81 1062
82 1062
83 1066
84 1070
85 1072
86 1090
87 1091
88 1102
89 1112
90 1113
91 1115
92 1117
93 1122
94 1125
95 1125
96 1132
97 1142
98 1147
99 1148
100 1161
101 1559
102 1596
103 1599
104 1602
105 1620
106 1622
107 1627
108 1642
109 1644
110 1659
111 1669
112 1685
113 1704
114 1713
115 1715
116 1725
117 1755
118 1801
119 2252
120 2544
};
\addplot+[color=blue, mark=none]  table {
1 27
2 28
3 28
4 28
5 28
6 28
7 28
8 29
9 29
10 29
11 29
12 29
13 29
14 29
15 29
16 29
17 29
18 30
19 30
20 30
21 30
22 31
23 31
24 31
25 31
26 31
27 31
28 31
29 31
30 31
31 31
32 31
33 31
34 31
35 32
36 32
37 32
38 32
39 32
40 32
41 32
42 32
43 32
44 32
45 32
46 33
47 33
48 33
49 33
50 33
51 33
52 33
53 33
54 33
55 33
56 33
57 33
58 33
59 34
60 34
61 34
62 34
63 34
64 34
65 35
66 35
67 35
68 35
69 35
70 35
71 36
72 37
73 37
74 38
75 38
76 38
77 39
78 39
79 39
80 39
81 40
82 40
83 41
84 42
85 42
86 42
87 43
88 43
89 43
90 44
91 47
92 48
93 48
94 49
95 49
96 49
97 49
98 50
99 52
100 52
101 52
102 53
103 55
104 55
105 60
106 61
107 63
108 63
109 65
110 68
111 72
112 1524
113 1552
114 1557
115 1613
116 1635
117 1655
118 1686
119 1701
120 1720
};
\addplot+[color=red, mark=none]  table {
1 27
2 27
3 28
4 28
5 28
6 28
7 29
8 29
9 30
10 30
11 30
12 30
13 30
14 30
15 30
16 30
17 30
18 30
19 31
20 31
21 31
22 31
23 31
24 31
25 31
26 31
27 32
28 32
29 32
30 32
31 32
32 33
33 33
34 33
35 33
36 33
37 33
38 33
39 33
40 34
41 34
42 34
43 34
44 34
45 34
46 34
47 34
48 34
49 35
50 35
51 35
52 35
53 35
54 35
55 35
56 35
57 35
58 35
59 36
60 36
61 36
62 36
63 37
64 37
65 37
66 37
67 37
68 38
69 38
70 38
71 38
72 38
73 39
74 41
75 41
76 41
77 41
78 41
79 42
80 44
81 45
82 45
83 46
84 48
85 48
86 48
87 49
88 49
89 49
90 51
91 52
92 52
93 52
94 52
95 53
96 53
97 55
98 55
99 56
100 57
101 58
102 59
103 61
104 61
105 62
106 66
107 66
108 66
109 68
110 68
111 76
112 1782
113 1802
114 1822
115 1857
116 1858
117 1895
118 1921
119 1992
120 2072
};
\end{axis}
\end{tikzpicture}        & \begin{tikzpicture}[xscale=0.25, yscale=0.20]
\pgfplotsset{compat=newest}

\begin{axis}[
    width=15cm,
    height=15cm,
    enlargelimits=false,
    ymode=log,
    ytick pos=left,
    ylabel={\huge Time in ms},
    xmin=0,
    xmax=120,
    xtick={1,20,40,60,80,100,120},
    xtick pos=left,
    xlabel={\huge \#Queries},
    tick label style={font=\huge},
    tick align=outside,
    legend cell align=left,
    legend style={at={(0,1)}, anchor=north west, font=\huge}
]
\addplot+[color=black, mark=none] table {
0 26542
120 26542
};
\addplot+[color=green, mark=none] table {
1 14874
2 14893
3 14899
4 14920
5 14958
6 14988
7 15005
8 15006
9 15027
10 15030
11 15040
12 15044
13 15047
14 15048
15 15065
16 15073
17 15076
18 15101
19 15111
20 15113
21 15113
22 15114
23 15115
24 15122
25 15126
26 15127
27 15136
28 15145
29 15146
30 15147
31 15153
32 15157
33 15167
34 15172
35 15172
36 15187
37 15190
38 15199
39 15208
40 15214
41 15227
42 15237
43 15239
44 15241
45 15245
46 15259
47 15259
48 15266
49 15270
50 15274
51 15300
52 15305
53 15308
54 15312
55 15318
56 15319
57 15330
58 15332
59 15339
60 15341
61 15355
62 15376
63 15378
64 15400
65 15402
66 15405
67 15410
68 15417
69 15427
70 15432
71 15435
72 15439
73 15441
74 15473
75 15476
76 15487
77 15538
78 15549
79 15583
80 15629
81 15641
82 15689
83 15766
84 15790
85 15915
86 15952
87 15996
88 16089
89 16163
90 16195
91 16196
92 16275
93 16383
94 16528
95 16689
96 16733
97 16872
98 17372
99 18567
100 18731
101 22292
102 22520
103 22534
104 22550
105 22618
106 22640
107 22881
108 22995
109 23015
110 23254
111 23307
112 23553
113 23565
114 23586
115 23594
116 23788
117 24021
118 24491
119 33359
120 35613
};
\addplot+[color=blue, mark=none]  table {
1 46
2 46
3 47
4 48
5 53
6 54
7 54
8 58
9 60
10 61
11 61
12 62
13 63
14 64
15 65
16 66
17 67
18 67
19 67
20 67
21 67
22 69
23 69
24 69
25 69
26 70
27 71
28 71
29 71
30 71
31 72
32 72
33 73
34 73
35 74
36 74
37 74
38 75
39 75
40 75
41 75
42 76
43 78
44 78
45 78
46 79
47 79
48 80
49 82
50 83
51 83
52 84
53 85
54 87
55 88
56 88
57 88
58 89
59 90
60 90
61 93
62 93
63 95
64 96
65 96
66 96
67 97
68 97
69 98
70 99
71 100
72 101
73 102
74 102
75 106
76 109
77 110
78 113
79 117
80 118
81 122
82 123
83 123
84 124
85 124
86 127
87 127
88 128
89 131
90 134
91 138
92 142
93 143
94 143
95 145
96 157
97 161
98 161
99 161
100 168
101 168
102 172
103 172
104 179
105 193
106 194
107 227
108 233
109 235
110 266
111 275
112 17789
113 17791
114 17876
115 17898
116 17971
117 18127
118 18397
119 18440
120 18618
};
\addplot+[color=red, mark=none]  table {
1 46
2 47
3 48
4 49
5 54
6 55
7 58
8 58
9 59
10 61
11 64
12 64
13 65
14 65
15 67
16 67
17 68
18 69
19 69
20 70
21 71
22 71
23 72
24 72
25 73
26 73
27 73
28 74
29 75
30 75
31 75
32 77
33 77
34 78
35 78
36 79
37 79
38 80
39 81
40 81
41 82
42 82
43 82
44 86
45 86
46 89
47 89
48 90
49 91
50 91
51 92
52 92
53 92
54 92
55 92
56 93
57 93
58 95
59 95
60 95
61 98
62 99
63 100
64 101
65 101
66 104
67 105
68 108
69 109
70 114
71 115
72 118
73 118
74 119
75 119
76 122
77 130
78 130
79 132
80 135
81 138
82 139
83 141
84 147
85 148
86 151
87 153
88 153
89 155
90 157
91 158
92 164
93 166
94 174
95 174
96 176
97 179
98 183
99 189
100 196
101 198
102 203
103 208
104 223
105 239
106 245
107 251
108 252
109 271
110 272
111 336
112 22196
113 22286
114 22374
115 22394
116 22406
117 22854
118 22996
119 23072
120 27054
};
\end{axis}
\end{tikzpicture}        & \begin{tikzpicture}[xscale=0.25, yscale=0.20]
\pgfplotsset{compat=newest}

\begin{axis}[
    width=15cm,
    height=15cm,
    enlargelimits=false,
    ymode=log,
    ytick pos=left,
    ylabel={\huge Time in ms},
    xmin=0,
    xmax=120,
    xtick={1,20,40,60,80,100,120},
    xtick pos=left,
    xlabel={\huge \#Queries},
    tick label style={font=\huge},
    tick align=outside,
    legend cell align=left,
    legend style={at={(0,1)}, anchor=north west, font=\huge}
]
\addplot+[color=black, mark=none] table {
0 147492
120 147492
};
\addplot+[color=green, mark=none] table {
1 84786
2 84815
3 84900
4 85163
5 85198
6 85228
7 85236
8 85293
9 85408
10 85412
11 85417
12 85540
13 85558
14 85735
15 85740
16 85743
17 85762
18 85852
19 85990
20 86001
21 86002
22 86003
23 86071
24 86108
25 86119
26 86147
27 86180
28 86190
29 86194
30 86284
31 86297
32 86329
33 86343
34 86362
35 86446
36 86460
37 86534
38 86538
39 86555
40 86581
41 86622
42 86632
43 86687
44 86735
45 86888
46 86902
47 86975
48 87030
49 87045
50 87055
51 87114
52 87182
53 87262
54 87308
55 87364
56 87425
57 87459
58 87472
59 87584
60 87622
61 87732
62 87830
63 87872
64 87993
65 88078
66 88094
67 88095
68 88140
69 88160
70 88192
71 88292
72 88447
73 88813
74 88852
75 89033
76 89211
77 89248
78 89289
79 89537
80 89786
81 90460
82 90919
83 91049
84 91124
85 91986
86 92162
87 92432
88 92753
89 92851
90 92901
91 94062
92 94284
93 95064
94 95259
95 95777
96 96745
97 99254
98 99850
99 103646
100 107711
101 118572
102 120927
103 120972
104 121052
105 121686
106 122175
107 122459
108 122500
109 123946
110 125121
111 127547
112 127982
113 128440
114 128661
115 130203
116 130322
117 134085
118 152140
119 177543
120 191057
};
\addplot+[color=blue, mark=none]  table {
1 96
2 97
3 99
4 100
5 114
6 115
7 115
8 126
9 128
10 134
11 137
12 143
13 147
14 159
15 163
16 164
17 167
18 168
19 169
20 169
21 170
22 170
23 171
24 172
25 173
26 177
27 177
28 179
29 180
30 185
31 193
32 194
33 194
34 198
35 202
36 205
37 206
38 209
39 213
40 218
41 222
42 223
43 225
44 225
45 227
46 230
47 231
48 233
49 251
50 255
51 258
52 261
53 263
54 264
55 265
56 277
57 281
58 284
59 287
60 302
61 307
62 313
63 319
64 321
65 322
66 327
67 336
68 340
69 347
70 351
71 364
72 367
73 372
74 379
75 432
76 451
77 472
78 481
79 484
80 488
81 504
82 505
83 527
84 529
85 544
86 555
87 560
88 564
89 579
90 613
91 668
92 687
93 712
94 718
95 735
96 776
97 804
98 834
99 842
100 866
101 900
102 947
103 1019
104 1042
105 1047
106 1191
107 1469
108 1503
109 1581
110 1686
111 1861
112 103873
113 104221
114 105150
115 105605
116 106022
117 106444
118 106732
119 107010
120 120921
};
\addplot+[color=red, mark=none]  table {
1 104
2 105
3 105
4 106
5 120
6 121
7 122
8 136
9 136
10 137
11 141
12 156
13 161
14 170
15 180
16 180
17 200
18 204
19 204
20 205
21 207
22 207
23 207
24 207
25 209
26 209
27 209
28 210
29 212
30 212
31 214
32 221
33 229
34 232
35 233
36 240
37 253
38 253
39 257
40 259
41 262
42 273
43 274
44 275
45 281
46 289
47 293
48 298
49 299
50 305
51 307
52 309
53 315
54 328
55 330
56 334
57 340
58 343
59 348
60 351
61 354
62 358
63 358
64 365
65 388
66 395
67 421
68 421
69 440
70 458
71 467
72 493
73 526
74 529
75 542
76 546
77 550
78 572
79 587
80 604
81 647
82 676
83 690
84 704
85 725
86 751
87 758
88 762
89 783
90 788
91 791
92 800
93 960
94 963
95 1016
96 1029
97 1106
98 1121
99 1208
100 1222
101 1228
102 1276
103 1282
104 1291
105 1344
106 1539
107 1718
108 1745
109 1912
110 1919
111 2629
112 117702
113 118739
114 118944
115 119698
116 122357
117 123276
118 131733
119 133021
120 138361
};
\end{axis}
\end{tikzpicture}        \\
            \rotatebox{90}{\hspace{0.6cm}\# Derived Facts} & \begin{tikzpicture}[xscale=0.25, yscale=0.20]
\pgfplotsset{compat=newest}

\begin{axis}[
    width=15cm,
    height=15cm,
    enlargelimits=false,
    ymode=log,
    xtick pos=left,
    ylabel={\huge \#Total facts},
    xmin=0,
    xmax=120,
    xtick={1,20,40,60,80,100,120},
    xtick pos=left,
    xlabel={\huge \#Queries},
    tick label style={font=\huge},
    tick align=outside,
    legend cell align=left,
    legend style={at={(0,1)}, anchor=north west, font=\huge}
]
\addplot+[color=black, mark=none] table {
0 1590000
120 1590000
};
\addplot+[color=green, mark=none] table {
1 1013002
2 1015002
3 1015002
4 1016002
5 1017002
6 1017002
7 1017002
8 1018002
9 1018002
10 1019002
11 1019002
12 1019002
13 1019002
14 1019002
15 1019002
16 1019003
17 1019003
18 1020002
19 1020002
20 1020003
21 1020003
22 1021002
23 1021002
24 1021002
25 1022002
26 1022002
27 1022002
28 1022002
29 1022002
30 1022002
31 1022002
32 1022002
33 1022002
34 1022002
35 1022002
36 1022002
37 1022002
38 1022002
39 1022002
40 1022002
41 1022002
42 1022003
43 1022003
44 1023002
45 1023002
46 1023002
47 1023003
48 1024002
49 1024002
50 1024002
51 1024002
52 1024002
53 1024002
54 1024002
55 1025002
56 1025002
57 1025002
58 1025002
59 1025002
60 1025002
61 1025003
62 1025003
63 1025004
64 1026002
65 1026003
66 1026003
67 1026004
68 1027002
69 1027002
70 1027002
71 1027003
72 1028002
73 1028003
74 1028003
75 1030003
76 1031003
77 1031004
78 1032003
79 1033003
80 1033003
81 1033003
82 1033004
83 1034003
84 1034003
85 1034003
86 1035003
87 1035003
88 1037004
89 1038003
90 1038004
91 1039004
92 1040003
93 1040003
94 1041003
95 1041004
96 1043004
97 1046004
98 1046004
99 1048004
100 1050004
101 1715002
102 1720002
103 1721002
104 1721002
105 1725002
106 1725004
107 1732003
108 1733003
109 1734003
110 1737002
111 1737002
112 1740003
113 1745003
114 1746003
115 1747003
116 1747003
117 1749002
118 1750002
119 1755004
120 2465003
};
\addplot+[color=blue, mark=none]  table {
1 4000
2 4000
3 4000
4 4000
5 4000
6 4000
7 5000
8 5000
9 5000
10 6000
11 6000
12 6000
13 6000
14 6000
15 7000
16 7000
17 7000
18 7000
19 7000
20 7000
21 7000
22 7000
23 7000
24 7000
25 7000
26 7000
27 7000
28 7000
29 7000
30 8000
31 8000
32 8000
33 8000
34 8000
35 8000
36 8000
37 9000
38 9000
39 9000
40 9000
41 9000
42 9000
43 9000
44 9000
45 9000
46 9000
47 9000
48 10000
49 10000
50 10000
51 10000
52 10000
53 10000
54 10000
55 11000
56 12000
57 12000
58 12000
59 12000
60 12000
61 13000
62 13000
63 13000
64 13000
65 13000
66 14000
67 14000
68 14000
69 14000
70 14000
71 14000
72 15000
73 16000
74 16000
75 16000
76 16000
77 17000
78 17000
79 17000
80 17000
81 18000
82 19000
83 19000
84 20000
85 20000
86 20000
87 20000
88 20000
89 21000
90 22000
91 22000
92 22000
93 22000
94 24000
95 24000
96 25000
97 26000
98 27000
99 27000
100 28000
101 29000
102 30000
103 30000
104 33000
105 34000
106 34000
107 41000
108 42000
109 44000
110 44000
111 56000
112 1481000
113 1482000
114 1482000
115 1482000
116 1482000
117 1482000
118 1482000
119 1482000
120 1482000
};
\addplot+[color=red, mark=none]  table {
1 4002
2 4002
3 4002
4 4002
5 4002
6 4002
7 5002
8 5002
9 5002
10 5003
11 5003
12 7002
13 7002
14 7002
15 7002
16 7002
17 8002
18 8002
19 8002
20 8002
21 8002
22 8002
23 8002
24 8002
25 8002
26 8002
27 8002
28 8002
29 8002
30 8002
31 8003
32 9002
33 9002
34 9002
35 9002
36 9002
37 9003
38 9003
39 10002
40 10003
41 11002
42 11002
43 11002
44 11002
45 11002
46 11002
47 11002
48 11002
49 11002
50 11003
51 12002
52 12002
53 12002
54 12002
55 12003
56 12003
57 12003
58 13002
59 13002
60 13002
61 14002
62 14002
63 14003
64 14004
65 15002
66 15002
67 15002
68 15003
69 17002
70 17002
71 17002
72 17003
73 17004
74 18002
75 18003
76 18003
77 19002
78 19003
79 21002
80 21003
81 21003
82 22003
83 22004
84 23003
85 23003
86 25003
87 25003
88 26003
89 27003
90 27004
91 28003
92 28003
93 28004
94 28004
95 30003
96 31003
97 33003
98 33004
99 33004
100 34003
101 34004
102 35003
103 36004
104 37003
105 38003
106 44004
107 45003
108 45003
109 50004
110 52004
111 63004
112 1704002
113 1705002
114 1707002
115 1707002
116 1716003
117 1721002
118 1722003
119 1723003
120 1728003
};
\end{axis}
\end{tikzpicture} & \begin{tikzpicture}[xscale=0.25, yscale=0.20]
\pgfplotsset{compat=newest}

\begin{axis}[
    width=15cm,
    height=15cm,
    enlargelimits=false,
    ymode=log,
    xtick pos=left,
    ylabel={\huge \#Total facts},
    xmin=0,
    xmax=120,
    xtick={1,20,40,60,80,100,120},
    xtick pos=left,
    xlabel={\huge \#Queries},
    tick label style={font=\huge},
    tick align=outside,
    legend cell align=left,
    legend style={at={(0,1)}, anchor=north west, font=\huge}
]
\addplot+[color=black, mark=none] table {
0 15900000
120 15900000
};
\addplot+[color=green, mark=none] table {
1 10130002
2 10150002
3 10150002
4 10160002
5 10170002
6 10170002
7 10170002
8 10180002
9 10180002
10 10190002
11 10190002
12 10190002
13 10190002
14 10190002
15 10190002
16 10190003
17 10190003
18 10200002
19 10200002
20 10200003
21 10200003
22 10210002
23 10210002
24 10210002
25 10220002
26 10220002
27 10220002
28 10220002
29 10220002
30 10220002
31 10220002
32 10220002
33 10220002
34 10220002
35 10220002
36 10220002
37 10220002
38 10220002
39 10220002
40 10220002
41 10220002
42 10220003
43 10220003
44 10230002
45 10230002
46 10230002
47 10230003
48 10240002
49 10240002
50 10240002
51 10240002
52 10240002
53 10240002
54 10240002
55 10250002
56 10250002
57 10250002
58 10250002
59 10250002
60 10250002
61 10250003
62 10250003
63 10250004
64 10260002
65 10260003
66 10260003
67 10260004
68 10270002
69 10270002
70 10270002
71 10270003
72 10280002
73 10280003
74 10280003
75 10300003
76 10310003
77 10310004
78 10320003
79 10330003
80 10330003
81 10330003
82 10330004
83 10340003
84 10340003
85 10340003
86 10350003
87 10350003
88 10370004
89 10380003
90 10380004
91 10390004
92 10400003
93 10400003
94 10410003
95 10410004
96 10430004
97 10460004
98 10460004
99 10480004
100 10500004
101 17150002
102 17200002
103 17210002
104 17210002
105 17250002
106 17250004
107 17320003
108 17330003
109 17340003
110 17370002
111 17370002
112 17400003
113 17450003
114 17460003
115 17470003
116 17470003
117 17490002
118 17500002
119 17550004
120 24650003
};
\addplot+[color=blue, mark=none]  table {
1 40000
2 40000
3 40000
4 40000
5 40000
6 40000
7 50000
8 50000
9 50000
10 60000
11 60000
12 60000
13 60000
14 60000
15 70000
16 70000
17 70000
18 70000
19 70000
20 70000
21 70000
22 70000
23 70000
24 70000
25 70000
26 70000
27 70000
28 70000
29 70000
30 80000
31 80000
32 80000
33 80000
34 80000
35 80000
36 80000
37 90000
38 90000
39 90000
40 90000
41 90000
42 90000
43 90000
44 90000
45 90000
46 90000
47 90000
48 100000
49 100000
50 100000
51 100000
52 100000
53 100000
54 100000
55 110000
56 120000
57 120000
58 120000
59 120000
60 120000
61 130000
62 130000
63 130000
64 130000
65 130000
66 140000
67 140000
68 140000
69 140000
70 140000
71 140000
72 150000
73 160000
74 160000
75 160000
76 160000
77 170000
78 170000
79 170000
80 170000
81 180000
82 190000
83 190000
84 200000
85 200000
86 200000
87 200000
88 200000
89 210000
90 220000
91 220000
92 220000
93 220000
94 240000
95 240000
96 250000
97 260000
98 270000
99 270000
100 280000
101 290000
102 300000
103 300000
104 330000
105 340000
106 340000
107 410000
108 420000
109 440000
110 440000
111 560000
112 14810000
113 14820000
114 14820000
115 14820000
116 14820000
117 14820000
118 14820000
119 14820000
120 14820000
};
\addplot+[color=red, mark=none]  table {
1 40002
2 40002
3 40002
4 40002
5 40002
6 40002
7 50002
8 50002
9 50002
10 50003
11 50003
12 70002
13 70002
14 70002
15 70002
16 70002
17 80002
18 80002
19 80002
20 80002
21 80002
22 80002
23 80002
24 80002
25 80002
26 80002
27 80002
28 80002
29 80002
30 80002
31 80003
32 90002
33 90002
34 90002
35 90002
36 90002
37 90003
38 90003
39 100002
40 100003
41 110002
42 110002
43 110002
44 110002
45 110002
46 110002
47 110002
48 110002
49 110002
50 110003
51 120002
52 120002
53 120002
54 120002
55 120003
56 120003
57 120003
58 130002
59 130002
60 130002
61 140002
62 140002
63 140003
64 140004
65 150002
66 150002
67 150002
68 150003
69 170002
70 170002
71 170002
72 170003
73 170004
74 180002
75 180003
76 180003
77 190002
78 190003
79 210002
80 210003
81 210003
82 220003
83 220004
84 230003
85 230003
86 250003
87 250003
88 260003
89 270003
90 270004
91 280003
92 280003
93 280004
94 280004
95 300003
96 310003
97 330003
98 330004
99 330004
100 340003
101 340004
102 350003
103 360004
104 370003
105 380003
106 440004
107 450003
108 450003
109 500004
110 520004
111 630004
112 17040002
113 17050002
114 17070002
115 17070002
116 17160003
117 17210002
118 17220003
119 17230003
120 17280003
};
\end{axis}
\end{tikzpicture} & \begin{tikzpicture}[xscale=0.25, yscale=0.20]
\pgfplotsset{compat=newest}

\begin{axis}[
    width=15cm,
    height=15cm,
    enlargelimits=false,
    ymode=log,
    xtick pos=left,
    ylabel={\huge \#Total facts},
    xmin=0,
    xmax=120,
    xtick={1,20,40,60,80,100,120},
    xtick pos=left,
    xlabel={\huge \#Queries},
    tick label style={font=\huge},
    tick align=outside,
    legend cell align=left,
    legend style={at={(0,1)}, anchor=north west, font=\huge}
]
\addplot+[color=black, mark=none] table {
0 79500000
120 79500000
};
\addplot+[color=green, mark=none] table {
1 50650002
2 50750002
3 50750002
4 50800002
5 50850002
6 50850002
7 50850002
8 50900002
9 50900002
10 50950002
11 50950002
12 50950002
13 50950002
14 50950002
15 50950002
16 50950003
17 50950003
18 51000002
19 51000002
20 51000003
21 51000003
22 51050002
23 51050002
24 51050002
25 51100002
26 51100002
27 51100002
28 51100002
29 51100002
30 51100002
31 51100002
32 51100002
33 51100002
34 51100002
35 51100002
36 51100002
37 51100002
38 51100002
39 51100002
40 51100002
41 51100002
42 51100003
43 51100003
44 51150002
45 51150002
46 51150002
47 51150003
48 51200002
49 51200002
50 51200002
51 51200002
52 51200002
53 51200002
54 51200002
55 51250002
56 51250002
57 51250002
58 51250002
59 51250002
60 51250002
61 51250003
62 51250003
63 51250004
64 51300002
65 51300003
66 51300003
67 51300004
68 51350002
69 51350002
70 51350002
71 51350003
72 51400002
73 51400003
74 51400003
75 51500003
76 51550003
77 51550004
78 51600003
79 51650003
80 51650003
81 51650003
82 51650004
83 51700003
84 51700003
85 51700003
86 51750003
87 51750003
88 51850004
89 51900003
90 51900004
91 51950004
92 52000003
93 52000003
94 52050003
95 52050004
96 52150004
97 52300004
98 52300004
99 52400004
100 52500004
101 85750002
102 86000002
103 86050002
104 86050002
105 86250002
106 86250004
107 86600003
108 86650003
109 86700003
110 86850002
111 86850002
112 87000003
113 87250003
114 87300003
115 87350003
116 87350003
117 87450002
118 87500002
119 87750004
120 123250003
};
\addplot+[color=blue, mark=none]  table {
1 200000
2 200000
3 200000
4 200000
5 200000
6 200000
7 250000
8 250000
9 250000
10 300000
11 300000
12 300000
13 300000
14 300000
15 350000
16 350000
17 350000
18 350000
19 350000
20 350000
21 350000
22 350000
23 350000
24 350000
25 350000
26 350000
27 350000
28 350000
29 350000
30 400000
31 400000
32 400000
33 400000
34 400000
35 400000
36 400000
37 450000
38 450000
39 450000
40 450000
41 450000
42 450000
43 450000
44 450000
45 450000
46 450000
47 450000
48 500000
49 500000
50 500000
51 500000
52 500000
53 500000
54 500000
55 550000
56 600000
57 600000
58 600000
59 600000
60 600000
61 650000
62 650000
63 650000
64 650000
65 650000
66 700000
67 700000
68 700000
69 700000
70 700000
71 700000
72 750000
73 800000
74 800000
75 800000
76 800000
77 850000
78 850000
79 850000
80 850000
81 900000
82 950000
83 950000
84 1000000
85 1000000
86 1000000
87 1000000
88 1000000
89 1050000
90 1100000
91 1100000
92 1100000
93 1100000
94 1200000
95 1200000
96 1250000
97 1300000
98 1350000
99 1350000
100 1400000
101 1450000
102 1500000
103 1500000
104 1650000
105 1700000
106 1700000
107 2050000
108 2100000
109 2200000
110 2200000
111 2800000
112 74050000
113 74100000
114 74100000
115 74100000
116 74100000
117 74100000
118 74100000
119 74100000
120 74100000
};
\addplot+[color=red, mark=none]  table {
1 200002
2 200002
3 200002
4 200002
5 200002
6 200002
7 250002
8 250002
9 250002
10 250003
11 250003
12 350002
13 350002
14 350002
15 350002
16 350002
17 400002
18 400002
19 400002
20 400002
21 400002
22 400002
23 400002
24 400002
25 400002
26 400002
27 400002
28 400002
29 400002
30 400002
31 400003
32 450002
33 450002
34 450002
35 450002
36 450002
37 450003
38 450003
39 500002
40 500003
41 550002
42 550002
43 550002
44 550002
45 550002
46 550002
47 550002
48 550002
49 550002
50 550003
51 600002
52 600002
53 600002
54 600002
55 600003
56 600003
57 600003
58 650002
59 650002
60 650002
61 700002
62 700002
63 700003
64 700004
65 750002
66 750002
67 750002
68 750003
69 850002
70 850002
71 850002
72 850003
73 850004
74 900002
75 900003
76 900003
77 950002
78 950003
79 1050002
80 1050003
81 1050003
82 1100003
83 1100004
84 1150003
85 1150003
86 1250003
87 1250003
88 1300003
89 1350003
90 1350004
91 1400003
92 1400003
93 1400004
94 1400004
95 1500003
96 1550003
97 1650003
98 1650004
99 1650004
100 1700003
101 1700004
102 1750003
103 1800004
104 1850003
105 1900003
106 2200004
107 2250003
108 2250003
109 2500004
110 2600004
111 3150004
112 85200002
113 85250002
114 85350002
115 85350002
116 85800003
117 86050002
118 86100003
119 86150003
120 86400003
};
\end{axis}
\end{tikzpicture} \\
                                                           & \Genseven-1k                                      & \Genseven-10k                                      & \Genseven-50k                                      \\
            \rotatebox{90}{\hspace{0.5cm}Query Times (ms)} & \begin{tikzpicture}[xscale=0.25, yscale=0.20]
\pgfplotsset{compat=newest}

\begin{axis}[
    width=15cm,
    height=15cm,
    enlargelimits=false,
    ymode=log,
    ytick pos=left,
    ylabel={\huge Time in ms},
    xmin=0,
    xmax=120,
    xtick={1,20,40,60,80,100,120},
    xtick pos=left,
    xlabel={\huge \#Queries},
    tick label style={font=\huge},
    tick align=outside,
    legend cell align=left,
    legend style={at={(0,1)}, anchor=north west, font=\huge}
]
\addplot+[color=black, mark=none] table {
0 5610
120 5610
};
\addplot+[color=green, mark=none] table {
1 4138
2 4153
3 4155
4 4166
5 4167
6 4167
7 4170
8 4174
9 4181
10 4181
11 4182
12 4184
13 4184
14 4185
15 4187
16 4188
17 4191
18 4192
19 4193
20 4194
21 4194
22 4195
23 4196
24 4196
25 4196
26 4196
27 4196
28 4198
29 4199
30 4199
31 4200
32 4200
33 4200
34 4201
35 4203
36 4204
37 4206
38 4208
39 4209
40 4209
41 4211
42 4211
43 4216
44 4217
45 4219
46 4220
47 4225
48 4226
49 4228
50 4228
51 4252
52 4262
53 4262
54 4263
55 4270
56 4275
57 4290
58 4320
59 4374
60 4515
61 4589
62 4960
63 4971
64 4979
65 5006
66 5012
67 5014
68 5024
69 5027
70 5029
71 5030
72 5031
73 5043
74 5045
75 5057
76 5061
77 5067
78 5071
79 5073
80 5108
81 5120
82 5121
83 5132
84 5171
85 5180
86 5202
87 5229
88 5234
89 5248
90 5252
91 5262
92 5278
93 5318
94 5320
95 5372
96 5394
97 5488
98 5493
99 5509
100 5511
101 5511
102 5515
103 5531
104 5573
105 5579
106 5582
107 5678
108 5700
109 5707
110 5750
111 5780
112 5791
113 5833
114 5894
115 5942
116 6032
117 6057
118 6067
119 6077
120 7478
};
\addplot+[color=blue, mark=none]  table {
1 36
2 36
3 36
4 36
5 36
6 36
7 37
8 37
9 37
10 37
11 37
12 37
13 38
14 38
15 38
16 38
17 38
18 38
19 38
20 39
21 39
22 39
23 39
24 39
25 39
26 39
27 39
28 39
29 40
30 40
31 40
32 40
33 41
34 41
35 41
36 41
37 41
38 41
39 41
40 41
41 41
42 41
43 42
44 42
45 42
46 42
47 43
48 43
49 43
50 43
51 44
52 44
53 44
54 44
55 44
56 44
57 44
58 45
59 45
60 45
61 45
62 45
63 45
64 45
65 45
66 46
67 46
68 47
69 48
70 48
71 48
72 48
73 49
74 49
75 49
76 50
77 51
78 53
79 53
80 54
81 54
82 55
83 55
84 56
85 57
86 57
87 59
88 59
89 59
90 59
91 63
92 64
93 64
94 65
95 66
96 67
97 68
98 70
99 70
100 72
101 89
102 4365
103 4395
104 4404
105 4414
106 4424
107 4433
108 4464
109 4500
110 4511
111 4533
112 4612
113 4678
114 4707
115 4789
116 4839
117 4856
118 4902
119 4968
120 5096
};
\addplot+[color=red, mark=none]  table {
1 36
2 36
3 36
4 37
5 37
6 37
7 38
8 38
9 38
10 38
11 38
12 38
13 39
14 39
15 39
16 39
17 39
18 39
19 39
20 39
21 39
22 40
23 40
24 40
25 40
26 40
27 40
28 40
29 40
30 40
31 41
32 41
33 41
34 41
35 41
36 41
37 41
38 41
39 42
40 42
41 42
42 42
43 42
44 43
45 43
46 43
47 43
48 43
49 43
50 44
51 44
52 44
53 44
54 44
55 44
56 45
57 45
58 45
59 46
60 46
61 46
62 46
63 46
64 47
65 47
66 47
67 47
68 47
69 47
70 47
71 48
72 48
73 49
74 49
75 50
76 50
77 51
78 51
79 52
80 53
81 54
82 55
83 57
84 57
85 58
86 60
87 60
88 61
89 62
90 63
91 64
92 66
93 67
94 69
95 69
96 70
97 72
98 72
99 73
100 74
101 90
102 5437
103 5440
104 5446
105 5460
106 5470
107 5486
108 5488
109 5496
110 5539
111 5542
112 5569
113 5578
114 5784
115 5832
116 5900
117 6053
118 6054
119 6228
120 6581
};
\end{axis}
\end{tikzpicture}        & \begin{tikzpicture}[xscale=0.25, yscale=0.20]
\pgfplotsset{compat=newest}

\begin{axis}[
    width=15cm,
    height=15cm,
    enlargelimits=false,
    ymode=log,
    ytick pos=left,
    ylabel={\huge Time in ms},
    xmin=0,
    xmax=120,
    xtick={1,20,40,60,80,100,120},
    xtick pos=left,
    xlabel={\huge \#Queries},
    tick label style={font=\huge},
    tick align=outside,
    legend cell align=left,
    legend style={at={(0,1)}, anchor=north west, font=\huge}
]
\addplot+[color=black, mark=none] table {
0 73805
120 73805
};
\addplot+[color=green, mark=none] table {
1 53922
2 54003
3 54045
4 54065
5 54092
6 54095
7 54102
8 54125
9 54169
10 54217
11 54239
12 54378
13 54379
14 54409
15 54422
16 54436
17 54442
18 54450
19 54451
20 54458
21 54477
22 54479
23 54487
24 54499
25 54500
26 54509
27 54525
28 54530
29 54550
30 54573
31 54580
32 54592
33 54621
34 54627
35 54673
36 54711
37 54724
38 54801
39 54804
40 54823
41 54838
42 54838
43 54844
44 54948
45 54958
46 54964
47 54974
48 54988
49 55016
50 55038
51 55075
52 55216
53 55459
54 55601
55 55785
56 55977
57 56715
58 56866
59 58071
60 58303
61 59245
62 60553
63 60608
64 60670
65 60700
66 60887
67 60946
68 61045
69 61088
70 61112
71 61226
72 61245
73 61412
74 61526
75 61772
76 61782
77 61785
78 61824
79 62045
80 62105
81 62164
82 62449
83 62647
84 63227
85 63422
86 63539
87 63646
88 63724
89 63984
90 64117
91 64242
92 64505
93 64508
94 64521
95 64770
96 65029
97 65059
98 65883
99 68210
100 69242
101 69327
102 69490
103 69624
104 69831
105 69988
106 70160
107 70245
108 70431
109 70783
110 70809
111 70943
112 71015
113 71024
114 71086
115 75825
116 76043
117 77039
118 77153
119 77735
120 98783
};
\addplot+[color=blue, mark=none]  table {
1 52
2 54
3 54
4 54
5 54
6 54
7 54
8 55
9 55
10 56
11 57
12 57
13 58
14 61
15 62
16 63
17 67
18 68
19 68
20 70
21 70
22 70
23 71
24 72
25 73
26 73
27 73
28 73
29 74
30 77
31 78
32 78
33 79
34 80
35 80
36 81
37 81
38 82
39 83
40 83
41 84
42 85
43 85
44 86
45 87
46 88
47 88
48 88
49 91
50 92
51 94
52 94
53 95
54 97
55 98
56 98
57 99
58 100
59 100
60 100
61 101
62 102
63 102
64 103
65 105
66 106
67 106
68 108
69 108
70 111
71 111
72 112
73 114
74 118
75 118
76 119
77 119
78 120
79 126
80 135
81 136
82 140
83 142
84 142
85 144
86 145
87 151
88 153
89 155
90 159
91 160
92 161
93 184
94 188
95 197
96 209
97 221
98 222
99 222
100 225
101 329
102 50798
103 50835
104 51067
105 51299
106 51350
107 51432
108 51495
109 51526
110 51881
111 51993
112 52104
113 52976
114 53034
115 53104
116 53144
117 54924
118 59369
119 59672
120 62045
};
\addplot+[color=red, mark=none]  table {
1 55
2 55
3 55
4 55
5 57
6 57
7 58
8 58
9 58
10 60
11 61
12 61
13 62
14 63
15 65
16 68
17 69
18 70
19 71
20 72
21 72
22 73
23 74
24 75
25 77
26 77
27 78
28 80
29 82
30 82
31 83
32 83
33 84
34 84
35 84
36 84
37 84
38 86
39 87
40 89
41 92
42 92
43 96
44 97
45 98
46 98
47 98
48 98
49 99
50 99
51 99
52 102
53 104
54 104
55 104
56 104
57 104
58 106
59 108
60 109
61 109
62 110
63 110
64 113
65 113
66 114
67 115
68 115
69 115
70 117
71 118
72 121
73 129
74 129
75 131
76 134
77 138
78 139
79 139
80 140
81 149
82 150
83 157
84 158
85 158
86 160
87 164
88 168
89 184
90 185
91 192
92 200
93 200
94 221
95 226
96 230
97 231
98 239
99 246
100 258
101 312
102 62144
103 62220
104 62488
105 62570
106 62865
107 62991
108 63030
109 63061
110 63106
111 63253
112 63439
113 63972
114 64202
115 65743
116 69977
117 70016
118 70204
119 70821
120 73876
};
\end{axis}
\end{tikzpicture}        & \begin{tikzpicture}[xscale=0.25, yscale=0.20]
\pgfplotsset{compat=newest}

\begin{axis}[
    width=15cm,
    height=15cm,
    enlargelimits=false,
    ymode=log,
    ytick pos=left,
    ylabel={\huge Time in ms},
    xmin=0,
    xmax=120,
    ymax=1000000,
    xtick={1,20,40,60,80,100,120},
    xtick pos=left,
    xlabel={\huge \#Queries},
    tick label style={font=\huge},
    tick align=outside,
    legend cell align=left,
    legend style={at={(0,1)}, anchor=north west, font=\huge}
]
\addplot+[color=black, mark=none] table {
0 397359
120 397359
};
\addplot+[color=green, mark=none] table {
1 265600
2 265774
3 266382
4 266979
5 267035
6 267112
7 267135
8 267237
9 267396
10 267400
11 267484
12 267678
13 267832
14 268123
15 268166
16 268324
17 268354
18 268462
19 268634
20 268706
21 268710
22 268887
23 268962
24 269241
25 269364
26 269584
27 269685
28 269703
29 269727
30 269789
31 269892
32 269951
33 270133
34 270334
35 270360
36 270399
37 270430
38 270495
39 270727
40 270728
41 270965
42 271159
43 271402
44 271436
45 271694
46 271815
47 272354
48 273112
49 273548
50 274154
51 274622
52 275147
53 275644
54 275678
55 277280
56 278204
57 279516
58 280222
59 281118
60 295031
61 295455
62 296189
63 296431
64 296875
65 297062
66 297377
67 297575
68 298071
69 299913
70 300420
71 300668
72 301375
73 301591
74 301957
75 302231
76 302601
77 302775
78 302943
79 304678
80 305808
81 307853
82 308000
83 308059
84 309499
85 309867
86 309966
87 311077
88 311102
89 312147
90 313640
91 314421
92 315049
93 315583
94 315758
95 316387
96 317244
97 320549
98 321931
99 322047
100 323604
101 324352
102 324954
103 325520
104 325856
105 326419
106 327609
107 332253
108 333813
109 337898
110 340258
111 341597
112 362149
113 363262
114 366328
115 374353
116 375297
117 378548
118 378640
119 379253
120 448892
};
\addplot+[color=blue, mark=none]  table {
1 67
2 91
3 93
4 93
5 94
6 94
7 94
8 95
9 95
10 96
11 96
12 97
13 98
14 107
15 109
16 110
17 121
18 123
19 124
20 125
21 125
22 126
23 130
24 130
25 131
26 139
27 140
28 148
29 152
30 155
31 155
32 157
33 158
34 159
35 162
36 174
37 179
38 183
39 183
40 183
41 190
42 193
43 194
44 199
45 200
46 203
47 209
48 210
49 218
50 221
51 221
52 222
53 239
54 241
55 242
56 246
57 248
58 251
59 251
60 259
61 262
62 272
63 274
64 277
65 287
66 292
67 316
68 325
69 340
70 346
71 352
72 353
73 358
74 367
75 368
76 375
77 383
78 387
79 391
80 431
81 436
82 478
83 512
84 524
85 553
86 580
87 580
88 582
89 587
90 588
91 641
92 667
93 917
94 920
95 937
96 964
97 994
98 1122
99 1135
100 1160
101 2107
102 240702
103 241069
104 242368
105 243249
106 243297
107 243723
108 244076
109 245180
110 245495
111 246146
112 246517
113 248463
114 249582
115 249615
116 250324
117 252387
118 252601
119 255104
120 281570
};
\addplot+[color=red, mark=none]  table {
1 96
2 99
3 99
4 99
5 100
6 100
7 100
8 100
9 100
10 101
11 101
12 101
13 103
14 112
15 114
16 114
17 127
18 129
19 131
20 133
21 133
22 139
23 143
24 143
25 159
26 163
27 165
28 174
29 180
30 182
31 183
32 184
33 184
34 185
35 185
36 185
37 187
38 190
39 192
40 206
41 206
42 216
43 225
44 230
45 238
46 239
47 242
48 246
49 247
50 250
51 256
52 256
53 259
54 263
55 266
56 276
57 280
58 284
59 284
60 284
61 287
62 290
63 295
64 309
65 313
66 334
67 343
68 344
69 351
70 359
71 359
72 362
73 375
74 421
75 431
76 436
77 446
78 457
79 459
80 475
81 489
82 496
83 509
84 579
85 584
86 599
87 665
88 668
89 706
90 869
91 958
92 962
93 1019
94 1060
95 1069
96 1145
97 1170
98 1256
99 1273
100 1325
101 1729
102 282457
103 282717
104 283645
105 284403
106 284729
107 286034
108 287553
109 287968
110 288751
111 288838
112 294148
113 296385
114 303193
115 306102
116 318589
117 319748
118 321122
119 321863
120 325115
};
\end{axis}
\end{tikzpicture}        \\
            \rotatebox{90}{\hspace{0.6cm}\# Derived Facts} & \begin{tikzpicture}[xscale=0.25, yscale=0.20]
\pgfplotsset{compat=newest}

\begin{axis}[
    width=15cm,
    height=15cm,
    enlargelimits=false,
    ymode=log,
    xtick pos=left,
    ylabel={\huge \#Total facts},
    xmin=0,
    xmax=120,
    xtick={1,20,40,60,80,100,120},
    xtick pos=left,
    xlabel={\huge \#Queries},
    tick label style={font=\huge},
    tick align=outside,
    legend cell align=left,
    legend style={at={(0,1)}, anchor=north west, font=\huge}
]
\addplot+[color=black, mark=none] table {
0 2787000
120 2787000
};
\addplot+[color=green, mark=none] table {
1 2357002
2 2363002
3 2363002
4 2363002
5 2365002
6 2366002
7 2368002
8 2369002
9 2371002
10 2375002
11 2375002
12 2375002
13 2375002
14 2375002
15 2375002
16 2375002
17 2375002
18 2375002
19 2375002
20 2375002
21 2375002
22 2375002
23 2375002
24 2375002
25 2376003
26 2377002
27 2377002
28 2377002
29 2377002
30 2377002
31 2377002
32 2378002
33 2378002
34 2378002
35 2378002
36 2378002
37 2378002
38 2378002
39 2378002
40 2378003
41 2378003
42 2378003
43 2379003
44 2380002
45 2380002
46 2380002
47 2380002
48 2380002
49 2380002
50 2380003
51 2381002
52 2381002
53 2381003
54 2382002
55 2382002
56 2382003
57 2386004
58 2387004
59 2394003
60 2396003
61 2397003
62 2997003
63 3002002
64 3007003
65 3010002
66 3012002
67 3012004
68 3015002
69 3015003
70 3016003
71 3020003
72 3021003
73 3024004
74 3024005
75 3025003
76 3027003
77 3028003
78 3029003
79 3029003
80 3029003
81 3032005
82 3036003
83 3037003
84 3039002
85 3040005
86 3041002
87 3043003
88 3045003
89 3046002
90 3049002
91 3049003
92 3050002
93 3050003
94 3051002
95 3053002
96 3054003
97 3056002
98 3056003
99 3056004
100 3058002
101 3058004
102 3059002
103 3060002
104 3061003
105 3070002
106 3075003
107 3076003
108 3077003
109 3672003
110 3686003
111 3690002
112 3697003
113 3699002
114 3703004
115 3714002
116 3723003
117 3733003
118 3736003
119 3738004
120 5104002
};
\addplot+[color=blue, mark=none]  table {
1 4000
2 4000
3 4000
4 4000
5 4000
6 4000
7 4000
8 4000
9 4000
10 4000
11 4000
12 4000
13 4000
14 4000
15 4000
16 4000
17 5000
18 5000
19 5000
20 5000
21 5000
22 6000
23 6000
24 6000
25 6000
26 6000
27 6000
28 6000
29 7000
30 7000
31 7000
32 7000
33 7000
34 7000
35 7000
36 8000
37 8000
38 8000
39 8000
40 8000
41 8000
42 9000
43 9000
44 9000
45 9000
46 9000
47 9000
48 10000
49 10000
50 10000
51 11000
52 11000
53 11000
54 12000
55 12000
56 12000
57 12000
58 12000
59 12000
60 12000
61 12000
62 12000
63 13000
64 13000
65 13000
66 14000
67 14000
68 15000
69 15000
70 15000
71 15000
72 16000
73 16000
74 17000
75 17000
76 17000
77 18000
78 18000
79 19000
80 20000
81 21000
82 21000
83 22000
84 22000
85 22000
86 23000
87 23000
88 23000
89 24000
90 25000
91 26000
92 27000
93 31000
94 32000
95 32000
96 34000
97 38000
98 39000
99 39000
100 40000
101 61000
102 2768000
103 2768000
104 2768000
105 2769000
106 2769000
107 2769000
108 2769000
109 2769000
110 2769000
111 2769000
112 2769000
113 2769000
114 2769000
115 2769000
116 2769000
117 2769000
118 2769000
119 2769000
120 2769000
};
\addplot+[color=red, mark=none]  table {
1 4002
2 4002
3 4002
4 4002
5 4002
6 4002
7 4002
8 4002
9 4002
10 4002
11 4002
12 4002
13 4002
14 4002
15 4002
16 4002
17 5002
18 5002
19 5002
20 5002
21 5002
22 6002
23 6003
24 7002
25 7002
26 7002
27 7002
28 7002
29 8002
30 8002
31 8002
32 8002
33 8002
34 8002
35 8002
36 8002
37 8002
38 8003
39 8003
40 9002
41 9002
42 9003
43 10002
44 11002
45 11002
46 11002
47 11002
48 11002
49 11002
50 11003
51 12002
52 12002
53 12002
54 13002
55 13002
56 13002
57 13002
58 13003
59 13003
60 13003
61 13003
62 14002
63 14002
64 14002
65 14003
66 14003
67 15002
68 15002
69 15002
70 15003
71 15003
72 17003
73 18002
74 18003
75 19004
76 20002
77 20003
78 21002
79 21002
80 21002
81 21002
82 21003
83 21004
84 25003
85 26003
86 27003
87 27003
88 27004
89 28004
90 30003
91 31003
92 33003
93 35003
94 38003
95 38004
96 40003
97 40005
98 41003
99 43005
100 46003
101 55003
102 2985002
103 2988003
104 2989002
105 2989003
106 2989003
107 2990002
108 2991002
109 2992003
110 2994003
111 2994003
112 2997003
113 3005003
114 3009003
115 3012003
116 3012003
117 3664002
118 3670003
119 3674004
120 3687004
};
\end{axis}
\end{tikzpicture} & \begin{tikzpicture}[xscale=0.25, yscale=0.20]
\pgfplotsset{compat=newest}

\begin{axis}[
    width=15cm,
    height=15cm,
    enlargelimits=false,
    ymode=log,
    xtick pos=left,
    ylabel={\huge \#Total facts},
    xmin=0,
    xmax=120,
    xtick={1,20,40,60,80,100,120},
    xtick pos=left,
    xlabel={\huge \#Queries},
    tick label style={font=\huge},
    tick align=outside,
    legend cell align=left,
    legend style={at={(0,1)}, anchor=north west, font=\huge}
]
\addplot+[color=black, mark=none] table {
0 27870000
120 27870000
};
\addplot+[color=green, mark=none] table {
1 23570002
2 23630002
3 23630002
4 23630002
5 23650002
6 23660002
7 23680002
8 23690002
9 23710002
10 23750002
11 23750002
12 23750002
13 23750002
14 23750002
15 23750002
16 23750002
17 23750002
18 23750002
19 23750002
20 23750002
21 23750002
22 23750002
23 23750002
24 23750002
25 23760003
26 23770002
27 23770002
28 23770002
29 23770002
30 23770002
31 23770002
32 23780002
33 23780002
34 23780002
35 23780002
36 23780002
37 23780002
38 23780002
39 23780002
40 23780003
41 23780003
42 23780003
43 23790003
44 23800002
45 23800002
46 23800002
47 23800002
48 23800002
49 23800002
50 23800003
51 23810002
52 23810002
53 23810003
54 23820002
55 23820002
56 23820003
57 23860004
58 23870004
59 23940003
60 23960003
61 23970003
62 29970003
63 30020002
64 30070003
65 30100002
66 30120002
67 30120004
68 30150002
69 30150003
70 30160003
71 30200003
72 30210003
73 30240004
74 30240005
75 30250003
76 30270003
77 30280003
78 30290003
79 30290003
80 30290003
81 30320005
82 30360003
83 30370003
84 30390002
85 30400005
86 30410002
87 30430003
88 30450003
89 30460002
90 30490002
91 30490003
92 30500002
93 30500003
94 30510002
95 30530002
96 30540003
97 30560002
98 30560003
99 30560004
100 30580002
101 30580004
102 30590002
103 30600002
104 30610003
105 30700002
106 30750003
107 30760003
108 30770003
109 36720003
110 36860003
111 36900002
112 36970003
113 36990002
114 37030004
115 37140002
116 37230003
117 37330003
118 37360003
119 37380004
120 51040002
};
\addplot+[color=blue, mark=none]  table {
1 40000
2 40000
3 40000
4 40000
5 40000
6 40000
7 40000
8 40000
9 40000
10 40000
11 40000
12 40000
13 40000
14 40000
15 40000
16 40000
17 50000
18 50000
19 50000
20 50000
21 50000
22 60000
23 60000
24 60000
25 60000
26 60000
27 60000
28 60000
29 70000
30 70000
31 70000
32 70000
33 70000
34 70000
35 70000
36 80000
37 80000
38 80000
39 80000
40 80000
41 80000
42 90000
43 90000
44 90000
45 90000
46 90000
47 90000
48 100000
49 100000
50 100000
51 110000
52 110000
53 110000
54 120000
55 120000
56 120000
57 120000
58 120000
59 120000
60 120000
61 120000
62 120000
63 130000
64 130000
65 130000
66 140000
67 140000
68 150000
69 150000
70 150000
71 150000
72 160000
73 160000
74 170000
75 170000
76 170000
77 180000
78 180000
79 190000
80 200000
81 210000
82 210000
83 220000
84 220000
85 220000
86 230000
87 230000
88 230000
89 240000
90 250000
91 260000
92 270000
93 310000
94 320000
95 320000
96 340000
97 380000
98 390000
99 390000
100 400000
101 610000
102 27680000
103 27680000
104 27680000
105 27690000
106 27690000
107 27690000
108 27690000
109 27690000
110 27690000
111 27690000
112 27690000
113 27690000
114 27690000
115 27690000
116 27690000
117 27690000
118 27690000
119 27690000
120 27690000
};
\addplot+[color=red, mark=none]  table {
1 40002
2 40002
3 40002
4 40002
5 40002
6 40002
7 40002
8 40002
9 40002
10 40002
11 40002
12 40002
13 40002
14 40002
15 40002
16 40002
17 50002
18 50002
19 50002
20 50002
21 50002
22 60002
23 60003
24 70002
25 70002
26 70002
27 70002
28 70002
29 80002
30 80002
31 80002
32 80002
33 80002
34 80002
35 80002
36 80002
37 80002
38 80003
39 80003
40 90002
41 90002
42 90003
43 100002
44 110002
45 110002
46 110002
47 110002
48 110002
49 110002
50 110003
51 120002
52 120002
53 120002
54 130002
55 130002
56 130002
57 130002
58 130003
59 130003
60 130003
61 130003
62 140002
63 140002
64 140002
65 140003
66 140003
67 150002
68 150002
69 150002
70 150003
71 150003
72 170003
73 180002
74 180003
75 190004
76 200002
77 200003
78 210002
79 210002
80 210002
81 210002
82 210003
83 210004
84 250003
85 260003
86 270003
87 270003
88 270004
89 280004
90 300003
91 310003
92 330003
93 350003
94 380003
95 380004
96 400003
97 400005
98 410003
99 430005
100 460003
101 550003
102 29850002
103 29880003
104 29890002
105 29890003
106 29890003
107 29900002
108 29910002
109 29920003
110 29940003
111 29940003
112 29970003
113 30050003
114 30090003
115 30120003
116 30120003
117 36640002
118 36700003
119 36740004
120 36870004
};
\end{axis}
\end{tikzpicture} & \begin{tikzpicture}[xscale=0.25, yscale=0.20]
\pgfplotsset{compat=newest}

\begin{axis}[
    width=15cm,
    height=15cm,
    enlargelimits=false,
    ymode=log,
    xtick pos=left,
    ylabel={\huge \#Total facts},
    xmin=0,
    xmax=120,
    xtick={1,20,40,60,80,100,120},
    xtick pos=left,
    xlabel={\huge \#Queries},
    tick label style={font=\huge},
    tick align=outside,
    legend cell align=left,
    legend style={at={(0,1)}, anchor=north west, font=\huge}
]
\addplot+[color=black, mark=none] table {
0 139350000
120 139350000
};
\addplot+[color=green, mark=none] table {
1 117850002
2 118150002
3 118150002
4 118150002
5 118250002
6 118300002
7 118400002
8 118450002
9 118550002
10 118750002
11 118750002
12 118750002
13 118750002
14 118750002
15 118750002
16 118750002
17 118750002
18 118750002
19 118750002
20 118750002
21 118750002
22 118750002
23 118750002
24 118750002
25 118800003
26 118850002
27 118850002
28 118850002
29 118850002
30 118850002
31 118850002
32 118900002
33 118900002
34 118900002
35 118900002
36 118900002
37 118900002
38 118900002
39 118900002
40 118900003
41 118900003
42 118900003
43 118950003
44 119000002
45 119000002
46 119000002
47 119000002
48 119000002
49 119000002
50 119000003
51 119050002
52 119050002
53 119050003
54 119100002
55 119100002
56 119100003
57 119300004
58 119350004
59 119700003
60 119800003
61 119850003
62 149850003
63 150100002
64 150350003
65 150500002
66 150600002
67 150600004
68 150750002
69 150750003
70 150800003
71 151000003
72 151050003
73 151200004
74 151200005
75 151250003
76 151350003
77 151400003
78 151450003
79 151450003
80 151450003
81 151600005
82 151800003
83 151850003
84 151950002
85 152000005
86 152050002
87 152150003
88 152250003
89 152300002
90 152450002
91 152450003
92 152500002
93 152500003
94 152550002
95 152650002
96 152700003
97 152800002
98 152800003
99 152800004
100 152900002
101 152900004
102 152950002
103 153000002
104 153050003
105 153500002
106 153750003
107 153800003
108 153850003
109 183600003
110 184300003
111 184500002
112 184850003
113 184950002
114 185150004
115 185700002
116 186150003
117 186650003
118 186800003
119 186900004
120 255200002
};
\addplot+[color=blue, mark=none]  table {
1 200000
2 200000
3 200000
4 200000
5 200000
6 200000
7 200000
8 200000
9 200000
10 200000
11 200000
12 200000
13 200000
14 200000
15 200000
16 200000
17 250000
18 250000
19 250000
20 250000
21 250000
22 300000
23 300000
24 300000
25 300000
26 300000
27 300000
28 300000
29 350000
30 350000
31 350000
32 350000
33 350000
34 350000
35 350000
36 400000
37 400000
38 400000
39 400000
40 400000
41 400000
42 450000
43 450000
44 450000
45 450000
46 450000
47 450000
48 500000
49 500000
50 500000
51 550000
52 550000
53 550000
54 600000
55 600000
56 600000
57 600000
58 600000
59 600000
60 600000
61 600000
62 600000
63 650000
64 650000
65 650000
66 700000
67 700000
68 750000
69 750000
70 750000
71 750000
72 800000
73 800000
74 850000
75 850000
76 850000
77 900000
78 900000
79 950000
80 1000000
81 1050000
82 1050000
83 1100000
84 1100000
85 1100000
86 1150000
87 1150000
88 1150000
89 1200000
90 1250000
91 1300000
92 1350000
93 1550000
94 1600000
95 1600000
96 1700000
97 1900000
98 1950000
99 1950000
100 2000000
101 3050000
102 138400000
103 138400000
104 138400000
105 138450000
106 138450000
107 138450000
108 138450000
109 138450000
110 138450000
111 138450000
112 138450000
113 138450000
114 138450000
115 138450000
116 138450000
117 138450000
118 138450000
119 138450000
120 138450000
};
\addplot+[color=red, mark=none]  table {
1 200002
2 200002
3 200002
4 200002
5 200002
6 200002
7 200002
8 200002
9 200002
10 200002
11 200002
12 200002
13 200002
14 200002
15 200002
16 200002
17 250002
18 250002
19 250002
20 250002
21 250002
22 300002
23 300003
24 350002
25 350002
26 350002
27 350002
28 350002
29 400002
30 400002
31 400002
32 400002
33 400002
34 400002
35 400002
36 400002
37 400002
38 400003
39 400003
40 450002
41 450002
42 450003
43 500002
44 550002
45 550002
46 550002
47 550002
48 550002
49 550002
50 550003
51 600002
52 600002
53 600002
54 650002
55 650002
56 650002
57 650002
58 650003
59 650003
60 650003
61 650003
62 700002
63 700002
64 700002
65 700003
66 700003
67 750002
68 750002
69 750002
70 750003
71 750003
72 850003
73 900002
74 900003
75 950004
76 1000002
77 1000003
78 1050002
79 1050002
80 1050002
81 1050002
82 1050003
83 1050004
84 1250003
85 1300003
86 1350003
87 1350003
88 1350004
89 1400004
90 1500003
91 1550003
92 1650003
93 1750003
94 1900003
95 1900004
96 2000003
97 2000005
98 2050003
99 2150005
100 2300003
101 2750003
102 149250002
103 149400003
104 149450002
105 149450003
106 149450003
107 149500002
108 149550002
109 149600003
110 149700003
111 149700003
112 149850003
113 150250003
114 150450003
115 150600003
116 150600003
117 183200002
118 183500003
119 183700004
120 184350004
};
\end{axis}
\end{tikzpicture} \\
                                                           & \Geneight-1k                                      & \Geneight-10k                                      & \Geneight-50k                                      \\
            \rotatebox{90}{\hspace{0.5cm}Query Times (ms)} & \begin{tikzpicture}[xscale=0.25, yscale=0.20]
\pgfplotsset{compat=newest}

\begin{axis}[
    width=15cm,
    height=15cm,
    enlargelimits=false,
    ymode=log,
    ytick pos=left,
    ylabel={\huge Time in ms},
    xmin=0,
    xmax=120,
    xtick={1,20,40,60,80,100,120},
    xtick pos=left,
    xlabel={\huge \#Queries},
    tick label style={font=\huge},
    tick align=outside,
    legend cell align=left,
    legend style={at={(0,1)}, anchor=north west, font=\huge}
]
\addplot+[color=black, mark=none] table {
0 3915
120 3915
};
\addplot+[color=green, mark=none] table {
1 4312
2 4321
3 4322
4 4325
5 4326
6 4327
7 4330
8 4332
9 4339
10 4339
11 4340
12 4340
13 4341
14 4343
15 4344
16 4347
17 4348
18 4351
19 4351
20 4353
21 4354
22 4355
23 4356
24 4356
25 4357
26 4357
27 4357
28 4360
29 4360
30 4361
31 4365
32 4365
33 4365
34 4367
35 4369
36 4369
37 4369
38 4370
39 4370
40 4371
41 4372
42 4373
43 4373
44 4373
45 4374
46 4375
47 4375
48 4375
49 4376
50 4376
51 4376
52 4376
53 4376
54 4377
55 4378
56 4379
57 4381
58 4381
59 4381
60 4382
61 4382
62 4382
63 4382
64 4383
65 4386
66 4387
67 4388
68 4392
69 4392
70 4393
71 4394
72 4396
73 4397
74 4397
75 4398
76 4398
77 4402
78 4402
79 4403
80 4404
81 4411
82 4416
83 4424
84 4431
85 4435
86 4438
87 4441
88 4453
89 4453
90 4465
91 4478
92 4502
93 4502
94 4508
95 4511
96 4512
97 4520
98 4531
99 4531
100 4532
101 4538
102 4540
103 4581
104 5012
105 5027
106 5032
107 5061
108 5093
109 5101
110 5101
111 5103
112 5129
113 5144
114 5164
115 5165
116 5177
117 5186
118 5232
119 5234
120 5413
};
\addplot+[color=blue, mark=none]  table {
1 24
2 24
3 24
4 24
5 25
6 25
7 25
8 25
9 26
10 26
11 26
12 26
13 26
14 26
15 27
16 27
17 27
18 27
19 27
20 27
21 27
22 27
23 27
24 27
25 27
26 28
27 28
28 28
29 28
30 28
31 28
32 28
33 28
34 28
35 28
36 28
37 28
38 28
39 28
40 28
41 28
42 28
43 28
44 29
45 29
46 29
47 29
48 29
49 29
50 29
51 29
52 29
53 29
54 29
55 29
56 29
57 29
58 29
59 29
60 29
61 30
62 30
63 30
64 30
65 30
66 30
67 30
68 30
69 30
70 31
71 31
72 31
73 31
74 32
75 32
76 32
77 32
78 33
79 33
80 33
81 33
82 33
83 33
84 34
85 34
86 34
87 34
88 34
89 34
90 35
91 35
92 35
93 35
94 36
95 37
96 37
97 38
98 40
99 42
100 43
101 43
102 43
103 43
104 44
105 44
106 44
107 45
108 45
109 47
110 47
111 53
112 58
113 59
114 68
115 3276
116 3312
117 3320
118 3334
119 3511
120 3849
};
\addplot+[color=red, mark=none]  table {
1 25
2 25
3 25
4 25
5 26
6 26
7 26
8 26
9 26
10 26
11 26
12 27
13 27
14 27
15 27
16 27
17 27
18 27
19 27
20 28
21 28
22 28
23 28
24 29
25 29
26 29
27 29
28 29
29 29
30 29
31 29
32 30
33 30
34 30
35 30
36 30
37 30
38 30
39 30
40 30
41 31
42 31
43 31
44 31
45 31
46 31
47 31
48 31
49 31
50 31
51 31
52 31
53 31
54 32
55 32
56 32
57 32
58 32
59 32
60 32
61 32
62 32
63 32
64 32
65 32
66 32
67 32
68 33
69 33
70 33
71 33
72 33
73 33
74 34
75 34
76 34
77 35
78 35
79 35
80 36
81 36
82 36
83 36
84 36
85 36
86 36
87 36
88 36
89 37
90 37
91 37
92 38
93 39
94 40
95 40
96 40
97 40
98 43
99 43
100 43
101 44
102 47
103 47
104 47
105 49
106 50
107 50
108 51
109 54
110 55
111 56
112 62
113 64
114 67
115 3739
116 3740
117 3774
118 3855
119 4022
120 5121
};
\end{axis}
\end{tikzpicture}        & \begin{tikzpicture}[xscale=0.25, yscale=0.20]
\pgfplotsset{compat=newest}

\begin{axis}[
    width=15cm,
    height=15cm,
    enlargelimits=false,
    ymode=log,
    ytick pos=left,
    ylabel={\huge Time in ms},
    xmin=0,
    xmax=120,
    xtick={1,20,40,60,80,100,120},
    xtick pos=left,
    xlabel={\huge \#Queries},
    tick label style={font=\huge},
    tick align=outside,
    legend cell align=left,
    legend style={at={(0,1)}, anchor=north west, font=\huge}
]
\addplot+[color=black, mark=none] table {
0 55001
120 55001
};
\addplot+[color=green, mark=none] table {
1 56417
2 56437
3 56485
4 56495
5 56505
6 56581
7 56626
8 56626
9 56653
10 56656
11 56659
12 56667
13 56698
14 56742
15 56764
16 56767
17 56772
18 56791
19 56803
20 56816
21 56830
22 56850
23 56881
24 56914
25 56927
26 56942
27 56944
28 56944
29 56949
30 56993
31 56997
32 57001
33 57016
34 57024
35 57042
36 57045
37 57046
38 57066
39 57081
40 57081
41 57092
42 57098
43 57115
44 57117
45 57122
46 57169
47 57171
48 57176
49 57186
50 57195
51 57196
52 57219
53 57220
54 57236
55 57238
56 57244
57 57248
58 57253
59 57268
60 57283
61 57308
62 57333
63 57343
64 57362
65 57364
66 57380
67 57421
68 57422
69 57427
70 57446
71 57458
72 57485
73 57509
74 57516
75 57546
76 57586
77 57622
78 57649
79 57655
80 57655
81 57778
82 57873
83 57941
84 57992
85 58026
86 58028
87 58186
88 58198
89 58300
90 58419
91 58553
92 58698
93 58922
94 59289
95 59310
96 59315
97 60010
98 61348
99 61467
100 61938
101 62186
102 62196
103 62358
104 62362
105 62410
106 62426
107 62582
108 63083
109 63172
110 63319
111 65346
112 65899
113 66091
114 66696
115 67062
116 67681
117 68626
118 69474
119 70565
120 73203
};
\addplot+[color=blue, mark=none]  table {
1 40
2 41
3 41
4 42
5 42
6 43
7 43
8 44
9 44
10 44
11 49
12 50
13 51
14 52
15 55
16 56
17 57
18 59
19 61
20 62
21 64
22 65
23 65
24 66
25 66
26 66
27 67
28 67
29 67
30 67
31 67
32 67
33 68
34 69
35 69
36 70
37 70
38 71
39 71
40 71
41 71
42 71
43 72
44 72
45 72
46 72
47 72
48 72
49 73
50 73
51 73
52 74
53 74
54 75
55 76
56 76
57 77
58 77
59 78
60 78
61 79
62 79
63 79
64 81
65 81
66 82
67 82
68 82
69 83
70 84
71 84
72 86
73 87
74 88
75 89
76 90
77 90
78 90
79 90
80 91
81 91
82 92
83 93
84 94
85 94
86 95
87 95
88 96
89 97
90 97
91 99
92 100
93 102
94 108
95 111
96 122
97 125
98 130
99 134
100 135
101 135
102 136
103 137
104 140
105 145
106 147
107 154
108 161
109 163
110 175
111 206
112 211
113 241
114 294
115 45321
116 45522
117 45920
118 46086
119 51256
120 51665
};
\addplot+[color=red, mark=none]  table {
1 40
2 42
3 43
4 45
5 45
6 46
7 47
8 47
9 48
10 48
11 51
12 52
13 52
14 52
15 56
16 59
17 64
18 66
19 68
20 69
21 71
22 71
23 71
24 71
25 71
26 71
27 72
28 72
29 72
30 72
31 73
32 74
33 74
34 74
35 76
36 76
37 77
38 78
39 79
40 80
41 81
42 81
43 82
44 82
45 83
46 83
47 84
48 84
49 85
50 85
51 85
52 85
53 86
54 86
55 86
56 86
57 87
58 87
59 89
60 89
61 89
62 90
63 90
64 91
65 92
66 92
67 92
68 94
69 94
70 94
71 95
72 96
73 97
74 97
75 97
76 98
77 99
78 100
79 101
80 102
81 105
82 105
83 106
84 108
85 108
86 109
87 112
88 113
89 114
90 116
91 118
92 119
93 120
94 122
95 129
96 135
97 141
98 147
99 147
100 153
101 155
102 157
103 162
104 168
105 169
106 170
107 171
108 180
109 209
110 222
111 267
112 277
113 291
114 292
115 47350
116 47419
117 48047
118 48208
119 48425
120 60177
};
\end{axis}
\end{tikzpicture}        & \begin{tikzpicture}[xscale=0.25, yscale=0.20]
\pgfplotsset{compat=newest}

\begin{axis}[
    width=15cm,
    height=15cm,
    enlargelimits=false,
    ymode=log,
    ytick pos=left,
    ylabel={\huge Time in ms},
    xmin=0,
    xmax=120,
    xtick={1,20,40,60,80,100,120},
    xtick pos=left,
    xlabel={\huge \#Queries},
    tick label style={font=\huge},
    tick align=outside,
    legend cell align=left,
    legend style={at={(0,1)}, anchor=north west, font=\huge}
]
\addplot+[color=black, mark=none] table {
0 288050
120 288050
};
\addplot+[color=green, mark=none] table {
1 289709
2 290548
3 290617
4 290790
5 291016
6 291167
7 291193
8 291236
9 291532
10 291711
11 291745
12 291967
13 292201
14 292268
15 292442
16 292581
17 292829
18 292835
19 293035
20 293303
21 293304
22 293371
23 293478
24 293586
25 293687
26 293942
27 294071
28 294106
29 294198
30 294208
31 294343
32 294489
33 294540
34 294641
35 294812
36 294876
37 295126
38 295349
39 295423
40 295491
41 295676
42 295709
43 295855
44 296009
45 296110
46 296127
47 296203
48 296369
49 296386
50 296534
51 296743
52 296797
53 296868
54 296895
55 296910
56 296996
57 297005
58 297140
59 297307
60 297454
61 297500
62 297646
63 297671
64 297689
65 297966
66 298316
67 298316
68 298416
69 298484
70 298570
71 298990
72 299067
73 299268
74 299392
75 299476
76 299506
77 299527
78 299682
79 300195
80 300249
81 300462
82 301120
83 301455
84 301774
85 301921
86 302111
87 302233
88 302516
89 302812
90 303009
91 303864
92 304875
93 305676
94 305682
95 305699
96 307869
97 314786
98 316810
99 316850
100 317314
101 319547
102 319597
103 319762
104 320520
105 320869
106 323083
107 328339
108 330948
109 331857
110 332285
111 333244
112 335190
113 336062
114 337485
115 338495
116 344538
117 347453
118 354471
119 359124
120 370282
};
\addplot+[color=blue, mark=none]  table {
1 82
2 83
3 83
4 83
5 83
6 84
7 85
8 85
9 85
10 87
11 100
12 101
13 102
14 102
15 112
16 114
17 115
18 127
19 136
20 139
21 141
22 143
23 143
24 146
25 146
26 147
27 148
28 149
29 149
30 154
31 155
32 156
33 160
34 166
35 171
36 176
37 177
38 178
39 178
40 178
41 180
42 180
43 180
44 181
45 181
46 181
47 181
48 182
49 183
50 184
51 187
52 188
53 189
54 192
55 192
56 193
57 199
58 199
59 199
60 207
61 211
62 211
63 213
64 214
65 214
66 217
67 217
68 219
69 224
70 224
71 225
72 226
73 230
74 237
75 245
76 250
77 256
78 261
79 264
80 268
81 274
82 280
83 283
84 285
85 289
86 291
87 294
88 294
89 298
90 307
91 309
92 339
93 350
94 402
95 413
96 441
97 470
98 484
99 506
100 511
101 511
102 537
103 542
104 553
105 598
106 630
107 670
108 706
109 777
110 885
111 1080
112 1151
113 1289
114 1716
115 223617
116 224061
117 226332
118 237687
119 238440
120 253216
};
\addplot+[color=red, mark=none]  table {
1 86
2 87
3 88
4 88
5 89
6 89
7 90
8 91
9 91
10 92
11 104
12 104
13 105
14 106
15 119
16 120
17 135
18 145
19 162
20 171
21 172
22 172
23 174
24 174
25 175
26 176
27 176
28 176
29 180
30 180
31 181
32 190
33 191
34 193
35 195
36 195
37 201
38 202
39 212
40 214
41 221
42 224
43 224
44 225
45 225
46 226
47 227
48 228
49 229
50 229
51 229
52 230
53 231
54 232
55 234
56 235
57 243
58 243
59 243
60 247
61 253
62 253
63 254
64 255
65 256
66 264
67 268
68 268
69 270
70 271
71 271
72 271
73 278
74 287
75 290
76 295
77 299
78 316
79 318
80 318
81 329
82 338
83 339
84 342
85 345
86 358
87 360
88 369
89 377
90 381
91 400
92 401
93 416
94 421
95 422
96 458
97 461
98 535
99 549
100 579
101 599
102 610
103 653
104 719
105 775
106 780
107 793
108 850
109 880
110 1005
111 1234
112 1393
113 1611
114 1695
115 242371
116 245183
117 247084
118 249663
119 249998
120 307233
};
\end{axis}
\end{tikzpicture}        \\
            \rotatebox{90}{\hspace{0.6cm}\# Derived Facts} & \input{figures/plot_generated28_1000_total_facts} & \input{figures/plot_generated28_10000_total_facts} & \input{figures/plot_generated28_50000_total_facts} \\
        \end{tabular}
    }
    \caption{The results for \textcolor{black}{\matrun}, \textcolor{blue}{\relevancerun}, \textcolor{green}{\magicrun}, and \textcolor{red}{\allrun} on second-order scenarios}\label{fig:so:results:2}
\end{figure}

Figures~\ref{fig:so:results:1} and~\ref{fig:so:results:2} show the results of
our experiments on second-order scenarios. Since the numbers of rules are
independent of the sizes of the base instances, we show the numbers of rules
just once. Table~\ref{tab:so:results} summarises the distribution of our
results.

Our results on second-order dependencies are broadly in agreement with the
results on first-order dependencies. Specifically, our techniques seem to be
effective in that they allow most queries to be answered orders of magnitude
faster than by computing the chase. The difference between magic sets and
relevance is less pronounced than in our first-order experiments: across the
board, \matrun is slower than \magicrun, which is in turn considerably slower
than \relevancerun. Moreover, the difference between \relevancerun and \allrun
is negligible; thus, while the magic sets do not considerably improve
performance, they do not hurt it either.

This behaviour is explained to some extent by the very high ratios of the
numbers of total and useful facts: as shown in Table~\ref{tab:so:results},
these often exceed a factor of ten, and already \matrun incurs considerable
overheads (unlike on first-order scenarios). We confirmed that this is mainly
due to the facts involving the $\fnpred{f}$ predicates produced by
$\D$-restricted function reflexivity axioms. In particular, as per
Definition~\ref{def:defun}, each axiom of the form \eqref{eq:Dfnref} introduces
a rule of the form \eqref{eq:defun:Dfnref}, which is transformed by
desingularisation into
\begin{align}
    \fnpred{f}(x_1,\dots,x_n,f(x_1,\dots,x_n)) \leftarrow \D(x_1) \wedge \dots \wedge \D(x_n).
\end{align}
In other words, these axioms generate the value of $f$ on all values of ${x_1,
\dots, x_n}$ that have been identified as relevant (i.e., they occur in facts
with the $\D$ predicate, which is the case whenever they occur in relational
facts). This is a considerable source of overhead, particularly when the
relevant terms are deep. The most effective way to reduce this overhead seems
to be to remove function symbols altogether, which explains why performance
improvements of \relevancerun are much more substantial than of \magicrun or
\allrun.

At present, it is not clear whether this inefficiency can be avoided. The chase
for second-order dependencies by \citet{DBLP:journals/corr/abs-1106-3745} uses
a similar step: it enumerates the domain of the Herbrand universe up to a depth
that is determined in advance from the graph used to ensure that the input is
weakly acyclic. Our approach is somewhat more efficient in that facts with the
$\fnpred{f}$ predicates are produced `on demand' (i.e., only for the relevant
terms).

Nevertheless, our results show that efficient goal-driven query answering over
generalised second-order dependencies is feasible and can lead to considerable
efficiency gains in practice.

\section{Conclusion}\label{sec:conclusion}

In this paper, we presented several techniques for goal-driven query answering
over first- and second-order dependencies with equality. Towards this goal, we
presented a revised formulation of the singularisation technique by
\citet{DBLP:conf/pods/Marnette09}, which overcomes the incompleteness on
second-order dependencies in the work by \citet{DBLP:journals/jodsn/CateHK16}.
Furthermore, we presented a novel relevance analysis technique that tries to
identify rules whose instances are all irrelevant to the query, and we
presented a variant of the magic sets technique for logic programs
\cite{DBLP:conf/pods/BancilhonMSU86, DBLP:journals/jlp/BeeriR91,
DBLP:journals/jlp/BalbinPRM91} that can identify the relevant instances of
rules with equality. In order to evaluate our techniques on second-order
dependencies, we presented (in Appendix~\ref{sec:generating}) a technique for
generating synthetic second-order scenarios that are guaranteed to exhibit
nontrivial inferences. Finally, we conducted an extensive empirical evaluation
and showed that our techniques can often be used to answer a query orders of
magnitude faster than by computing the chase of the input dependencies in full.

We see several interesting avenues for future work. On the theoretical side, it
would be interesting to make our techniques applicable to classes of decidable
dependencies with nonterminating chase (e.g., guarded dependencies). The key
difficulty in relevance analysis is that the fixpoint of the base instance
abstraction may not always be finite. Reasoning in such cases often involves
computing finite representations of infinite models, but it is unclear how to
perform backward chaining over such representations. The key difficulty in the
magic sets is to ensure that the transformation result falls into a decidable
dependency class. On the practical side, a key challenge is to develop ways to
reduce the number of auxiliary facts (i.e., facts involving the $\D$,
$\fnpred{f}$, or magic predicates), as well as ways to reduce the overheads of
applying magic rules.

\section*{Acknowledgements}

This work was funded by the EPSRC grant AnaLOG (EP/P025943/1). For the purpose
of Open Access, the authors have applied a CC BY public copyright licence to
any Author Accepted Manuscript (AAM) version arising from this submission.

\bibliographystyle{ACM-Reference-Format}
\bibliography{references}

\ifappendix

    \clearpage

    \appendix

    \makeatletter
    \counterwithin{theorem}{section}
    \renewcommand{\theproposition}{\@Alph\c@section.\arabic{theorem}}
    \renewcommand{\thelemma}{\@Alph\c@section.\arabic{theorem}}
    \renewcommand{\thedefinition}{\@Alph\c@section.\arabic{theorem}}
    \makeatother

    \section{Proof of Theorem~\ref{thm:sg}}\label{sec:proof:sg}

\thmSG*

\begin{proof}
For arbitrary $\Sigma$, $\Sigma'$, $B$, and $\predQ(\vec a)$ as in the claim,
we show that ${\Sigma \cup B \not\modelsEq \predQ(\vec a)}$ if and only if
${\Sigma' \cup \SG{\Sigma'} \cup B \not\models \predQ(\vec a)}$. The proof of
the ($\Rightarrow$) direction is straightforward. If ${\Sigma \cup B
\not\modelsEq \predQ(\vec a)}$, then there exists an interpretation $I$ such
that ${I \modelsEq \Sigma \cup B}$ and ${I \not\modelsEq \predQ(\vec a)}$. Now
let $I'$ be the interpretation obtained from $I$ by simply setting ${\D^{I'} =
\Delta^{I'} = \Delta^I}$. Predicate $\equals$ is `true' equality in $I$, which
clearly ensures ${I' \models \Sigma' \cup \SG{\Sigma'} \cup B}$ and ${I'
\not\models \predQ(\vec a)}$. This implies ${\Sigma' \cup \SG{\Sigma'} \cup B
\not\models \predQ(\vec a)}$, as required.

For the ($\Leftarrow$) direction, we assume that ${\Sigma' \cup \SG{\Sigma'}
\cup B \not\models \predQ(\vec a)}$, and we show that this implies ${\Sigma
\cup B \not\modelsEq \predQ(\vec a)}$. We achieve this goal in six steps.
First, we use a variant of the chase to construct a Herbrand interpretation
$J^\infty$ such that ${J^\infty \models \Sigma' \cup \SG{\Sigma'} \cup B}$ and
${J^\infty \not\models \predQ(\vec a)}$; note that $\equals$ is interpreted in
$J^\infty$ as an `ordinary' predicate without any special meaning. Second, we
show that the interpretation of predicate $\D$ in $J^\infty$ satisfies a
particular property. Third, we convert $J^\infty$ into an interpretation $I'$
that satisfies the unrestricted functional reflexivity axioms \eqref{eq:fnref}.
Fourth, we define a function $\rho$ that maps each domain element of
$\Delta^{I'}$ into a representative domain element, and we use $\rho$ to
convert $I'$ into an interpretation $I$ where predicate $\equals$ is
interpreted as `true' equality. Fifth, we prove that ${I \modelsEq \Sigma \cup
B}$ and ${I \not\modelsEq \predQ(\vec a)}$, which in turn implies ${\Sigma \cup
B \not\modelsEq \predQ(\vec a)}$.

\medskip

\emph{(Step 1.)} We construct the Herbrand interpretation $J^\infty$ using a
variant of the chase that can handle both function symbols and first-order
quantifiers, which we summarise next. We assume we have a countably infinite
set of \emph{labelled nulls} that is pairwise disjoint with the sets of
constants, function symbols, and predicates. Labelled nulls can be used to
construct terms just like constants; for example, $f(x,c,n)$ is a term
containing variable $x$, constant $c$, and labelled null $n$. Note that $B$ is
a base instance and thus does not contain labelled nulls. Now let ${J^0, J^1,
\dots}$ be a sequence of instances such that ${J^0 = B \cup \{ \D(c) \mid
\text{ constant } c \text{ occurs in } \Sigma' \cup B \}}$ and where each
$J^{i+1}$ with ${i \geq 0}$ is obtained by extending $J^i$ as follows.
\begin{itemize}
    \item For each dependency ${\forall \vec x.(\varphi(\vec x) \rightarrow
    \exists \vec y.\psi(\vec x, \vec y))}$ in ${\Sigma' \cup \SG{\Sigma'}}$ and
    each substitution $\sigma$ mapping $\vec x$ to the terms of $J^i$ such that
    ${\varphi(\vec x)\sigma \subseteq J^i}$ and no extension $\sigma'$ of
    $\sigma$ to the variables of $\vec y$ satisfies ${\psi(\vec x, \vec
    y)\sigma' \subseteq J^i}$, add to $J^{i+1}$ all facts of ${\psi(\vec x,
    \vec y)\sigma''}$ where substitution $\sigma''$ extends $\sigma$ by mapping
    all variables of $\vec y$ to fresh labelled nulls not occurring in $J^i$.

    \item For each fact $R(t_1,\dots,t_n)$ with $R$ different from $\equals$
    and $\predQ$ introduced in the previous step, add to $J^{i+1}$ facts
    ${\D(t_1), \dots, \D(t_n)}$.
\end{itemize}
Let ${J^\infty = \bigcup_{i=0}^\infty J^i}$. Like the chase for SO-TGDs
\cite{DBLP:journals/tods/FaginKPT05}, this chase variant interprets function
symbols by themselves and introduces functional ground terms. This is in fact
the main difference from the chase variant outlined in
Section~\ref{sec:so-dependencies:chase}, which interprets functional terms
using labelled nulls. The approach outlined above is possible because equality
is treated as an ordinary predicate and our objective is to construct a
Herbrand model. Furthermore, like the restricted chase
\cite{DBLP:journals/tcs/FaginKMP05} for first-order dependencies, the approach
above introduces labelled nulls to satisfy the first-order existential
quantifiers. Finally, the approach above applies axioms \eqref{eq:dom-c} and
\eqref{eq:dom-R} eagerly for convenience. By the standard properties of the
mentioned chase variants, $J^\infty$ is a Herbrand interpretation that
satisfies ${J^\infty \models \Sigma' \cup \SG{\Sigma'} \cup B}$ and ${J^\infty
\not\models \predQ(\vec a)}$.

\medskip

\emph{(Step 2.)} We show by induction on ${i \geq 0}$ that each $J^i$ satisfies
the following auxiliary property, which is then clearly also satisfied for ${i
= \infty}$.
\begin{quote}
    ($\ast$): For each (relational or equality) fact ${F \in J^i}$, each
    argument $t$ of $F$, and each proper subterm $s$ of $t$, it is the case
    that ${\D(s) \in J^i}$; moreover, if the predicate of $F$ is different from
    $\equals$ and $\predQ$, then ${\D(t) \in J^i}$ holds as well.
\end{quote}
The base case ${i = 0}$ holds by the definition of $J^0$ and the fact $J^0$
contains no functional terms and no equality facts. Now assume that $J^i$
satisfies ($\ast$) and consider an arbitrary fact ${F \in J^{i+1} \setminus
J^i}$. Fact $F$ clearly satisfies ($\ast$) if it is produced by a dependency in
$\SG{\Sigma'}$. The only remaining possibility is that $F$ is produced by a
dependency ${\delta' = \forall \vec x'.[\varphi'(\vec x') \rightarrow \exists
\vec y.\psi'(\vec x', \vec y)]}$ in $\Sigma'$ by substitutions $\sigma$ and
$\sigma''$ where ${\varphi'(\vec x)\sigma \subseteq J^i}$ and $\sigma''$
extends $\sigma$ by mapping all variables of $\vec y$ to fresh labelled nulls.
Dependency $\delta'$ can be one of the following two forms.
\begin{itemize}
    \item Assume that $\psi'(\vec x',\vec y)$ does not contain the query
    predicate $\predQ$. Dependency $\delta'$ is obtained from a safe dependency
    ${\delta \in \Sigma}$ by singularisation, which does not affect safety.
    Thus, each variable ${x \in \vec x'}$ occurs as an argument of a relational
    atom of $\varphi(\vec x')$, and this atom is matched in the instance $J^i$
    to a relational fact satisfying ($\ast$); thus, $x\sigma$ occurs as an
    argument of a relational fact in $J^i$, and ${\D(s) \in J^i \subseteq
    J^{i+1}}$ holds for each subterm $s$ of $x\sigma$. Thus, fact $F$ is
    produced by instantiating a relational or an equality atom of $\psi(\vec
    x',\vec y)$ of the form $R(t_1,\dots,t_n)$ (where $R$ can be $\equals$).
    Each term $t_i$ is of depth at most one, so ${\D(s) \in J^i \subseteq
    J^{i+1}}$ holds for each proper subterm $s$ of $t_i\sigma$. If ${R =
    {\equals}}$, fact $F$ clearly satisfies ($\ast$). If ${R \neq {\equals}}$,
    then our definition of the chase ensures ${\D(t_i\sigma) \in J^i \subseteq
    J^{i+1}}$ for each $t_i$, so $F$ again satisfies ($\ast$).

    \item The only remaining possibility is that $\psi'(\vec x',\vec y)$ is a
    query atom for the form $\predQ(x_1',\dots,x_n')$, in which case
    $\varphi'(\vec x')$ contains ${x_1 \equals x_1' \wedge \dots \wedge x_n
    \equals x_n'}$. Now each atom ${x_i \equals x_i'}$ is matched in $J^i$ to
    an equality fact that satisfies ($\ast$), so ${\D(s) \in J^i}$ holds each
    proper subterm $s$ of $x_i'\sigma$. Clearly, fact ${F =
    \predQ(x_1',\dots,x_n')\sigma''}$ satisfies ($\ast$).
\end{itemize}

\medskip

\emph{(Step 3.)} Our objective is to convert the Herbrand interpretation
$J^\infty$ into an interpretation $I$ that interprets $\equals$ as `true
equality', and we shall achieve this by replacing all terms in $J^\infty$ that
belong to the same equivalence class in $J^\infty$ with a representative.
However, interpretation $J^\infty$ satisfies only the $\D$-restricted
functional reflexivity axioms \eqref{eq:Dfnref}, but not their unrestricted
version. As a result, $J^\infty$ can contain a fact ${s \equals t}$ without
also containing ${f(s) \equals f(t)}$, which makes choosing the representatives
while satisfying unrestricted functional reflexivity difficult. We overcome
this by first converting $J^\infty$ into an auxiliary interpretation $I'$ where
we `remap' each term not occurring in a relational fact into a fresh, distinct
domain element $\omega$. This ensures that $I'$ satisfies both ${\Sigma' \cup
\SG{\Sigma'} \cup B}$ and the unrestricted functional reflexivity axioms.

In particular, let $T$ be the set of all terms occurring in a (relational or
equality) fact of $J^\infty$, and let $\omega$ be a fresh labelled null not
occurring in $T$. The $\equals$ predicate is reflexive, symmetric, and
transitive in $J^\infty$ so it partitions the ground terms of $J^\infty$ into
equivalence classes. For each equivalence class $C$, we select a representative
for $C$ as follows: select arbitrary ${s \in C}$ such that ${\D(s) \in
J^\infty}$ if one exists, and select arbitrary ${s \in C}$ otherwise. For each
ground term $t$ occurring in $J^\infty$, let $\hat{t}$ be a representative for
the equivalence class that $t$ belongs to. Finally, let ${\hat{\omega} =
\omega}$.

Interpretation $I'$ is defined as shown below, where $c$ ranges over all
constants, $f$ ranges over all function symbols, $n$ is the arity of $f$, tuple
${\langle \alpha_1, \dots, \alpha_n \rangle}$ ranges over
$\big(\Delta^{I'}\big)^n$, $R$ ranges over relational predicates, and $m$ is
the arity of $R$.
\begin{align*}
    \Delta^{I'}                     & = T \cup \{ \omega \}     \\
    c^{I'}                          & = c \\
    f^{I'}(\alpha_1,\dots,\alpha_n) & = \begin{cases}
                                            f(\alpha_1,\dots,\alpha_n)              & \text{if } f(\alpha_1,\dots,\alpha_n) \in T \\
                                            f(\hat{\alpha_1},\dots,\hat{\alpha_n})  & \text{if } f(\alpha_1,\dots,\alpha_n) \not\in T \text{ and } \D(\hat{\alpha_i}) \in J^\infty \text{ for each } i \in \{ 1, \dots, n \} \\
                                            \omega                                  & \text{otherwise} \\
                                        \end{cases} \\
    R^{I'}                          & = \{ \langle t_1, \dots, t_m \rangle \mid R(t_1, \dots, t_m) \in J^\infty \} \\
    \equals^{I'}                    & = \{ \langle t_1, t_2 \rangle \mid t_1 \equals t_2 \in J^\infty \} \cup \{ \langle \omega, \omega \rangle \}
\end{align*}
To see that ${I' \models \Sigma' \cup \SG{\Sigma'} \cup B}$ holds, consider an
arbitrary generalised first-order dependency ${\delta' = \forall \vec
x'.[\varphi'(\vec x') \rightarrow \exists \vec y.\psi'(\vec x', \vec y)]}$ in
${\Sigma' \cup \SG{\Sigma'}}$ and a valuation $\pi'$ such that ${I',\pi'
\models \varphi'(\vec x')}$. If ${\delta' \in \SG{\Sigma'}}$, then ${I',\pi'
\models \delta'}$ because ${\D^{I'} = \D^{J^\infty}}$ and $\omega$ is not
connected in $\equals^{I'}$ to any element of $T$. Now assume that $\psi'(\vec
x', \vec y)$ does not contain the query predicate $\predQ$. For each variable
$x$ in $\varphi'(\vec x')$, variable $x$ occurs in $\varphi'(\vec x')$ in a
relational atom since $\delta'$ is safe, and ${\pi'(x) \in T}$ since no tuple
in any $R^{I'}$ contains $\omega$. Thus, ${J^\infty,\pi' \models \exists \vec
y.\psi'(\vec x', \vec y)}$ clearly implies ${I',\pi' \models \exists \vec
y.\psi'(\vec x', \vec y)}$. Furthermore, assume that $\psi'(\vec x', \vec y)$
is of the form $\predQ(x_1',\dots,x_n')$, so $\varphi'(\vec x')$ is of the form
${\phi \wedge x_1 \equals x_1' \wedge \dots \wedge x_n \equals x_n'}$. Formula
$\phi$ satisfies the safety condition so ${\pi'(x) \in T}$ holds for each
variable $x$ in $\phi$; moreover, $\omega$ is not connected in $\equals^{I'}$
to any element of $T$, so ${\pi'(x_i') \in T}$ for each ${i \in \{ 1, \dots, n
\}}$; thus, we have ${I',\pi' \models \predQ(x_1',\dots,x_n')}$. Finally, we
show that $I'$ satisfies the unrestricted functional reflexivity axiom
\eqref{eq:fnref} for each $n$-ary function symbol $f$. To this end, we consider
arbitrary valuation $\pi'$ such that ${I',\pi' \models x_1 \equals x_1' \wedge
\dots \wedge x_n \equals x_n'}$, and we show ${\langle f^{I'}(u_1,\dots,u_n),
f^{I'}(v_1,\dots,v_n) \rangle \in {\equals}^{I'}}$, where ${u_i = \pi'(x_i)}$
and ${v_i = \pi'(x_i')}$ for each ${i \in N}$ and ${N = \{ 1, \dots, n \}}$.
\begin{itemize}
    \item Assume that ${\{ \D(\hat{u}_i), \D(\hat{v}_i) \} \subseteq J^\infty}$
    for each ${i \in N}$. If ${f(u_1,\dots,u_n) \in T}$, let ${u_i' = u_i}$ for
    each ${i \in N}$; then, property ($\ast$) ensures ${\{ \D(u_1'), \dots,
    \D(u_n') \} \subseteq J^\infty}$. If ${f(u_1,\dots,u_n) \not\in T}$, let
    ${u_i' = \hat{u}_i}$ for each ${i \in N}$; then, ${\D(u_i') \in J^\infty}$
    for each ${i \in N}$ holds by our assumption. Analogously, if
    ${f(v_1,\dots,v_n) \in T}$, let ${v_i' = v_i}$ for each ${i \in N}$, and
    otherwise let ${v_i' = \hat{v}_i}$; we again have ${\D(v_i') \in J^\infty}$
    for each ${i \in N}$. These definitions clearly ensure
    ${f^{I'}(u_1,\dots,u_n) = f(u_1',\dots,u_n')}$ and ${f^{I'}(v_1,\dots,v_n)
    = f(v_1',\dots,v_n')}$, as well as ${u_i' \equals v_i' \in J^\infty}$ for
    each ${i \in N}$. But then, the $\D$-restricted functional reflexivity
    axiom \eqref{eq:Dfnref} for $f$ ensures ${f(u_1',\dots,u_n') \equals
    f(v_1',\dots,v_n') \in J^\infty}$ Consequently, we have ${\langle
    f(u_1',\dots,u_n'), f(v_1',\dots,v_n') \rangle \in {\equals}^{I'}}$, as
    required.

    \item Assume that there exists ${i \in N}$ such that either ${\D(\hat{u}_i)
    \not\in J^\infty}$ or ${\D(\hat{v}_i) \not\in J^\infty}$. Since $\hat{u}_i$
    and $\hat{v}_i$ are in the same equivalence class, we have ${\{
    \D(\hat{u}_i), \D(\hat{v}_i) \} \cap J^\infty = \emptyset}$. But then,
    property ($\ast$) and the definition of representatives ensure
    ${f(u_1,\dots,u_n) \not\in T}$ and ${f(v_1,\dots,v_n) \not\in T}$. Thus,
    the definition of $f^{I'}$ ensures ${f^{I'}(u_1,\dots,u_n) =
    f^{I'}(v_1,\dots,v_n) = \omega}$, and so ${\langle \omega, \omega \rangle
    \in {\equals}^{I'}}$, as required.
\end{itemize}

\medskip

\emph{(Step 4.)} Dependencies \eqref{eq:ref}, \eqref{eq:sym}, and
\eqref{eq:trans} ensure that $\equals^{I'}$ is an equivalence relation on
$\Delta^{I'}$. Let ${\rho : \Delta^{I'} \to \Delta^{I'}}$ be an arbitrary
function that maps each ${\alpha \in \Delta^{I'}}$ to an element ${\rho(\alpha)
\in \Delta^{I'}}$ that is unique associated with the equivalence class of
$\alpha$---that is, for all ${\alpha,\beta \in \Delta^{I'}}$ we have
${\rho(\alpha) = \rho(\beta)}$ if and only if $\alpha$ and $\beta$ belong to
the same equivalence class. The fact that $I'$ satisfies all equality
dependencies including the unrestricted functional reflexivity ensures the
following property.
\begin{quote}
    ($\blacklozenge$): For all ${\{ \alpha_1, \dots, \alpha_n, \beta_1, \dots,
    \beta_n \} \subseteq \Delta^{I'}}$ such that ${\rho(\alpha_i) =
    \rho(\beta_i)}$ for each ${i \in \{ 1, \dots, n \}}$, we have
    ${\rho(f^{I'}(\alpha_1, \dots, \alpha_n)) = \rho(f^{I'}(\beta_1, \dots,
    \beta_n))}$.
\end{quote}
We use $\rho$ to convert $I'$ into an interpretation $I$ where $\equals$ is
interpreted as `true' equality as follows, where $c$ ranges over all constants,
$f$ ranges over all function symbols, $n$ is the arity of $f$, tuple ${\langle
\alpha_1, \dots, \alpha_n \rangle}$ ranges over $\big(\Delta^I\big)^n$, $R$
ranges over relational predicates, and $m$ is the arity of $R$.
\begin{align*}
    \Delta^I                        & = \{ \rho(\alpha) \mid \alpha \in \Delta^{I'} \} \\
    c^I                             & = \rho(c^{I'})                                    & R^I       & = \{ \langle \rho(\alpha_1), \dots, \rho(\alpha_m) \rangle \mid \langle \alpha_1, \dots, \alpha_m \rangle \in R^{I'} \} \\
    f^I( \alpha_1, \dots, \alpha_n) & = \rho(f^{I'}(\alpha_1, \dots, \alpha_n))         & \equals^I & = \{ \langle \alpha, \alpha \rangle \mid \alpha \in \Delta^I \}
\end{align*}

\medskip

\emph{(Step 5.)} It is obvious that ${I \modelsEq B}$, and we next prove ${I
\modelsEq \Sigma}$ and ${I \not\modelsEq \predQ(\vec a)}$. The construction of
$I$ ensures that $I$ satisfies the properties of equality. Now consider an
arbitrary generalised FO dependency ${\delta = \forall \vec x.[\varphi(\vec x)
\rightarrow \exists \vec y.\psi(\vec x, \vec y)] \in \Sigma}$, its
singularisation ${\delta' = \forall \vec x'.[\varphi'(\vec x') \rightarrow
\exists \vec y.\psi'(\vec x', \vec y)] \in \Sigma'}$, and a valuation $\pi$
such that ${I,\pi \modelsEq \varphi(\vec x)}$. We define a valuation $\pi'$ on
$I'$ as follows.
\begin{itemize}
    \item For each relational atom $R(t_1,\dots,t_n)$ in $\varphi(\vec x)$ and
    the corresponding atom $R(x_1',\dots,x_n')$ in $\varphi'(\vec x')$, choose
    any tuple ${\langle \alpha_1, \dots, \alpha_n \rangle \in R^{I'}}$ such
    that ${\rho(\alpha_i) = t_i^{I,\pi}}$ for each ${i \in \{ 1, \dots, n \}}$,
    and define ${\pi'(x_i') = \alpha_i}$ for each such $i$. This definition is
    correctly formed since ${I,\pi \modelsEq R(t_1,\dots,t_n)}$ and the
    construction of $I$ guarantee the existence of at least one such ${\langle
    \alpha_1, \dots, \alpha_n \rangle}$, and singularisation ensures that each
    $x_i'$ occurs exactly once in exactly one relational atom of $\varphi'(\vec
    x')$.

    \item Due to safety, each variable $x'$ in $\varphi'(\vec x')$ not covered
    by the previous case necessarily occurs in $\varphi'(\vec x')$ in a unique
    equality atom of the form ${x \equals x'}$ that is introduced because
    $\psi(\vec x, \vec y)$ contains the query predicate $\predQ$. We define
    ${\pi'(x') = \pi'(x)}$, and we call such $x'$ a \emph{query variable}.
\end{itemize}
Interpretation $I$ and valuation $\pi'$ satisfy the following auxiliary
property.
\begin{quote}
    ($\lozenge$): For each term $t$ occurring in $\varphi(\vec x)$, we have
    ${t^{I,\pi} = \rho(t^{I',\pi'})}$.
\end{quote}
If $t$ is a variable $x$ occurring in $\varphi(\vec x)$, the definition of
$\pi'$ clearly ensures ${\pi(x) = \rho(\pi'(x))}$. Moreover, if $t$ is a
constant $c$, the definition of $I$ ensures ${c^{I,\pi} = \rho(c^{I',\pi'})}$.
Finally, if $t$ is of the form ${t = f(t_1,\dots,t_n)}$ where each $t_i$ is a
constant or a variable occurring in $\varphi(\vec x)$, we get the required
property as follows, where the second equality holds by the definition of $f^I$
and the fact that ${t_i^{I,\pi} = \rho(t_i^{I',\pi'})}$ for each ${i \in \{ 1,
\dots, n \}}$, and the third equality holds due to property ($\blacklozenge$).
\begin{align*}
    t^{I,\pi} = f^I(t_1^{I,\pi},\dots,t_n^{I,\pi}) = \rho(f^{I'}(\rho(t_1^{I,\pi}),\dots,\rho(t_n^{I,\pi}))) = \rho(f^{I'}(t_1^{I',\pi'},\dots,t_n^{I',\pi'})) = \rho(t^{I',\pi'})
\end{align*}

We are now ready to show that ${I',\pi' \models \varphi'(\vec
x')}$. The first item in the definition of $\pi'$ clearly ensures ${I',\pi'
\models R(x_1',\dots,x_n')}$ for each relational atom in $\varphi'(\vec x')$.
We next show that $I'$ and $\pi'$ satisfy each equality of $\varphi'(\vec x')$.
Note that singularisation changes neither the equality atoms nor the functional
terms occurring in $\varphi(\vec x)$.
\begin{itemize}
    \item Consider any equality atom ${t_1 \equals t_2}$ in $\varphi'(\vec x')$
    that also occurs in $\varphi(\vec x)$. Then, ${I,\pi \modelsEq t_1 \equals
    t_2}$ implies ${t_1^{I,\pi} = t_2^{I,\pi}}$, and property ($\lozenge$)
    ensures ${\rho(t_1^{I',\pi'}) = t_1^{I,\pi}}$ and ${\rho(t_2^{I',\pi'}) =
    t_2^{I,\pi}}$. Hence, we have ${\rho(t_1^{I',\pi'}) =
    \rho(t_2^{I',\pi'})}$, so terms $t_1^{I',\pi'}$ and $t_2^{I',\pi'}$ belong
    to the same equivalence class of $\equals^{I'}$ by the definition of
    $\rho$; thus, ${I',\pi' \models t_1 \equals t_2}$ holds, as required.

    \item Each equality atom in $\varphi'(\vec x')$ not occurring in
    $\varphi(\vec x)$ is of the form ${x_i' = t_i}$ and was obtained by
    applying a singularisation step to the argument $t_i$ of a relational atom
    $R(t_1,\dots,t_n)$; moreover, $R(t_1,\dots,t_n)$ is eventually reduced to a
    relational atom $R(x_1',\dots,x_n')$ in $\varphi'(\vec x')$. We have the
    following possibilities for $t_i$.
    \begin{itemize}
        \item Assume $t_i$ is a variable $z_k$ occurring in another relational
        atom ${S(z_1,\dots,z_m)}$ of $\varphi'(\vec x')$. Formula $\varphi(\vec
        x)$ then contains a relational atom ${S(u_1,\dots,u_m)}$ where ${t_i =
        u_k = z_k}$. But then, the definition of $\pi'$ ensures ${\pi(z_k) =
        \rho(\pi'(z_k)) = \rho(\pi'(x_i'))}$, and the definition of $\rho$
        ensures ${I',\pi' \models x_i' \equals z_k}$, as required.

        \item Assume $t_i$ is a constant. The definition of $\pi'$ and $I$
        ensure ${\rho(\pi'(x_i')) = t_i^{I,\pi}}$. Thus, $\pi'(x_i')$ and
        $t_i^{I',\pi'}$ belong to the same equivalence class of $\equals^{I'}$
        by the definition of $\rho$, and ${I',\pi' \models x_i' \equals t_i}$
        holds, as required.

        \item Assume $t_i$ is a query variable. The definition of $\pi'$ then
        ensures ${\pi'(t_i) = \pi'(x_i')}$, which in turn ensures ${I',\pi'
        \models x_i' \equals t_i}$, as required.
    \end{itemize}
\end{itemize}

Since ${I' \models \Sigma'}$ by our initial assumption, ${I',\pi' \models
\varphi'(\vec x')}$ implies ${I',\pi' \models \exists \vec y.\psi'(\vec x',
\vec y)}$, so there exists a valuation $\xi'$ obtained by extending $\pi'$ to
the variables of $\vec y$ such that ${I',\xi' \models \psi'(\vec x', \vec y)}$;
since $\exists \vec y.\psi'(\vec x', \vec y)$ is already satisfied in
$J^\infty$, we can safely assume that the range of $\xi'$ does not contain the
special labelled null $\omega$. Let $\xi$ be the valuation obtained by
extending $\pi$ so that ${\xi(y) = \rho(\xi'(y))}$ for each variable ${y \in
\vec y}$. Now consider an arbitrary argument $t$ of a relational or equality
atom in $\psi(\vec x, \vec y)$. If $t$ is a variable in $\vec y$, the
definition of $\xi'$ ensures ${t^{I,\xi} = \rho(t^{I',\xi'})}$. Alternatively,
$t$ is a term of depth at most one constructed from constants, variables in
$\vec x$, and function symbols, so property ($\lozenge$) and the definition of
$\xi$ ensure ${t^{I,\xi} = \rho(t^{I',\xi'})}$. We next prove ${I,\xi \modelsEq
\psi(\vec x, \vec y)}$ by considering all possible forms of an atom in
${\psi(\vec x, \vec y)}$.
\begin{itemize}
    \item Consider an arbitrary relational atom $R(t_1,\dots,t_n)$ in
    $\psi(\vec x, \vec y)$. Then ${I',\xi' \models R(t_1,\dots,t_n)}$ implies
    ${\langle t_1^{I',\xi'}, \dots, t_n^{I',\xi'} \rangle \in R^{I'}}$, and so
    the definition of $I$ ensures ${\langle \rho(t_1^{I',\xi'}), \dots,
    \rho(t_n^{I',\xi'}) \rangle \in R^I}$. Finally, ${\rho(t_i^{I',\xi'}) =
    t_i^{I,\xi}}$ holds for each $i$, so ${\langle t_1^{I,\xi}, \dots,
    t_n^{I,\xi} \rangle \in R^I}$, which ensures ${I,\xi \modelsEq
    R(t_1,\dots,t_n)}$, as required.

    \item Consider an equality atom ${t_1 \equals t_2}$ in $\psi(\vec x, \vec
    y)$. Then ${I',\xi' \models t_1 \equals t_2}$ ensures that $t_1^{I',\xi'}$
    and $t_2^{I',\xi'}$ belong to the same equivalence class of $\equals^{I'}$,
    so ${\rho(t_1^{I',\xi'}) = \rho(t_2^{I',\xi'})}$. Moreover, we again have
    ${\rho(t_i^{I',\xi'}) = t_i^{I,\xi}}$ for each ${i \in \{ 1, 2 \}}$, which
    implies ${t_1^{I,\xi} = t_2^{I,\xi}}$. Thus, ${I,\xi \modelsEq t_1 \equals
    t_2}$, as required.
\end{itemize}
We have thus shown that ${I,\pi \modelsEq \forall \vec x.[\varphi(\vec x)
\rightarrow \exists \vec y.\psi(\vec x, \vec y)]}$. This entire argument holds
for each valuation $\pi$ and each generalised first-order dependency in
$\Sigma$, so ${I \modelsEq \Sigma \cup B}$, as required.

We finally show ${I \not\modelsEq \predQ(\vec a)}$. Assume for contradiction
that ${I \modelsEq \predQ(\vec a)}$ for ${\vec a = a_1, \dots, a_n}$, so
${\langle a_1^I, \dots, a_n^I \rangle \in \predQ^I}$. By the definition of $I$,
there exists a tuple ${\langle \alpha_1, \dots, \alpha_n \rangle \in
\predQ^{I'}}$ such that ${\rho(\alpha_i) = a_i^I}$ for each ${i \in \{ 1,
\dots, n \}}$. The construction of $J^\infty$, the fact that the query
predicate $\predQ$ never occurs in the $\varphi'(\vec x')$, and that fact that
the interpretations of $\predQ$ in $J^\infty$ and $I'$ coincide ensure that
$J^\infty$ and $I'$ are minimal with respect to $\predQ$---that is, there exist
a valuation $\pi'$ on $I'$ and a dependency in $\Sigma'$ of the form ${\delta'
= \forall \vec x'.[\varphi'(\vec x') \wedge \phi \rightarrow
\predQ(x_1',\dots,x_n')]}$ where ${\phi = x_1 \equals x_1' \wedge \dots \wedge
x_n \equals x_n'}$ such that ${I',\pi' \models \varphi'(\vec x') \wedge \phi}$
and ${\pi'(x_i') = \alpha_i}$. Thus, ${I',\pi' \models \varphi'(\vec x') \wedge
\phi}$ and ${\pi'(x_i') = \alpha_i}$ for each $i$. Now let $\pi''$ be the
valuation that coincides with $\pi'$ on all variables apart from
${x_1',\dots,x_n'}$ where ${\pi''(x_i') = a_i^{I'}}$ for each ${i \in \{ 1,
\dots, n \}}$. Variables ${x_1',\dots,x_n'}$ do not occur in $\varphi'(\vec
x')$, which ensures ${I',\pi'' \models \varphi'(\vec x')}$. Moreover, for each
$i$, elements $\pi''(x_i')$ and $\pi'(x_i)$ belong to the same equivalence
class of $\Delta^{I'}$, so we have ${\rho(\pi''(x_i')) = \rho(\pi'(x_i)) =
a_i^I}$; this, in turn, implies ${I',\pi'' \models \phi}$. But then, we have
${I' \models \predQ(\vec a)}$, which contradicts our initial assumption
${J^\infty \not\models \predQ(\vec a)}$ and the observation that ${\predQ^{I'}
= \predQ^{J^\infty}}$. Thus, ${I \not\modelsEq \predQ(\vec a)}$ holds, as
required.
\end{proof}

    \section{Proof of Theorem~\ref{thm:relevance}}\label{sec:proof:relevance}

Before proving Theorem~\ref{thm:relevance}, we first address a technical
detail: the relevant program $\R$ produced by Algorithm~\ref{alg:relevance} may
contain fewer true function symbols than the input program $P$; thus, $\FR{\R}$
may not contain all $\D$-restricted functional reflexivity rules of $\FR{P}$,
and it is not obvious that the rules in ${\FR{P} \setminus \FR{\R}}$ are not
relevant to the query answers. The following lemma establishes a property that
we use later to show that the rules in ${\FR{P} \setminus \FR{\R}}$ are indeed
irrelevant.

\begin{lemma}\label{lem:Dfin-Rfin-fnsym}
    Let $\Dfin$ and $\Rfin$ be the final sets $\Di{}$ and $\R$ obtained by
    applying Algorithm~\ref{alg:relevance} to a well-formed program $P$ and a
    base instance $B$. Then, for each function symbol $f$ occurring in a fact
    of $\Dfin$, there exists a rule ${r \in \Rfin}$ that contains $f$.
\end{lemma}

\begin{proof}
Transformations in
lines~\ref{alg:relevance:UNA:start}--\ref{alg:relevance:UNA:end} do not remove
function symbols, so it suffices to prove an analogous claim for the
intermediate set of rules $\Rint$ just before
line~\ref{alg:relevance:UNA:start}. The query predicate $\predQ$ does not occur
in a rule body so facts of the form $\predQ(\vec t)$ are added to set
$\mathcal{D}$ only in line~\ref{alg:relevance:init}; thus, no fact in $\Dfin$
with the $\predQ$ predicate contains a function symbol. We next define a notion
of a function symbol $f$ being \emph{relevant} for a fact $F$. In particular,
if $F$ is relational, then $f$ is relevant for $F$ if $f$ occurs in an argument
of $F$. Furthermore, if $F$ is of the form ${t_1 \equals t_2}$, then $f$ is
relevant for $F$ if (i)~the outermost symbols (i.e., the constants or the
top-level function symbols) of $t_1$ and $t_2$ are different and there exists
${i \in \{ 1, 2 \}}$ such that $f$ occurs in $t_i$, or (ii)~${t_1 \equals t_2}$
is of the form ${g(u_1,\dots,u_n) \equals g(v_1,\dots,v_n)}$ for some function
symbol $g$ (which may or may not be $f$) and there exists ${i \in \{ 1, \dots,
n \}}$ such that $f$ is relevant for the fact ${u_i \equals v_i}$.

Next, we show that, for each function symbol that is relevant for a fact of
$\Dfin$, set $\Rint$ contains a rule with that function symbol. Let ${I^0, I^1,
\dots}$ be the sequence of instances used to compute ${\fixpoint{P'}{B'} = I}$
in line~\ref{alg:relevance:abstraction-fixpoint}, and consider an arbitrary
fact of $\Dfin$ and an arbitrary function symbol $f$ relevant for this fact.
Instance $B'$ does not contain function symbols and ${\Dfin \subseteq I}$, so
there exist ${i > 0}$ and a fact ${F \in I^i \cap \Dfin}$ such that $f$ is
relevant for $F$, and, for each ${j < i}$ and each fact ${F' \in I^j \cap
\Dfin}$, function symbol $f$ is not relevant for $F'$. But then, there exists a
rule ${r \in P'}$ and a substitution $\nu$ such that ${\body{r}\nu \subseteq
I^{i-1}}$ and ${F = \head{r}\nu \in I^i \setminus I^{i-1}}$.
Line~\ref{alg:relevance:add-T-D} ensures that $\Dfin$ contains each fact of
$\body{r}\nu$ apart from possibly facts of the form ${c \equals c}$. Now $f$ is
clearly not relevant for facts of the form ${c \equals c}$, so our assumption
on $i$ ensures that $f$ is not relevant for any fact of $\body{r}\nu$. Thus,
$r$ clearly cannot be a domain rule \eqref{eq:dom-c} or \eqref{eq:dom-R}, the
reflexivity rule \eqref{eq:ref}, the symmetry rule \eqref{eq:sym}, or the
transitivity rule \eqref{eq:trans} since these rules are function-free.
Moreover, $r$ cannot be a $\D$-restricted functional reflexivity rule
\eqref{eq:Dfnref}: then $F$ is of the form ${g(u_1,\dots,u_n) \equals
g(v_1,\dots,v_n)}$, and thus $f$ is relevant for $F$ only if $f$ is relevant
for some ${u_i \equals v_i}$, which cannot be the case. The only remaining
possibility is ${r \in P}$. All facts in $\Dfin$ with the query predicate
$\predQ$ consist of only constants, so the predicate of $F$ is not $\predQ$;
moreover, rule $r$ is well-formed, so Definition~\ref{def:well-formed} ensures
that $\body{r}$ is well-formed too. Thus, for each variable $x$ in $r$, there
exists a relational atom ${A \in \body{r}}$ that contains $x$. Function symbol
$f$ is not relevant for $A\nu$ by our observations thus far, which means that
no argument of $A\nu$ contains $f$; but then, $x\nu$ does not contain $f$
either. Thus, the only way for $\head{r}\nu$ to contain $f$ is if $\head{r}$
contains $f$. Finally, rule $r$ is added to $\Rint$ in
line~\ref{alg:relevance:add-r} of Algorithm~\ref{alg:relevance}, which
completes our claim.

We now show our main claim. Let ${\Di{0}, \Di{1}, \dots, \Dfin}$ be the
sequence of sets $\Di{}$ from Algorithm~\ref{alg:relevance} where $\Di{0}$ is
the set in line~\ref{alg:relevance:init}, and each $\Di{i+1}$ is obtained from
$\Di{i}$ after processing one rule $r$ and one substitution $\nu$ in
line~\ref{alg:relevance:rule:start}. We show by induction on $i$ that, for each
$i$, each fact ${F \in \Di{i}}$, and each function symbol $f$ occurring in $F$,
there exists a rule ${r \in \Rint}$ that contains $f$. The base case holds
trivially since $\Di{0}$ does not contain function symbols. For the induction
step, assume that $\Di{i}$ satisfies the claim, and consider a fact ${F \in
\Di{i}}$, rule ${r \in P \cup \DOM{P} \cup \ST \cup \FR{P}}$, and substitution
$\nu$ processed in line~\ref{alg:relevance:rule:start}. If $r$ is the symmetry
rule \eqref{eq:sym} or a $\D$-restricted functional reflexivity rule
\eqref{eq:Dfnref}, then each fact added to $\Di{i+1}$ in
line~\ref{alg:relevance:add-T-D} consists of subterms of the arguments of $F$,
so $\Di{i+1}$ satisfies the claim by the inductive hypothesis. If $r$ is the
transitivity rule \eqref{eq:trans}, the only interesting case is if term
$x_2\nu$ contains a function symbol $f$ that occurs in neither $x_1\nu$ nor
$x_3\nu$; but then, $f$ is relevant for ${x_1\nu \equals x_2\nu}$ or ${x_2\nu
\equals x_3\nu}$, so the claim from the previous paragraph ensures our claim.
If $r$ is a domain rule \eqref{eq:dom-c}, then the body of $r$ is empty so
${\Di{i+1} = \Di{i}}$ and the claim holds vacuously. The only remaining
possibility is ${r \in P}$ or $r$ is a domain rule \eqref{eq:dom-R}. Now
consider an arbitrary function symbol $f$ not occurring in $\Di{i}$ that occurs
in a fact of $\body{r}\nu$ added to $\Di{i+1}$ in
line~\ref{alg:relevance:add-T-D}. If $r$ contains $f$, then $r$ is not a domain
rule \eqref{eq:dom-R}; thus, $r$ is added to $\Rint$ in
line~\ref{alg:relevance:add-r} and the claim holds. Otherwise, rule $r$ is
well-formed, so it is of one of the two forms mentioned in
Definition~\ref{def:well-formed}. If $\head{r}$ does not contain the query
predicate $\predQ$, then $\body{r}$ is well-formed, so each variable of $r$
occurs in $\body{r}$ in a relational atom; but then, there exists a relational
fact ${F' \in \body{r}\nu}$ such that $f$ is relevant for $F'$, so the previous
paragraph ensures that $f$ occurs in $\Rint$. Finally, if $\head{r}$ is of the
form ${\predQ(x_1',\dots,x_n')}$, then $\body{r}$ is of the form ${\phi \wedge
x_1 \equals x_1' \wedge \dots \wedge x_n \equals x_n'}$. Since $\phi$ is
well-formed, the facts in $\phi\nu$ are handled as above. Moreover, each fact
involving $\predQ$ in $\Di{i}$ consists of constants only, so $x_i'\nu$ is a
constant for each ${i \in \{ 1, \dots, n \}}$. Thus, if $f$ occurs in some fact
${x_i \equals x_i'}$ added to $\Di{i+1}$ in line~\ref{alg:relevance:add-T-D},
then $f$ is relevant for ${x_i \equals x_i'}$ so the previous paragraph ensures
that $f$ occurs in $\Rint$.
\end{proof}

We are now ready to prove the correctness of our relevance analysis algorithm.

\thmRelevance*

\begin{proof}
Consider a run of Algorithm~\ref{alg:relevance} on a program $P$ and a base
instance $B$. Let $B'$ be the base instance used in
line~\ref{alg:relevance:abstraction}; let $\eta$ be any mapping of constants to
constants such that ${\eta(B') \subseteq B}$ and ${\eta(c) = c}$ for each
constant $c$ occurring in $P$; let $I$ and $P'$ be as in
line~\ref{alg:relevance:abstraction-fixpoint}; let $\Rint$ be the intermediate
set of rules just before line~\ref{alg:relevance:UNA:start}; and let $\Rfin$
and $\Dfin$ be the final sets. For $t$ a term, let $\eta(t)$ be the result of
replacing in $t$ each occurrence of constant $c$ on which $\eta$ is defined
with $\eta(c)$. The algorithm ensures ${\Rint \subseteq P}$, so set $\Rint$ is
well-formed; moreover, the transformation in line~\ref{alg:relevance:UNA:desg}
clearly preserves this property. Note that ${\SG{\Rint} = \SG{\Rfin} \subseteq
\SG{P}}$ and ${\Dfin \subseteq I}$. Also, the query predicate $\predQ$ does not
occur in a rule body so facts of the form $\predQ(\vec t)$ are added to set
$\mathcal{D}$ only in line~\ref{alg:relevance:init}; thus, no fact with the
$\predQ$ predicate occurring in $\Dfin$ contains a function symbol.

\medskip

($\Leftarrow$) Let ${I^0, I^1, \dots}$ be the sequence of instances used to
compute $\fixpoint{\Rfin \cup \SG{\Rfin}}{B}$ as described in
Section~\ref{sec:preliminaries}. We show by induction on $i$ that, for each
fact ${F \in I^i}$, we have ${\Rint \cup \SG{\Rint} \cup B \models F}$; then,
${\Rint \subseteq P}$ and ${\SG{\Rint} \subseteq \SG{P}}$ clearly imply ${P
\cup \SG{P} \cup B \models F}$. The base case is trivial since ${I^0 = B}$. For
the induction step, we assume that the claim holds for $I^i$, and we consider a
fact ${\head{r}\sigma \in I^{i+1} \setminus I^i}$ derived by a rule ${r \in
\Rfin \cup \SG{\Rfin}}$ and a substitution $\sigma$ satisfying ${\body{r}\sigma
\subseteq I^i}$. The inductive assumption then ensures ${\Rint \cup \SG{\Rint}
\cup B \models \body{r}\sigma}$. If ${r \in \SG{\Rint}}$, then ${\Rint \cup
\SG{\Rint} \cup B \models \head{r}\sigma}$ holds trivially; hence, we next
consider the case when ${r \not\in \Rint}$ and ${r \in \Rfin}$. Then, rule $r$
is obtained from some rule ${r' \in \Rint}$ by transformations in
line~\ref{alg:relevance:UNA:desg}. Let $\sigma'$ be the extension of $\sigma$
such that ${\sigma'(x) = t\sigma}$ for each variable $x$ that was replaced with
a term $t$. Each body atom ${A' \in \body{r'}}$ that was not removed in
line~\ref{alg:relevance:UNA:desg} corresponds to some ${A \in \body{r}}$ that
satisfies ${A'\sigma' = A\sigma \in I^i}$. Now consider an arbitrary atom ${A'
\in \body{r'}}$ of the form ${x \equals t}$ or ${t \equals x}$ that was removed
in line~\ref{alg:relevance:UNA:desg}. The definition of $\sigma'$ ensures
${x\sigma' = t\sigma'}$. Moreover, the condition in
line~\ref{alg:relevance:UNA:form} ensures that $\body{r'}$ contains a
relational atom ${R(\dots,x,\dots) \in \body{r'}}$; this atom is relational, so
it was not removed in line~\ref{alg:relevance:UNA:desg} and thus satisfies
${R(\dots,x,\dots)\sigma' \in I^i}$; hence, ${\Rint \cup \SG{\Rint} \cup B
\models R(\dots,x,\dots)\sigma'}$ holds by the inductive assumption. Moreover,
$\SG{\Rint}$ contains the domain rules \eqref{eq:dom-R} for $R$, so ${\Rint
\cup \SG{\Rint} \cup B \models \D(x)\sigma'}$; moreover, $\SG{\Rint}$ also
contains the reflexivity rule \eqref{eq:ref}, so ${\Rint \cup \SG{\Rint} \cup B
\models x\sigma' \equals x\sigma'}$. Finally, ${x\sigma' = t\sigma'}$ ensures
${\Rint \cup \SG{\Rint} \cup B \models A'\sigma'}$. This holds for each atom
${A' \in \body{r'}}$ that was removed in line~\ref{alg:relevance:UNA:desg}, so
${\Rint \cup \SG{\Rint} \cup B \models \body{r'}\sigma'}$ holds. Thus, we have
${\Rint \cup \SG{\Rint} \cup B \models \head{r'}\sigma'}$, and the definition
of $\sigma'$ ensures ${\head{r'}\sigma' = \head{r}\sigma}$. Hence, we have
${\Rint \cup \SG{\Rint} \cup B \models \head{r}\sigma}$, as required.

\medskip

($\Rightarrow$) Let ${I^0, I^1, \dots}$ be the sequence of instances used to
compute $\fixpoint{P \cup \SG{P}}{B}$ as described in
Section~\ref{sec:preliminaries}. We prove the claim in three steps.

\smallskip

\emph{Step 1.} We show by induction on $i$ that ${\eta(I^i) \subseteq I}$
holds. For the base case, we have ${I^0 = \eta(B') \subseteq I}$, as required.
For the induction step, assume that ${\eta(I^i) \subseteq I}$ holds, and
consider a fact ${\head{r}\sigma \in I^{i+1} \setminus I^i}$ derived by a rule
${r \in P'}$ and a substitution $\sigma$ satisfying ${\body{r}\sigma \subseteq
I^i}$. Let $\sigma'$ be the substitution defined as ${\sigma'(x) =
\eta(\sigma(x))}$ on each variable $x$ from the domain of $\sigma$. Since
$\eta$ is identity on the constants of $r$, we clearly have ${\eta(r\sigma) =
r\sigma'}$. But then, the inductive assumption ensures ${\body{r}\sigma' =
\eta(\body{r}\sigma) \subseteq \eta(I^i) \subseteq I}$, which implies
${\head{r}\sigma' = \eta(\head{r}\sigma) \in I}$, as required.

\smallskip

\emph{Step 2.} Let ${\Rstar = \DOM{P} \cup \Rfl \cup \ST \cup \FR{\Rfin}}$, and
let us say that a fact $F$ is \emph{$\eta$-relevant} if ${\eta(F) \in \Dfin}$,
and $F$ is not of the form ${t \equals t}$ for some term $t$. We prove by
induction on $i$ that, for each $\eta$-relevant fact ${F \in I^i}$, we have
${\Rfin \cup \Rstar \cup B \models F}$. Note that $\Rstar$ and $\SG{\Rfin}$
differ only in the use of $\DOM{P}$ and $\DOM{\Rfin}$, respectively; we deal
with this detail in the third step.

The induction base holds trivially. Now assume that $I^i$ satisfies this
property and consider an arbitrary rule ${r \in P \cup \SG{P}}$ and
substitution $\sigma$ such that ${\body{r}\sigma \subseteq I^i}$,
${\head{r}\sigma \in I^{i+1} \setminus I^i}$, and $\head{r}\sigma$ is
$\eta$-relevant. Thus, $\head{r}\sigma$ is not of the form ${t \equals t}$, so
$r$ is not the reflexivity rule \eqref{eq:ref}, and we have ${r \in P \cup
\DOM{P} \cup \ST \cup \FR{P}}$. Moreover, ${\eta(\head{r}\sigma) \in \Dfin}$
ensures that fact ${\eta(\head{r}\sigma)}$ was added to $\mathcal{T}$ in
line~\ref{alg:relevance:add-T-D}, so the fact was extracted from $\mathcal{T}$
at some point in line~\ref{alg:relevance:choose-F}. Furthermore,
${\body{r}\sigma \subseteq I^i}$, ${\eta(I^i) \subseteq I}$, and the
observation that $\eta$ is identity on the constants of $P$ ensures that $r$
and substitution $\nu$ satisfying ${r\nu = \eta(r\sigma)}$ were considered in
line~\ref{alg:relevance:rule:start}; thus, line~\ref{alg:relevance:add-r}
ensures ${r \in \Rint \cup \DOM{P} \cup \ST \cup \FR{P}}$. Moreover, $\eta$
does not affect the function symbols, so each function symbol of
$\head{r}\sigma$ occurs in $\eta(\head{r}\sigma)$; thus,
Lemma~\ref{lem:Dfin-Rfin-fnsym} ensures ${r \in \Rint \cup \Rstar}$.
Line~\ref{alg:relevance:add-T-D} ensures that each atom of $\body{r}\nu$ not of
the form ${c \equals c}$ is contained in $\Dfin$. We next prove the following
property.
\begin{quote}
    ($\ast$): $\Rfin \cup \Rstar \cup B \models \body{r}\sigma$.
\end{quote}

First, consider a ground atom ${F \in \body{r}\sigma}$ not of the form ${t
\equals t}$. If atom $F$ is relational, or if $F$ contains function symbols, or
if ${P \cup B}$ does not satisfy UNA, then $\eta(F)$ is not of the form ${c
\equals c}$ for $c$ a constant; hence, $\eta(F)$ is added to $\Dfin$ in
line~\ref{alg:relevance:add-T-D}, so the inductive assumption ensures ${\Rfin
\cup \Rstar \cup B \models F}$.

Second, consider a ground atom ${t \equals t \in \body{r}\sigma}$ obtained by
instantiating some atom ${A \in \body{r}}$---that is, ${A\sigma = t \equals
t}$. If $r$ is the symmetry rule \eqref{eq:sym}, then ${\head{r}\sigma = t
\equals t}$, which contradicts ${\head{r}\sigma \in I^{i+1} \setminus I^i}$. If
$r$ is the transitivity rule \eqref{eq:trans}, then let $A'$ be the other atom
from $\body{r}$; regardless of how we choose $A$ and $A'$, we have ${A'\sigma
\in I^i}$, and ${\head{r}\sigma = A'\sigma}$, which again contradicts
${\head{r}\sigma \in I^{i+1} \setminus I^i}$. Thus, ${r \not\in \ST}$, so the
rule $r$ is well-formed. Now assume that $\head{r}$ does not contain the
$\predQ$ predicate. For each constant $c$ in $A$, the domain rules in $\Rstar$
ensure ${\Rfin \cup \Rstar \cup B \models \D(c)}$, and the reflexivity rule in
$\Rstar$ ensures ${\Rfin \cup \Rstar \cup B \models c \equals c}$. Furthermore,
consider a variable ${x \in \vars{A}}$. Rule $r$ is well-formed, so there
exists a relational atom ${A' \in \body{r}}$ such that ${x \in \vars{A'}}$. The
previous paragraph ensures ${\Rfin \cup \Rstar \cup B \models A'\sigma}$; thus,
the domain rules for the predicate of $A'$ in $\Rstar$ ensure ${\Rfin \cup
\Rstar \cup B \models \D(x)\sigma}$, and the reflexivity rule in $\Rstar$
ensures ${\Rfin \cup \Rstar \cup B \models x\sigma \equals x\sigma}$. Thus, if
either argument of $A$ is a variable or a constant, our observations thus far
ensure ${\Rfin \cup \Rstar \cup B \models A\sigma}$. Otherwise, at least one
argument of $A$ contains a function symbol $f$, and $A\sigma$ is of the form
${f(\vec v) \equals f(\vec v)}$. Our observations ensure ${\Rfin \cup \Rstar
\cup B \models \D(v)}$ and ${\Rfin \cup \Rstar \cup B \models v \equals v}$ for
each term ${v \in \vec v}$. Moreover, $A\sigma$ is not of the form ${c \equals
c}$, so $\eta(A\sigma)$ is added to $\Dfin$ in
line~\ref{alg:relevance:add-T-D}; but then, Lemma~\ref{lem:Dfin-Rfin-fnsym}
ensures that $\FR{\Rfin}$ contains a $\D$-restricted functional reflexivity
rule for $f$, which in turn ensures ${\Rfin \cup \Rstar \cup B \models
A\sigma}$. Finally, if $\head{r}$ contains the $\predQ$ predicate, then the
only case not considered thus far is when $A$ is of the form ${x \equals x'}$
where only variable $x$ occurs in a relational atom; but then, the
corresponding domain rule ensures ${\Rfin \cup \Rstar \cup B \models \D(t)}$,
and the reflexivity rule in $\Rstar$ ensures ${\Rfin \cup \Rstar \cup B \models
A\sigma}$.

Thus, property ($\ast$) holds. If ${r \in \Rstar}$, this immediately proves the
inductive step. Otherwise, we have ${r \in \Rint}$, so there exists a rule ${r'
\in \Rfin}$ obtained from $r$ by transformations in
line~\ref{alg:relevance:UNA:desg}. Now consider an arbitrary body equality atom
${F \in \body{r}}$. If $F$ is not of the form ${c \equals c}$ for some constant
$c$, then either ${P \cup B}$ does not satisfy UNA, or $\eta(F)$ is not of the
form ${c \equals c}$ for $c$ a constant; hence, ${\langle r,i \rangle}$ is
added to $\mathcal{B}$ in line~\ref{alg:relevance:add-B}, so the body atom of
$r'$ corresponding to $F$ is not eliminated in
line~\ref{alg:relevance:UNA:desg}. Thus, we have ${\body{r'}\sigma \subseteq
\body{r}\sigma}$, so ${\Rfin \cup \Rstar \cup B \models \head{r}\sigma}$.

\smallskip

\emph{Step 3.} For each fact of the form ${\predQ(\vec a)}$ such that ${P \cup
\SG{P} \cup B \models \predQ(\vec a)}$, line~\ref{alg:relevance:init} makes
sure that ${\eta(\predQ(\vec a)) \in \Dfin}$; thus, the property from the
previous step ensures ${\Rfin \cup \Rstar\cup B \models \predQ(\vec a)}$. We
now argue that ${\Rfin \cup \SG{\Rfin} \cup B \models \predQ(\vec a)}$ holds as
well. In particular, note that ${\SG{\Rfin} \subseteq \Rstar}$, and that
${\Rstar \setminus \SG{\Rfin}}$ consists of zero or more domain rules of the
form \eqref{eq:dom-c} for a constant $c$ not occurring in $\R$, and domain
rules of the form \eqref{eq:dom-R} for a predicate $R$ not occurring in $\R$.
Such rules can derive facts of the form $\D(c)$ for a constant $c$ occurring in
the base instance $B$ in a fact with a predicate not occurring in ${\Rfin \cup
\SG{\Rfin}}$, which can eventually derive facts of the form ${c \equals c}$ and
${f(\vec u) \equals t}$ where $\vec u$ contain such constant $c$. Now consider
an arbitrary rule ${r \in \Rfin}$, and assume that the head of $r$ does not
contain $\predQ$. Furthermore, consider a variable $x$ occurring in an equality
in $\body{r}$. Since $r$ is well-formed, variable $x$ occurs in a relational
atom ${A \in \body{r}}$; since $R$ does not occur in $\Rfin$, no fact that can
be match $A$ can contain constant $c$, so $x$ cannot be matched to a term
containing $c$. In other words, facts ${c \equals c}$ and ${f(\vec u) \equals
t}$ mentioned earlier cannot contribute to a derivation made by $r$. Finally,
if $\head{r}$ contains $\predQ$, then the body of $r$ consists of a well-formed
conjunction $\phi$ and equalities $x \equals x'$. Since $\phi$ is well-formed,
facts of the form ${c \equals c}$ and ${f(\vec u) \equals t}$ cannot be matched
to $\phi$ as mentioned above. Furthermore, for each body equality $x \equals
x'$, variable $x$ occurs in a relational atom in $\phi$ so $x$ cannot be
matched to a term containing $c$. Finally, $x'$ can potentially be matched to a
functional term, but then $r$ does not derive a fact containing only constants.
\end{proof}

    \section{Proof of Theorem~\ref{thm:magic}}\label{sec:proof:magic}

\thmMagic*

\begin{proof}
($\Leftarrow$) Each rule ${r' \in \R}$ where $\head{r'}$ does not contain a
magic predicate is obtained from some rule ${r \in P}$ by appending an atom
with a magic predicate to the body of $r$. Furthermore, each function symbol in
$\R$ occurs in $P$, which ensures ${\FR{\R} \subseteq \FR{P}}$. Thus, for each
fact $F$ not containing a magic predicate such that ${\R \cup \SG{\R} \cup B
\models F}$, we clearly have ${P \cup \SG{P} \cup B \models F}$.

\medskip

($\Rightarrow$) Let ${\Rstar = \DOM{P} \cup \Rfl \cup \ST \cup \FR{\R}}$. We
prove in two steps that ${P \cup \SG{P} \cup B \models \predQ(\vec a)}$ implies
${\R \cup \SG{\R} \cup B \models \predQ(\vec a)}$.

\smallskip

\emph{Step 1.} Let ${I^0, I^1, \dots}$ be the sequence of instances used to
compute $\fixpoint{P \cup \SG{P}}{B}$ as specified in
Section~\ref{sec:preliminaries}. We prove by induction on $i$ that, for each
fact ${R(\vec u) \in I^i}$ not of the form ${t \equals t}$ and each magic
predicate $\mgc{R}{\alpha}$ for $R$ occurring in the body of a rule in $\R$,
properties \ref{thm:magic:eqb} and \ref{thm:magic:other} hold:
\begin{enumerate}\renewcommand{\theenumi}{(M\arabic{enumi})}\renewcommand{\labelenumi}{\theenumi}
    \item\label{thm:magic:eqb}
    if ${\alpha = \adEqb}$ (and thus ${R = {\equals}}$), then ${\R \cup \Rstar
    \cup B \cup \{ \mgc{\equals}{\adEqb}(\vec u^\beta) \} \models
    {\equals}(\vec u)}$ for each ${\beta \in \{ \ad{bf}, \ad{fb} \}}$; and

    \item\label{thm:magic:other}
    if ${\alpha \neq \adEqb}$ (and thus ${R \neq {\equals}}$), then ${\R \cup
    \Rstar \cup B \cup \{ \mgc{R}{\alpha}(\vec u^\alpha) \} \models R(\vec u)}$.
\end{enumerate}

The base case holds trivially for each fact in ${I^0 = B}$. For the induction
step, assume that $I^i$ satisfies properties \ref{thm:magic:eqb} and
\ref{thm:magic:other}, and consider an arbitrary fact ${R(\vec u) \in I^{i+1}
\setminus I^i}$ derived by a rule ${r \in P \cup \DOM{P} \cup \Rfl \cup \ST
\cup \FR{P}}$ and substitution $\sigma$---that is, ${\sigma(\body{r}) \subseteq
I^i}$ and ${R(\vec u) = \head{r}\sigma}$. The inductive claim holds vacuously
if $R(\vec u)$ is of the form ${t \equals t}$. Now assume that $R(\vec u)$ is
not of the form ${t \equals t}$, so $r$ cannot be the reflexivity rule
\eqref{eq:ref}. We consider the remaining possibilities for $r$.

Assume ${r \in P}$ and consider an arbitrary magic predicate $\mgc{R}{\alpha}$
for $R$ that occurs in the body of a rule in $\R$; moreover, consider arbitrary
${\beta \in \{ \ad{bf}, \ad{fb} \}}$ if ${\alpha = \adEqb}$ (and thus ${R =
{\equals}}$), and let ${\beta = \alpha}$ otherwise. Since $\mgc{R}{\alpha}$
occurs in the body of a rule in $\R$, predicate $\mgc{R}{\alpha}$ was extracted
at some point from $\mathcal{T}$ in line~\ref{alg:magic:R:start}, so rule $r$
was considered in line~\ref{alg:magic:r:start}. Thus, the algorithm calls
$\mathsf{process}(r,\alpha,\beta)$ in
lines~\ref{alg:magic:process:bf}--\ref{alg:magic:process:fb} or
line~\ref{alg:magic:process}. Let $R(\vec t)$ be the head of $r$; thus, ${\vec
u = \vec t\sigma}$. Also, let ${\langle R_1(\vec t_1), \dots, R_n(\vec t_n)
\rangle}$ and ${\langle \gamma_1, \dots, \gamma_n \rangle}$ be the permutation
of atoms and the accompanying adornments produced by applying the sideways
information passing to $\body{r}$ and $\vars{\vec t^\beta}$ in
line~\ref{alg:magic:SIPS}. Note that conditions \eqref{eq:SIPS:1} and
\eqref{eq:SIPS:2} can always be satisfied for ${r \in P}$: conjunction
$\body{r}$ consists of a well-formed conjunction $\phi$ and zero or more
equality atoms of the form ${x \equals x'}$ where $x$ occurs in $\phi$; thus,
each variable occurring in an equality atom of $\phi$ occurs in a relational
atom of $\phi$ so we can always place these relational atoms before the
relevant equality atom in the order. Furthermore, the rule added to $\R$ in
line~\ref{alg:magic:mod-rule} contains all constants, predicates, and function
symbols of $r$, so $\Rstar$ contains the domain rule \eqref{eq:dom-c} for each
constant in $r$, the domain rule \eqref{eq:dom-R} for each predicate in $r$,
and the $\D$-restricted functional reflexivity rule \eqref{eq:Dfnref} for each
function symbol in $r$. We next prove by another induction on ${0 \leq j \leq
n}$ the following property:
\begin{quote}
    ($\ast$): ${\R \cup \Rstar \cup B \cup \{ \mgc{R}{\alpha}(\vec u^\beta) \} \models \mgc{R}{\alpha}(\vec t^\beta)\sigma \wedge R_1(\vec t_1)\sigma \wedge \dots \wedge R_j(\vec t_j)\sigma}$.
\end{quote}
The base case for $j=0$ is trivial because ${\vec u^\beta = \vec
t^\beta\sigma}$. Thus, we consider arbitrary ${j \in \{ 1, \dots, n \}}$ and
assume that property ($\ast$) holds for $j-1$. We consider the possible forms
of fact $R_j(\vec t_j)\sigma$.
\begin{itemize}
    \item Assume ${R_j(\vec t_j)\sigma}$ is of the form ${t \equals t}$; thus,
    $R_j(\vec t_j)$ is of the form ${t_1 \equals t_2}$ and ${\gamma_j \in \{
    \ad{bf}, \ad{fb} \}}$. We next consider ${\gamma_j = \ad{bf}}$; the case of
    ${\gamma_j = \ad{fb}}$ is analogous. For each constant $c$ in $t_1$, set
    $\Rstar$ contains a domain rule \eqref{eq:dom-c} that ensures ${\R \cup
    \Rstar \cup B \cup \{ \mgc{R}{\alpha}(\vec u^\beta) \} \models \D(c)}$.
    Furthermore, for each variable ${x \in \vars{t_1}}$, property
    \eqref{eq:SIPS:2} of Definition~\ref{def:SIPS} ensures that there exists a
    relational atom $R_{j'}(\vec t_{j'})$ with ${j' < j}$ such that ${x \in
    \vec t_{j'}}$; but then, the inductive property \ref{thm:magic:other}
    ensures ${\R \cup \Rstar \cup B \cup \{ \mgc{R}{\alpha}(\vec u^\beta) \}
    \models R_{j'}(\vec t_{j'})\sigma}$, so the domain rules \eqref{eq:dom-R}
    for $R_{j'}$ and the reflexivity rule \eqref{eq:ref} in $\Rstar$ ensure
    ${\R \cup \Rstar \cup B \cup \{ \mgc{R}{\alpha}(\vec u^\beta) \} \models
    \D(x)\sigma}$. Thus, for each subterm $s$ of $t_1$ that is a variable or a
    constant, the reflexivity rule \eqref{eq:ref} of $\Rstar$ ensures ${\R \cup
    \Rstar \cup B \cup \{ \mgc{R}{\alpha}(\vec u^\beta) \} \models s\sigma
    \equals s\sigma}$; hence, ($\ast$) holds if $t_1$ is a variable or a
    constant. Finally, if $t_1$ is a functional term of the form $f(\vec s)$,
    then our observation about the proper subterms of $t_1$ and the
    $\D$-restricted functional reflexivity rule \eqref{eq:Dfnref} of $\Rstar$
    ensure ${\R \cup \Rstar \cup B \cup \{ \mgc{R}{\alpha}(\vec u^\beta) \}
    \models f(\vec s)\sigma \equals f(\vec s)\sigma}$, so property ($\ast$)
    holds as well.

    \item Assume ${R_j \neq {\equals}}$ and that $R_j$ does not occur in the
    head of a rule in $P$. Then ${R_j(\vec t_j)\sigma \in B}$, so ($\ast$)
    holds trivially.

    \item Assume ${R_j(\vec t_j)\sigma}$ is not of the form ${t \equals t}$ and
    that $R_j$ is processed in
    lines~\ref{alg:magic:body:start}--\ref{alg:magic:body:end}. Let $S$ be the
    magic predicate introduced in line~\ref{alg:magic:S}. The magic rule added
    to $\R$ in line~\ref{alg:magic:magic-rule} and property ($\ast$) together
    ensure ${\R \cup \Rstar \cup B \cup \{ \mgc{R}{\alpha}(\vec u^\beta) \}
    \models S(\vec t_j^{\gamma_j})\sigma}$, so ${\R \cup \Rstar \cup B \cup \{
    \mgc{R}{\alpha}(\vec u^\beta) \} \models R_j(\vec t_j)\sigma}$ holds by
    ${R_j(\vec t_j)\sigma \in I^i}$ and the inductive assumption for properties
    \ref{thm:magic:eqb} and \ref{thm:magic:other}.
\end{itemize}
This completes the proof of ($\ast$). Thus, if ${R = {\equals}}$ and ${\alpha =
\adEqb}$, then ($\ast$) and the rule added to $\R$ in
line~\ref{alg:magic:mod-rule} ensure that $R(\vec u)$ satisfies property
\ref{thm:magic:eqb}; in all other cases, ($\ast$) and the rule added to $\R$ in
line~\ref{alg:magic:mod-rule} ensure that $R(\vec u)$ satisfies property
\ref{thm:magic:other}.

Assume ${r \in \DOM{P}}$---that is, $r$ is of the form ${\D(x_i) \leftarrow
T(\vec x)}$ and ${R(\vec u) = \D(x_i)\sigma}$. Rule $r$ is processed in
line~\ref{alg:magic:process}, so, in the same way as in ($\ast$), we have ${\R
\cup \Rstar \cup B \cup \{ \mgc{\D}{\alpha}(x_i^\beta) \} \models T(\vec
x)\sigma}$. But then, ${r \in \Rstar}$ ensures that fact $\D(x_i)\sigma$
satisfies property \ref{thm:magic:other}.

In all remaining cases, $\head{r}$ is an equality atom, so $\mgc{R}{\alpha}$ is
of the form $\mgc{\equals}{\adEqb}$. Furthermore, we assumed that
$\mgc{R}{\alpha}$ occurs in the body of a rule in $\R$, so the condition in
line~\ref{alg:magic:eq:start} was satisfied; thus, $\R$ contains the rules from
lines~\ref{alg:magic:eq:mgc:1}--\ref{alg:magic:eq:end} of
Algorithm~\ref{alg:magic}.

Assume ${r \in \FR{P}}$ for an $n$-ary function symbol $f$.
Lines~\ref{alg:magic:eq:mgc:1}--\ref{alg:magic:eq:end} ensure that $f$ occurs
in $\R$, so we have ${r \in \FR{\R}}$. We next consider the case ${\beta =
\ad{bf}}$. We next argue that the following entailments hold for each ${j \in
\{ 1, \dots, n\}}$.
\begin{displaymath}
\begin{array}{@{}l@{\;}l@{}}
    \R \cup \Rstar \cup B \cup \{ \mgc{\equals}{\adEqb}(\vec u^\ad{bf}) \} \models  & \mgc{\equals}{\adEqb}(f(x_1,\dots,x_n))\sigma \; \wedge \\
                                                                                    & \mgc{\D}{\ad{b}}(x_j)\sigma \wedge \D(x_j)\sigma \;\wedge \\
                                                                                    & \mgc{\equals}{\adEqb}(x_j)\sigma \wedge x_j\sigma \equals x_j'\sigma \wedge \\
                                                                                    & \mgc{\D}{\ad{b}}(x_j')\sigma \wedge \D(x_j')\sigma \\
\end{array}
\end{displaymath}
Fact $\mgc{\equals}{\adEqb}(f(x_1,\dots,x_n))\sigma$ is entailed trivially
because ${\vec u^\ad{bf} = f(x_1,\dots,x_n)\sigma}$ by our assumption. This and
the rule added to $\R$ in line~\ref{alg:magic:eq:mgc:2} ensures that
$\mgc{\D}{\ad{b}}(x_j)\sigma$ is entailed, which by the inductive property
\ref{thm:magic:other} ensures that $\D(x_j)\sigma$ is entailed. But then, the
rule added to $\R$ in line~\ref{alg:magic:eq:mgc:3} ensures that
$\mgc{\equals}{\adEqb}(x_j)\sigma$ is entailed. Fact ${x_j\sigma \equals
x_j'\sigma}$ is entailed by
\begin{itemize}
    \item $\D(x_j)\sigma$ and the reflexivity rule \eqref{eq:ref} in $\Rstar$
    if ${x_j\sigma = x_j'\sigma}$, and

    \item $\mgc{\equals}{\adEqb}(x_j)\sigma$ and the inductive property
    \ref{thm:magic:eqb} otherwise.
\end{itemize}
Fact $\mgc{\D}{\ad{b}}(x_j')\sigma$ is entailed by the rule added to $\R$ in
line~\ref{alg:magic:eq:mgc:4}, so the inductive property \ref{thm:magic:other}
ensures that $\D(x_j')\sigma$ is entailed. These entailments hold for each ${j
\in \{ 1, \dots, n\}}$, so ${r \in \SG{P}}$ ensures that
${f(x_1,\dots,x_n)\sigma \equals f(x_1',\dots,x_n')\sigma}$ satisfies property
\ref{thm:magic:eqb}. The case for ${\beta = \ad{fb}}$ is analogous because
$\Rstar$ contains the symmetry rule \eqref{eq:sym}.

Assume that $r$ is the symmetry rule \eqref{eq:sym}, and furthermore assume
that ${\beta = \ad{bf}}$. We assumed that $R(\vec u)$ is not of the form ${t
\equals t}$, so we have ${x_1\sigma \neq x_2\sigma}$. Thus, the inductive
property \ref{thm:magic:eqb} ensures ${\R \cup \Rstar \cup B \cup \{
\mgc{\equals}{\adEqb}(\vec u^\ad{bf}) \} \models x_1\sigma \equals x_2\sigma}$;
but then, the symmetry rule \eqref{eq:sym} in $\Rstar$ ensures ${\R \cup \Rstar
\cup B \cup \{ \mgc{\equals}{\adEqb}(\vec u^\ad{bf}) \} \models x_2\sigma
\equals x_1\sigma}$. The case for ${\beta = \ad{fb}}$ is analogous.

Assume that $r$ is the transitivity rule \eqref{eq:trans}, and furthermore
assume that ${\beta = \ad{bf}}$. If ${x_1\sigma = x_2\sigma}$ or ${x_2\sigma =
x_3\sigma}$, then $r$ derives ${x_2\sigma \equals x_3\sigma}$ or ${x_1\sigma
\equals x_2\sigma}$, respectively; however, this contradicts our assumption
that ${x_1\sigma \equals x_3\sigma \in I^{i+1} \setminus I^i}$. Thus,
${x_1\sigma \neq x_2\sigma}$ and ${x_2\sigma \neq x_3\sigma}$ holds. We now
show that the following entailments hold.
\begin{displaymath}
    \R \cup \Rstar \cup B \cup \{ \mgc{\equals}{\adEqb}(\vec u^\ad{bf}) \} \models \mgc{\equals}{\adEqb}(x_1)\sigma \wedge x_1\sigma \equals x_2\sigma \wedge \mgc{\equals}{\adEqb}(x_2)\sigma \wedge x_2\sigma \equals x_3\sigma
\end{displaymath}
Fact $\mgc{\equals}{\adEqb}(x_1)\sigma$ is entailed trivially because our
assumption ensures ${\vec u^\ad{bf} = x_1\sigma}$. This, ${x_1\sigma \neq
x_2\sigma}$, and the inductive property \ref{thm:magic:eqb} entail ${x_1\sigma
\equals x_2\sigma}$. Then, the rule added to $\R$ in
line~\ref{alg:magic:eq:mgc:1} entails $\mgc{\equals}{\adEqb}(x_2)\sigma$, which
together with ${x_2\sigma \neq x_3\sigma}$ and the inductive property
\ref{thm:magic:eqb} entail ${x_2\sigma \equals x_3\sigma}$. Finally, the
transitivity rule \eqref{eq:trans} in $\Rstar$ ensures that the fact
${x_1\sigma \equals x_3\sigma}$ satisfies property \ref{thm:magic:eqb}. The
case of ${\beta = \ad{fb}}$ is analogous because $\Rstar$ contains the symmetry
rule \eqref{eq:sym}.

\smallskip

\emph{Step 2.} Let ${\alpha = f \cdots f}$. Line~\ref{alg:magic:init} of
Algorithm~\ref{alg:magic} ensures ${\mgc{\predQ}{\alpha} \leftarrow \; \in \R}$
and ${\mgc{\predQ}{\alpha} \in \Di{}}$. The former ensures ${\R \cup \Rstar
\cup B \models \mgc{\predQ}{\alpha}}$. The latter ensures that
$\mgc{\predQ}{\alpha}$ is extracted from $\Di{}$ in
line~\ref{alg:magic:R:start} and passed to the $\process$ function in
line~\ref{alg:magic:process}; thus, line~\ref{alg:magic:mod-rule} ensures that
$\mgc{\predQ}{\alpha}$ occurs in the body of a rule in $\R$. But then, for each
fact of the form ${\predQ(\vec a)}$ such that ${P \cup \SG{P} \cup B \models
\predQ(\vec a)}$, property \ref{thm:magic:other} ensures ${\R \cup \Rstar \cup
B \models \predQ(\vec a)}$. We now argue that ${\R \cup \SG{\R} \cup B \models
\predQ(\vec a)}$ holds as well. In particular, note that ${\SG{\R} \subseteq
\Rstar}$, and that ${\Rstar \setminus \SG{\R}}$ consists of zero or more domain
rules of the form \eqref{eq:dom-c} for a constant $c$ not occurring in $\R$,
and domain rules of the form \eqref{eq:dom-R} for a predicate $R$ not occurring
in $\R$. Such rules can derive facts of the form $\D(c)$ for a constant $c$
occurring in the base instance $B$ in a facts with a predicate not occurring in
${\R \cup \SG{\R}}$, which can in turn lead to the derivation of facts of the
form ${c \equals c}$ and ${f(\vec u) \equals t}$ where $\vec u$ contain such
constant $c$. However, the rules of $\R$ not containing the magic predicate in
the head are well-formed, so they cannot apply to such facts in exactly the
same way as in the proof of Theorem~\ref{thm:relevance}. Hence, facts ${c
\equals c}$ and ${f(\vec u) \equals t}$ cannot contribute to the derivation of
the fact $\predQ(\vec a)$.
\end{proof}

    \section{Proof of Theorem~\ref{thm:magic:termination}}\label{sec:proof:magic:termination}

\thmMagicTermination*

\begin{proof}
Let $M$ be the maximum depth of an atom in ${\fixpoint{P_1}{B}}$, and let
${I^0, I^1, \dots}$ be the sequence of instances used to compute
${\fixpoint{P_2}{B}}$ as defined in Section~\ref{sec:preliminaries}. We show by
induction on $k$ that, for each fact ${F \in I^k}$,
\begin{itemize}
    \item ${\dep{F} \leq M+1}$ if $F$ contains a magic predicate, and

    \item ${\dep{F} \leq M}$ and ${P_1 \cup B \models F}$ otherwise.
\end{itemize}
The base case ${I^0 = B}$ is trivial, so we assume that the claim holds for
some $I^k$ and consider a fact ${\head{r}\sigma \in I^{k+1} \setminus I^k}$
derived by a rule ${r \in P_2}$ and substitution $\sigma$ satisfying
${\body{r}\sigma \subseteq I^k}$. We consider the possible forms of the rule
$r$.

Assume that $r$ was added to $\R$ in line~\ref{alg:magic:mod-rule}. Then, $P_1$
contains a rule $r'$ such that ${\body{r'} \subseteq \body{r}}$, and the
induction assumption ensures ${P_1 \cup B \models \body{r'}\sigma}$. Thus, we
have ${P_1 \cup B \models F}$ and consequently fact $F$ is of depth at most $M$.

Assume that ${r \in \SG{\R}}$. Definition~\ref{def:magic-SG} ensures that
domain rules \eqref{eq:dom-R} are not instantiated in $\SG{\R}$ for the magic
predicates, so ${r \in P_1}$ and the property holds as in the previous
paragraph.

Assume that $r$ was added to $\R$ in line~\ref{alg:magic:magic-rule} for some
${i \in \{ 1, \dots, n \}}$. Let $R(\vec t)$ be as in
line~\ref{alg:magic:fn-process}, let ${\langle R_1(\vec t_1), \dots, R_n(\vec
t_n) \rangle}$ be as in line~\ref{alg:magic:SIPS}, and let $S(\vec
t_i^{\gamma_i})$ be as in line~\ref{alg:magic:magic-rule} for $i$. The
inductive assumption ensures ${\dep{\mgc{R}{\alpha}(\vec t^\beta)\sigma} \leq M
+ 1}$ and ${\dep{R_j(\vec t_j)\sigma} \leq M}$ for each $j$ with ${1 \leq j <
i}$. We consider the possible forms of atom $R_i(\vec t_i)$.
\begin{itemize}
    \item Assume predicate $R_i$ is relational. Since program $P$ is
    well-formed, terms $\vec t_i$ do not contain a function symbol, so atom
    $S(\vec t_i^{\gamma_i})$ does not contain a function symbol either.
    Moreover, each variable of $S(\vec t_i^{\gamma_i})$ occurs in
    $\mgc{R}{\alpha}(\vec t^\beta)$ or some $R_j(\vec t_j)$ with ${1 \leq j <
    i}$ due to property \eqref{eq:SIPS:1} of Definition~\ref{def:SIPS}, so the
    variable is mapped to a term of depth at most $M+1$. Thus, we have
    ${\dep{S(\vec t_i^{\gamma_i})\sigma} \leq M+1}$, as required.

    \item Assume ${R_i = {\equals}}$ so atom $R_i(\vec t_i)$ is of the form
    ${t_i^1 \equals t_i^2}$. We consider only the case of ${\beta = \ad{bf}}$
    as the case for ${\beta = \ad{fb}}$ is analogous. Property
    \eqref{eq:SIPS:2} of Definition~\ref{def:SIPS} ensures that, for each ${x
    \in \vars{t_i^1}}$, there exists a relational atom $R_j(\vec t_j)$ such
    that ${1 \leq j < i}$ and ${x \in \vars{\vec t_j}}$; now, we have already
    shown that ${\dep{R_j(\vec t_j)\sigma} \leq M}$, so ${\dep{x\sigma} \leq
    M}$ holds as well. Consequently, atom $S(\vec t_i^{\gamma_i})$ is of depth
    at most one, so ${\dep{S(\vec t_i^{\gamma_i})\sigma} \leq M+1}$.
\end{itemize}

Assume that $r$ was added to $\R$ in
lines~\ref{alg:magic:eq:mgc:1}--\ref{alg:magic:eq:end}. The head of $r$ then
contains a magic predicate; moreover, each term occurring in $\head{r}$ is a
variable occurring in $\body{r}$, so ${\dep{\head{r}\sigma} \leq M+1}$ clearly
holds by the inductive hypothesis.
\end{proof}

    \section{Proof of Theorem~\ref{thm:desg}}\label{sec:proof:desg}

\thmDESG*

\begin{proof}
Let $\Sigma_1$, $\Sigma_2$ and $P_3$--$P_7$ be as specified in
Algorithm~\ref{alg:answer-query}. Program $P_6$ is safe, and removal of
function symbols ensures that all equalities in the body involve only
variables. Desingularisation is thus applicable to each body equality atom, and
applying it removes the atom while also preserving safety. Moreover, for each
fact of the form $\predQ(\vec a)$, we have ${\{ \Sigma \} \cup B \modelsEq
\predQ(\vec a)}$ if and only if ${P_6 \cup \SG{P_6} \cup B \models \predQ(\vec
a)}$; this holds by Proposition~\ref{prop:fol}, Theorem~\ref{thm:sg},
Proposition~\ref{prop:sk}, Theorem~\ref{thm:relevance},
Theorem~\ref{thm:magic}, and Proposition~\ref{prop:defun}. Now consider
arbitrary fact $\predQ(\vec a)$.

\medskip

($\Leftarrow$) Let ${P' = \sk{\fol{\Sigma}}}$ and ${\Rstar = \EQ{P'} \cup
\EQ{P_7}}$. The properties of equality axiomatisation ensure ${\{ \Sigma \}
\cup B \modelsEq \predQ(\vec a)}$ if and only if ${P' \cup \EQ{P'} \cup B
\models \predQ(\vec a)}$, and the observation about the relationship between
$\Sigma$ and $P_6$ from the previous paragraph ensures ${P' \cup \EQ{P'} \cup B
\models \predQ(\vec a)}$ if and only if ${P_6 \cup \SG{P_6} \cup B \models
\predQ(\vec a)}$. The entailment relationship of logic programming then ensures
${P' \cup \EQ{P'} \cup \Rstar \cup B \models \predQ(\vec a)}$ if and only if
${P_6 \cup \SG{P_6} \cup \Rstar \cup B \models \predQ(\vec a)}$ as well;
moreover, for each rule ${r \in \SG{P_6}}$, there exists a rule ${r' \in
\Rstar}$ such that ${\head{r'} = \head{r}}$ and ${\body{r'} \subseteq
\body{r}}$, so $\SG{P_6}$ is redundant. All of these observations together
ensure ${\{ \Sigma \} \cup B \modelsEq \predQ(\vec a)}$ if and only if ${P'
\cup \Rstar \cup B \models \predQ(\vec a)}$ if and only if ${P_6 \cup \Rstar
\cup B \models \predQ(\vec a)}$. Now assume that ${P_7 \cup \EQ{P_7} \cup B
\models \predQ(\vec a)}$, and let ${I^0, I^1, \dots}$ be the sequence of
instances used to compute $\fixpoint{P_7 \cup \EQ{P_7}}{B}$ as defined in
Section~\ref{sec:preliminaries}. To prove ${P_6 \cup \Rstar \cup B \models
\predQ(\vec a)}$, we show by induction on $i$ that ${P_6 \cup \Rstar \cup B
\models F}$ for each fact ${F \in I^i}$. The base case is obvious, so assume
that this property holds for $I^i$ and consider a fact ${\head{r}\sigma \in
I^{i+1} \setminus I^i}$ derived by a rule ${r \in P_7 \cup \EQ{P_7}}$ and
substitution $\sigma$ such that ${\sigma(\body{r}) \subseteq I^i}$. If ${r \in
\EQ{P_7} \subseteq \Rstar}$, we clearly have ${P_6 \cup \Rstar \models
\head{r}\sigma}$. Furthermore, if ${r \in P_7}$, then there exists a rule ${r'
\in P_6}$ such that $r$ is obtained from $r'$ by desingularisation. Let
$\sigma'$ be the substitution obtained from $\sigma$ by setting ${\sigma'(x) =
t\sigma}$ for each variable $x$ replaced with term $t$ when
Definition~\ref{def:desg} is applied to $r$. For each relational body atom
${R(\vec t) \in \body{r}}$, the inductive assumption and ${R(\vec t)\sigma \in
I^i}$ ensure ${P_6 \cup \Rstar \models R(\vec t)\sigma}$; moreover, rule $r'$
contains a body atom ${R'(\vec t') \in \body{r'}}$ that satisfies ${R(\vec
t)\sigma = R'(\vec t')\sigma'}$. Furthermore, as mentioned at the beginning of
this proof, each equality in $\body{r'}$ is of the form ${x_1 \equals x_2}$;
moreover, $r'$ is well-formed, so $x_1$ and $x_2$ occur in relational atoms in
$\body{r'}$. Also, the definition of $\sigma'$ ensures ${x_1\sigma' =
x_2\sigma'}$, so the domain rules \eqref{eq:dom-R} and the reflexivity rule
\eqref{eq:ref} in $\Rstar$ ensure ${P_6 \cup \Rstar \models x_1\sigma' \equals
x_2\sigma'}$. Thus, we have ${P_6 \cup \Rstar \models \body{r'}\sigma'}$, which
ensures ${P_6 \cup \Rstar \models \head{r'}\sigma'}$. Finally, the
transformation from Definition~\ref{def:desg} ensures ${\head{r}\sigma =
\head{r}\sigma'}$.

\medskip

($\Rightarrow$) Assume ${P_6 \cup \SG{P_6} \cup B \models \predQ(\vec a)}$, and
let ${I^0, I^1, \dots}$ be the sequence of instances used to compute
$\fixpoint{P_6 \cup \SG{P_6}}{B}$ as defined in
Section~\ref{sec:preliminaries}. To prove ${P_7 \cup \EQ{P_7} \cup B \models
\predQ(\vec a)}$, we show by induction on $i$ that ${P_7 \cup \EQ{P_7} \cup B
\models F}$ for each fact ${F \in I^i}$. The base case ${I^0 = B}$ is obvious,
so assume that this property holds for $I^i$ and consider a fact
${\head{r}\sigma \in I^{i+1} \setminus I^i}$ derived by a rule ${r \in P_6 \cup
\SG{P_6}}$ and substitution $\sigma$ such that ${\sigma(\body{r}) \subseteq
I^i}$. First, assume that ${r \in \SG{P_6}}$. Then there exists a rule ${r' \in
\EQ{P_6} = \EQ{P_7}}$ such that ${\head{r'} = \head{r}}$ and ${\body{r'}
\subseteq \body{r}}$; the only case where ${\body{r'} \subsetneq \body{r}}$ is
when $r$ is a $\D$-restricted functional reflexivity rule \eqref{eq:Dfnref}.
Thus, we have ${P_7 \cup \EQ{P_7} \cup B \models \head{r}\sigma}$. Second,
assume that ${r \in P_6}$. The inductive assumption ensures ${P_7 \cup \EQ{P_7}
\cup B \models R(\vec t)\sigma}$ for each atom ${R(\vec t) \in \body{r}}$. Let
${r' \in P_7}$ be the rule obtained from $r$ as specified in
Definition~\ref{def:desg}. Now consider an arbitrary relational atom ${A \in
\body{r}}$, and let $A'$ be the corresponding atom in $\body{r'}$. Atom $A'$ is
obtained from $A$ by zero or more replacements of a variable $x$ with $t$ due
to ${x \equals t \in \body{r}}$. For each such ${x \equals t}$, we have
${x\sigma \equals t\sigma \in I^i}$, so the induction assumption ensures ${P_7
\cup \EQ{P_7} \cup B \models x\sigma \equals t\sigma}$. Moreover, we have
${A\sigma \in I^i}$, so the induction assumption ensures ${P_7 \cup \EQ{P_7}
\cup B \models A\sigma}$. But then, the congruence rules of $\EQ{P_7}$ clearly
ensure ${P_7 \cup \EQ{P_7} \cup B \models A'\sigma}$. Since this holds for
arbitrary ${A \in \body{r}}$ and the corresponding ${A' \in \body{r'}}$, we
have ${P_7 \cup \EQ{P_7} \cup B \models \head{r'}\sigma}$. Finally, note that
${\head{r'}\sigma = \head{r}\sigma}$, which ensures the claim.
\end{proof}

    \section{Generating Benchmarks for Second-Order Dependencies}\label{sec:generating}

In this appendix we discuss our approach for generating second-order
dependencies and datasets. We use the standard notion of a \emph{position},
which is a possibly empty sequence of integers that identifies an occurrence of
a subexpression in a term or atom. For example, given the atom $R(x,f(y,z),w)$,
the empty position identifies the atom itself; positions $1$, $2$, and $3$
identify the variable $x$, the term $f(y,z)$, and the variable $w$,
respectively; position $2.1$ identifies the variable $y$; and position $2.2$
identifies the variable $z$. Our approach is shown in
Algorithms~\ref{alg:generator}--\ref{alg:createRule}, where
Algorithm~\ref{alg:generator} provides the entry point. It takes a number of
parameters that control the generation process, and it produces a set $\R$ of
conjuncts of a generalised SO dependency and a base instance $B$. All elements
of $\R$ are of the form \eqref{eq:generalised-SO-dep:conjunct}, but with no
existential quantification over individual variables, and with just one atom
per head. In line with the terminology used thus far, we call the elements of
$\R$ \emph{rules}.

\begin{algorithm}[!tb]
\caption{$\benchmarkGenerator$}\label{alg:generator}
\begin{algorithmic}[1]\footnotesize
    \Statex \begin{tabular}{@{}l@{\;}l@{\;}l@{}}
                \textbf{Global:}    & $\maxPredArity$           & : the maximum predicate arity \\
                                    & $\maxFnSymArity$          & : the maximum function symbol arity \\
                                    & $\preds$                  & : the size of a set of predicates of arity at most $\maxPredArity$ \\
                                    & $\fnSyms$                 & : the size of a set of function symbols of arity at most $\maxFnSymArity$ \\
                                    & $\constants$              & : the size of a set of constants \\
                                    & $\seedQFacts$             & : the number of seed query facts \\
                                    & $\maxTermDepth$           & : the maximum term depth in derivations considered \\
                                    & $\maxNumberRules$         & : the maximum number of rules to produce in the first phase \\
                                    & $\maxNumberRulesPerFact$  & : the maximum number of rules that derive a certain fact \\
                                    & $\maxNumberRelBodyAtoms$  & : the maximum number of relational body atoms per rule \\
                                    & $\maxNumberEqBodyAtoms$   & : the maximum number of equality body atoms per rule \\
                                    & $\maxNumberTuples$        & : the `blowup' factor used to generate a dataset from the derivation leaves \\
            \end{tabular}
    \Statex
    \State $\R \defeq \emptyset$
    \State Set $\Gamma$ to the empty graph
    \State $\mathcal{D} \defeq \mathcal{T} \defeq \{ \predQ_1(\vec a_1), \dots, \predQ_{\seedQFacts}(\vec a_{\seedQFacts}) \}$ where vectors of constants $\vec a_i$ are randomly chosen    \label{alg:generator:seed}
    \While{$\mathcal{T} \neq \emptyset$ \textbf{and} $|\R| < \maxNumberRules$}                                                                                                              \label{alg:generator:1:start}
        \State Choose and remove a fact $F$ from $\mathcal{T}$
        \State Randomly choose a number $\numberRulesPerFact$ between one and $\maxNumberRulesPerFact$
        \For{$i \in \{ 1, \dots, \numberRulesPerFact \}$}
            \If{$F$ is an equality fact and this condition is randomly satisfied}
                 \State $\overline{\rho} \defeq \createGroundEqualityRule(F)$                                                                                                               \label{alg:generator:1:eq}
            \Else
                \State $\langle \rho, \overline{\rho} \rangle \defeq \createRule(F)$                                                                                                        \label{alg:generator:1:std}
                \If{the fixpoint of $\R \cup \{ \rho \} \cup \SG{\R \cup \{ \rho \}}$ and the critical instance cannot be proved finite}
                    \State \textbf{continue}
                \ElsIf{no rule in $\R$ subsumes $\rho$}
                    \State Remove from $\R$ each rule that is subsumed by $\rho$
                    \State Add $\rho$ to $\R$
                \EndIf
            \EndIf
            \For{\textbf{each} $F_i \in \body{\overline{\rho}}$}                                                                                                                            \label{alg:generator:1:body:start}
                \If{$F_i \not\in \mathcal{D}$}
                    \State Add $F_i$ to $\mathcal{T}$ and $\mathcal{D}$
                    \State Extend $\Gamma$ to contain vertices $F_i$ and $F$, as well as an edge from $F_i$ to $F$
                \EndIf
            \EndFor                                                                                                                                                                         \label{alg:generator:1:body:end}
        \EndFor
    \EndWhile                                                                                                                                                                               \label{alg:generator:1:end}
    \State $B \defeq \emptyset$
    \State $\mathcal{D} \defeq \mathcal{T} \defeq \{ F \mid F \text{ has zero incoming edges in } \Gamma \}$                                                                                \label{alg:generator:2:start}
    \While{$\mathcal{T} \neq \emptyset$}
        \State Choose and remove a fact $S(t_1,\dots,t_n)$ from $\mathcal{T}$
        \If{at least one argument of $S(t_1,\dots,t_n)$ contains a function symbol}
            \State $\langle \rho, \overline{\rho} \rangle \defeq \createTransferRule(S(t_1,\dots,t_n))$                                                                                     \label{alg:generator:2:transfer}
            \State Add $\rho$ to $\R$
            \State Let $R(u_1,\dots,u_m)$ be the (only) atom in $\body{\overline{\rho}}$
            \If{$R(u_1,\dots,u_m) \not\in \mathcal{D}$}
                \State Add $R(u_1,\dots,u_m)$ to $\mathcal{T}$ and $\mathcal{D}$
            \EndIf
        \Else
            \For{\textbf{each} $i \in \{ 1, \dots, \maxNumberTuples \}$}
                \State Add $S(t_1^{(i)},\dots,t_n^{(i)})$ to $B$
            \EndFor
        \EndIf
    \EndWhile                                                                                                                                                                               \label{alg:generator:2:end}
    \State \Return $\langle \R, B \rangle$
\end{algorithmic}
\end{algorithm}

\begin{algorithm}[!tb]
\caption{$\createTransferRule(S(t_1,\dots,t_n))$}\label{alg:createTransferRule}
\begin{algorithmic}[1]\footnotesize
    \State $\vec s \defeq \vec s' \defeq \vec t' \defeq \emptyset$
    \For{\textbf{each } $i \in \{ 1, \dots, n \}$}
        \If{$t_i$ is of the form $f(u_1,\dots,u_m)$}
            \State Let $\vec x$ be a tuple of $m$ fresh variables
            \State Append $f(\vec x)$ to $\vec s$, append $\vec x$ to $\vec s'$, and append $u_1,\dots,u_m$ to $\vec t'$
        \Else
            \State Let $x$ be a fresh variable
            \State Append $x$ to both $\vec{s}$ and $\vec{s}'$, and append $t_i$ to $\vec{t}'$
        \EndIf
    \EndFor
    \State Let $R$ be a fresh predicate of arity $|\vec s'|$
    \State $\rho \defeq S(\vec s) \leftarrow R(\vec s')$ \quad and \quad $\overline{\rho} \defeq S(t_1, \dots, t_n) \leftarrow R(\vec t')$
    \State \Return $\langle \rho, \overline{\rho} \rangle$
\end{algorithmic}
\end{algorithm}

Several steps of the algorithm involve randomly selecting predicates, function
symbols, and constants, as well as generating ground terms. These choices are
governed as follows. To select predicates, the algorithm maintains a cache of
previously selected predicates; then, when a predicate is to be selected, the
algorithm either randomly reuses a predicate from the cache, or it chooses a
new predicate of maximum arity $\maxPredArity$ out of a fixed set of $\preds$
predicates. Function symbols and constants are handled analogously; the number
and the arity of function symbols are bounded by $\fnSyms$ and
$\maxFnSymArity$, respectively, and the number of constants is bounded by
$\constants$. In addition, several steps of the algorithm also involve
generating ground terms of depth at most $d$. To this end, the algorithm
maintains a cache of previously seen terms, and then either randomly reuses a
term from the cache, selects a constant if $d$ is zero, or selects a function
symbol and recursively generates arguments of depth at most $d-1$.

The algorithm initialises in line~\ref{alg:generator:seed} the `ToDo' set
$\mathcal{T}$ with $\seedQFacts$ randomly generated \emph{seed} facts of the
form $\predQ_i(\vec a_i)$. Predicates $\predQ_i$ are all different, while
arguments $\vec a_i$ are randomly generated as outlined earlier and are thus
likely to overlap. The algorithm then randomly produces chains of rule
instances that derive the seed facts, and lifts these instances to rules with
variables. It also produces a base instance $B$ while ensuring that at least
$\maxNumberTuples$ copies of each seed fact are derived when $\R$ is applied to
$B$. Since the arguments of the seed facts overlap, the derivations of the seed
facts are likely to overlap as well. Thus, by using $\predQ_i$ as query
predicates, we can check how well our techniques from
Section~\ref{sec:answering} identify the inferences relevant to a query. The
algorithm realises this idea in two phases. In the first phase, it generates up
to $\maxNumberRules$ rules matching the form of equation
\eqref{eq:generalised-SO-dep:conjunct}, and it also constructs an acyclic graph
$\Gamma$ that records the derivation dependencies: the vertices of $\Gamma$ are
facts occurring in the rule instances, and the edges reflect which facts are
used to produce other facts. In the second phase, the algorithm `unfolds' the
function symbols in the facts occurring in the leaves of $\Gamma$ and populates
the base instance $B$. We next describe the two phases in detail.

The first phase comprises
lines~\ref{alg:generator:1:start}--\ref{alg:generator:1:end}, where the
algorithm processes each fact $F$ in the `ToDo' list $\mathcal{T}$ and
generates between one and $\maxNumberRulesPerFact$ ground rules
$\overline{\rho}$ that derive $F$ (i.e., that have $F$ in the head). Each
$\overline{\rho}$ can be either a rule that axiomatises equality
(line~\ref{alg:generator:1:eq}) or a conjunct of a generalised SO dependency
(line~\ref{alg:generator:1:std}). In the latter case, a nonground rule $\rho$
is also produced such that $\overline{\rho}$ is an instance of $\rho$, and the
rule $\rho$ is added to $\R$ provided that it passes two checks. First, the
algorithm verifies whether the fixpoint of ${\R \cup \{ \rho \}}$ is finite,
which can be determined using sound but incomplete checks such as
model-summarising or model-faithful acyclicity by
\citet{DBLP:journals/jair/GrauHKKMMW13}. Second, the algorithm checks that
$\rho$ is not subsumed by another rule in $\R$ and is thus not redundant. If
both checks are satisfied, then all rules in $\R$ subsumed by $\rho$ are
removed in order to keep $\R$ free of redundant rules, and $\rho$ is added to
$\R$. Finally, all facts from the body of the nonground rule $\overline{\rho}$
are added to the `ToDo' set $\mathcal{T}$, and set $\mathcal{D}$ is used to
make sure that no fact is processed more than once
(lines~\ref{alg:generator:1:body:start}--\ref{alg:generator:1:body:end}). Also,
the graph $\Gamma$ is updated to record that the facts in the body of
$\overline{\rho}$ derive $F$.

The second phase consists of
lines~\ref{alg:generator:2:start}--\ref{alg:generator:2:end}. It generates the
base instance $B$ and the rules that use the facts in $B$ to `feed' the rule
instances considered in the first phase. The `ToDo' set $\mathcal{T}$ is
initialised to contain all facts occurring in the leaves of $\Gamma$, and the
algorithm next considers each fact $F$ in $\mathcal{T}$. If $F$ contains no
function symbols, the algorithm adds $\maxNumberTuples$ copies of $F$ to the
base instance $B$. Otherwise, Algorithm~\ref{alg:createTransferRule} is used to
`unfold' top-level function symbols in $F$ using so-called `transfer' rules.
Specifically, the latter algorithm produces a ground rule $\overline{\rho}$ of
the form ${F \leftarrow F'}$ such that all functional terms in $F'$ are of
depth one less than in $F$, as well as a corresponding nonground rule $\rho$.
Fact $F'$ is then added to the `ToDo' set so that it may be `unfolded' further
or added to the base instance. This is best understood by an example where $F$
is $S(a,f(g(b)))$. Algorithm~\ref{alg:createTransferRule} produces a ground
`transfer' rule ${S(a,f(g(b))) \leftarrow X(a,g(b))}$ and the corresponding
nonground rule ${S(x,f(y)) \leftarrow X(x,y)}$. Fact $X(a,g(b))$ contains a
function symbol, so it gives rise to a further `transfer' rule ${X(a,g(b))
\leftarrow Y(a,b)}$ and the corresponding nonground rule ${X(x,g(y)) \leftarrow
Y(x,y)}$. Finally, fact $Y(a,b)$ contains no function symbols, so
$\maxNumberTuples$ copies of $Y(a,b)$ are added to $B$. Predicates $X$ and $Y$
are fresh---that is, they are not chosen randomly as outlined earlier so that
the inferences by transfer rules do not interact with the inferences of the
rules generated in the first phase.

What remains to be discussed is how the rules deriving $F$ in
lines~\ref{alg:generator:1:eq} and~\ref{alg:generator:1:std} are produced. For
the former, Algorithm~\ref{alg:createGroundEqualityRule} randomly selects a
functional reflexivity rule, a transitivity rule, or a symmetry rule. In the
case of transitivity, the algorithm randomly generates a term $t_3$ of depth at
most $\maxTermDepth$. Equality axiomatisation is not meant to be included in
the resulting program $\R$, so Algorithm~\ref{alg:createGroundEqualityRule}
creates only a ground rule $\overline{\rho}$ that derives $F$.

\begin{algorithm}[!tb]
\caption{$\createGroundEqualityRule(t_1 \equals t_2)$}\label{alg:createGroundEqualityRule}
\begin{algorithmic}[1]\footnotesize
    \If{$t_1 = f(u_1,\dots,u_n)$ and $t_2 = f(u_1',\dots,u_n')$ with $f$ a true function symbol, and this condition is randomly satisfied}
        \State \Return $f(u_1,\dots,u_n) \equals f(u'_1,\dots,u'_n) \leftarrow u_1 \equals u_1' \wedge \dots \wedge u_n \equals u_n'$
    \ElsIf{this condition is randomly satisfied}
        \State Let $t_3$ be a random ground term of depth at most $\maxTermDepth$
        \State \Return $t_1 \equals t_2 \leftarrow t_1 \equals t_3 \wedge t_3 \equals t_2$
    \Else
        \State \Return $t_1 \equals t_2 \leftarrow t_2 \equals t_1$
    \EndIf
\end{algorithmic}
\end{algorithm}

\begin{algorithm}[!tb]
\caption{$\createRule(F)$ where $F = R_0(t_1,\dots,t_{a_0})$ and $a_0$ is the arity of $R_0$}\label{alg:createRule}
\begin{algorithmic}[1]\footnotesize
    \State Set $G$ to the empty graph, and randomly choose numbers $\relBodyAtoms$ and $\eqBodyAtoms$ between one and $\maxNumberRelBodyAtoms$ and $\maxNumberEqBodyAtoms$, respectively
    \For{\textbf{each} ${j \in \{ 1, \dots, a_0 \}}$}                                                                                                                                           \label{alg:createRule:h:start}
        \If{$t_j$ is of the form $f(s_1,\dots,s_m)$, and either the depth of $t_j$ is $\maxTermDepth + 1$ or this branch is chosen randomly}
            \State $\FN{0}{j} \defeq f$                                                                                                                                                         \label{alg:createRule:h:fn:start}
            \For{\textbf{each} $k \in \{ 1, \dots, m \}$}
                \State Add the \emph{head position} vertex $v = \langle 0, j.k \rangle$ to $G$, let $\vNGT{v}$ be a fresh variable, and let $\vGT{v} \defeq s_k$
            \EndFor                                                                                                                                                                             \label{alg:createRule:h:fn:end}
        \Else
            \State $\FN{0}{j} \defeq \bot$                                                                                                                                                      \label{alg:createRule:h:var:start}
            \State Add the \emph{head position} vertex $v = \langle 0, j \rangle$ to $G$, let $\vNGT{v}$ be a fresh variable, and let $\vGT{v} \defeq t_j$
        \EndIf                                                                                                                                                                                  \label{alg:createRule:h:var:end}
    \EndFor                                                                                                                                                                                     \label{alg:createRule:h:end}
    \For{\textbf{each} $i \in \{ 1, \dots, \relBodyAtoms \}$}                                                                                                                                   \label{alg:createRule:br:start}
        \State Set $R_i$ to a randomly chosen predicate
        \For{\textbf{each} $j \in \{ 1, \dots, a_i \}$ where $a_i$ is the arity of $R_i$}
            \State $\FN{i}{j} \defeq \bot$
            \State Add the \emph{relational body position} vertex $\langle i, j \rangle$ to $G$
        \EndFor
    \EndFor                                                                                                                                                                                     \label{alg:createRule:br:end}
    \For{\textbf{each} $i \in \{ \relBodyAtoms + 1, \dots, \relBodyAtoms + \eqBodyAtoms \}$}                                                                                                    \label{alg:createRule:eq:start}
        \State $R_i \defeq {\equals}$ and $a_i \defeq 2$
        \For{\textbf{each} $j \in \{ 1, 2 \}$}
            \If{this condition is randomly satisfied}
                \State $\FN{i}{j} \defeq \bot$
                \State Add the \emph{equality body position} vertex $\langle i, j \rangle$ to $G$
            \Else
                \State Set $\FN{i}{j}$ to a randomly chosen function symbol
                \For{\textbf{each} $k \in \{ 1, \dots, m \}$ where $m$ is the arity of $\FN{i}{j}$}
                    \State Add the \emph{equality body position} vertex $\langle i, j.k \rangle$ to $G$
                \EndFor
            \EndIf
        \EndFor
    \EndFor                                                                                                                                                                                     \label{alg:createRule:eq:end}
    \State Randomly extend $G$ with edges such that (i)~graph $G$ is acyclic, (ii)~each head position vertex has zero incoming edges
    \Statex \hspace{0.25cm} and at least one outgoing edge to a relational body position vertex, (iii)~each relational body position vertex has at
    \Statex \hspace{0.25cm} most one incoming edge, and (iv)~each equality body position vertex has one incoming edge and no outgoing edges.                                                    \label{alg:createRule:G}
    \For{\textbf{each} relational body position vertex $v$ with no incoming edges}                                                                                                              \label{alg:createRule:leaves:start}
        \State Let $\vNGT{v}$ be a fresh variable
        \State Let $\vGT{v}$ be a random ground term of depth at most $\maxTermDepth$                                                                                                           \label{alg:createRule:leaves:end}
    \EndFor
    \While{$G$ contains a vertex $v$ such that $\vNGT{v}$ and $\vGT{v}$ are undefined}                                                                                                          \label{alg:createRule:prop:start}
        \State Choose a vertex $v$ in $G$ such that $\vNGT{v}$ and $\vGT{v}$ are undefined
        \State Choose a vertex $v'$ in $G$ such that $\vNGT{v'}$ and $\vGT{v'}$ are defined and $G$ contains an edge from $v'$ to $v$
        \State ${\vNGT{v} \defeq \vNGT{v'}}$ \quad and \quad ${\vGT{v} \defeq \vGT{v'}}$
    \EndWhile                                                                                                                                                                                   \label{alg:createRule:prop:end}
    \For{\textbf{each} ${i \in \{ 0, \dots, \relBodyAtoms + \eqBodyAtoms \}}$ and \textbf{each} ${j \in \{ 1, \dots, a_i \}}$}                                                                  \label{alg:createRule:arg:start}
        \If{$\FN{i}{j} = \bot$}
            \State Let $v$ be the vertex of $G$ of the form $\langle i,j \rangle$
            \State $\NGA{i}{j} \defeq \vNGT{v}$ \quad and \quad $\GA{i}{j} \defeq \vGT{v}$
        \Else
            \State Let $m$ be the arity of $\FN{i}{j}$ and let $v_1,\dots,v_m$ be all vertices of $G$ of the form  $\langle i, j.k \rangle$ for $k \in \{ 1, \dots, m \}$
            \State $\NGA{i}{j} \defeq \FN{i}{j}(\vNGT{v_1}, \dots, \vNGT{v_m})$ \quad and \quad $\GA{i}{j} \defeq \FN{i}{j}(\vGT{v_1}, \dots, \vGT{v_m})$
        \EndIf
    \EndFor                                                                                                                                                                                     \label{alg:createRule:arg:end}
    \State $\rho \defeq R_0(\NGA{0}{1},\dots,\NGA{0}{a_0}) \leftarrow \bigwedge\limits_{i = 1}^{\relBodyAtoms + \eqBodyAtoms} R_i(\NGA{i}{1},\dots,\NGA{i}{a_i})$                               \label{alg:createRule:rules1}
    \State $\overline{\rho} \defeq R_0(\GA{0}{1},\dots,\GA{0}{a_0}) \leftarrow \bigwedge\limits_{i = 1}^{\relBodyAtoms + \eqBodyAtoms} R_i(\GA{i}{1},\dots,\GA{i}{a_i})$                        \label{alg:createRule:rules2}
    \State \Return $\langle \rho, \overline{\rho} \rangle$
\end{algorithmic}
\end{algorithm}

Algorithm~\ref{alg:createRule} creates a nonground rule $\rho$ that may be
added to $\R$ and the corresponding instance $\overline{\rho}$. Roughly
speaking, this is done by generating a graph $G$ containing a vertex of the
form ${\langle i, p \rangle}$ for each head or body atom of $\rho$ and each
position $p$ in the atom: vertices of the form ${\langle 0, p \rangle}$
identify positions in the head of $\rho$, and vertices of the form ${\langle i,
p \rangle}$ with ${i \geq 1}$ identify positions in the $i$-th body atom of
$\rho$. The edges of $G$ will describe the joins among the positions of the
atoms of $\rho$ and $\overline{\rho}$. Each vertex $v$ in $G$ will be
associated with a variable $\vNGT{v}$ and a ground term $\vGT{v}$. Finally, for
each argument $j$ of an atom $i$, the algorithm will define $F_{i,j}$ as either
a function symbol or a special symbol $\bot$. Intuitively, ${F_{i,j} \neq
\bot}$ means that the $j$-th argument of the $i$-th atom of $\rho$ and
$\overline{\rho}$ are ${F_{i,j}(\vNGT{\langle i,j.1 \rangle}, \dots,
\vNGT{\langle i,j.m \rangle})}$ and ${F_{i,j}(\vGT{\langle i,j.1 \rangle},
\dots, \vGT{\langle i,j.m \rangle})}$, respectively; in contrast, ${F_{i,j} =
\bot}$ means that the $j$-th argument of the $i$-th atom of $\rho$ and
$\overline{\rho}$ are ${\vNGT{\langle i,j \rangle}}$ and ${\vGT{\langle i,j
\rangle}}$, respectively. Thus, $G$, $F_{i,j}$, $\vNGT{v}$, and $\vGT{v}$
correctly describe the arguments of all atoms in $\rho$ and $\overline{\rho}$,
and the main objective of Algorithm~\ref{alg:createRule} is to randomly select
these values while avoiding any contradictions.

To this end, Algorithm~\ref{alg:createRule} analyses each argument $t_j$ of ${F
= \head{\overline{\rho}}}$. If $t_j$ is a functional term involving a function
symbol $f$, the algorithm can choose whether the $j$-th argument of
$\head{\rho}$ should be a term of the form $f(\vec x)$ or a variable. In the
former case
(lines~\ref{alg:createRule:h:fn:start}--\ref{alg:createRule:h:fn:end}),
$F_{0,j}$ is set to $f$, a vertex ${v = \langle 0,j.k \rangle}$ is introduced
for each argument $k$ of $f$, a fresh variable $\vNGT{v}$ is introduced for
each such $v$, and $\vGT{v}$ is set to the $k$-th argument of $t_j$. In the
latter case
(lines~\ref{alg:createRule:h:var:start}--\ref{alg:createRule:h:var:end}), only
vertex ${v = \langle 0,j \rangle}$ is introduced, $\FN{0}{j}$ is set to $\bot$,
a fresh variable $\vNGT{v}$ is introduced for the $j$-th argument of
$\head{\rho}$, and $\vGT{v}$ is set to $t_j$. If $t_j$ does not contain a
function symbol (i.e., if $t_j$ is a constant), then only the latter case is
possible.

In lines~\ref{alg:createRule:br:start}--\ref{alg:createRule:br:end}
and~\ref{alg:createRule:eq:start}--\ref{alg:createRule:eq:end}, the algorithm
creates the vertices of $G$ that correspond to the relational and equality body
atoms, respectively, of $\rho$ and $\overline{\rho}$. Relational atoms contain
variable arguments only (cf.\ Section~\ref{sec:preliminaries}), so only
vertices of the form ${\langle i,j \rangle}$ are added for relational atoms. In
contrast, equality atoms can contain functional terms, so for each argument of
an equality atom the algorithm chooses whether to introduce a functional term.
These steps do \emph{not} yet determine $\vNGT{v}$ and $\vGT{v}$.

Graph $G$ is next extended in line~\ref{alg:createRule:G} with edges that
encode the propagation of terms among the positions of the body atoms of $\rho$
and $\overline{\rho}$. In
lines~\ref{alg:createRule:leaves:start}--\ref{alg:createRule:leaves:end}, for
each `leaf' vertex $v$ (i.e., for each vertex with no incoming edges),
$\vNGT{v}$ is assigned a fresh variable, and $\vGT{v}$ is assigned a random
ground term of depth at most $\maxTermDepth$. These choices are next propagated
via the edges of $G$ to all the remaining terms of $G$ in
lines~\ref{alg:createRule:prop:start}--\ref{alg:createRule:prop:end}. To avoid
contradictions, graph $G$ must satisfy the properties in
line~\ref{alg:createRule:G}: it must be acyclic; the head positions cannot have
an incoming edge because their values have already been determined, and the
corresponding variables must occur in relational body positions to satisfy the
safety requirement; relational body positions can receive at most one value
from another position; and equality body positions must receive exactly one
value and should not propagate it further to satisfy the safety requirement.
Thus, the structure of the rule has been determined after
line~\ref{alg:createRule:prop:end}: if vertices $v_1$ and $v_2$ satisfy
${\vNGT{v_1} = \vNGT{v_2}}$, then the atom positions corresponding to $v_1$ and
$v_2$ participate in a join.
Lines~\ref{alg:createRule:arg:start}--\ref{alg:createRule:arg:end} then
construct a nonground term $\NGA{i}{j}$ and a ground term $\GA{i}{j}$ for the
$j$-th argument of the $i$-th atom of $\rho$ and $\overline{\rho}$ using
$\FN{i}{j}$ and the corresponding vertices of $G$. Finally, the rules $\rho$
and $\overline{\rho}$ are constructed in lines~\ref{alg:createRule:rules1}
and~\ref{alg:createRule:rules2}, respectively.

Let $\R$ and $B$ be the output of Algorithm~\ref{alg:generator}. Clearly, when
computing the fixpoint of ${\R \cup \SG{\R}}$ on $B$, all rule instances
considered during the generation process will derive copies of the considered
facts; however, the rules $\rho$ considered in lines~\ref{alg:generator:1:std}
and~\ref{alg:generator:2:transfer} can give rise to other rule instances too.
Thus, computing the fixpoint of ${\R \cup \SG{\R}}$ on $B$ considers at least
the rule instances produced during benchmark generation (but it often considers
additional rule instances as well). This, in turn, ensures that the benchmarks
provide nontrivial workloads for the techniques from
Section~\ref{sec:answering}.

\fi

\end{document}